\documentclass[11pt]{article}
\usepackage{fullpage,times}

\pdfoutput=1

\usepackage[T1]{fontenc}
\usepackage[latin9]{inputenc}
\usepackage{float}
\usepackage{color}
\usepackage{boxedminipage}

\usepackage{amsmath,amssymb}

\usepackage{graphicx}
\usepackage{ltxtable} 
\newcolumntype{L}{>{\centering\arraybackslash} m{0.04\columnwidth}} 
\newcolumntype{R}{>{\centering\arraybackslash} m{0.48\columnwidth}} 
\newcolumntype{S}{>{\centering\arraybackslash} m{0.32\columnwidth}} 

\newcommand{\rE}{{\mathbb E}}

 \newtheorem{lemma}{Lemma}
 \newtheorem{proposition}{Proposition}
 
 \newtheorem{theorem}{Theorem}
 
 \newtheorem{definition}{Definition}
 
 \newtheorem{remark}{Remark}

\newcommand{\BlackBox}{\rule{1.5ex}{1.5ex}}  
\newenvironment{proof}{\par\noindent{\bf Proof\ }}{\hfill\BlackBox\\[2mm]}

\DeclareMathOperator*{\E}{\mathbb{E}}

\DeclareMathOperator*{\argmin}{argmin} 
 
\newcommand{\reals}{\mathbb{R}}

\newcommand{\figref}[1]{Figure~\ref{#1}}

\newcommand{\secref}[1]{Section~\ref{#1}}

\newcommand{\thmref}[1]{Theorem~\ref{#1}}

\newcommand{\lemref}[1]{Lemma~\ref{#1}}

\newcommand{\propref}[1]{Proposition~\ref{#1}}

\newenvironment{myalgo}[1]%
{
\begin{center}
\begin{boxedminipage}{0.8\linewidth}
\begin{center}
\textbf{\texttt{#1}}
\end{center}
\rm
\begin{tabbing}
....\=...\=...\=...\=...\=  \+ \kill
} %
{\end{tabbing} 
\end{boxedminipage} \end{center} 
}

\title{Stochastic Dual Coordinate Ascent Methods for Regularized Loss Minimization}
\author{Shai Shalev-Shwartz\\
School of Computer Science and Engineering \\
Hebrew University, Jerusalem, Israel
 \and 
Tong Zhang \\
Department of Statistics \\
Rutgers University, NJ, USA}
\date{}

\begin{document}

\maketitle

\begin{abstract}
  Stochastic Gradient Descent (SGD) has become popular for solving
  large scale supervised machine learning optimization problems such
  as SVM, due to their strong theoretical guarantees.  While the
  closely related Dual Coordinate Ascent (DCA) method has been
  implemented in various software packages, it has so far lacked good
  convergence analysis.  This paper presents a new analysis of
  Stochastic Dual Coordinate Ascent (SDCA) showing that this class of
  methods enjoy strong theoretical guarantees that are comparable or
  better than SGD. This analysis justifies the effectiveness of SDCA
  for practical applications.
\end{abstract}

\section{Introduction}

We consider the following generic optimization problem associated with regularized loss
minimization of linear predictors: Let $x_1,\ldots,x_n$ be vectors in $\reals^d$, let
$\phi_1,\ldots,\phi_n$ be a sequence of scalar convex functions, and
let $\lambda > 0$ be a regularization parameter. Our
goal is to solve $\min_{w \in \reals^d} P(w)$ where\footnote{Throughout this paper, we only consider the $\ell_2$-norm.} 
\begin{equation} \label{eqn:PrimalProblem}
P(w) = \left[ \frac{1}{n} \sum_{i=1}^n \phi_i( w^\top x_i) + \frac{\lambda}{2} \|w\|^2 \right] .
\end{equation}
For example, given labels $y_1,\ldots,y_n$ in $\{\pm 1\}$, the SVM
problem (with linear kernels and no bias term) is obtained by setting
$\phi_i(a) = \max\{0,1-y_i a\}$. Regularized logistic regression is
obtained by setting $\phi_i(a) = \log(1+\exp(-y_i a))$. Regression
problems also fall into the above. For example, ridge regression is
obtained by setting $\phi_i(a) = (a-y_i)^2$, regression with the
absolute-value is obtained by setting $\phi_i(a) = |a-y_i|$, and
support vector regression is obtained by setting $\phi_i(a) =
\max\{0,|a-y_i|-\nu\}$, for some predefined insensitivity parameter
$\nu > 0$.

Let $w^*$ be the optimum of \eqref{eqn:PrimalProblem}. We say that a
solution $w$ is $\epsilon_P$-sub-optimal if $P(w)-P(w^*) \le \epsilon_P$. We
analyze the runtime of optimization procedures as a function of the
time required to find an $\epsilon_P$-sub-optimal solution. 

A simple approach for solving SVM is stochastic gradient descent (SGD)
\cite{robbins1951stochastic,murata1998statistical,le2004large,Zhang04,BottouBo08,ShalevSiSr07}.
SGD finds an $\epsilon_P$-sub-optimal solution in time
$\tilde{O}(1/(\lambda \epsilon_P))$. This runtime does not depend on
$n$ and therefore is favorable when $n$ is very large.  However, the
SGD approach has several disadvantages. It does not have a clear
stopping criterion; it tends to be too aggressive at the beginning of
the optimization process, especially when $\lambda$ is very small;
while SGD reaches a moderate accuracy quite fast, its convergence
becomes rather slow when we are interested in more accurate solutions.

An alternative approach is dual coordinate ascent (DCA), which solves
a \emph{dual} problem of \eqref{eqn:PrimalProblem}. Specifically, for
each $i$ let $\phi_i^* : \reals \to \reals$ be the convex conjugate of $\phi_i$, namely,
$\phi_i^*(u) = \max_z (z u - \phi_i(z))$. The dual problem is 
\begin{equation} \label{eqn:DualProblem}
\max_{\alpha \in \reals^n} D(\alpha) ~~~\textrm{where}~~~ D(\alpha) = 
\left[ \frac{1}{n} \sum_{i=1}^n -\phi_i^*(-\alpha_i) -
  \frac{\lambda}{2} \left\|\tfrac{1}{\lambda n}\sum_{i=1}^n \alpha_i x_i \right\|^2 \right] ~.
\end{equation}
The dual objective in (\ref{eqn:DualProblem}) has a different dual
variable associated with each example in
the training set.  At each iteration of DCA, the dual objective is
optimized with respect to a single dual variable, while the rest of
the dual variables are kept in tact. 
 
If we define 
\begin{equation} \label{eqn:walpha}
w(\alpha) = \frac{1}{\lambda n} \sum_{i=1}^n \alpha_i x_i ,
\end{equation}
then it is known that $w(\alpha^*)=w^*$, where $\alpha^*$ is an optimal solution of (\ref{eqn:DualProblem}). 
It is also known that $P(w^*)=D(\alpha^*)$ which immediately implies that for all $w$ and $\alpha$, we have
$P(w) \geq D(\alpha)$, and hence the duality gap defined as
\[
P(w(\alpha))-D(\alpha)
\]
can be
regarded as an upper bound of the primal sub-optimality
$P(w(\alpha))-P(w^*)$.

We focus on a \emph{stochastic} version of DCA, abbreviated by SDCA,
in which at each round we choose which dual coordinate to optimize
uniformly at random.  The purpose of this paper is to develop
theoretical understanding of the convergence of the duality gap for
SDCA.

We analyze SDCA either for $L$-Lipschitz loss functions or for
$(1/\gamma)$-smooth loss functions, which are defined as follows.
Throughout the paper, we will use $\phi'(a)$ to denote a sub-gradient of a convex function $\phi(\cdot)$,
and use $\partial \phi(a)$ to denote its sub-differential.
\begin{definition}
  A function $\phi_i: \reals \to \reals$ is $L$-Lipschitz if for all $a, b \in \reals$, we have
  \[
  |\phi_i(a)- \phi_i(b)| \leq L\,|a-b| .
  \]
  A function $\phi_i: \reals \to \reals$ is $(1/\gamma)$-smooth if it
  is differentiable and its derivative is $(1/\gamma)$-Lipschitz. An
  equivalent condition is that for all $a, b \in \reals$, we have
  \[
  \phi_i(a) \leq \phi_i(b) + \phi_i'(b) (a-b) + \frac{1}{2\gamma} (a-b)^2 ,
  \]
where $\phi_i'$ is the derivative of $\phi_i$. 
\end{definition}
It is well-known that if $\phi_i(a)$ is $(1/\gamma)$-smooth, then $\phi_i^*(u)$ is $\gamma$ strongly convex:
for all $u, v \in \reals$ and $s \in [0,1]$:
\[
- \phi_i^*( s u + (1-s) v)\geq - s \phi_i^*(u) - (1-s) \phi_i^*(v) + \frac{\gamma s (1-s)}{2} (u-v)^2 .
\]

Our main findings are: in order to achieve a duality gap of
$\epsilon$,
\begin{itemize}
\item For $L$-Lipschitz loss functions, we obtain the rate of $\tilde{O}(n
  + L^2/(\lambda \epsilon))$.
\item For $(1/\gamma)$-smooth loss functions, we obtain the rate of  $\tilde{O}((n + 1/(\lambda\gamma)) \log (1/\epsilon))$.
\item For loss functions which are almost everywhere smooth (such as the hinge-loss),
we can obtain rate better than the above rate for Lipschitz loss. See
\secref{sec:refined-analysis} for a precise statement.
\end{itemize}

\section{Related Work}

DCA methods are related to decomposition methods
\cite{Platt98,Joachims98}. While several experiments have shown that
decomposition methods are inferior to SGD for large scale SVM
\cite{ShalevSiSr07,BottouBo08}, \cite{HsiehChLiKeSu08}
recently argued that SDCA outperform the SGD approach in some regimes.
For example, this occurs when we need relatively high solution
accuracy so that either SGD or SDCA has to be run for more than a few
passes over the data.

However, our theoretical understanding of SDCA is not satisfying.
Several authors (e.g.  \cite{MangasarianMu99,HsiehChLiKeSu08}) proved
a linear convergence rate for solving SVM with DCA (not necessarily
stochastic). The basic technique is to adapt the linear convergence of
coordinate ascent that was established by \cite{LuoTs92}.  The linear convergence means that it achieves a rate
of $(1-\nu)^k$ after $k$ passes over the data, where $\nu >0$.  This
convergence result tells us that after an unspecified number of
iterations, the algorithm converges faster to the optimal solution than SGD.

However, there are two problems with this analysis. First, the linear
convergence parameter, $\nu$, may be very close to zero and the
initial unspecified number of iterations might be very large.  In
fact, while the result of \cite{LuoTs92} does not explicitly specify
$\nu$, an examine of their proof shows that $\nu$ is proportional to
the smallest nonzero eigenvalue of $X^\top X$, where $X$ is the $n
\times d$ data matrix with its $i$-th row be the $i$-th data point
$x_i$.  For example if two data points $x_i \neq x_j$ becomes closer
and closer, then $\nu \to 0$.  This dependency is problematic in the
data laden domain, and we note that such a dependency does not occur
in the analysis of SGD.

Second, the analysis only deals with the sub-optimality of the
\emph{dual} objective, while our real goal is to bound the
sub-optimality of the \emph{primal} objective.  Given a dual solution
$\alpha \in \reals^n$ its corresponding primal solution is $w(\alpha)$
(see \eqref{eqn:walpha}).  The problem is that even if $\alpha$ is
$\epsilon_D$-sub-optimal in the dual, for some small $\epsilon_D$, the
primal solution $w(\alpha)$ might be far from being optimal. For SVM,
\cite[Theorem 2]{HuKeScSt06} showed that 
in order to obtain a primal $\epsilon_P$-sub-optimal solution, we need
a dual $\epsilon_D$-sub-optimal solution with $\epsilon_D=O(\lambda
\epsilon_P^2)$; therefore a convergence result for dual solution can
only translate into a primal convergence result with worse convergence
rate. Such a treatment is unsatisfactory, and this
is what we will avoid in the current paper.

Some analyses of stochastic coordinate ascent provide solutions to the
first problem mentioned above. For example, \cite{CollinsGlKoCaBa08} analyzed an exponentiated gradient dual
coordinate ascent algorithm. The algorithm analyzed there
(exponentiated gradient) is different from the standard DCA algorithm
which we consider here, and the proof techniques are quite different.
Consequently their results are not directly comparable to results we
obtain in this paper.  Nevertheless we note that for SVM, their
analysis shows a convergence rate of $O(n/\epsilon_D)$ in order to
achieve $\epsilon_D$-sub-optimality (on the dual) while our analysis
shows a convergence of $O(n \log \log n + 1/\lambda \epsilon)$ to
achieve $\epsilon$ duality gap; for logistic regression, their
analysis shows a convergence rate of $O((n+1/\lambda) \log(1/\epsilon_D))$ in
order to achieve $\epsilon_D$-sub-optimality on the dual while our
analysis shows a convergence of $O((n+1/\lambda)\log(1/\epsilon))$ to
achieve $\epsilon$ duality gap.

In addition, \cite{ShalevTe09}, and later \cite{Nesterov10} have
analyzed randomized versions of coordinate descent for unconstrained
and constrained minimization of smooth convex functions. \cite[Theorem
4]{HsiehChLiKeSu08} applied these results to the dual SVM
formulation. However, the resulting convergence rate is
$O(n/\epsilon_D)$ which is, as mentioned before, inferior to the
results we obtain here.  Furthermore, neither of these analyses can be
applied to logistic regression due to their reliance on the smoothness
of the dual objective function which is not satisfied for the dual
formulation of logistic regression.  We shall also point out again
that all of these bounds are for the dual sub-optimality, while as
mentioned before, we are interested in the primal sub-optimality.

In this paper we derive new bounds on the duality gap (hence, they
also imply bounds on the primal sub-optimality) of SDCA. These bounds
are superior to earlier results, and our analysis only holds for
randomized (stochastic) dual coordinate ascent. As we will see from
our experiments, randomization is important in practice.  In fact, the
practical convergence behavior of (non-stochastic) cyclic dual
coordinate ascent (even with a random ordering of the data) can be
slower than our theoretical bounds for SDCA, and thus cyclic DCA is
inferior to SDCA.  In this regard, we note that some of the earlier
analysis such as \cite{LuoTs92} can be applied both to stochastic and
to cyclic dual coordinate ascent methods with similar results. This
means that their analysis, which can be no better than the behavior of
cyclic dual coordinate ascent, is inferior to our analysis.

Recently, \cite{lacoste2012stochastic} derived a stochastic coordinate
ascent for structural SVM based on the Frank-Wolfe
algorithm. Specifying one variant of their algorithm to binary classification with
the hinge loss, yields the SDCA algorithm for the hinge-loss. The rate
of convergence \cite{lacoste2012stochastic} derived for their algorithm
is the same as the rate we derive for SDCA with a Lipschitz loss
function.

Another relevant approach is the Stochastic Average Gradient (SAG),
that has recently been analyzed in \cite{LSB12-sgdexp}. There, a
convergence rate of $\tilde{O}(n \log(1/\epsilon))$ rate is shown, for
the case of smooth losses, assuming that $n \ge
\frac{8}{\lambda\,\gamma}$. This matches our guarantee in the regime $n \ge
\frac{8}{\lambda\,\gamma}$.

The following table summarizes our results in comparison to previous analyses.
Note that for SDCA with Lipschitz loss, we observe a faster practical convergence rate, which is explained with our refined analysis in Section~\ref{sec:refined-analysis}.
\begin{center}
\begin{tabular}{l|c|c}
\multicolumn{3}{c}{Lipschitz loss} \\ \hline
Algorithm & type of convergence & rate \\ \hline
SGD & primal & $\tilde{O}(\tfrac{1}{\lambda \epsilon})$ \\
online EG \cite{CollinsGlKoCaBa08} (for SVM) & dual &
$\tilde{O}(\tfrac{n}{\epsilon})$ \\
Stochastic Frank-Wolfe \cite{lacoste2012stochastic} & prima-dual &
$\tilde{O}(n + \tfrac{1}{\lambda \epsilon})$  \\
SDCA & primal-dual & $\tilde{O}(n + \tfrac{1}{\lambda \epsilon})$ or faster\\ \hline
\end{tabular}
\\[0.5cm]
\begin{tabular}{l|c|c}
\multicolumn{3}{c}{Smooth loss} \\ \hline
Algorithm & type of convergence & rate \\ \hline
SGD & primal & $\tilde{O}(\tfrac{1}{\lambda \epsilon})$ \\
online EG \cite{CollinsGlKoCaBa08} (for logistic regression) & dual &
$\tilde{O}((n+\tfrac{1}{\lambda}) \log \tfrac{1}{\epsilon})$ \\
SAG \cite{LSB12-sgdexp} (assuming $n \ge \frac{8}{\lambda\,\gamma}$) & primal & $\tilde{O}((n + \tfrac{1}{\lambda}) \log \tfrac{1}{\epsilon})$ \\ 
 SDCA & primal-dual & $\tilde{O}((n + \tfrac{1}{\lambda}) \log \tfrac{1}{\epsilon})$ \\ \hline
\end{tabular}
\end{center}

\section{Basic Results}

The generic algorithm we analyze is described below. In the
pseudo-code, the parameter $T$ indicates the number of iterations
while the parameter $T_0$ can be chosen to be a number between $1$ to
$T$. Based on our analysis, a good choice of $T_0$ is to be $T/2$. In
practice, however, the parameters $T$ and $T_0$ are not required as
one can evaluate the duality gap and terminate when it is sufficiently
small. 
\begin{myalgo}{Procedure SDCA$(\alpha^{(0)})$} 
\textbf{Let} $w^{(0)}=w(\alpha^{(0)})$ \\
\textbf{Iterate:} for $t=1,2,\dots,T$: \+ \\
 Randomly pick $i$ \\
 Find $\Delta \alpha_i$ to maximize
$-\phi_i^*(-(\alpha_i^{(t-1)} + \Delta \alpha_i) ) - \frac{\lambda
  n}{2} \|w^{(t-1)}+ (\lambda n)^{-1} \Delta \alpha_i x_i\|^2$. \\
 $\alpha^{(t)} \leftarrow \alpha^{(t-1)} + \Delta \alpha_i e_i$ \\
$w^{(t)} \leftarrow w^{(t-1)} + (\lambda n)^{-1} \Delta \alpha_i x_i $
\- \\
\textbf{Output (Averaging option):} \+ \\
Let $\bar{\alpha}  = \frac{1}{T-T_0} \sum_{i=T_0+1}^T \alpha^{(t-1)}$ \\
Let $\bar{w}  = w(\bar{\alpha}) = \frac{1}{T-T_0} \sum_{i=T_0+1}^T w^{(t-1)}$ \\
return $\bar{w}$  \- \\
\textbf{Output (Random option):} \+ \\
Let $\bar{\alpha}=\alpha^{(t)}$ and $\bar{w}  = w^{(t)}$ for some random $t \in T_0+1,\ldots,T$ \\
return $\bar{w}$ 
\end{myalgo}

We analyze the algorithm based on different assumptions on the loss
functions. To simplify the statements of our theorems, we always assume the following:
\begin{enumerate}
\item For all $i$, $\|x_i\| \le 1$
\item For all $i$ and $a$, $\phi_i(a) \ge 0$
\item For all $i$, $\phi_i(0) \le 1$
\end{enumerate}

\begin{theorem} \label{thm:Lipschitz}
Consider Procedure SDCA with $\alpha^{(0)}=0$. 
Assume that $\phi_i$ is $L$-Lipschitz for all $i$.
To obtain a duality gap of $\E [P(\bar{w})-D(\bar{\alpha})] \leq \epsilon_P$, it suffices to have a total number of
iterations of
\[
T \geq T_0 + n + \frac{4 \,L^2}{\lambda \epsilon_P} \geq 
\max(0, \lceil n \log(0.5 \lambda n L^{-2} ) \rceil ) + n + \frac{20 \,L^2}{\lambda \epsilon_P} ~.
\]
Moreover, when $t \geq T_0$, we have dual sub-optimality bound of
$\E [D(\alpha^*) - D(\alpha^{(t)})] \leq \epsilon_P/2$.
\end{theorem}

\begin{remark}
If we choose the average version, we may simply take $T=2T_0$. Moreover, we note that
\thmref{thm:Lipschitz} holds for both averaging or for choosing $w$ at random from
$\{T_0+1,\ldots,T\}$. This means that calculating the duality gap at
few random points would lead to the same type of guarantee with high
probability.  This approach has the advantage over averaging, since it
is easier to implement the stopping condition (we simply check the
duality gap at some random stopping points. This is in contrast to
averaging in which we need to know $T,T_0$ in advance). 
\end{remark}

\begin{remark} \label{remark:svm}
The above theorem applies to the hinge-loss function, $\phi_i(u) =
\max\{0,1-y_i a\}$. However, for the hinge-loss, the constant $4$ in
the first inequality can be replaced by $1$ (this is because the
domain of the dual variables is positive, hence the constant $4$ in \lemref{lem:GboundLip}
can be replaced by $1$). We therefore obtain the bound:
 \[
T \geq T_0 + n + \frac{L^2}{\lambda \epsilon_P} \geq 
\max(0, \lceil n \log(0.5 \lambda n L^{-2} ) \rceil ) + n + \frac{5 \,L^2}{\lambda \epsilon_P} ~.
\]
\end{remark}

\begin{theorem} \label{thm:smooth} 
  Consider Procedure SDCA with $\alpha^{(0)}=0$.  Assume that $\phi_i$ is
  $(1/\gamma)$-smooth for all $i$.  To obtain an expected duality gap
  of $\E [P(w^{(T)})-D(\alpha^{(T)})] \leq \epsilon_P$, it suffices to have a total number of
  iterations of
\[
T \geq \left(n +
  \tfrac{1}{\lambda \gamma}\right) \, \log( (n + \tfrac{1}{\lambda \gamma})   \cdot \tfrac{1}{\epsilon_P}) .
\]
Moreover, to obtain an expected duality gap of $\E [P(\bar{w})-D(\bar{\alpha})] \leq \epsilon_P$, it suffices to have a total number of iterations of $T > T_0$ where 
\[
T_0 \geq \left(n +
  \tfrac{1}{\lambda \gamma}\right) \, \log( (n + \tfrac{1}{\lambda \gamma})   \cdot \tfrac{1}{(T-T_0)\epsilon_P}) .
\]
\end{theorem}

\begin{remark}
If we choose $T=2T_0$, and assume that $T_0 \geq n + 1/(\lambda \gamma)$,
then the second part of \thmref{thm:smooth} implies a requirement of
\[
T_0 \geq 
 \left(n +
  \tfrac{1}{\lambda \gamma}\right) \, \log( \tfrac{1}{\epsilon_P}) ,
\]
which is slightly weaker than the first part of \thmref{thm:smooth} when $\epsilon_P$ is relatively large. 
\end{remark}

\begin{remark}
  \cite{BottouBo08} analyzed the runtime of SGD and other algorithms
  from the perspective of the time required to achieve a certain level
  of error on the test set. To perform such analysis, we also need to
  take into account the \emph{estimation error}, namely, the
  additional error we suffer due to the fact that the training
  examples defining the regularized loss minimization problem are only
  a finite sample from the underlying distribution.  The estimation
  error of the primal objective behaves like
  $\Theta\left(\frac{1}{\lambda n} \right)$ (see
  \cite{ShalevSr08,SridharanSrSh08}). Therefore, an interesting regime is when
  $\frac{1}{\lambda n} = \Theta(\epsilon)$. In that case, the bound
  for both Lipschitz and smooth functions would be
  $\tilde{O}(n)$. However, this bound on the estimation error is for
  the worst-case distribution over examples. Therefore, another
  interesting regime is when we would like $\epsilon \ll
  \frac{1}{\lambda n}$, but still $\frac{1}{\lambda n} = O(1)$
  (following the practical observation that $\lambda = \Theta(1/n)$
  often performs well). In that case, smooth functions still yield the
  bound $\tilde{O}(n)$, but the dominating term for Lipschitz
  functions will be $\frac{1}{\lambda\,\epsilon}$.
\end{remark}

\begin{remark}
The runtime of SGD is $\tilde{O}(\frac{1}{\lambda \epsilon} )$. This can be better than SDCA if $n
  \gg \frac{1}{\lambda \epsilon}$. However, in that case, SGD in fact
  only looks at $n' = \tilde{O}(\frac{1}{\lambda \epsilon} )$
  examples, so we can run SDCA on these $n'$ examples and obtain
  basically the same rate. For smooth functions, SGD
  can be much worse than SDCA if $\epsilon \ll \frac{1}{\lambda n}$.
\end{remark}

\section{Using SGD at the first epoch}
\label{sec:sgd}
From the convergence analysis, SDCA may not perform as well as
SGD for the first few epochs (each epoch means one pass over the data). 
The main reason is that SGD takes a larger step size than SDCA earlier on, which helps its performance. 
It is thus natural to combine SGD and SDCA, where
the first epoch is performed using a modified
stochastic gradient descent rule. 
We show that the expected dual sub-optimality at the end of the first epoch is $\tilde{O}(1/(\lambda n))$.  
This result can be combined with SDCA to obtain a faster convergence when $\lambda \gg \log n/n$.

We first introduce convenient notation. 
Let $P_t$ denote the primal objective for the first $t$ examples in
the training set, 
\[
P_t(w) = \left[ \frac{1}{t} \sum_{i=1}^t \phi_i( w^\top x_i) + \frac{\lambda}{2} \|w\|^2 \right] .
\]
The corresponding dual objective is
\[
D_t(\alpha) = 
\left[ \frac{1}{t} \sum_{i=1}^t -\phi_i^*(-\alpha_i) -
  \frac{\lambda}{2} \left\|\tfrac{1}{\lambda t}\sum_{i=1}^t \alpha_i x_i \right\|^2 \right] ~.
\]
Note that $P_n(w)$ is the primal objective given in
\eqref{eqn:PrimalProblem} and that $D_n(\alpha)$ is the dual objective
given in \eqref{eqn:DualProblem}.

The following algorithm is a modification of SGD. The idea is to greedily decrease the dual sub-optimality for problem
$D_t(\cdot)$ at each step $t$. This is different from DCA which works with $D_n(\cdot)$ at each step $t$.

\begin{myalgo}{Procedure Modified-SGD}
\textbf{Initialize:} $w^{(0)}=0$ \\
\textbf{Iterate:} for $t=1,2,\dots,n$: \+ \\
 Find $\alpha_t$ to maximize
 $-\phi_t^*(-\alpha_t ) - \frac{\lambda
  t}{2} \|w^{(t-1)}+ (\lambda t)^{-1} \alpha_t x_t\|^2$. \\
Let $w^{(t)} = \frac{1}{\lambda t} \sum_{i=1}^t \alpha_i x_i$ \- \\
return $\alpha$
\end{myalgo}

We have the following result for the convergence of dual objective:
\begin{theorem} \label{thm:SGD}
  Assume that $\phi_i$ is $L$-Lipschitz for all $i$.
  In addition, assume that
  $(\phi_i,x_i)$ are iid samples from the same distribution for all $i=1,\ldots,n$. 
  At the end of Procedure Modified-SGD, we have
  \[
  \E[D(\alpha^*)  - D(\alpha) ] \le  \frac{2L^2\log(e n)}{\lambda n} .
  \]
  Here the expectation is with respect to the random sampling of $\{(\phi_i,x_i): i=1,\ldots,n\}$.
\end{theorem}

\begin{remark}
   When $\lambda$ is relatively large, the convergence rate in \thmref{thm:SGD} for modified-SGD is better than 
   what we can prove for SDCA. This is because Modified-SGD employs a larger step size at each step $t$
   for $D_t(\alpha)$  than the corresponding step size in SDCA for $D(\alpha)$.
   However, the proof requires us to assume that $(\phi_i,x_i)$ are randomly drawn from a certain distribution,
   while this extra randomness assumption is not needed for the convergence of SDCA.
\end{remark}

\begin{myalgo}{Procedure SDCA with SGD Initialization}
\textbf{Stage 1:} call Procedure Modified-SGD and obtain $\alpha$ \\
\textbf{Stage 2:} call Procedure SDCA with parameter $\alpha^{(0)}=\alpha$ \\
\end{myalgo}

\begin{theorem} \label{thm:Lipschitz-sgd}
Assume that $\phi_i$ is $L$-Lipschitz for all $i$. In addition, assume that
$(\phi_i,x_i)$ are iid samples from the same distribution for all $i=1,\ldots,n$. 
Consider Procedure SDCA with SGD Initialization.
To obtain a duality gap of $\E [P(\bar{w})-D(\bar{\alpha})] \leq \epsilon_P$ at Stage 2, it suffices to have a total number of
SDCA iterations of
\[
T \geq T_0 + n + \frac{4 \,L^2}{\lambda \epsilon_P} \geq 
\lceil n \log(\log (e n)) \rceil + n + \frac{20 \,L^2}{\lambda \epsilon_P} ~.
\]
Moreover, when $t \geq T_0$, we have duality sub-optimality bound of
$\E [D(\alpha^*) - D(\alpha^{(t)})] \leq \epsilon_P/2$.
 \end{theorem}

\begin{remark}
  For Lipschitz loss, ideally we would like to have a computational complexity of $O(n + L^2/(\lambda \epsilon_P))$. 
  \thmref{thm:Lipschitz-sgd} shows that SDCA with SGD at first epoch can achieve no worst than 
  $O(n \log (\log n) + L^2/(\lambda \epsilon_P))$, which is very close to the ideal bound. 
  The result is better than that of vanilla SDCA in \thmref{thm:Lipschitz} when $\lambda$ is relatively large, which
  shows a complexity of   $O(n \log (n) + L^2/(\lambda \epsilon_P))$. The difference is caused by small step-sizes
  in the vanilla SDCA, and its negative effect can be observed in practice. That is, the vanilla SDCA tends to have a
  slower convergence rate than SGD in the first few iterations when $\lambda$ is relatively large. 
\end{remark}

\begin{remark}
  Similar to Remark~\ref{remark:svm}, for the hinge-loss, the constant
  $4$ in \thmref{thm:Lipschitz-sgd} can be reduced to 1, and the
  constant $20$ can be reduced to $5$.
\end{remark}

\section{Refined Analysis for Almost Smooth Loss}
\label{sec:refined-analysis}
Our analysis shows that for smooth loss, SDCA converges faster than SGD (linear versus sub-linear convergence). 
For non-smooth loss, the analysis does not show any advantage of SDCA over SGD. 
This does not explain the practical observation that SDCA converges faster than SGD asymptotically even for
SVM. This section tries to refine the analysis for Lipschitz loss and shows potential advantage of SDCA over SGD asymptotically.
Note that the refined analysis of this section relies on quantities that depend on the underlying data distribution, and thus the results are more complicated than 
those presented earlier. Although precise interpretations of these results will be complex, we will discuss them qualitatively after the theorem statements, 
and use them to explain the advantage of SDCA over SGD for non-smooth losses.

Although we note that for SVM, Luo and Tseng's analysis \cite{LuoTs92} shows linear convergence of the form
$(1-\nu)^k$ for dual sub-optimality after $k$ passes over the data, as we mentioned, $\nu$ is
proportional to the smallest nonzero eigenvalue of the
data Gram matrix $X^\top X$, and hence can be arbitrarily bad
when two data points $x_i \neq x_j$ becomes very close to each other. Our analysis uses a completely
different argument that avoids this dependency on the data Gram matrix. 

The main intuition behind our analysis is that many non-smooth loss functions are nearly smooth everywhere. For example, the
hinge loss $\max(0,1-u y_i)$ is smooth at any point $u$ such that $u y_i$ is not close to $1$. Since a smooth loss has a strongly
convex dual (and the strong convexity of the dual is directly used in our proof to obtain fast rate for smooth loss), the refined analysis in this section
relies on the following refined dual strong convexity condition that holds for nearly everywhere smooth loss functions.
\begin{definition} 
  For each $i$, we define $\gamma_i(\cdot) \geq 0$ so that for all dual variables $a$ and $b$, and
  $u \in \partial \phi_i^*(-b)$,
  we have
  \begin{equation}
    \phi_i^*(-a) - \phi_i^*(-b) + u (a-b) \geq \gamma_i(u) |a-b|^2 .
    \label{eqn:refined-sc}
  \end{equation}
\end{definition}
For the SVM loss, we have $\phi_i(u)=\max(0,1-u y_i)$, and 
$\phi^*_i(-a)= - a y_i$, with $a y_i \in [0,1]$ and $y_i \in \{\pm 1\}$.
It follows that 
\[
\phi_i^*(-a) - \phi_i^*(-b) + u (a-b) = (b-a) y_i + u (a-b) = |u y_i -1| |a-b| \geq |u y_i -1| \cdot |a-b|^2 .
\]
Therefore we may take $\gamma_i(u)=|u y_i -1|$.

For the absolute deviation loss, we have $\phi_i(u)=|u -y_i|$, and $\phi^*(-a)=-a y_i$ with $a \in [-1,1]$.
It follows that $\gamma_i(u)=|u-y_i|$.

\begin{proposition} \label{prop:refined-sc}
Under the assumption of \eqref{eqn:refined-sc}.
Let $\gamma_i= \gamma_i(w^{*\top} x_i)$,  we have the following
dual strong convexity inequality:
\begin{equation}
D(\alpha^*)-D(\alpha) \geq 
\frac{1}{n} \sum_{i=1}^n \gamma_i |\alpha_i-\alpha_i^*|^2 + \frac{\lambda}{2} (w-w^*)^\top (w -w^*) .
\label{eqn:dual-rsc}
\end{equation}
Moreover, given $w \in \reals^d$ and $-a_i  \in \partial \phi_i(w^{\top} x_i)$, we have
\begin{equation}
  |(w^{*} -w)^\top x_i| \geq \gamma_i  |a_i - \alpha_i^*| .
\label{eqn:primal-src}
\end{equation}
\end{proposition}

For SVM, we can take $\gamma_i=|w^{*\top}x_i y_i-1|$, and for the absolute deviation loss, we may
take $\gamma_i=|w^{*\top} x_i - y_i|$.
Although some of $\gamma_i$ can be close to zero, in practice, most $\gamma_i$ will be away from zero, which
means $D(\alpha)$ is strongly convex at nearly all points.
Under this assumption, we may establish a convergence result for the dual sub-optimality. 

\begin{theorem}\label{thm:Lipschitz-refined-dual}
Consider Procedure SDCA with $\alpha^{(0)}=0$. 
Assume that $\phi_i$ is $L$-Lipschitz for all $i$ and it satisfies \eqref{eqn:dual-rsc}.
Define $N(u)= \#\{i: \gamma_i < u\}$.
To obtain a dual-suboptimality of $\E [D(\alpha^*)-D(\alpha^t)] \leq \epsilon_D$, it suffices to have a total number of
iterations of
\[
t \geq 2(n/s) \log (2/\epsilon_D) ,
\]
where $s \in [0,1]$ satisfies $\epsilon_D \geq 8 L^2 (s/\lambda n) N(s/\lambda n)/n$.
\end{theorem}

\begin{remark}
  if $N(s/\lambda n)/n$ is small, then \thmref{thm:Lipschitz-refined-dual} is superior to \thmref{thm:Lipschitz}
  for the convergence of the dual objective function.
  We consider three scenarios. The first scenario is when $s=1$. If $N(1/\lambda n)/n$ is small, and
  $\epsilon_D \geq 8 L^2 (1/\lambda n) N(1/\lambda n)/n$, then the convergence is linear.
  The second scenario is when there exists $s_0$ so that $N(s_0/\lambda n)=0$ (for SVM, it means that
  $\lambda n |w^{*\top} x_i y_i -1| \geq s_0$ for all $i$), and since $\epsilon_D \geq 8 L^2 (s_0/\lambda n) N(s_0/\lambda n)/n$,
  we again have a linear convergence of $ (2n/s_0) \log (2/\epsilon_D)$.
  In the third scenario, we assume that $N(s/\lambda n)/n = O [(s/\lambda n)^\nu]$ for some $\nu>0$, we can take
  $\epsilon_D = O((s/\lambda n)^{1+\nu})$ and obtain
  \[
  t \geq O( \lambda^{-1} \epsilon_D^{-1/(1+\nu)} \log (2/\epsilon_D)) .
  \]
  The $\log (1/\epsilon_D)$ factor can be removed in this case with
  a slightly more complex analysis. This result is again superior to \thmref{thm:Lipschitz} for dual convergence.
\end{remark}

The following result shows fast convergence of duality gap using \thmref{thm:Lipschitz-refined-dual}.
\begin{theorem} \label{thm:Lipschitz-refined-gap}
Consider Procedure SDCA with $\alpha^{(0)}=0$. 
Assume that $\phi_i$ is $L$-Lipschitz for all $i$ and it satisfies \eqref{eqn:refined-sc}.
Let $\rho \leq 1$ be the largest eigenvalue of the matrix $n^{-1}\sum_{i=1} x_i x_i^\top$.
Define $N(u)= \#\{i: \gamma_i < u\}$. Assume that at time $T_0 \geq n$, we have dual suboptimality of
$\E [D(\alpha^*)-D(\alpha^{(T_0)})] \leq \epsilon_D$, and define
\[
\tilde{\epsilon}_P= \inf_{\gamma >0}
\left[
  \frac{N(\gamma)}{n} 4 L^2 +\frac{2\epsilon_D}{\min(\gamma,\lambda \gamma^2/(2\rho))} \right] ,
\]
then at time $T=2 T_0$, we have  
\[
\E [P(\bar{w})-D(\bar{\alpha})]\leq \epsilon_D + \frac{\tilde{\epsilon}_P}{2\lambda T_0} .
\]
\end{theorem}

If for some $\gamma$, $N(\gamma)/n$ is small, then \thmref{thm:Lipschitz-refined-gap} 
is superior to \thmref{thm:Lipschitz}. Although the general dependency may be complex, the improvement over \thmref{thm:Lipschitz}
can be more easily seen in the special case that  
$N(\gamma)=0$ for some  $\gamma>0$. In fact, in this case we have $\tilde{\epsilon}_P= O(\epsilon_D)$, and thus
\[
\E [P(\bar{w})-D(\bar{\alpha})] = O(\epsilon_D) . 
\]
This means that the convergence rate for duality gap in \thmref{thm:Lipschitz-refined-gap}  is linear as implied by the linear convergence of $\epsilon_D$ in \thmref{thm:Lipschitz-refined-dual}.

\section{Examples}

We will specify the SDCA algorithms for a few common loss
functions. For simplicity, we only specify the algorithms without SGD
initialization. In practice, instead of complete randomization, we may
also run in epochs, and each epoch employs a random permutation of the
data. We call this variant SDCA-Perm.

\begin{myalgo}{Procedure SDCA-Perm$(\alpha^{(0)})$} 
\textbf{Let} $w^{(0)}=w(\alpha^{(0)})$ \\
\textbf{Let} $t=0$ \\
\textbf{Iterate:} for epoch $k=1,2,\ldots$ \+ \\
\textbf{Let} $\{i_1,\ldots, i_n\}$ be a random permutation of $\{1,\ldots,n\}$ \\
\textbf{Iterate:} for  $j=1,2,\ldots,n$: \+ \\
$t \leftarrow t +1$ \\
$i = i_j$ \\
 Find $\Delta \alpha_i$ to increase dual \hspace{1in} (*) \\
 $\alpha^{(t)} \leftarrow \alpha^{(t-1)} + \Delta \alpha_i e_i$ \\
 $w^{(t)} \leftarrow w^{(t-1)} + (\lambda n)^{-1} \Delta \alpha_i x_i $
 \- \- \\
\textbf{Output (Averaging option):} \+ \\
Let $\bar{\alpha}  = \frac{1}{T-T_0} \sum_{i=T_0+1}^T \alpha^{(t-1)}$ \\
Let $\bar{w}  = w(\bar{\alpha}) = \frac{1}{T-T_0} \sum_{i=T_0+1}^T w^{(t-1)}$ \\
return $\bar{w}$  \- \\
\textbf{Output (Random option):} \+ \\
Let $\bar{\alpha}=\alpha^{(t)}$ and $\bar{w}  = w^{(t)}$ for some random $t \in T_0+1,\ldots,T$ \\
return $\bar{w}$ 
\end{myalgo}

\subsection*{Lipschitz loss}

Hinge loss is used in SVM. We have $\phi_i(u)=\max\{0,1-y_i u\}$ and
$\phi_i^*(-a)=-a y_i$ with $a y_i \in [0,1]$.  Absolute deviation loss
is used in quantile regression. We have $\phi_i(u) = |u-y_i|$ and
$\phi_i^*(-a)=-a y_i$ with $a \in [-1,1]$.

For the hinge loss, step (*) in Procedure SDCA-Perm has a closed form
solution as
\[
\Delta \alpha_i  = y_i \max\left(0, \min\left(1, \frac{1-  x_i^\top w^{(t-1)} y_i}{\|x_i\|^2/(\lambda n)}  + \alpha^{(t-1)}_i y_i \right) \right) -\alpha^{(t-1)}_i .
\]

For absolute deviation loss, step (*) in Procedure SDCA-Perm has a closed form
solution as 
\[
\Delta \alpha_i  = \max\left(-1, \min\left(1, \frac{y_i -  x_i^\top w^{(t-1)}}{\|x_i\|^2/(\lambda n)}   + \alpha^{(t-1)}_i \right) \right) -\alpha^{(t-1)}_i .
\]

Both hinge loss and absolute deviation loss are
$1$-Lipschitz. Therefore, we expect a convergence behavior of no worse
than
\[
O\left( n \log n + \frac{1}{\lambda \epsilon} \right) 
\]
without SGD initialization based on \thmref{thm:Lipschitz}.  The refined analysis in \secref{sec:refined-analysis} suggests a rate that
can be significantly better, and this is confirmed with our empirical
experiments.

\subsection*{Smooth loss}
Squared loss is used in ridge regression. We have
$\phi_i(u) = (u-y_i)^2$, and $\phi_i^*(-a)= -a y_i + a^2/4$.
Log loss is used in logistic regression. We have $\phi_i(u) = \log(1+\exp(-y_i u))$, and
$\phi_i^*(-a)= a y_i \log (a y_i) + (1- a y_i) \log (1-a y_i)$ with $a y_i \in [0,1]$. 

For squared loss, step (*) in Procedure SDCA-Perm has a closed form
solution as
\[
\Delta \alpha_i  = \frac{y_i -  x_i^\top w^{(t-1)} - 0.5 \alpha^{(t-1)}_i}{0.5 + \|x_i\|^2/(\lambda n)} .
\]

For log loss, step (*) in Procedure SDCA-Perm does not have a closed
form solution. However, one may start with the approximate solution,  
\[
\Delta \alpha_i  = \frac{ (1 + \exp(x_i^\top w^{(t-1)} y_i))^{-1} y_i - \alpha^{(t-1)}_i}{\max(1,0.25 + \|x_i\|^2/(\lambda n))} ,
\]
and further use several steps of Newton's update to get a more accurate solution.

Finally, we present a smooth variant of the hinge-loss, as defined below.
Recall that the hinge loss function (for positive labels) is $\phi(u)
= \max\{0,1-u\}$ and we have $\phi^*(-a)=-a $ with $a \in [0,1]$.  
Consider adding to $\phi^*$ the term $\frac{\gamma}{2} a^2 $ which
yields the $\gamma$-strongly convex function
\[
\tilde{\phi}_\gamma^*(a) = \phi^*(a) + \frac{\gamma}{2} a^2 ~.
\]
Then, its conjugate, which is defined below, is $(1/\gamma)$-smooth. We refer to it as the \emph{smoothed hinge-loss} (for positive labels):
\begin{align} \nonumber
\tilde{\phi}_\gamma(x) &= \max_{a \in [-1,0]} \left[ ax - a - \frac{\gamma}{2} a^2 \right] 
\\ \label{eqn:smoothHinge}
&= \begin{cases}
0  & x > 1 \\
1-x-\gamma/2  & x < 1-\gamma \\
\frac{1}{2\gamma}(1-x)^2    & \textrm{o.w.}
\end{cases} 
\end{align}

For the smoothed hinge loss, step (*) in Procedure SDCA-Perm has a closed form
solution as
\[
\Delta \alpha_i  = y_i \max\left(0, \min\left(1, \frac{1-  x_i^\top
      w^{(t-1)} y_i - \gamma\,\alpha^{(t-1)}_i\,y_i}{\|x_i\|^2/(\lambda n)+\gamma}  + \alpha^{(t-1)}_i y_i \right) \right) -\alpha^{(t-1)}_i .
\]

Both log loss and squared loss are $1$-smooth. The smoothed-hinge loss
is $1/\gamma$ smooth. Therefore we expect a convergence behavior of no worse than
\[
O\left( \left(n + \frac{1}{\gamma\,\lambda}\right) \log \frac{1}{\epsilon} \right) .
\]
This is confirmed in our empirical experiments.

\section{Proofs}

We denote by $\partial \phi_i(a)$ the set of sub-gradients of $\phi_i$
at $a$. We use the notation $\phi_i'(a)$ to denote some sub-gradient
of $\phi_i$ at $a$.  For convenience, we list the following simple
facts about primal and dual formulations, which will used in the
proofs.  For each $i$, we have
\[
-\alpha_i^* \in \partial \phi_i(w^{*\top} x_i) , \quad
w^{*\top} x_i \in \partial \phi_i^*(-\alpha_i^*) ,
\]
and 
\[
w^* = \frac{1}{\lambda n} \sum_{i=1}^n \alpha_i^* x_i .
\]

The proof of our basic results stated in \thmref{thm:smooth} and \thmref{thm:Lipschitz} relies on the fact that
for SDCA, it is possible to lower bound the expected increase in dual objective
by the duality gap. This key observation is stated in Lemma~\ref{lem:key}. Note that the
duality gap can be further lower bounded using dual suboptimality. Therefore 
Lemma~\ref{lem:key} implies a recursion for dual suboptimality which can be solved
to obtain the convergence of dual objective.
We can then apply Lemma~\ref{lem:key} again, and the convergence of dual objective
implies an upper bound of the duality gap, which leads to the basic theorems.
The more refined results in \secref{sec:sgd} and \secref{sec:refined-analysis} use similar strategies but with Lemma~\ref{lem:key} 
replaced by its variants.

\subsection{Proof of \thmref{thm:smooth}}
The key lemma, which estimates the expected increase in dual objective in terms of the duality gap, can be stated as follows.
\begin{lemma} \label{lem:key}
Assume that $\phi^*_i$ is $\gamma$-strongly-convex (where $\gamma$ can
be zero). Then, for any iteration $t$ and any $s \in [0,1]$ we have
\[
\E[D(\alpha^{(t)})-D(\alpha^{(t-1)})] \ge  \frac{s}{n}\,
\E[P(w^{(t-1)})-D(\alpha^{(t-1)})] - \left(\frac{s}{n}\right)^2
\frac{G^{(t)}}{2\lambda} ~,
\]
where
\[
G^{(t)} = \frac{1}{n} \sum_{i=1}^n \left(\|x_i\|^2 -
      \frac{\gamma(1-s)\lambda n}{s}\right) \E[(u^{(t-1)}_i-\alpha^{(t-1)}_i)^2] ,
\]
and $-u^{(t-1)}_i \in \partial \phi_i(x_i^\top w^{(t-1)})$.
\end{lemma}
\begin{proof}
Since only the $i$'th element of $\alpha$ is updated, the improvement in the dual objective can be written as
\[
n[D(\alpha^{(t)}) - D(\alpha^{(t-1)})] = \underbrace{\left(-\phi^*_i(-\alpha^{(t)}_i) -
\frac{\lambda n}{2} \|w^{(t)}\|^2\right)}_A - \underbrace{\left(-\phi^*_i(-\alpha^{(t-1)}_i) - \frac{\lambda n}{2} \|w^{(t-1)}\|^2\right)}_B
\]
By the definition of the update we have for all $s \in [0,1]$ that
\begin{align} \nonumber
A &=  \max_{\Delta \alpha_i} -\phi^*_i(-(\alpha^{(t-1)}_i + \Delta\alpha_i)) - \frac{\lambda
  n}{2} \|w^{(t-1)} + (\lambda n)^{-1} \Delta\alpha_i x_i\|^2 \\
&\ge -\phi^*_i(-(\alpha^{(t-1)}_i + s(u^{(t-1)}_i - \alpha^{(t-1)}_i) )) - \frac{\lambda
  n}{2} \|w^{(t-1)} + (\lambda n)^{-1} s(u^{(t-1)}_i - \alpha^{(t-1)}_i)
x_i\|^2 
\label{eqn:PC1}
\end{align}
From now on, we omit the superscripts and subscripts. 
Since $\phi^*$ is $\gamma$-strongly convex, we have that
\begin{equation} \label{eqn:PC2}
\phi^*(-(\alpha+ s(u - \alpha) )) = \phi^*(s (-u) + (1-s) (-\alpha))
\le s \phi^*(-u) + (1-s) \phi^*(-\alpha) - \frac{\gamma}{2} s (1-s) (u-\alpha)^2
\end{equation}
Combining this with \eqref{eqn:PC1} and rearranging terms we obtain that
\begin{align*} 
A &\ge -s \phi^*(-u) - (1-s) \phi^*(-\alpha) + \frac{\gamma}{2} s (1-s)
(u-\alpha)^2
- \frac{\lambda
  n}{2} \|w + (\lambda n)^{-1} s(u - \alpha) x\|^2  \\
&= -s \phi^*(-u) - (1-s) \phi^*(-\alpha) + \frac{\gamma}{2} s (1-s)
(u-\alpha)^2
- \frac{\lambda
  n}{2} \|w\|^2 - s(u-\alpha)w^\top x \\
& \quad - \frac{s^2(u-\alpha)^2}{2\lambda n} \|x\|^2 \\
&= \underbrace{-s(\phi^*(-u)+uw^\top x)}_{s\,\phi(w^\top x)} + \underbrace{(-\phi^*(-\alpha) - \frac{\lambda
  n}{2} \|w\|^2)}_B + \frac{s}{2}\left(\gamma(1-s)-\frac{s
  \|x\|^2}{\lambda n}\right)(u-\alpha)^2 \\
& \quad + s(\phi^*(-\alpha)+\alpha w^\top x) ,
\end{align*}
where we used $-u \in \partial \phi(w^\top x)$ which yields
$\phi^*(-u) = -u w^\top x - \phi(w^\top x)$. Therefore
\begin{equation} \label{eqn:PC3}
A-B \ge s\left[\phi(w^\top x) + \phi^*(-\alpha) + \alpha w^\top x +
\left(\frac{\gamma(1-s)}{2} - \frac{s
  \|x\|^2}{2\lambda n}\right)(u-\alpha)^2 \right] ~.
\end{equation}
Next note that 
\begin{align*} 
P(w)-D(\alpha) &= \frac{1}{n} \sum_{i=1}^n \phi_i(w^\top x_i) +
  \frac{\lambda}{2} w^\top w - \left(-\frac{1}{n} \sum_{i=1}^n
  \phi^*_i(-\alpha_i) - \frac{\lambda}{2} w^\top w\right) \\
&= \frac{1}{n} \sum_{i=1}^n \left( \phi_i(w^\top x_i) +
  \phi^*_i(-\alpha_i) 
+ \alpha_i w^\top x_i\right) 
\end{align*}
Therefore, if we take expectation of \eqref{eqn:PC3} w.r.t. the choice
of $i$ we obtain that
\[
\frac{1}{s}\, \E[A-B] \ge  \E[P(w)-D(\alpha)] - \frac{s}{2\lambda
    n} \cdot \underbrace{\frac{1}{n} \sum_{i=1}^n \left(\|x_i\|^2 -
      \frac{\gamma(1-s)\lambda n}{s}\right) \rE (u_i-\alpha_i)^2 }_{= G^{(t)}} .
\]
We have obtained that
\begin{equation} \label{eqn:DualSObyGap}
\frac{n}{s}\, \E[D(\alpha^{(t)})-D(\alpha^{(t-1)})] \ge
\E[P(w^{(t-1)})-D(\alpha^{(t-1)})] - \frac{s\,G^{(t)}}{2\lambda n} ~.
\end{equation}
Multiplying both sides by $s/n$ concludes the proof of the lemma.
\end{proof}

We also use the following simple lemma:
\begin{lemma} \label{lem:LBdual}
For all $\alpha$, $D(\alpha) \le P(w^*) \le P(0) \le 1$. In addition, 
$D(0) \ge 0$. 
\end{lemma}
\begin{proof}
The first inequality is by weak duality, the second is by the
optimality of $w^*$, and the third by the assumption that $\phi_i(0)
\le 1$. For the last inequality we use
$-\phi^*_i(0) =- \max_z (0-\phi_i(z)) = \min_z \phi_i(z) \ge 0$, 
which yields $D(0) \ge 0$. 
\end{proof}

Equipped with the above lemmas we are ready to prove
\thmref{thm:smooth}. 
\begin{proof}[Proof of \thmref{thm:smooth}]
The assumption that $\phi_i$ is $(1/\gamma)$-smooth implies that
$\phi_i^*$ is $\gamma$-strongly-convex. 
We will apply \lemref{lem:key} with $s =
\frac{\lambda n \gamma}{1 + \lambda n \gamma } \in [0,1]$. Recall that
$\|x_i\| \le 1$. Therefore, 
the choice of $s$ implies that $\|x_i\|^2 -
      \frac{\gamma(1-s)\lambda n}{s} \le 0$, and hence $G^{(t)} \le 0$ for
all $t$. This yields, 
\[
\E[D(\alpha^{(t)})-D(\alpha^{(t-1)})] \ge  \frac{s}{n}\,
\E[P(w^{(t-1)})-D(\alpha^{(t-1)})] ~.
\]
But since $\epsilon_D^{(t-1)} := D(\alpha^*)-D(\alpha^{(t-1)}) \le P(w^{(t-1)})-D(\alpha^{(t-1)})$ and $D(\alpha^{(t)})-D(\alpha^{(t-1)})
= \epsilon_D^{(t-1)} - \epsilon_D^{(t)}$, we obtain that 
\[
\E[ \epsilon_D^{(t)} ] \le \left(1 -
  \tfrac{s}{n}\right)\E[\epsilon_D^{(t-1)}] \le \left(1 -
  \tfrac{s}{n}\right)^t \E[\epsilon_D^{(0)}] \le \left(1 -
  \tfrac{s}{n}\right)^t \le \exp(-st/n) = \exp\left(-\frac{\lambda \gamma t}{1
    + \lambda \gamma n}\right)~.
\]
This would be smaller than $\epsilon_D$ if 
\[
t \ge \left(n +
  \tfrac{1}{\lambda \gamma}\right) \, \log(1/\epsilon_D) ~.
\]
It implies that
\begin{equation}
\E[P(w^{(t)})-D(\alpha^{(t)})]  \le \frac{n}{s}
\E[\epsilon_D^{(t)} - \epsilon_D^{(t+1)}] \le \frac{n}{s} \E[\epsilon_D^{(t)}] . \label{eqn:dgap-bound-smooth}
\end{equation}
So, requiring $\epsilon_D^{(t)} \le \frac{s}{n} \epsilon_P$ we obtain
a duality gap of at most $\epsilon_P$. This means that we should
require
\[
t \ge \left(n +
  \tfrac{1}{\lambda \gamma}\right) \, \log( (n + \tfrac{1}{\lambda \gamma})   \cdot \tfrac{1}{\epsilon_P}) ~,
\]
which proves the first part of \thmref{thm:smooth}. 

Next, we sum \eqref{eqn:dgap-bound-smooth} over $t=T_0,\ldots,T-1$ to obtain
\[
\E\left[ \frac{1}{T-T_0} \sum_{t=T_0}^{T-1} (P(w^{(t)})-D(\alpha^{(t)}))\right] \le 
\frac{n}{s(T-T_0)} \E[D(\alpha^{(T)})-D(\alpha^{(T_0)})] .
\]
Now, if we choose $\bar{w},\bar{\alpha}$ to be either the average
vectors or a randomly chosen vector over $t \in \{T_0+1,\ldots,T\}$,
then the above implies
\[
\E[ P(\bar{w})-D(\bar{\alpha})] \le 
\frac{n}{s(T-T_0)} \E[D(\alpha^{(T)})-D(\alpha^{(T_0)})] 
\le \frac{n}{s(T-T_0)} \E[\epsilon_D^{(T_0)})] . 
\]
It follows that in order to obtain a result of 
$\E[ P(\bar{w})-D(\bar{\alpha})] \le \epsilon_P$, we only need to have
\[
\E[\epsilon_D^{(T_0)})] \leq \frac{s (T-T_0) \epsilon_P}{n} = \frac{(T-T_0) \epsilon_P}{n + \frac{1}{\lambda \gamma}} .
\]
This implies the second part of \thmref{thm:smooth}, and concludes the proof.
\end{proof}

\subsection{Proof of \thmref{thm:Lipschitz}}
Next, we turn to the case of Lipschitz loss function. We rely on the following lemma. 
\begin{lemma} \label{lem:LipConjDom}
Let $\phi : \reals \to \reals$ be an $L$-Lipschitz function. Then,
for any $\alpha$ s.t. $|\alpha| > L$ we have that $\phi^*(\alpha) =
\infty$. 
\end{lemma}
\begin{proof}
Fix some $\alpha > L$. By definition of the conjugate we have
\begin{align*}
\phi^*(\alpha)  &= \sup_x [\alpha\,x - \phi(x)] \\
&\ge -\phi(0) + \sup_{x } [\alpha\,x - (\phi(x) - \phi(0))] \\
&\ge -\phi(0) + \sup_{x } [\alpha\,x - L |x-0|] \\
&\ge -\phi(0) + \sup_{x > 0} (\alpha-L)\,x  = \infty ~.
\end{align*}
Similar argument holds for $\alpha < -L$. 
\end{proof}

A direct corollary of the above lemma is:
\begin{lemma} \label{lem:GboundLip}
Suppose that for all $i$, $\phi_i$ is $L$-Lipschitz. Let $G^{(t)}$ be
as defined in \lemref{lem:key} (with $\gamma=0$). Then, $G^{(t)} \le 4\,L^2$. 
\end{lemma}
\begin{proof}
Using \lemref{lem:LipConjDom} we know that $|\alpha^{(t-1)}_i| \le L$, and in
addition by the relation of Lipschitz and sub-gradients we have $|u^{(t-1)}_i|
\le L$. Thus, $(u^{(t-1)}_i-\alpha^{(t-1)}_i)^2 \le 4L^2$, and the proof follows. 
\end{proof}

We are now ready to prove \thmref{thm:Lipschitz}.
\begin{proof}[Proof of \thmref{thm:Lipschitz}]
Let $G = \max_t G^{(t)}$ and note that by \lemref{lem:GboundLip} we
have $G \le 4L^2$. \lemref{lem:key}, with $\gamma=0$, tells us that
\begin{equation} \label{eqn:dpeqnL}
\E[D(\alpha^{(t)})-D(\alpha^{(t-1)})] \ge  \frac{s}{n}\,
\E[P(w^{(t-1)})-D(\alpha^{(t-1)})] - \left(\frac{s}{n}\right)^2
\frac{G}{2\lambda} ~,
\end{equation}
which implies that
\[
\E[\epsilon_D^{(t)}] \le \left(1 - \tfrac{s}{n}\right) 
\E[\epsilon_D^{(t-1)}] + \left(\tfrac{s}{n}\right)^2
\tfrac{G}{2\lambda} ~.
\]
We next show that the above yields
\begin{equation} \label{eqn:DualSOL}
\E[\epsilon_D^{(t)}] \le \frac{2 G}{\lambda(2 n + t-t_0)} ~
\end{equation}
for all $t \ge t_0 = \max(0,\lceil n \log(2 \lambda n \epsilon_D^{(0)}/G ) \rceil)$.
Indeed, let us choose $s=1$, then at $t=t_0$, we have
\[
\E[\epsilon_D^{(t)}] \le \left(1 - \tfrac{1}{n}\right)^t \epsilon_D^{(0)} +
\tfrac{G}{2\lambda n^2} \tfrac{1}{1 - (1-1/n)} \le e^{-t/n} \epsilon_D^{(0)} +
\tfrac{G}{2\lambda n}
\le \tfrac{G}{\lambda n} ~.
\]
This implies that \eqref{eqn:DualSOL} holds at $t=t_0$.
For $t > t_0$ we use an inductive argument. 
Suppose the claim holds for $t-1$, therefore
\[
\E[\epsilon_D^{(t)}] \le \left(1 - \tfrac{s}{n}\right) 
\E[\epsilon_D^{(t-1)}] + \left(\tfrac{s}{n}\right)^2
\tfrac{G}{2\lambda} \le 
\left(1 - \tfrac{s}{n}\right) \tfrac{2 G}{\lambda(2n + t -1-t_0)}
 + \left(\tfrac{s}{n}\right)^2
\tfrac{G}{2\lambda} .
\]
Choosing $s = 2n/(2n-t_0+t-1) \in [0,1]$ yields
\begin{align*}
\E[\epsilon_D^{(t)}] &\le
\left(1 - \tfrac{2}{2n-t_0+t-1}\right) \tfrac{2 G}{\lambda(2n-t_0 + t -1)}
 + \left(\tfrac{2}{2n-t_0+t-1}\right)^2
\tfrac{G}{2\lambda} \\
&= \tfrac{2 G}{\lambda(2n-t_0 + t -1)}\left(1 - \tfrac{1}{2n-t_0 + t -1}\right) \\
&= \tfrac{2 G}{\lambda(2n-t_0 + t -1)}\tfrac{2n-t_0+t-2}{2n-t_0 + t -1} \\
&\le \tfrac{2 G}{\lambda(2n-t_0 + t -1)}\tfrac{2n-t_0+t-1}{2n-t_0 + t} \\
&= \tfrac{2 G}{\lambda(2n-t_0 + t)} ~.
\end{align*}
This provides a bound on the dual sub-optimality. We next turn to
bound the duality gap. 
Summing \eqref{eqn:dpeqnL} over $t=T_0+1,\ldots,T$ and rearranging terms we
obtain that
\[
\E\left[ \frac{1}{T-T_0} \sum_{t=T_0+1}^T (P(w^{(t-1)})-D(\alpha^{(t-1)}))\right] \le 
\frac{n}{s(T-T_0)} \E[D(\alpha^{(T)})-D(\alpha^{(T_0)})] + \frac{s\,G}{2\lambda n} 
\]
Now, if we choose $\bar{w},\bar{\alpha}$ to be either the average
vectors or a randomly chosen vector over $t \in \{T_0+1,\ldots,T\}$,
then the above implies
\[
\E[ P(\bar{w})-D(\bar{\alpha})] \le 
\frac{n}{s(T-T_0)} \E[D(\alpha^{(T)})-D(\alpha^{(T_0)})] +
\frac{s\,G}{2\lambda n}  ~.
\]
If $T \ge n+T_0$ and $T_0 \geq t_0$, we can set $s = n/(T-T_0)$ and combining with
\eqref{eqn:DualSOL} we obtain
\begin{align*}
\E[ P(\bar{w})-D(\bar{\alpha})] &\le 
\E[D(\alpha^{(T)})-D(\alpha^{(T_0)})] + \frac{G}{2\lambda (T-T_0)}  \\
&\le \E[D(\alpha^*)-D(\alpha^{(T_0)})]+ \frac{G}{2\lambda (T-T_0)} \\
&\le \frac{2G}{\lambda(2n-t_0+T_0)} + \frac{G}{2\lambda (T-T_0)} ~.
\end{align*}
A sufficient condition for the above to be smaller than $\epsilon_P$
is that $T_0 \ge \frac{4G}{\lambda\epsilon_P} - 2n + t_0$ and $T \ge T_0 +
\frac{G}{\lambda\epsilon_P}$. It also implies that $\E[D(\alpha^*)-D(\alpha^{(T_0)})] \leq \epsilon_P/2$.
Since we also need $T_0 \ge t_0$ and $T-T_0
\ge n$, the overall number of required iterations can be
\[
T_0 \geq \max\{t_0 , 4G/(\lambda \epsilon_P) -2 n + t_0\} ,
\quad T-T_0 \geq \max\{n,G/(\lambda \epsilon_P)\} .
\]
We conclude the proof by noticing that $\epsilon_D^{(0)} \leq 1$ using \lemref{lem:LBdual}, which
implies that $t_0 \leq \max(0,\lceil n \log(2 \lambda n/G ) \rceil)$.
\end{proof}

\subsection{Proof of \thmref{thm:SGD}}
We assume that $(\phi_t,x_t)$ are randomly drawn from a distribution $D$, and define
the population optimizer
\[
w^*_D = \argmin_w P_D(w) , \qquad P_D(w) = \rE_{(\phi,x) \sim D} \left[ \phi(w^\top x) + \frac{\lambda}{2} \|w\|^2 \right] .
\]
By definition, we have $P(w^*) \leq P(w^*_D)$ for any specific realization of $\{(\phi_t,x_t): t =1,\ldots,n\}$.
Therefore 
\[
\E P(w^*) \leq  \E P(w^*_D) = \E P_D(w^*_D) ,
\]
where the expectation is with respect to the choice of examples, and note that both $P(\cdot)$ and $w^*$ are sample dependent. 

After each step $t$, we let $\alpha^{(t)}=[\alpha_1,\ldots,\alpha_t]$, and
let $-u \in \partial \phi_{t+1}(x_{t+1}^\top w^{(t)})$.
We have, for all $t$, 
\begin{align*}
&(t+1)D_{t+1}(\alpha^{(t+1)}) - tD_t(\alpha^{(t)}) =
-\phi_{t+1}^*(-\alpha^{(t+1)}_{t+1})-(t+1)\frac{\lambda}{2} \|w^{(t+1)}\|^2
+ t \frac{\lambda}{2} \|w^{(t)}\|^2 \\
&= -\phi_{t+1}^*(-\alpha^{(t+1)}_{t+1})-\frac{1}{2(t+1)\lambda}
\|\lambda t w^{(t)} + \alpha^{(t+1)}_{t+1} x_{t+1}\|^2
+  \frac{1}{2t \lambda} \|\lambda t w^{(t)}\|^2 \\
&\ge -\phi^*_{t+1}(-u)-\frac{1}{2(t+1)\lambda}
\|\lambda t w^{(t)} + u x_{t+1}\|^2
+  \frac{1}{2t \lambda} \|\lambda t w^{(t)}\|^2 \\
&= - \phi_{t+1}^*(-u) - \frac{t}{t+1} x_{t+1}^\top w^{(t)}\,u +
\frac{1}{2\lambda}\left(\frac{1}{t}-\frac{1}{t+1}\right) \|\lambda t
w^{(t)}\|^2 - \frac{u^2\|x_{t+1}\|^2}{2(t+1)\lambda} \\ \nonumber
&= - \phi_{t+1}^*(-u) - x_{t+1}^\top w^{(t)}\,u +
\left(1-\frac{t}{t+1}\right) x_{t+1}^\top w^{(t)}\,u  +
\frac{1}{2(t+1)\lambda}\left(\frac{\|\lambda t
w^{(t)}\|^2}{t}- u^2\|x_{t+1}\|^2\right) \\
&= \phi_{t+1}(x_{t+1}^\top w^{(t)})+
\frac{1}{2(t+1)\lambda}\left(2\lambda  x_{t+1}^\top w^{(t)}\,u+ \frac{\|\lambda t
w^{(t)}\|^2}{t}- u^2\|x_{t+1}\|^2\right) \\
&= \phi_{t+1}(x_{t+1}^\top w^{(t)}) + \frac{\lambda}{2} \|w^{(t)}\|^2
+ 
\frac{1}{2(t+1)\lambda}\left(2\lambda  x_{t+1}^\top w^{(t)}\,u - \|\lambda 
w^{(t)}\|^2- u^2\|x_{t+1}\|^2\right) \\
&= \phi_{t+1}(w^{(t)}~^\top x_{t+1}) + \frac{\lambda}{2} \|w^{(t)}\|^2
- 
\frac{\|\lambda  w^{(t)} - u x_{t+1}\|^2}{2(t+1)\lambda} ~.
\end{align*}
The inequality above can be obtained by noticing that the choice of $-\alpha^{(t+1)}_{t+1}$ maximizes the dual objective. 
In the derivation of the equalities  we have used basic algebra as well as the equation
$- \phi_{t+1}^*(-u) - x_{t+1}^\top w^{(t)}\,u =\phi_{t+1}(x_{t+1}^\top w^{(t)})$ which follows
from $-u \in \partial \phi_{t+1}(x_{t+1}^\top w^{(t)})$.
Next we note that $\|\lambda w^{(t)} - u x_{t+1}\| \le 2L$ (where we used
the triangle inequality, the definition of $w^{(t)}$, and
\lemref{lem:LipConjDom}). 
Therefore,
\[
(t+1)D_{t+1}(\alpha^{(t+1)}) - tD_t(\alpha^{(t)}) \ge \phi_{t+1}(w^{(t)}~^\top x_{t+1}) + \frac{\lambda}{2} \|w^{(t)}\|^2
- 
\frac{2L^2}{(t+1)\lambda} ~.
\]
Taking expectation \emph{with respect to the choice of the examples},
and note that the $(t+1)$'th example does not depend on $w^{(t)}$ we
obtain that
\begin{align*}
&\E[(t+1)D_{t+1}(\alpha^{(t+1)}) - tD_t(\alpha^{(t)})] \\
\ge& \E[P_D(w^{(t)})] - \frac{2L^2}{(t+1)\lambda} 
\ge \E[P_D(w^{*}_D)] - \frac{2L^2}{(t+1)\lambda} \\
\ge& \E [P(w^{*})] - \frac{2L^2}{(t+1)\lambda} 
= \E [D(\alpha^{*})] - \frac{2L^2}{(t+1)\lambda} .
\end{align*}
Using \lemref{lem:LBdual} we know that $D_t(\alpha^{(t)}) \ge 0$ for
all $t$. Therefore, by summing the above over $t$ we obtain that
\[
\E[n D(\alpha^{(n)})] 
\ge n \E[D(\alpha^*)] -
\frac{2L^2\log(e n)}{\lambda}  ,
\]
which yields
\[
\E[D(\alpha^*)  - D(\alpha^{(n)}) ] \le 
\frac{2L^2\log(e n)}{\lambda n} .
\]

\subsection{Proof of \thmref{thm:Lipschitz-sgd}}
The proof is identical to the proof of \thmref{thm:Lipschitz}.
We just need to notice that at the end of the first stage, we have
$\rE \epsilon_D^{(0)} \leq 2 L^2 \log(e n)/(\lambda n)$. It 
implies that $t_0 \leq \max(0,\lceil n \log(2 \lambda n \cdot 2 L^2 \log(e n)/(\lambda nG) ) \rceil)$.

\subsection{Proof of \propref{prop:refined-sc}}

Consider any feasible dual variable $\alpha$ and the corresponding $w=w(\alpha)$. 
Since
\[
w= \frac{1}{\lambda n} \sum_{i=1}^n \alpha_i x_i , \qquad
w^*= \frac{1}{\lambda n} \sum_{i=1}^n \alpha^*_i x_i , 
\]
we have
\[
\lambda (w-w^*)^\top w^* = \frac{1}{n} \sum_{i=1}^n (\alpha_i-\alpha^*_i) w^{*\top} x_i .
\]
Therefore 
\begin{align*}
& D(\alpha^*)-D(\alpha) \\
=&
 \frac{1}{n} \sum_{i=1}^n \left[\phi^*_i(-\alpha_i) - \phi^*_i(-\alpha_i^*) + (\alpha_i-\alpha^*_i) w^{*\top} x_i \right]
+ \frac{\lambda}{2} [ w^\top w - w^{*\top} w^* - 2 (w-w^*)^\top w^* ] \\
=& \frac{1}{n} \sum_{i=1}^n \left[\phi^*_i(-\alpha_i) - \phi^*_i(-\alpha_i^*) + (\alpha_i-\alpha^*_i) w^{*\top} x_i \right]
+ \frac{\lambda}{2} (w-w^*)^\top (w -w^*) .
\end{align*}
Since $w^{*\top} x_i \in \partial \phi^*_i(-\alpha_i^*)$, we have
\[
\phi^*_i(-\alpha_i) - \phi^*_i(-\alpha_i^*) + (\alpha_i-\alpha^*_i) w^{*\top} x_i \geq \gamma_i (\alpha_i-\alpha_i^*)^2 .
\]
By combining the previous two displayed inequalities, we obtain the first desired bound.

Next, we let $u=w^{*\top} x_i$, $v=w^\top x_i$. Since $- a_i \in \partial \phi_i(v)$
and $-\alpha_i^* \in \partial \phi_i(u)$, it follows that
$u \in \partial \phi_i^*(-\alpha_i^*)$ and 
$v \in \partial \phi_i^*(-a_i)$. Therefore 
\begin{align*}
& |u-v| \cdot |\alpha_i^* - a_i|\\
=& \underbrace{[\phi_i^*(-a_i) - \phi_i^*(-\alpha_i^*) + u (a_i-\alpha_i^*)]}_{\geq 0}
+\underbrace{[\phi_i^*(-\alpha_i^*) - \phi_i^*(-a_i) + v (\alpha_i^*-a_i)]}_{\geq 0} \\
\geq&   \phi_i^*(-a_i) - \phi_i^*(-\alpha_i^*) + u (a_i-\alpha_i^*) \geq \gamma_i(u) |a_i-\alpha_i^*|^2 .
\end{align*}
This implies the second bound.

\subsection{Proof of \thmref{thm:Lipschitz-refined-dual}}
The following lemma is very similar to \lemref{lem:key} with nearly identical proof, 
but it focuses only on the convergence of dual objective function using \eqref{eqn:dual-rsc}.
\begin{lemma} \label{lem:key-dual}
Assume that \eqref{eqn:dual-rsc} is valid. Then 
for any iteration $t$ and any $s \in [0,1]$ we have
\[
\E[D(\alpha^{(t)})-D(\alpha^{(t-1)})] \ge  \frac{s}{2n}\,
\E[D(\alpha^*)-D(\alpha^{(t-1)})] 
+ \frac{ 3s \lambda}{4n} \|w^*-w^{(t-1)}\|^2 
- \left(\frac{s}{n}\right)^2 \frac{G_*^{(t)}(s)}{2\lambda} ~,
\]
where
\[
G_*^{(t)}(s) = \frac{1}{n} \sum_{i=1}^n \left(\|x_i\|^2 -
      \frac{\gamma_i \lambda n}{s}\right)\E[(\alpha_i^*-\alpha^{(t-1)}_i)^2].
\]
\end{lemma}
\begin{proof}
Since only the $i$'th element of $\alpha$ is updated, the improvement in the dual objective can be written as
\[
n[D(\alpha^{(t)}) - D(\alpha^{(t-1)})] = \underbrace{\left(-\phi^*(-\alpha^{(t)}_i) -
\frac{\lambda n}{2} \|w^{(t)}\|^2\right)}_{A_i} - 
\underbrace{\left(-\phi^*(-\alpha^{(t-1)}_i) - \frac{\lambda n}{2} \|w^{(t-1)}\|^2\right)}_{B_i} .
\]

By the definition of the update we have for all $s \in [0,1]$ that
\begin{align*} 
A_i &=  \max_{\Delta \alpha_i} -\phi^*(-(\alpha^{(t-1)}_i + \Delta\alpha_i)) - \frac{\lambda
  n}{2} \|w^{(t-1)} + (\lambda n)^{-1} \Delta\alpha_i x_i\|^2 \\
&\ge -\phi^*(-(\alpha^{(t-1)}_i + s(\alpha^*_i - \alpha^{(t-1)}_i) )) - \frac{\lambda
  n}{2} \|w^{(t-1)} + (\lambda n)^{-1} s(\alpha^*_i - \alpha^{(t-1)}_i)
x_i\|^2 .
\end{align*}
We can now apply the Jensen's inequality to obtain
\begin{align*}
A_i &\ge -s \phi_i^*(-\alpha^*_i) - (1-s) \phi_i^*(-\alpha^{(t-1)}_i) 
- \frac{\lambda
  n}{2} \|w^{(t-1)} + (\lambda n)^{-1} s(\alpha^*_i - \alpha^{(t-1)}_i) x_i\|^2  \\
&= -s [\phi_i^*(-\alpha^*_i) - \phi_i^*(-\alpha^{(t-1)}_i)] \;
 \underbrace{- \phi_i^*(-\alpha^{(t-1)}_i)
- \frac{\lambda  n}{2} \|w^{(t-1)}\|^2 }_{B_i} 
- s(\alpha^*_i-\alpha^{(t-1)}_i) x_i^\top w^{(t-1)} \\
& \qquad
- \frac{s^2(\alpha^*_i-\alpha^{(t-1)}_i)^2}{2\lambda n} \|x_i\|^2 .
\end{align*}
By summing over $i=1,\ldots,n$, we obtain
\begin{align*}
\sum_{i=1}^n [A_i-B_i] 
 \ge& -s \sum_{i=1}^n [\phi_i^*(-\alpha^*_i) - \phi_i^*(-\alpha^{(t-1)}_i)] 
- s \sum_{i=1}^n (\alpha^*_i-\alpha^{(t-1)}_i) x_i^\top w^{(t-1)} \\
& \qquad - \frac{s^2}{2\lambda n} \sum_{i=1}^n (\alpha^*_i-\alpha^{(t-1)}_i)^2\|x_i\|^2 \\
 =& -s \sum_{i=1}^n [\phi_i^*(-\alpha^*_i) - \phi_i^*(-\alpha^{(t-1)}_i) + \lambda (w^*-w^{(t-1)})^\top w^{(t-1)}] \\
& \qquad - \frac{s^2}{2\lambda n} \sum_{i=1}^n (\alpha^*_i-\alpha^{(t-1)}_i)^2\|x_i\|^2 ,
\end{align*}
where the equality follows from $\sum_{i=1}^n (\alpha^*_i-\alpha^{(t-1)}_i) x_i= \lambda n (w^*-w^{(t-1)})$.
By rearranging the terms on the right hand side using
$(w^*-w^{(t-1)})^\top w^{(t-1)} =  \|w^*\|^2/2 - \|w^{(t-1)}\|^2/2 - \|w^*-w^{(t-1)}\|^2/2$,
we obtain
\begin{align*}
& \sum_{i=1}^n [A_i-B_i] \\
\ge & -s \sum_{i=1}^n 
\left[\phi_i^*(-\alpha^*_i) - \phi_i^*(-\alpha^{(t-1)}_i) + \frac{\lambda}{2} \|w^*\|^2 - \frac{\lambda}{2} \|w^{(t-1)}\|^2 \right] 
+ \frac{\lambda s n}{2} \|w^*-w^{(t-1)}\|^2 
\\
& \qquad - \frac{s^2}{2\lambda n} \sum_{i=1}^n (\alpha^*_i-\alpha^{(t-1)}_i)^2\|x_i\|^2 \\
 =& s n [D(\alpha^*) -D(\alpha^{(t-1)})] + \frac{ s \lambda n}{2} \|w^*-w^{(t-1)}\|^2 
 - \frac{s^2}{2\lambda n} \sum_{i=1}^n (\alpha^*_i-\alpha^{(t-1)}_i)^2\|x_i\|^2 .
\end{align*}
We can now apply \eqref{eqn:dual-rsc} to obtain
\begin{align*}
\sum_{i=1}^n [A_i-B_i] 
 \geq& \frac{s n}{2} [D(\alpha^*) -D(\alpha^{(t-1)})]+ \frac{ 3s \lambda n}{4} \|w^*-w^{(t-1)}\|^2 \\
&  - \frac{s^2}{2\lambda n} \sum_{i=1}^n (\alpha^*_i-\alpha^{(t-1)}_i)^2 (\|x_i\|^2 - \gamma_i \lambda n /s) .
\end{align*}
This implies the desired result.
\end{proof}

\begin{lemma} \label{lem:GboundLip-refined}
Suppose that for all $i$, $\phi_i$ is $L$-Lipschitz. Let $G_*^{(t)}$ be  as defined in \lemref{lem:key-dual}.
Then
\[
G_*^{(t)}(s) \leq \frac{4 L^2 N(s/(\lambda n))}{n} .
\]
\end{lemma}
\begin{proof}
Similarly to the proof of \lemref{lem:GboundLip}, we know that
$(\alpha^*_i-\alpha^{(t-1)}_i)^2 \le 4L^2$. Moreover, $\|x_i\|^2 \leq 1$, and
$\|x_i\|^2 - \frac{\gamma_i \lambda n}{s} \leq 0$ when $\gamma_i \geq s/(\lambda n)$.
Therefore there are no more than $N(s/(\lambda n))$ data points $i$ such that 
$\|x_i\|^2 - \frac{\gamma_i \lambda n}{s}$ is positive. 
The desired result follows from these facts. 
\end{proof}  

\begin{proof}[Proof of \thmref{thm:Lipschitz-refined-dual}]
Let $\epsilon_D^{(t)}=\E [D(\alpha^{*})-D(\alpha^{(t)})]$, and $G_*(s) = 4 L^2 N(s/\lambda n)/n$.
We obtain from \lemref{lem:key-dual} and \lemref{lem:GboundLip-refined} that
\[
\epsilon_D^{(t)} \le  (1- s/(2n)) \epsilon_D^{(t-1)} + \left(\frac{s}{n}\right)^2 \frac{G_*(s)}{2\lambda} ~.
\]
It follows that for all $t > 0$ we have
\begin{align*}
\epsilon_D^{(t)} \le&  (1- s/(2n))^{t} \epsilon_D^{(0)} + \frac{1}{1-(1-s/(2n))} \left(\frac{s}{n}\right)^2 \frac{G_*(s)}{2\lambda} \\
\le&  e^{-s t /2n} + \left(\frac{s}{n}\right) \frac{G_*(s)}{\lambda} 
\le  e^{-s t /2n} + \epsilon_D/2 .
\end{align*}
It follows that when 
\[
t \geq  (2 n/s) \log (2 /\epsilon_D) ,
\]
we have $\epsilon_D^{(t)} \leq \epsilon_D$. 
\end{proof}

\subsection{Proof of \thmref{thm:Lipschitz-refined-gap}}
Let $\epsilon_D^{(t)}=\E [D(\alpha^*)-D(\alpha^{(t)})]$.
From \propref{prop:refined-sc}, we know that for all $t \geq T_0$:
\begin{align*}
\epsilon_D^{(t)} 
\geq &
\frac{1}{n} \sum_{i=1}^n 
\left[\gamma_i \E |\alpha_i^{(t)}-\alpha_i^*|^2 + \frac{\lambda}{2\rho} \E ((w^{(t)}-w^*)^\top x_i)^2 \right] \\
\geq &
\frac{1}{n} \sum_{i=1}^n 
\left[\gamma_i \E |\alpha_i^{(t)}-\alpha_i^*|^2 + \frac{\lambda \gamma_i^2}{2\rho} \E (u^{(t-1)}_i-\alpha^*_i)^2 \right] ,
\end{align*}
where $-u^{(t-1)}_i \in \partial \phi_i(x_i^\top w^{(t)})$.
It follows that given any $\gamma >0$, we have
\begin{align*}
& \frac{1}{n} \sum_{i=1}^n \E |\alpha_i^{(t)}-u_i^{(t-1)}|^2 \\
\leq& \frac{N(\gamma)}{n} \sup_i \E |\alpha_i^{(t)}-u^{(t-1)}_i|^2
+\frac{2}{n} \sum_{i: \gamma_i \geq \gamma} \left[\E |\alpha_i^{(t)}-\alpha_i^*|^2 +\E (u^{(t-1)}_i-\alpha^*_i)^2 \right]\\
\leq& \frac{N(\gamma)}{n} \sup_i \E |\alpha_i^{(t)}-u^{(t-1)}_i|^2
+\frac{\frac{2}{n} \sum_{i=1}^n 
\left[\gamma_i \E |\alpha_i^{(t)}-\alpha_i^*|^2+ \frac{\lambda \gamma_i^2}{2\rho} \E (u^{(t-1)}_i-\alpha^*_i)^2 \right]
}{\min(\gamma,\lambda \gamma^2/(2\rho))} \\
\leq& \frac{N(\gamma)}{n} 4 L^2 
+\frac{2\epsilon_D^{(t)}}{\min(\gamma,\lambda \gamma^2/(2\rho))} ,
\end{align*}
where \lemref{lem:GboundLip} is used for the last inequality.
Since $\gamma$ is arbitrary and $\epsilon_D^{(t)} \leq \epsilon_D$, it follows that
\[
\frac{1}{n} \sum_{i=1}^n \E |\alpha_i^{(t)}-u_i^{(t-1)}|^2 \leq \tilde{\epsilon}_P .
\]
Now plug into \lemref{lem:key}, 
we obtain for all $t \geq T_0+1$:
\begin{align*}
& \epsilon_D^{(t-1)}- \epsilon_D^{(t)} \\
\ge  &
\frac{s}{n}\, \E [P(w^{(t-1)})-D(\alpha^{(t-1)})]
- \left(\frac{s}{n}\right)^2
\frac{1}{2\lambda n } \sum_{i=1}^n \E[(u^{(t-1)}_i-\alpha^{(t-1)}_i)^2] \\
\ge  &
\frac{s}{n}\, \E [P(w^{(t-1)})-D(\alpha^{(t-1)})]
- \left(\frac{s}{n}\right)^2
\frac{\tilde{\epsilon}_P}{2\lambda} .
\end{align*}
By taking $s= n/T_0$, and summing over $t=T_0+1,\ldots,2T_0=T$, we obtain
\[
\epsilon_D
\ge
 \epsilon_D^{(T_0)}- \epsilon_D^{(T)} 
\ge 
\E [P(\bar{w})-D(\bar{\alpha})] 
- \frac{\tilde{\epsilon}_P}{2\lambda T_0} 
.
\]
This proves the desired bound.

\section{Experimental Results}

In this section we demonstrate the tightness of our theory. All our
experiments are performed with the smooth variant of the hinge-loss
defined in \eqref{eqn:smoothHinge}, where the value of $\gamma$ is
taken from the set $\{0,0.01,0.1,1\}$. Note that for $\gamma=0$ we
obtain the vanilla non-smooth hinge-loss.

In the experiments, we use $\epsilon_D$ to denote the dual sub-optimality, and $\epsilon_P$ to denote the primal sub-optimality (note that this is different than the notation in our analysis which uses $\epsilon_P$ to denote the duality gap). It follows that $\epsilon_D + \epsilon_P$ is the duality gap.

\subsection{Data}

The experiments were performed on three large datasets with very
different feature counts and sparsity, which were kindly provided by Thorsten
Joachims. The astro-ph dataset classifies abstracts of papers from the physics
ArXiv according to whether they belong in the astro-physics section; CCAT is a
classification task taken from the Reuters RCV1 collection; and cov1 is class 1
of the covertype dataset of Blackard, Jock \& Dean. The following table
provides details of the dataset characteristics.
\begin{center}
\begin{tabular}{|r|c|c|c|c|}
	\hline
Dataset & Training Size & Testing Size & Features & Sparsity  \\ \hline 
astro-ph & $29882$ & $32487$ & $99757$ & $0.08\%$ \\
CCAT & $781265$ & $23149$ & $47236$ & $0.16\%$ \\
cov1 & $522911$ & $58101$ & $54$ & $22.22\%$ \\
\hline
\end{tabular}
\end{center}

\subsection{Linear convergence for Smooth Hinge-loss}

Our first experiments are with $\phi_\gamma$ where we set $\gamma=1$. 
The goal of the experiment is to show that the convergence is indeed linear. 
We ran the SDCA algorithm for solving the regularized loss minimization problem with different values of regularization parameter $\lambda$. \figref{fig:smooth} shows the results. 
Note that a logarithmic scale is used for the vertical
axis. Therefore, a straight line corresponds to linear convergence. 
We indeed observe linear convergence for the duality gap.

\subsection{Convergence for non-smooth Hinge-loss}

Next we experiment with the original hinge loss, which is
$1$-Lipschitz but is not smooth.  We again ran the SDCA algorithm for
solving the regularized loss minimization problem with different
values of regularization parameter $\lambda$. \figref{fig:non-smooth}
shows the results.  
As expected, the overall
convergence rate is slower than the case of a smoothed hinge-loss.
However, it is also apparent that for large values of $\lambda$ a
linear convergence is still exhibited, as expected according to our
refined analysis. The bounds plotted are based on \thmref{thm:Lipschitz}, which are slower than what 
we observe, as expected from the refined analysis in \secref{sec:refined-analysis}.

\subsection{Effect of smoothness parameter}

We next show the effect of the smoothness parameter.
\figref{fig:gamma} shows the effect of the smoothness parameter on the
rate of convergence. As can be seen, the convergence becomes faster as
the loss function becomes smoother. However, the difference is more
dominant when $\lambda$ decreases.

\figref{fig:SDCAgamma} shows the effect of the smoothness parameter on
 the zero-one test error. It is noticeable that even though the
 non-smooth hinge-loss is considered a tighter approximation of the zero-one
 error, in most cases, the
 smoothed hinge-loss actually provides a lower test error than the
 non-smooth hinge-loss. In any case, it is apparent  that the smooth
 hinge-loss decreases the zero-one test error faster than the
 non-smooth hinge-loss.

\subsection{Cyclic vs. Stochastic vs. Random Permutation}

In \figref{fig:cyclic} we compare choosing dual variables at random
with repetitions (as done in SDCA) vs. choosing dual variables using a
random permutation at each epoch (as done in SDCA-Perm) vs. choosing
dual variables in a fixed cyclic order (that was chosen once at
random). As can be seen, a cyclic order does not lead to linear
convergence and yields actual convergence rate much slower than the
other methods and even worse than our bound. As mentioned before, some
of the earlier analyses such as \cite{LuoTs92} can be applied both to
stochastic and to cyclic dual coordinate ascent methods with similar
results. This means that their analysis, which can be no better than
the behavior of cyclic dual coordinate ascent, is inferior to our
analysis.  Finally, we also observe that SDCA-Perm is sometimes faster
than SDCA.

\subsection{Comparison to SGD}

We next compare SDCA to Stochastic Gradient Descent (SGD). In
particular, we implemented SGD with the update rule $w^{(t+1)} =
(1-1/t)w^{(t)} - \tfrac{1}{\lambda t} \phi'_i(w^{(t)\,\top} x_i) x_i$,
where $i$ is chosen uniformly at random and $\phi'_i$ denotes a
sub-gradient of $\phi_i$.  One clear advantage of SDCA is the
availability of a clear stopping condition (by calculating the duality
gap). In \figref{fig:SGDsmooth} and \figref{fig:SGD} we present the
primal sub-optimality of SDCA, SDCA-Perm, and SGD. As can be seen,
SDCA converges faster than SGD in most regimes.  SGD can be better if
both $\lambda$ is high and one performs a very small number of
epochs. This is in line with our theory of \secref{sec:sgd}. However,
SDCA quickly catches up.

\begin{figure}

\begin{center}
\begin{tabular}{ @{} L | @{} S @{} S @{} S @{} }
$\lambda$ & \scriptsize{astro-ph} & \scriptsize{CCAT} & \scriptsize{cov1}\\ \hline
$10^{-3}$ & 
\includegraphics[width=0.31\textwidth]{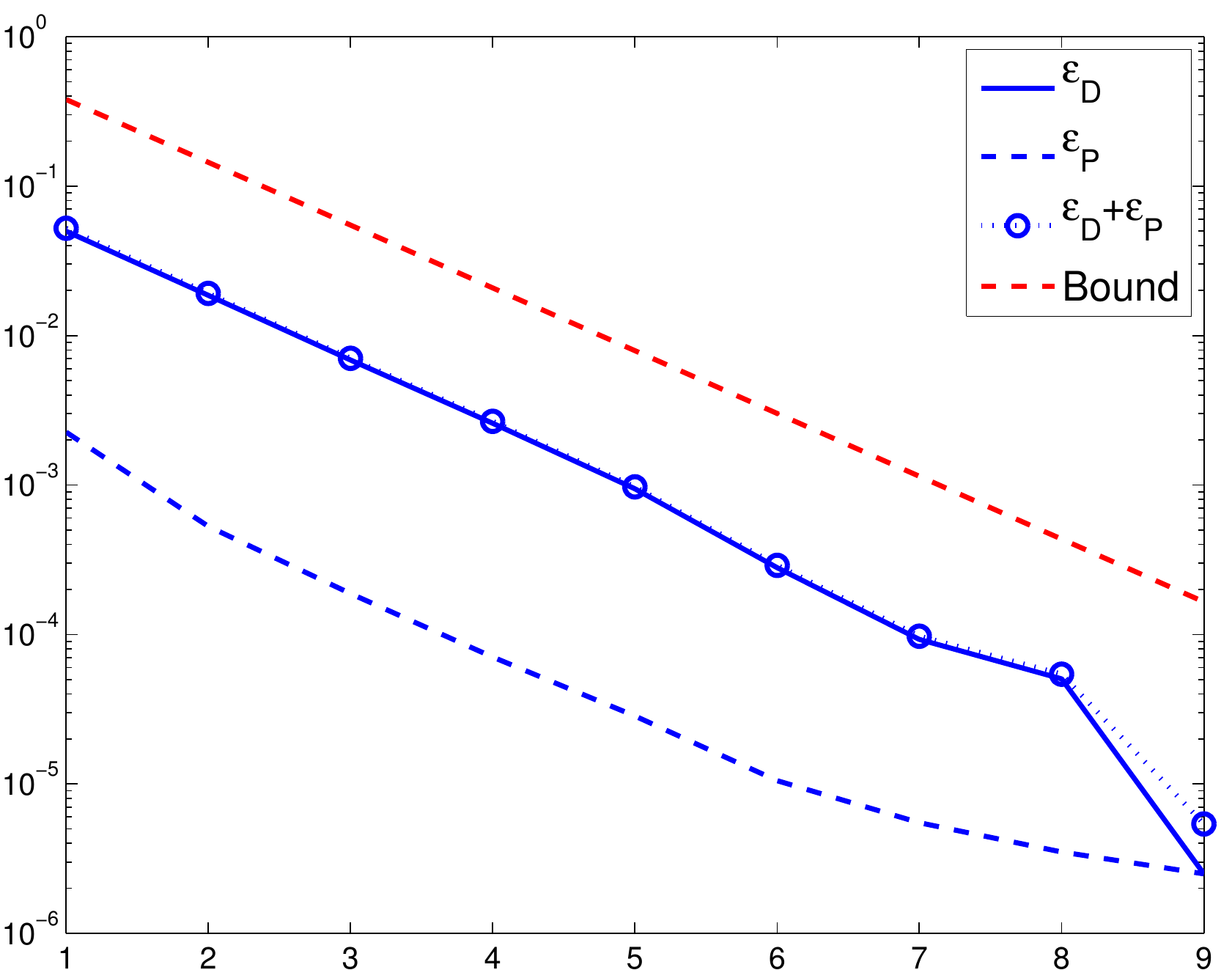} &
\includegraphics[width=0.31\textwidth]{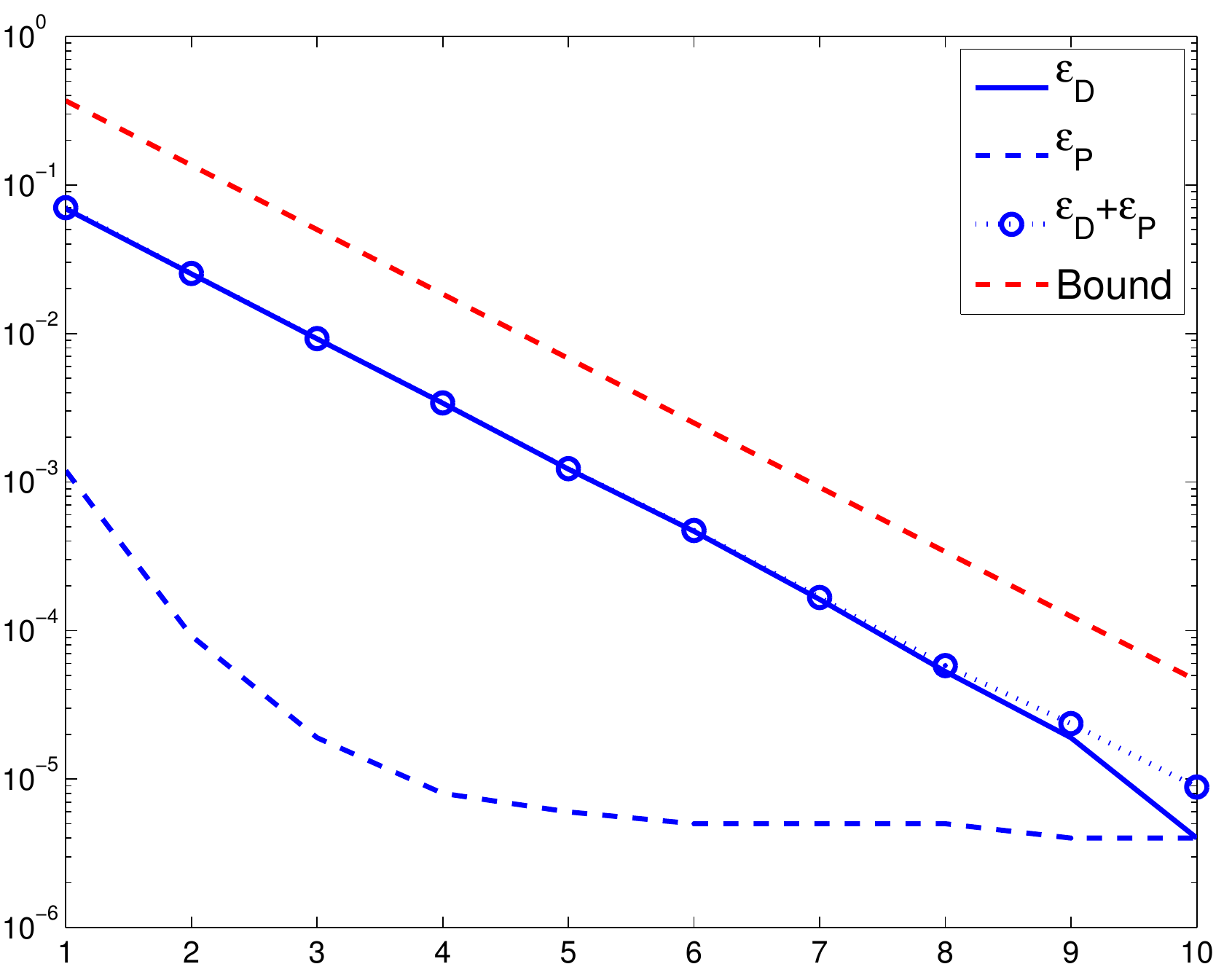} &
\includegraphics[width=0.31\textwidth]{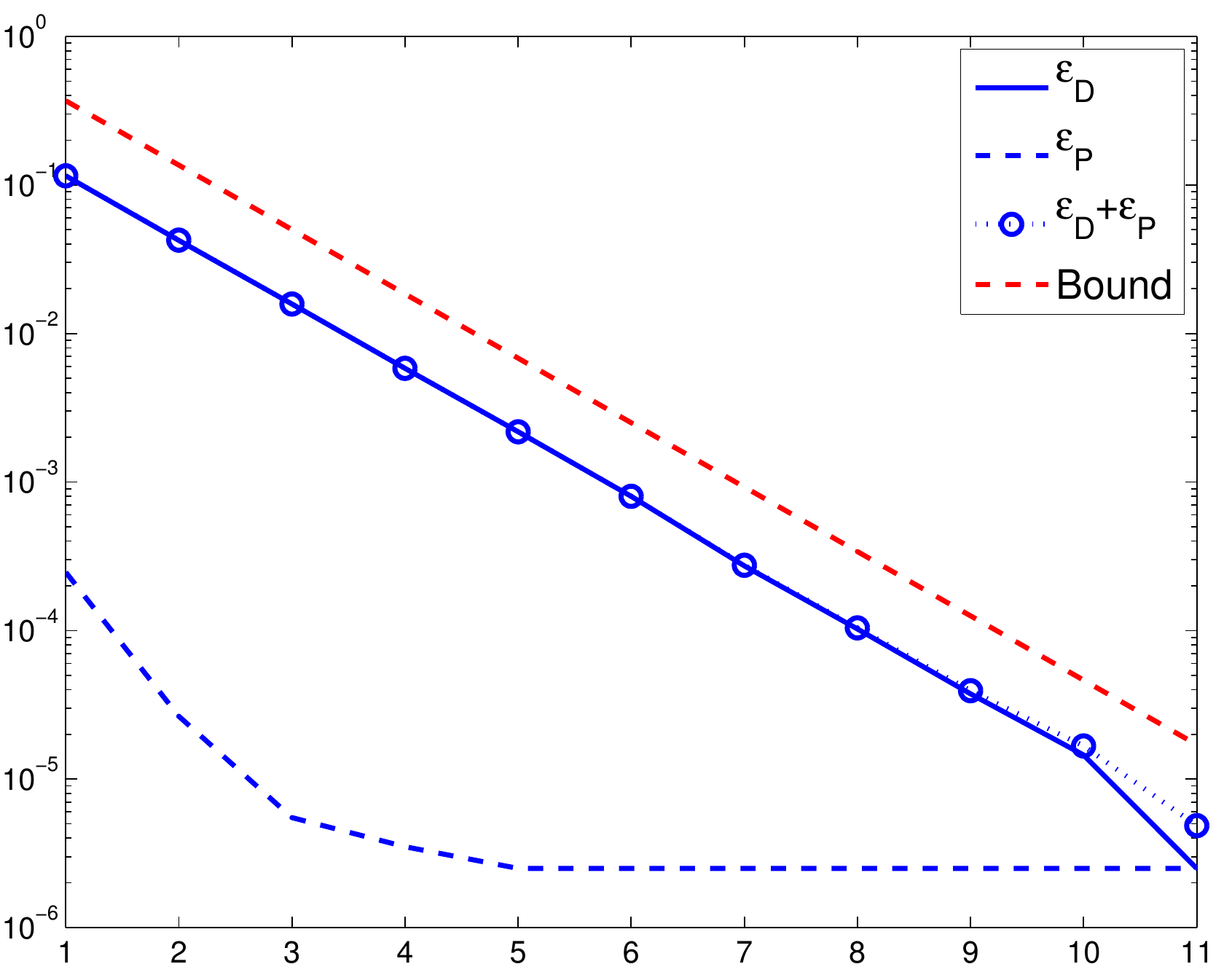}\\
$10^{-4}$ &
\includegraphics[width=0.31\textwidth]{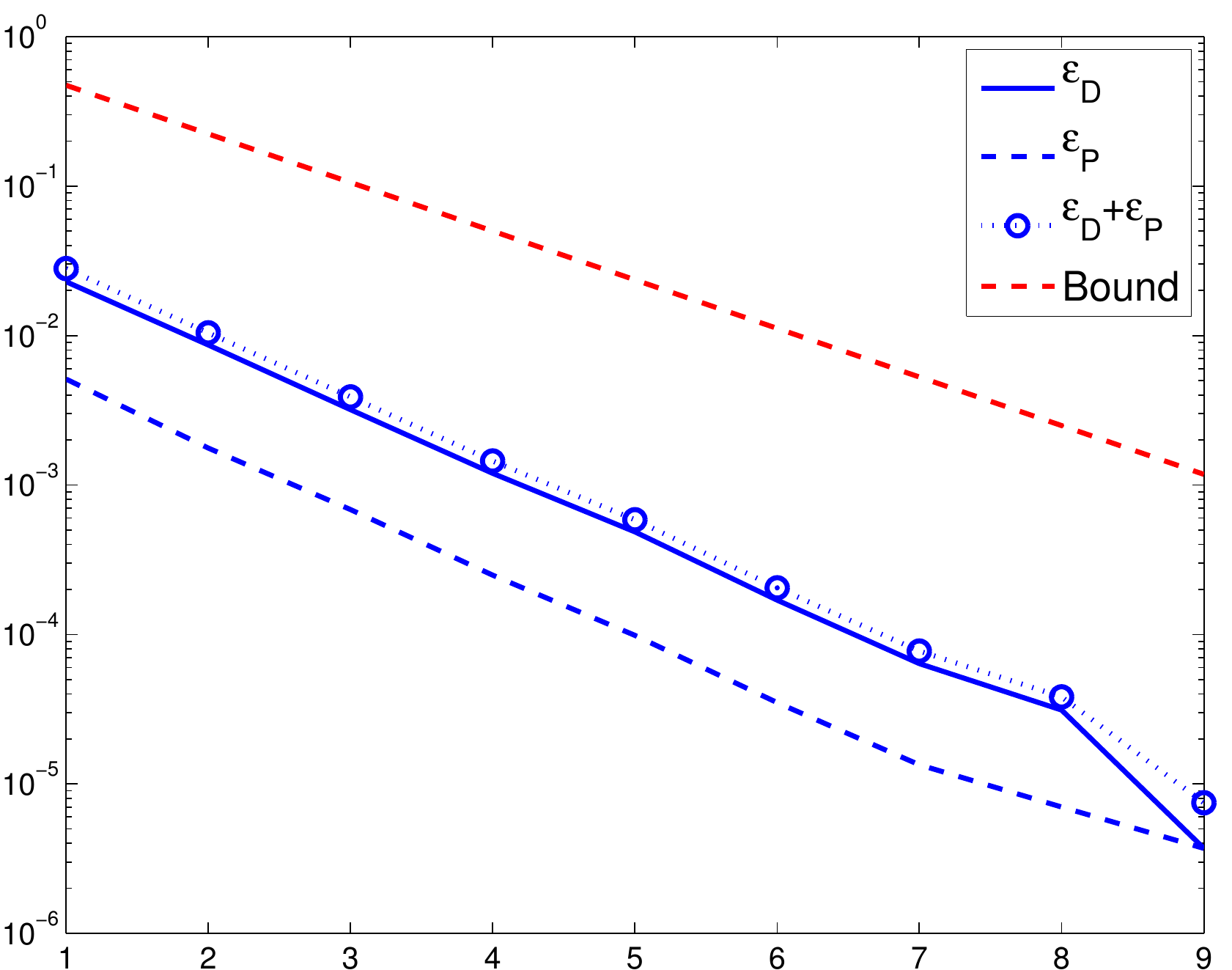} &
\includegraphics[width=0.31\textwidth]{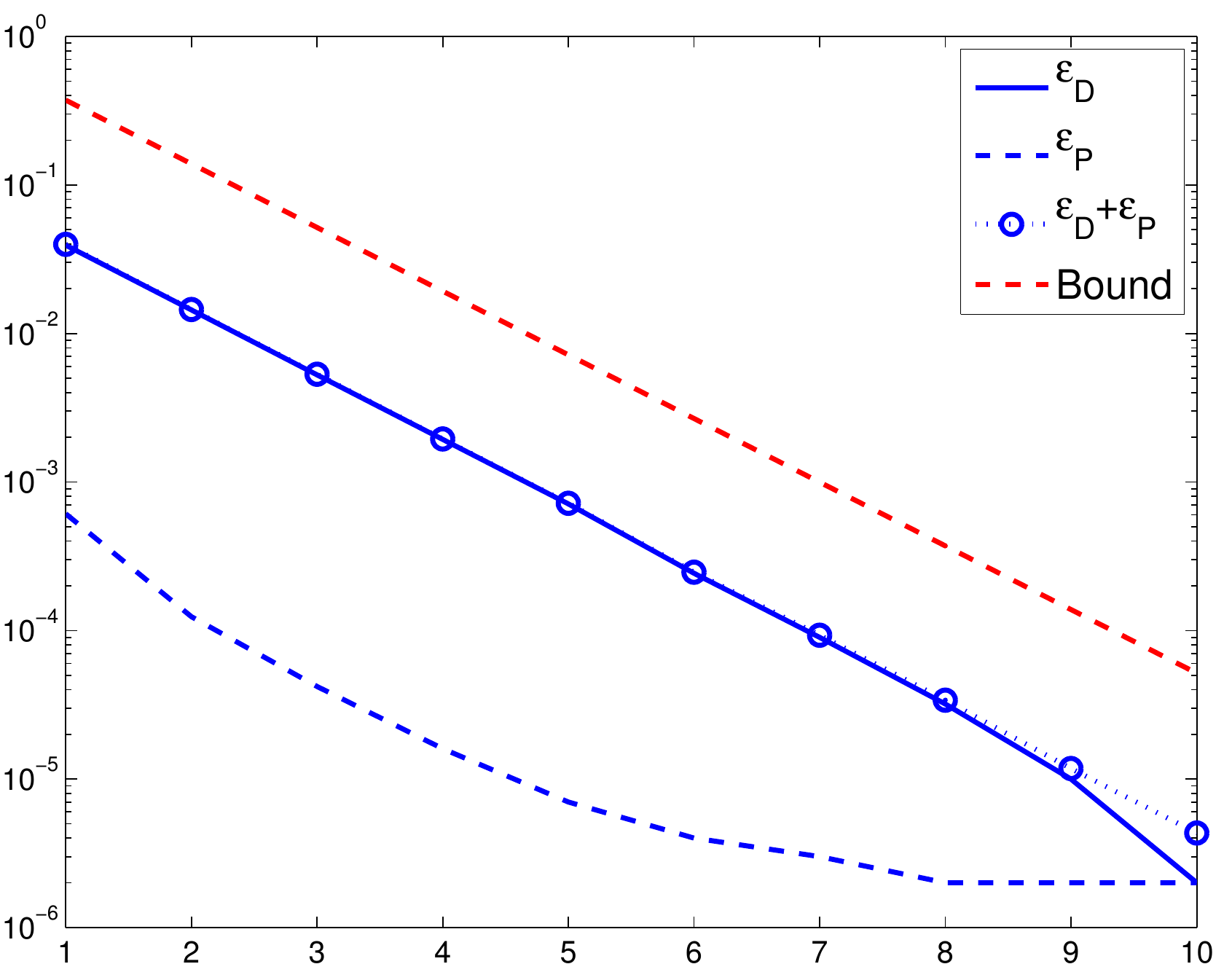} &
\includegraphics[width=0.31\textwidth]{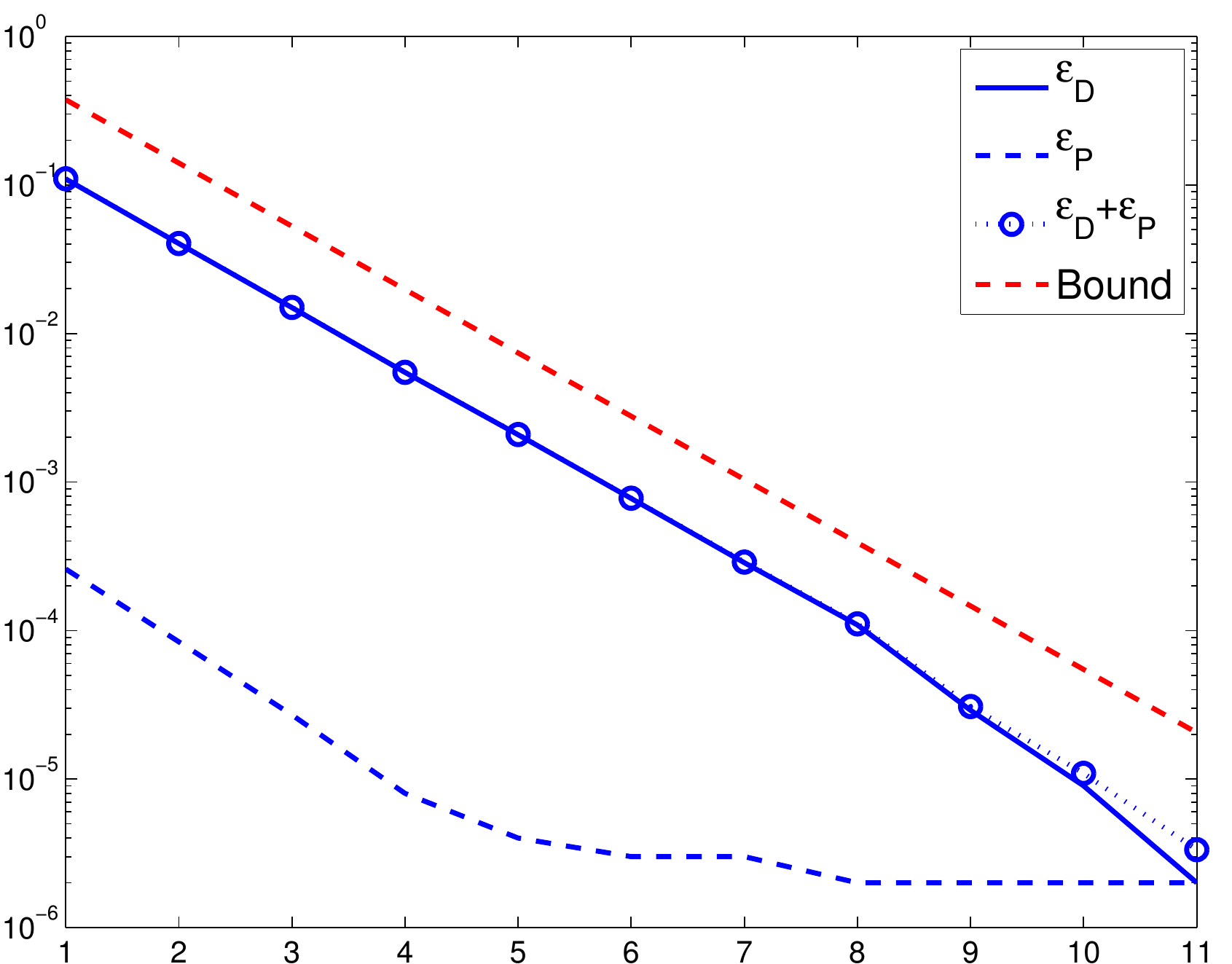}\\
$10^{-5}$ &
\includegraphics[width=0.31\textwidth]{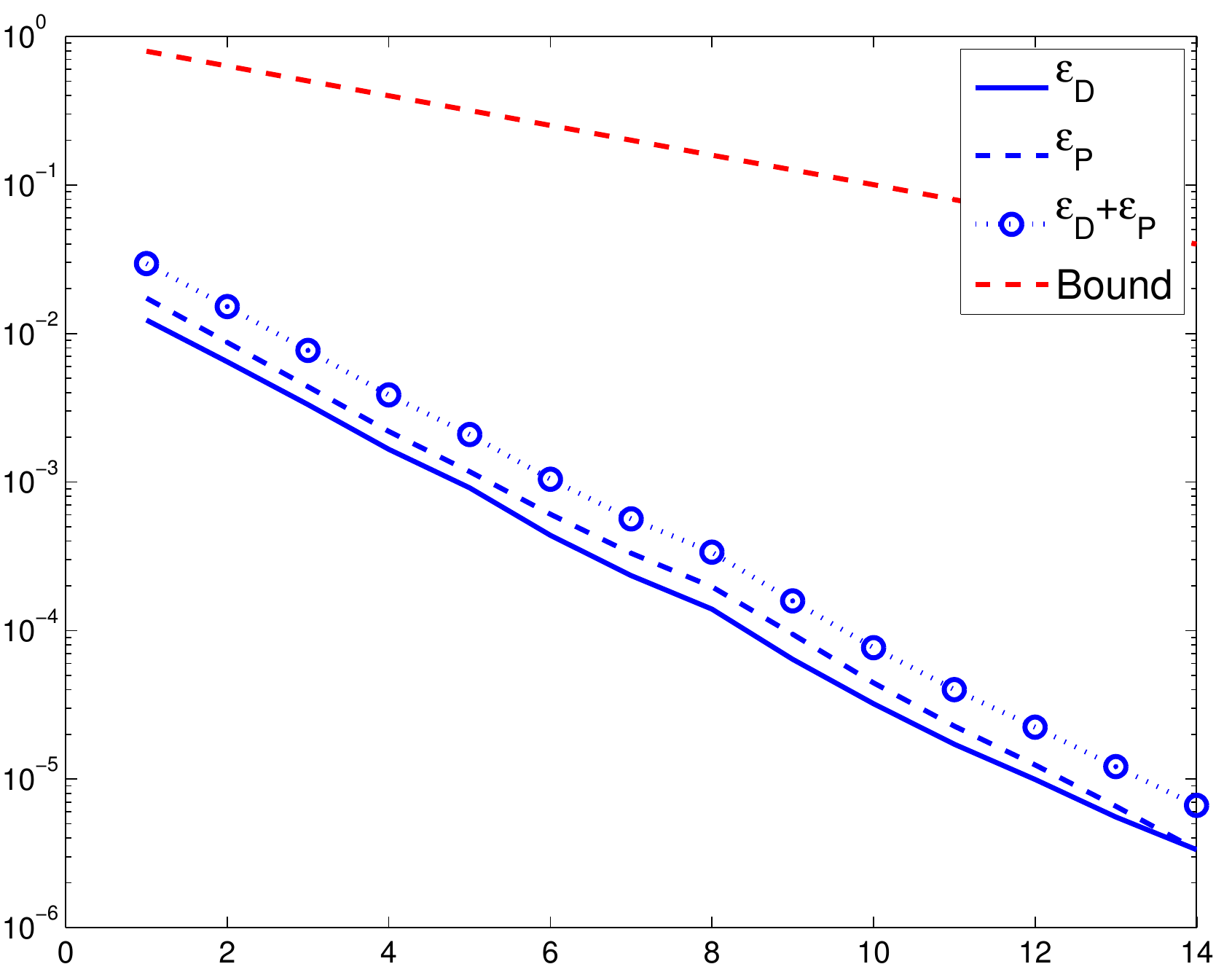} &
\includegraphics[width=0.31\textwidth]{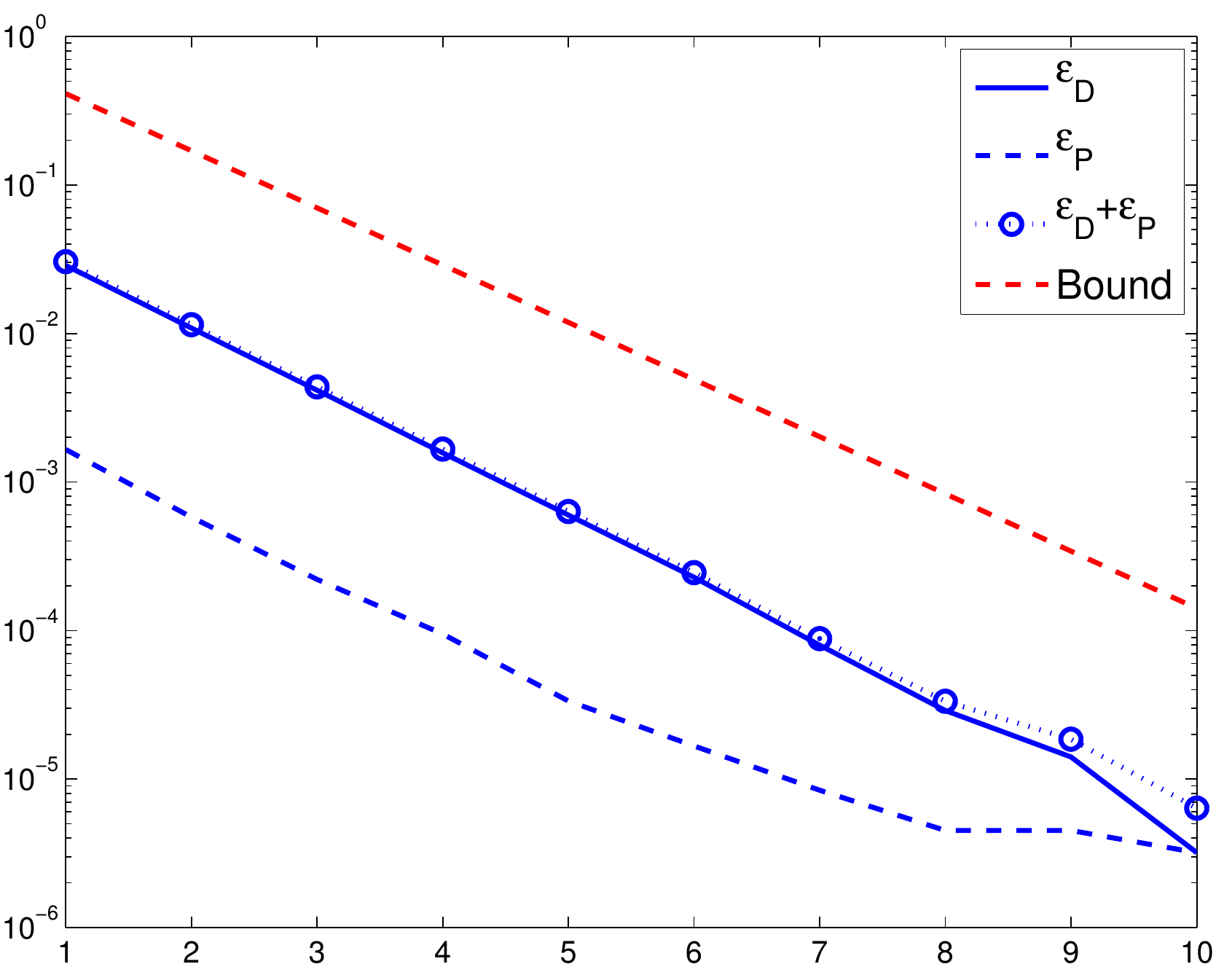} &
\includegraphics[width=0.31\textwidth]{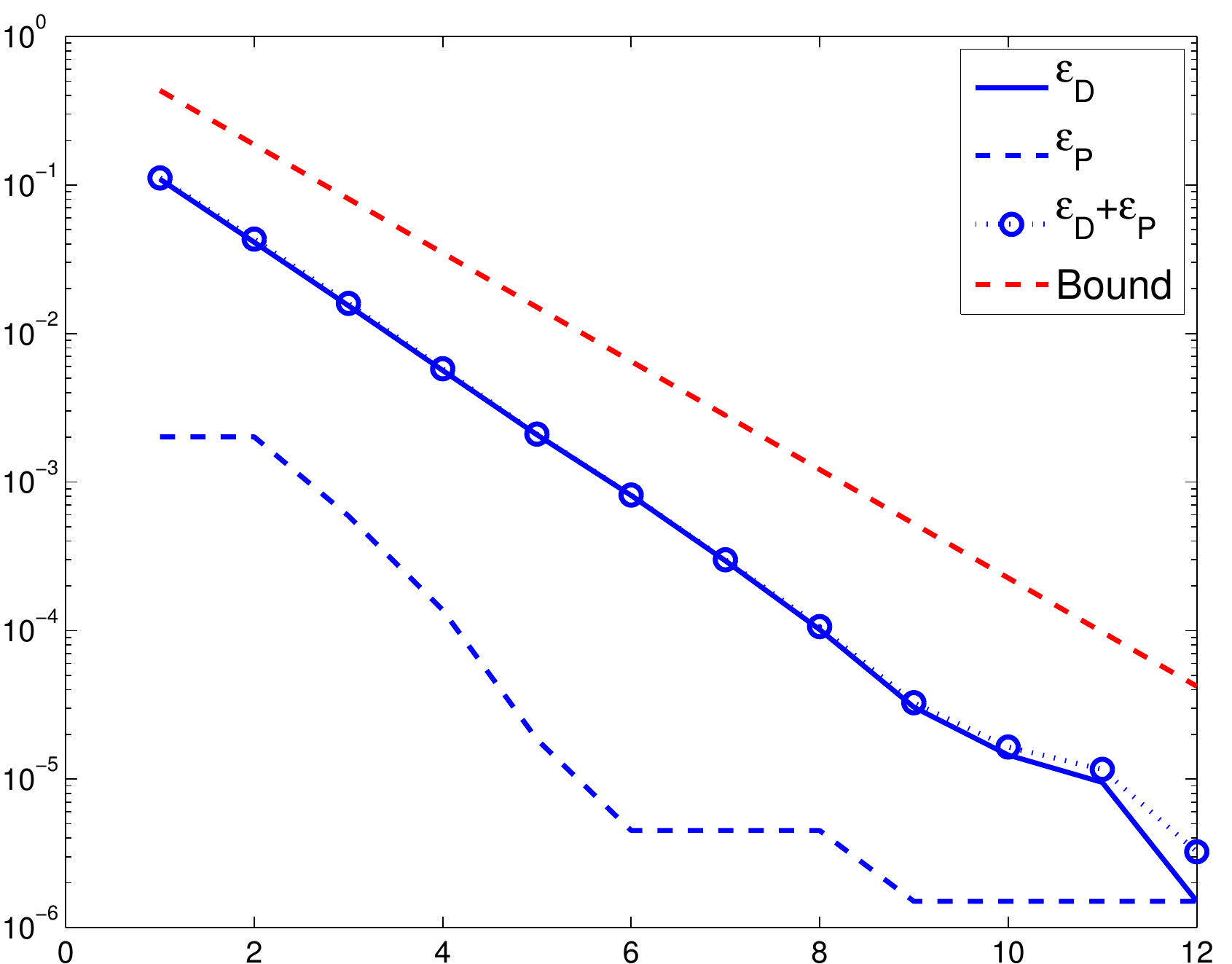}\\
$10^{-6}$ &
\includegraphics[width=0.31\textwidth]{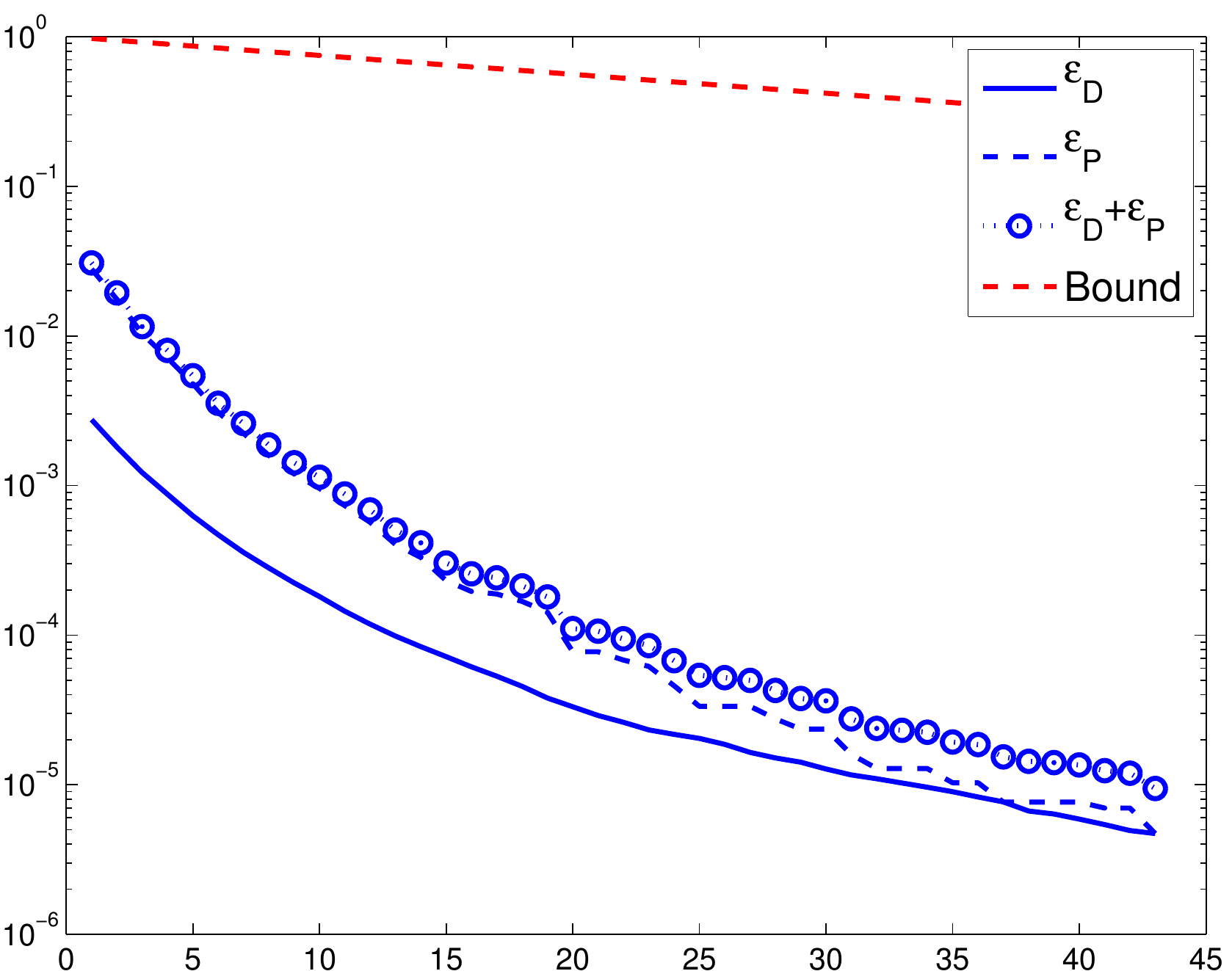} &
\includegraphics[width=0.31\textwidth]{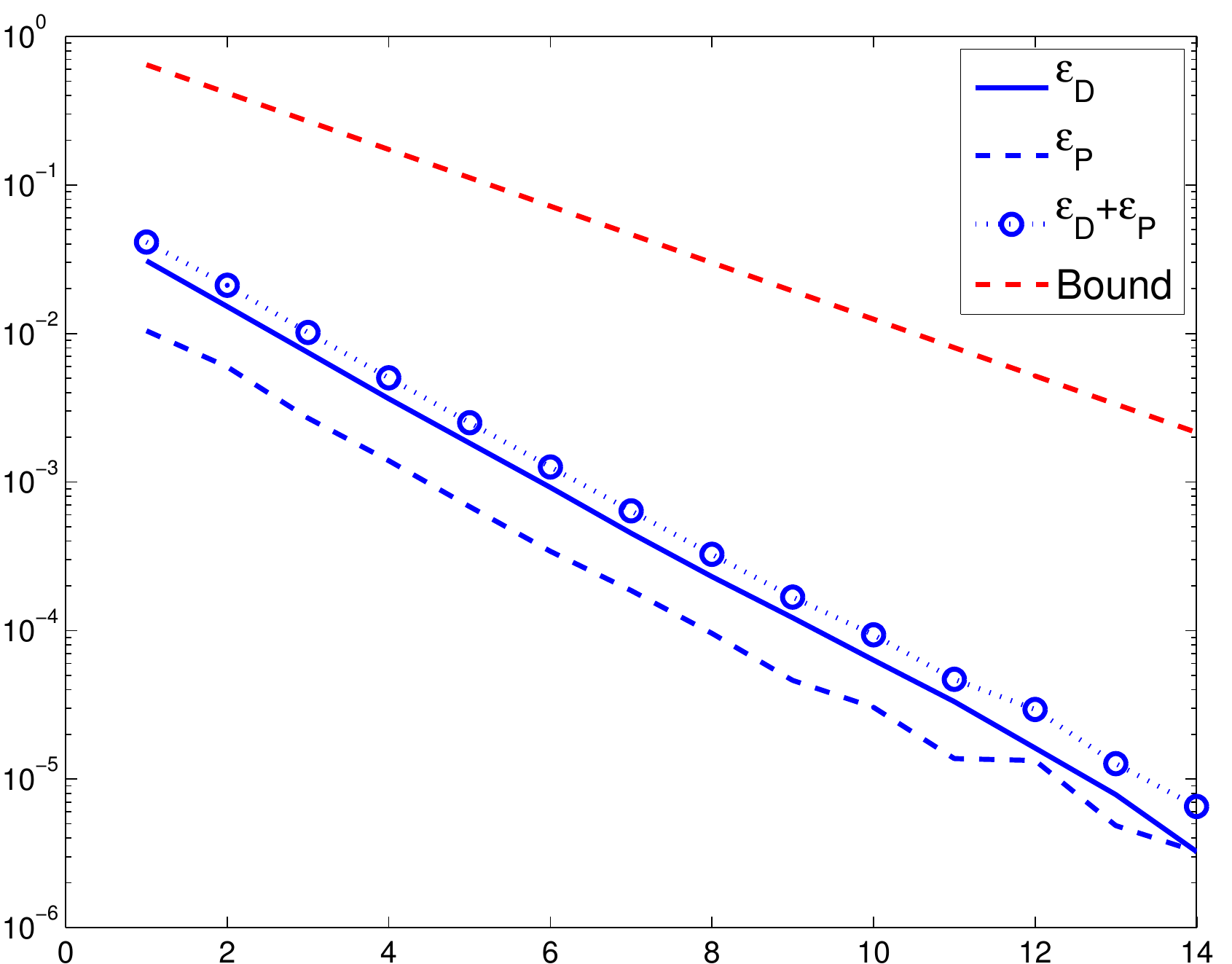} &
\includegraphics[width=0.31\textwidth]{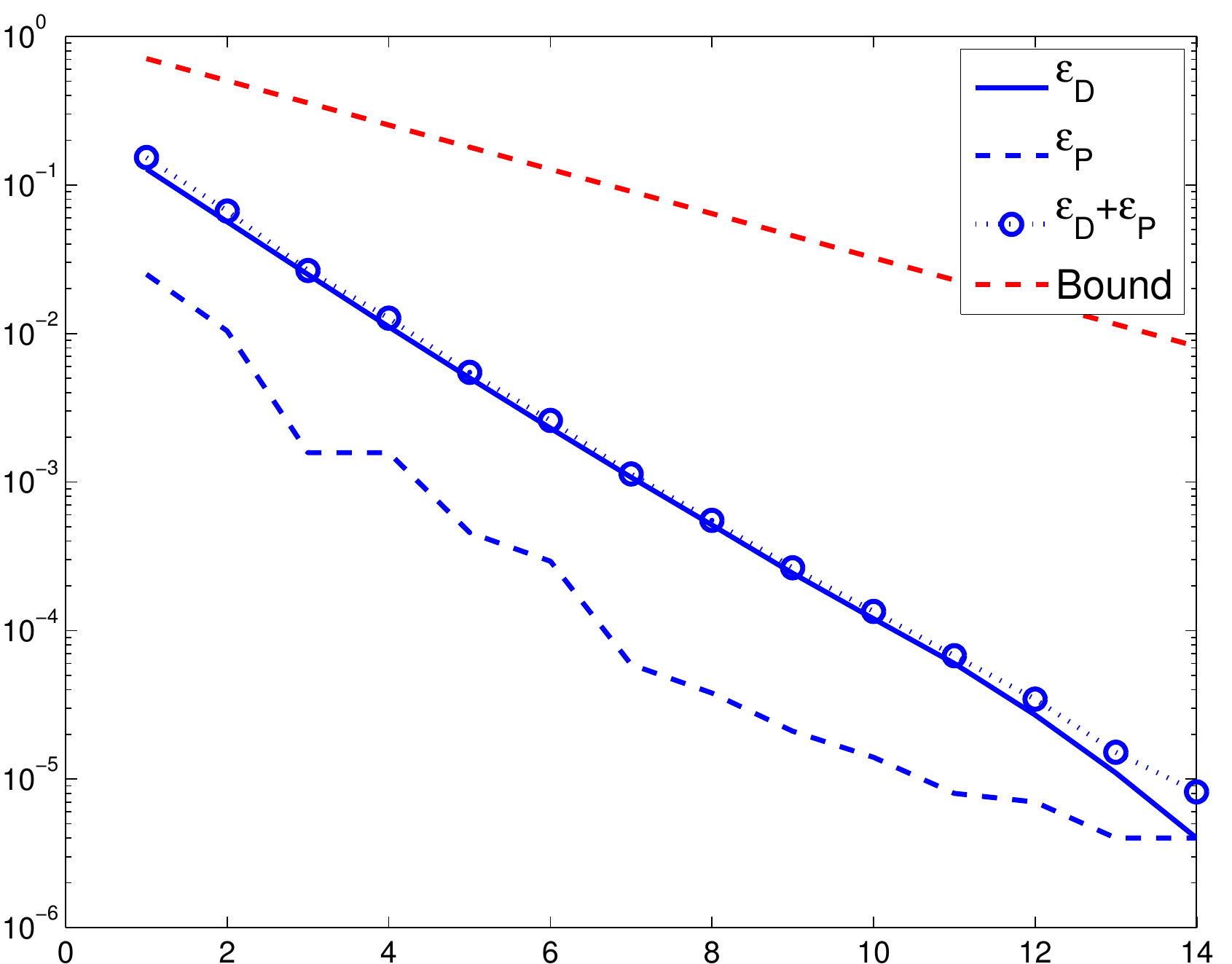}\\
\end{tabular}
\end{center}

\caption{\label{fig:smooth}Experiments with the smoothed hinge-loss
  ($\gamma=1$). The primal and dual sub-optimality, the duality gap,
  and our bound are depicted as a function of the number of epochs, on
  the astro-ph (left), CCAT (center) and cov1 (right) datasets. In all
  plots the horizontal axis is the number of iterations divided by
  training set size (corresponding to the number of epochs through the
  data).}

\end{figure}

\begin{figure}

\begin{center}
\begin{tabular}{ @{} L | @{} S @{} S @{} S @{} }
$\lambda$ & \scriptsize{astro-ph} & \scriptsize{CCAT} & \scriptsize{cov1}\\ \hline
$10^{-3}$ & 
\includegraphics[width=0.31\textwidth]{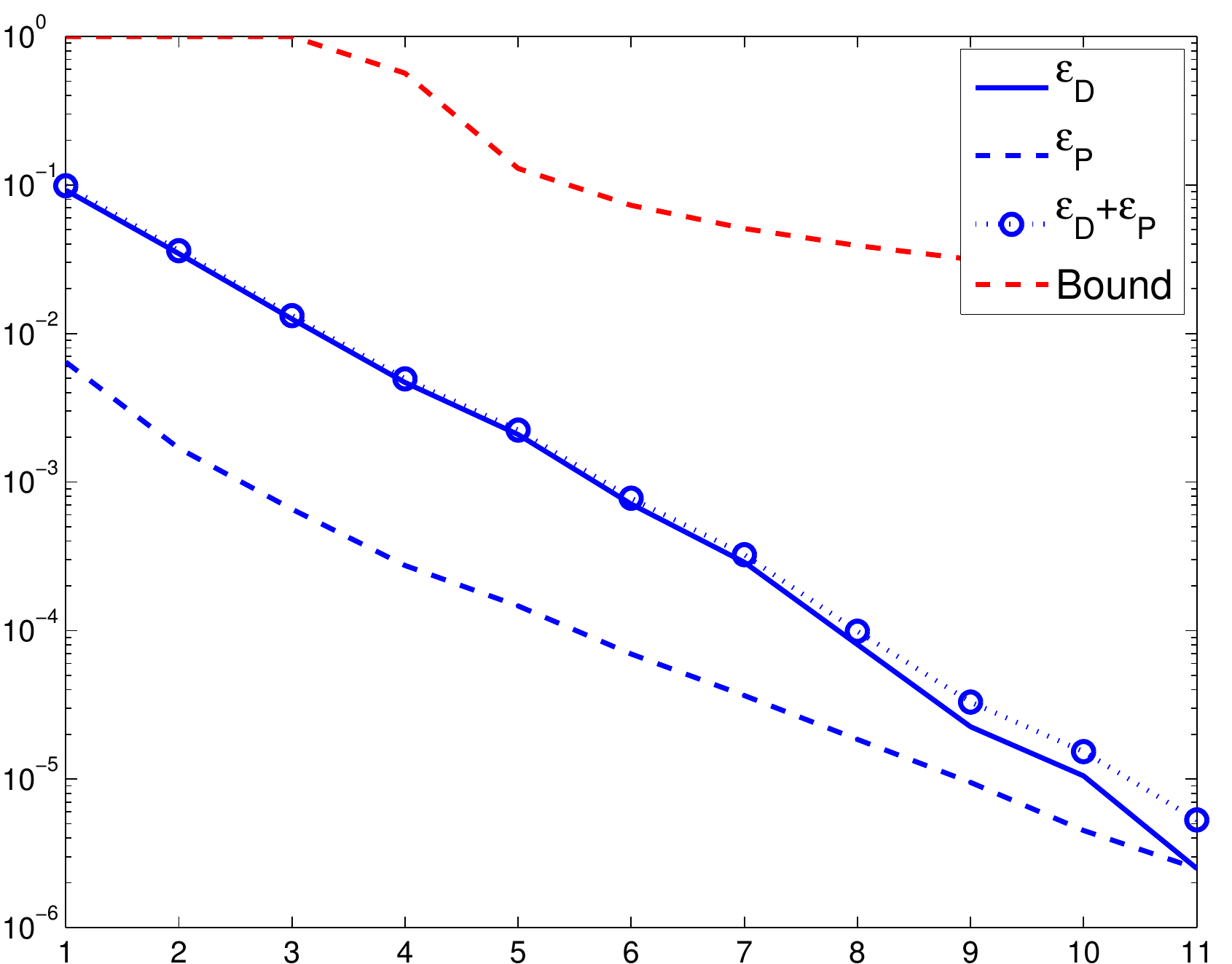} &
\includegraphics[width=0.31\textwidth]{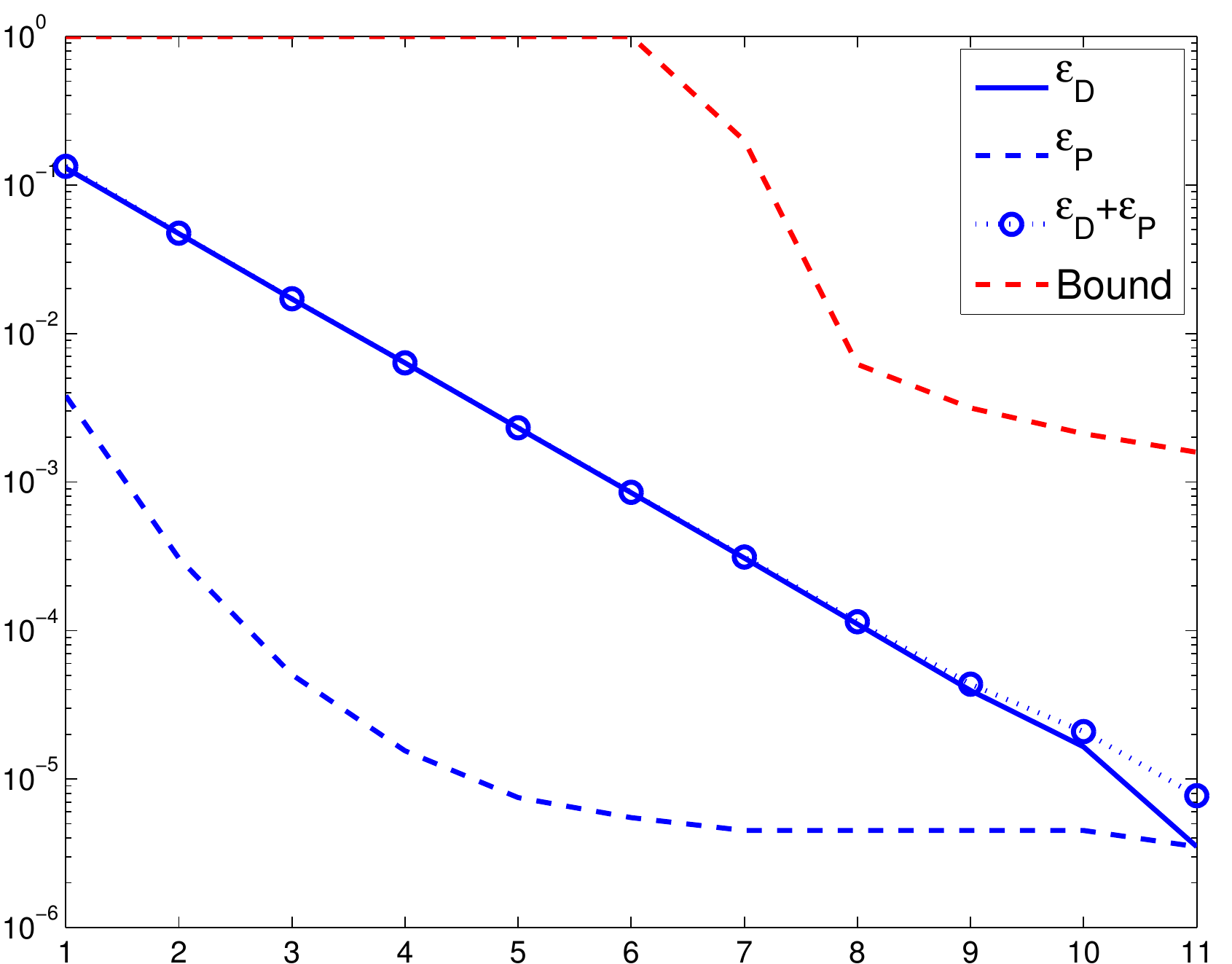} &
\includegraphics[width=0.31\textwidth]{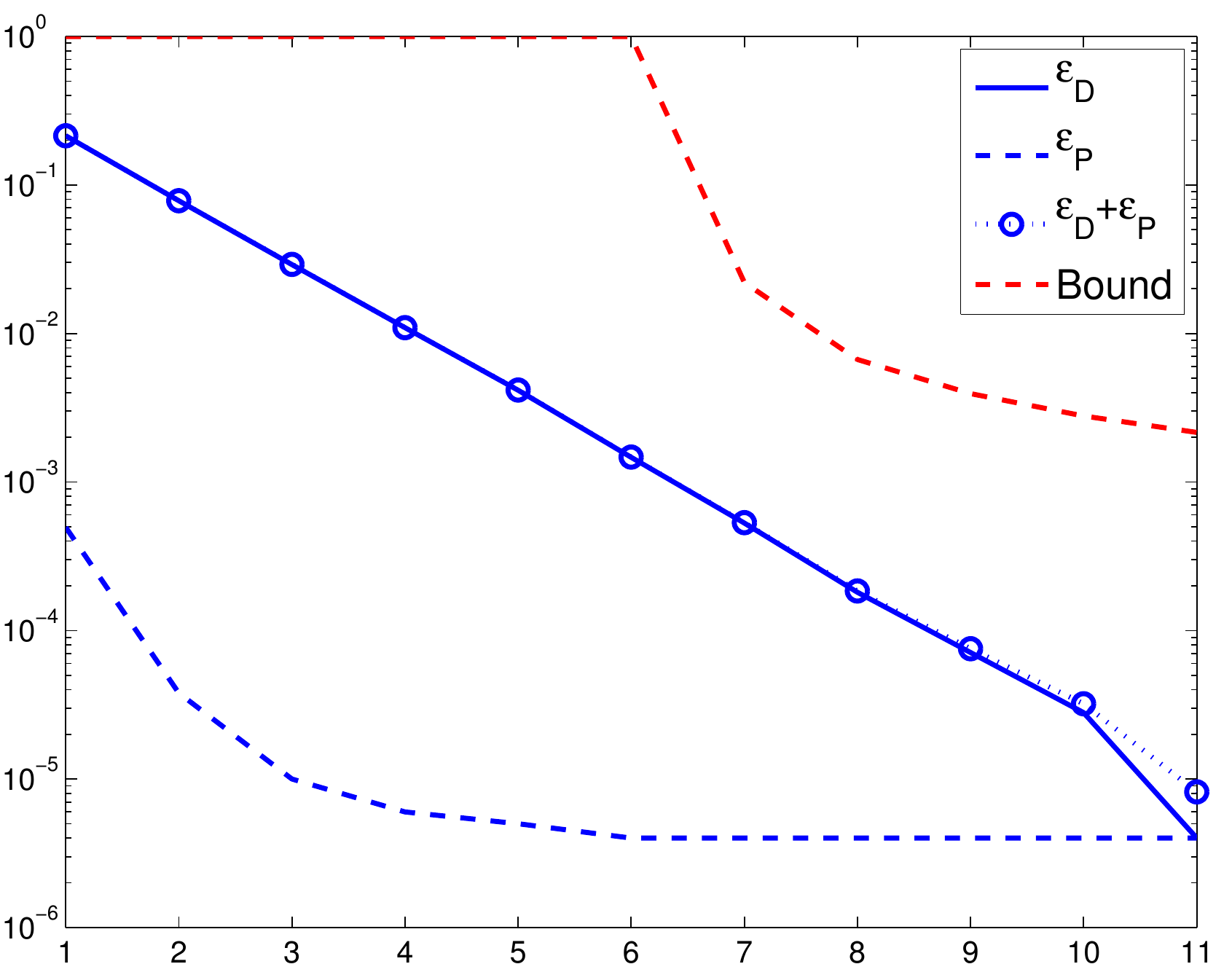}\\
$10^{-4}$ &
\includegraphics[width=0.31\textwidth]{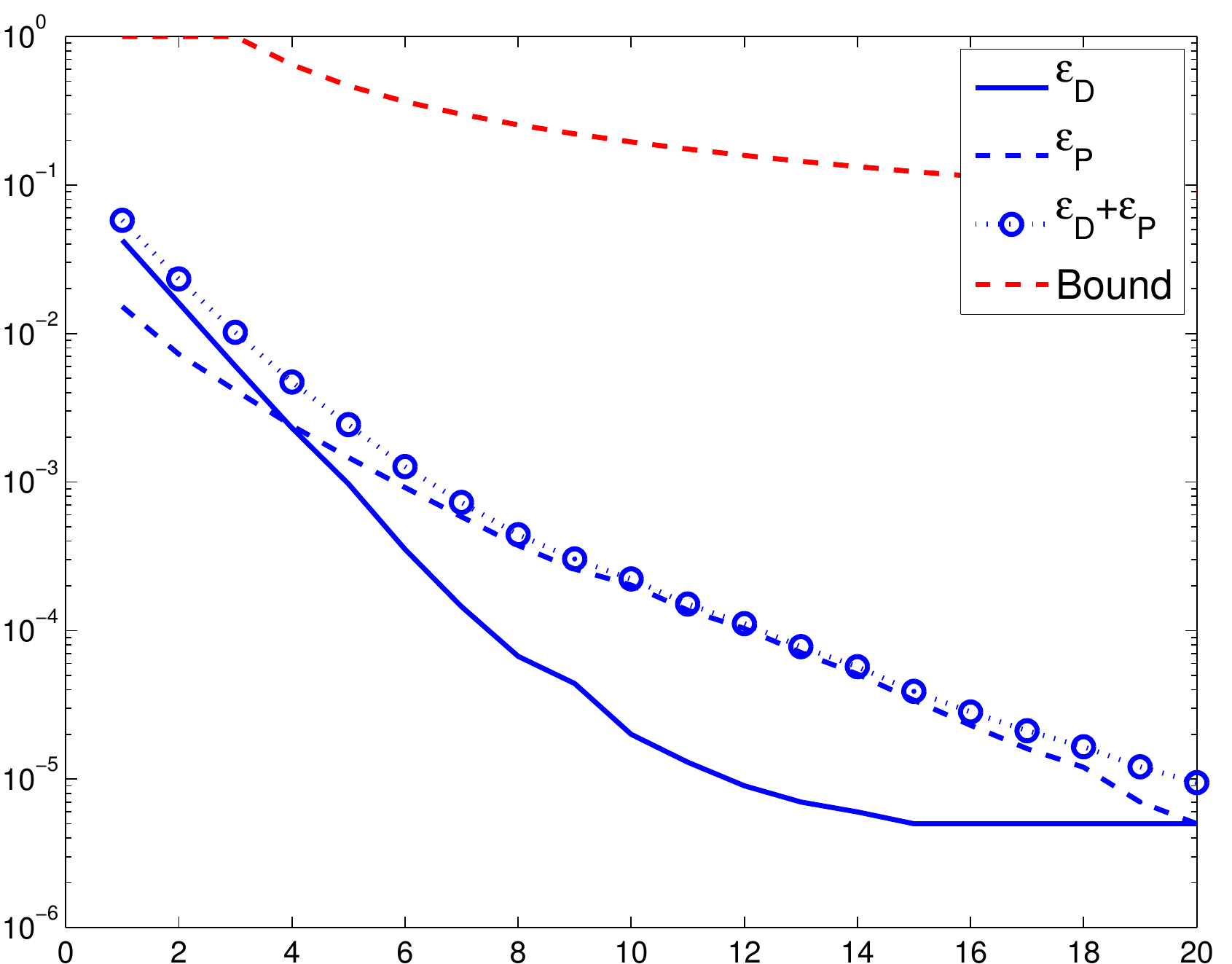} &
\includegraphics[width=0.31\textwidth]{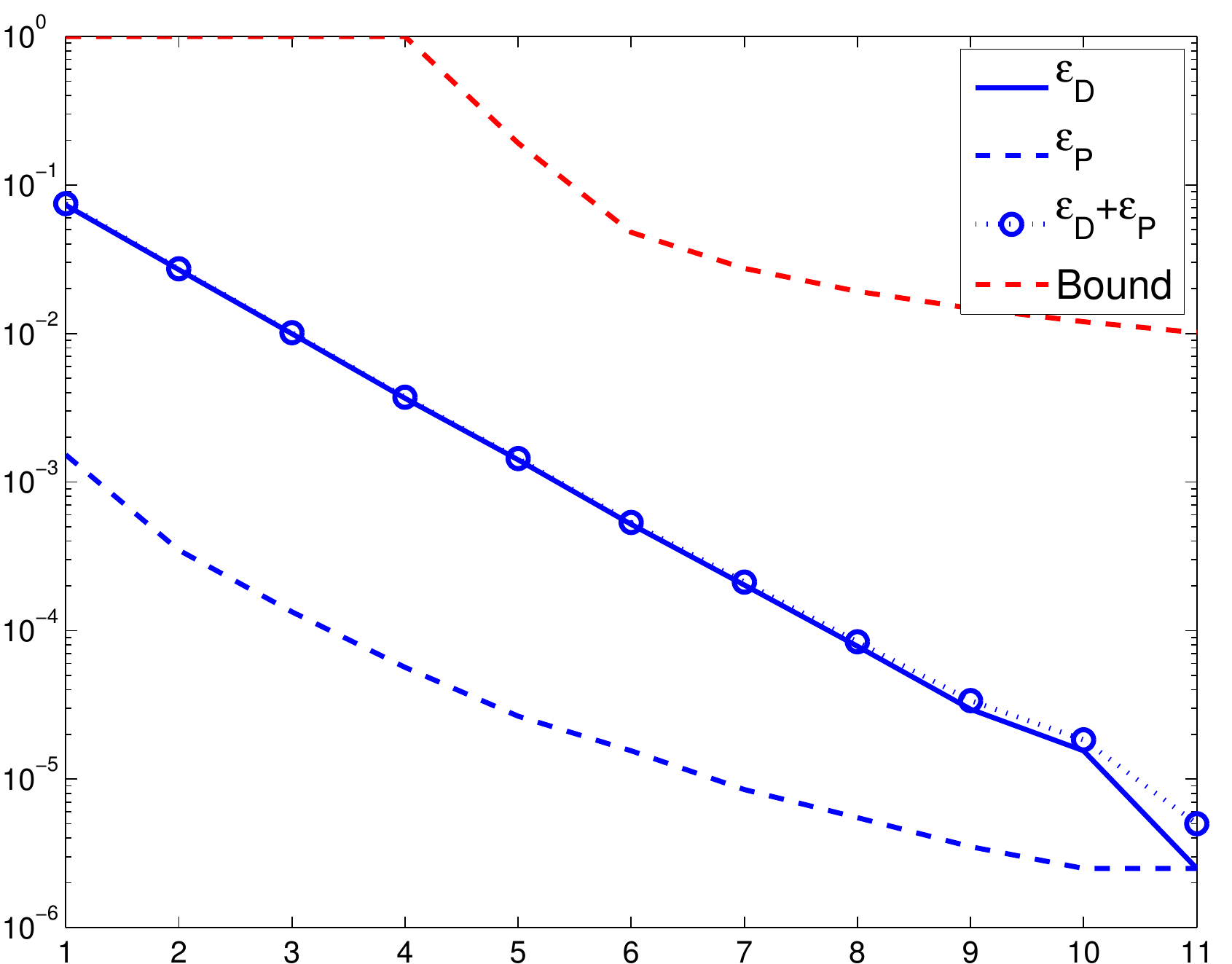} &
\includegraphics[width=0.31\textwidth]{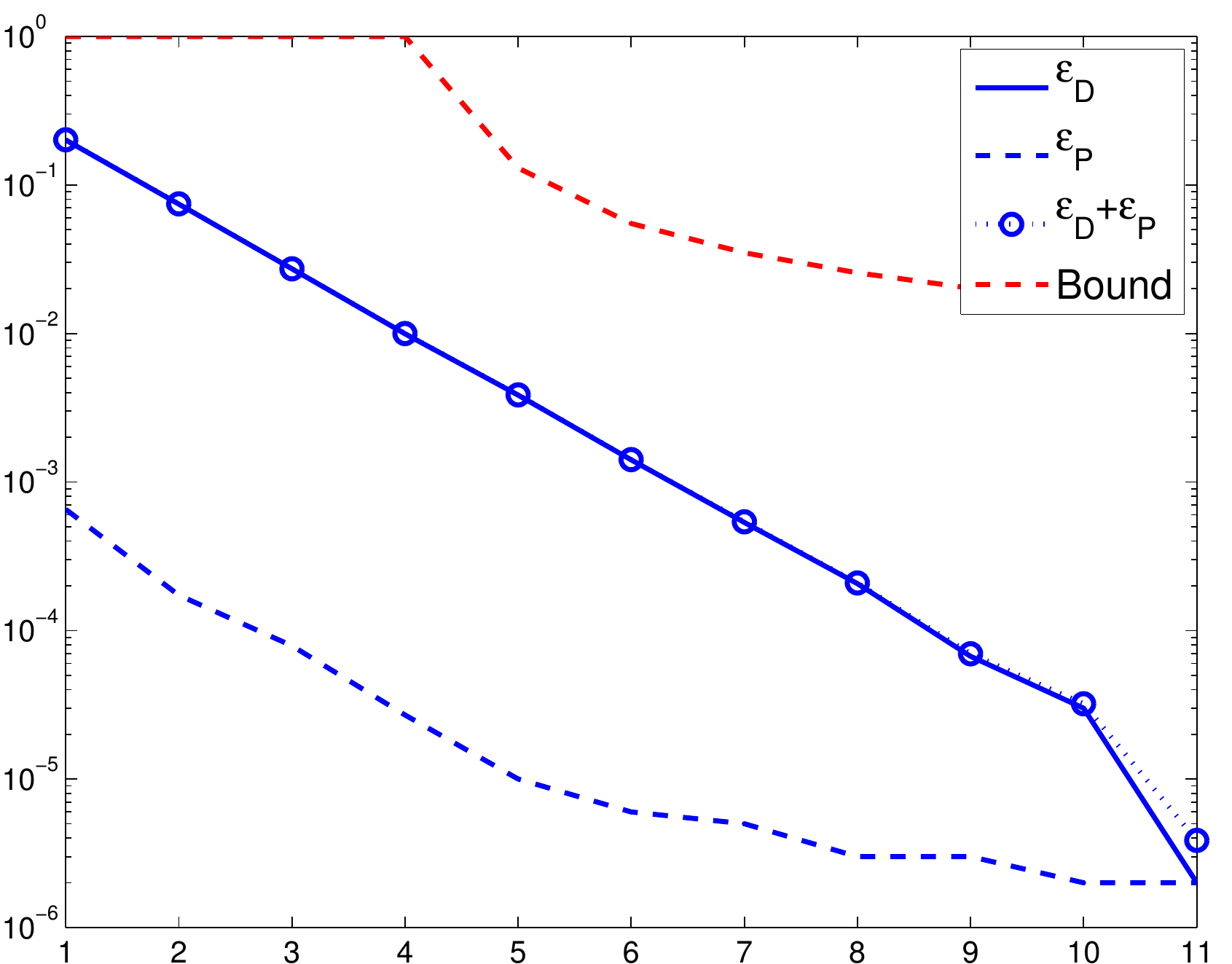}\\
$10^{-5}$ &
\includegraphics[width=0.31\textwidth]{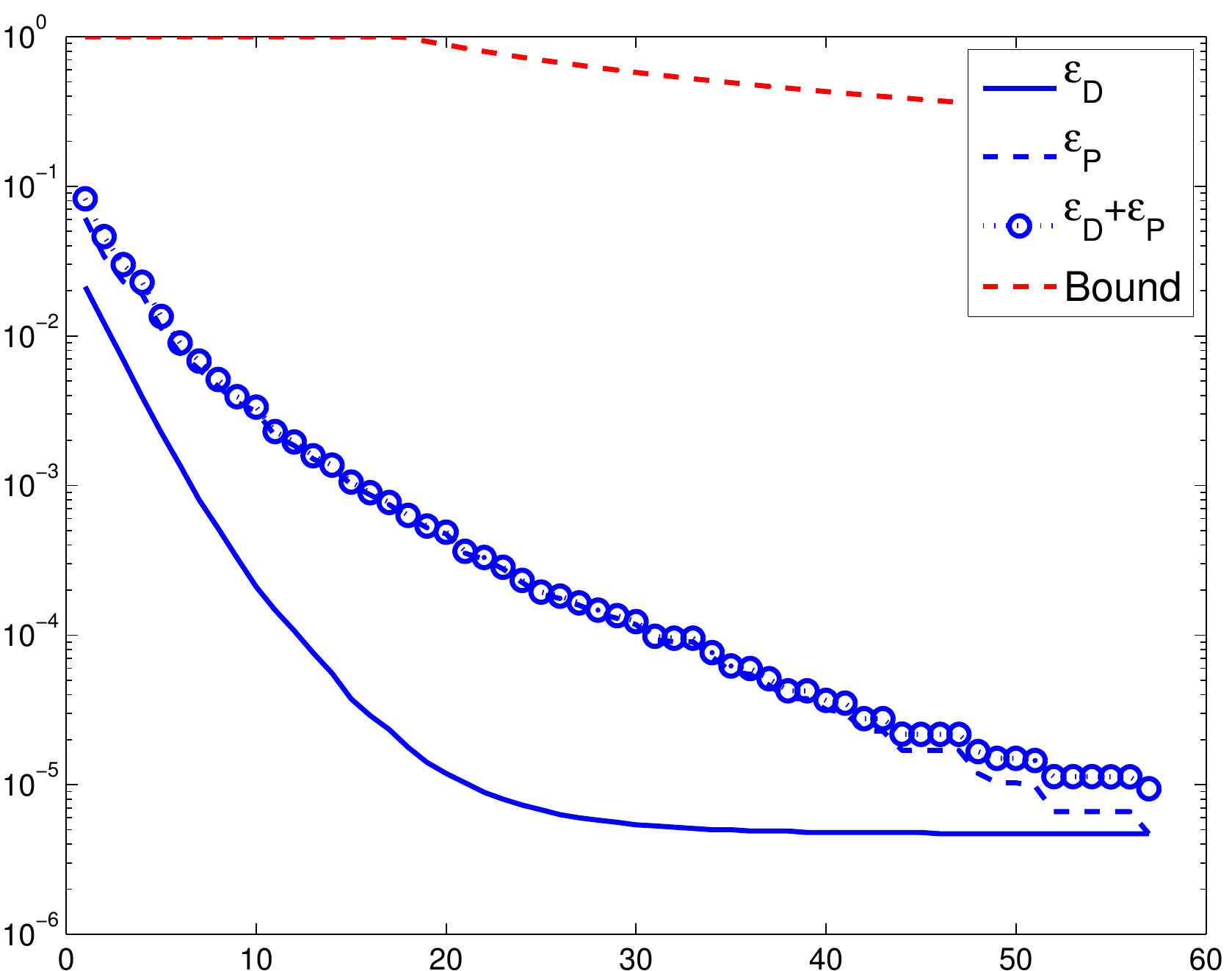} &
\includegraphics[width=0.31\textwidth]{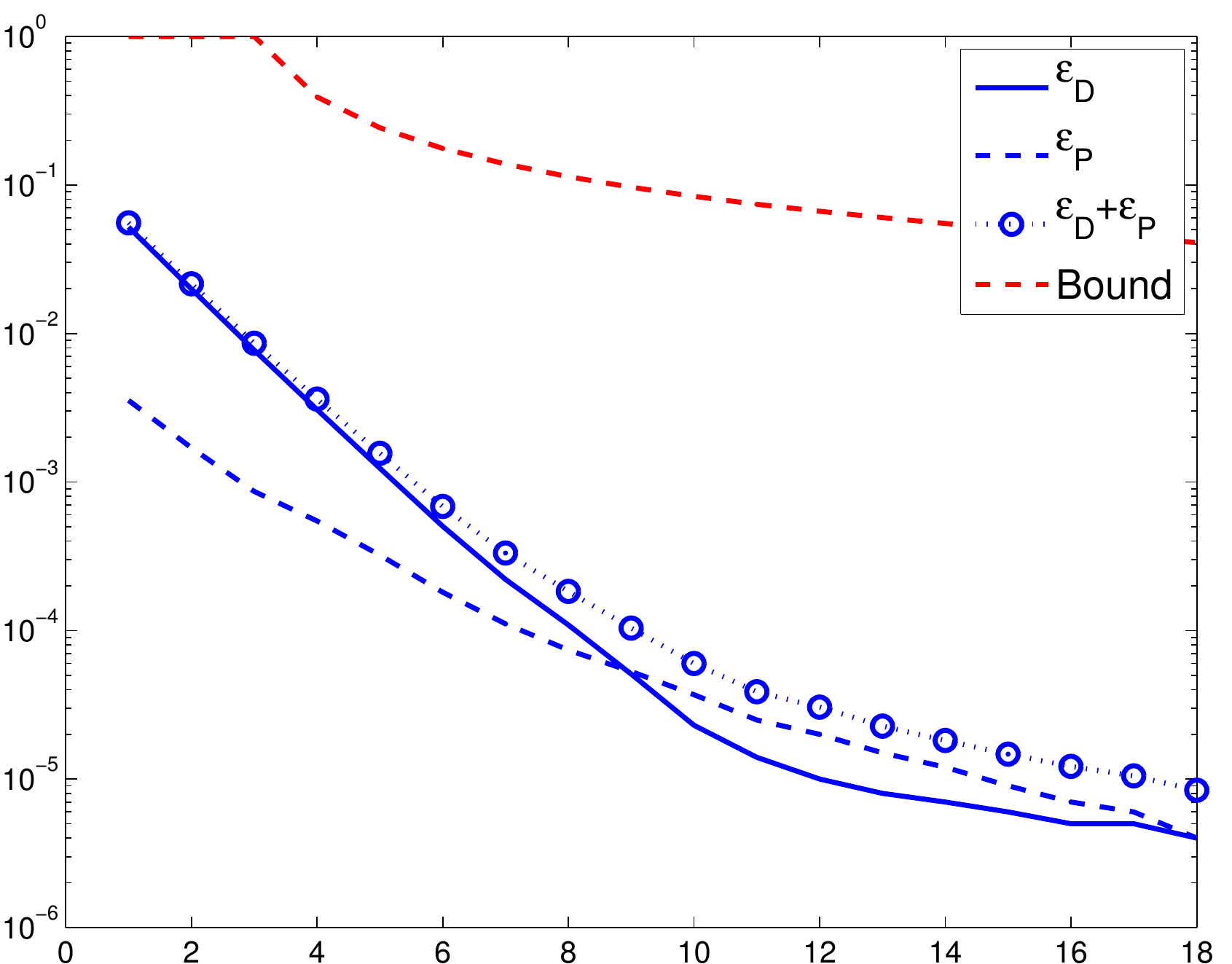} &
\includegraphics[width=0.31\textwidth]{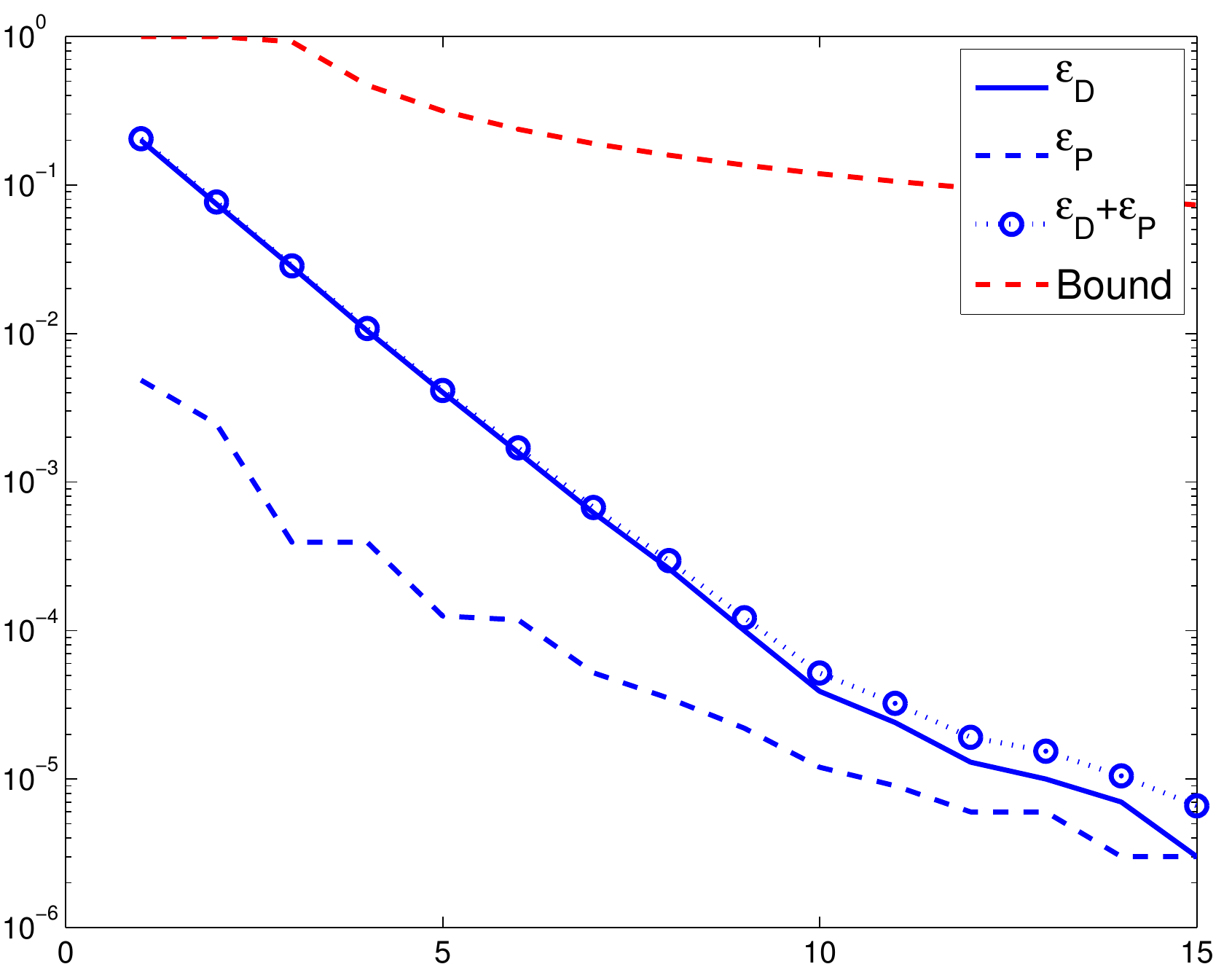}\\
$10^{-6}$ &
\includegraphics[width=0.31\textwidth]{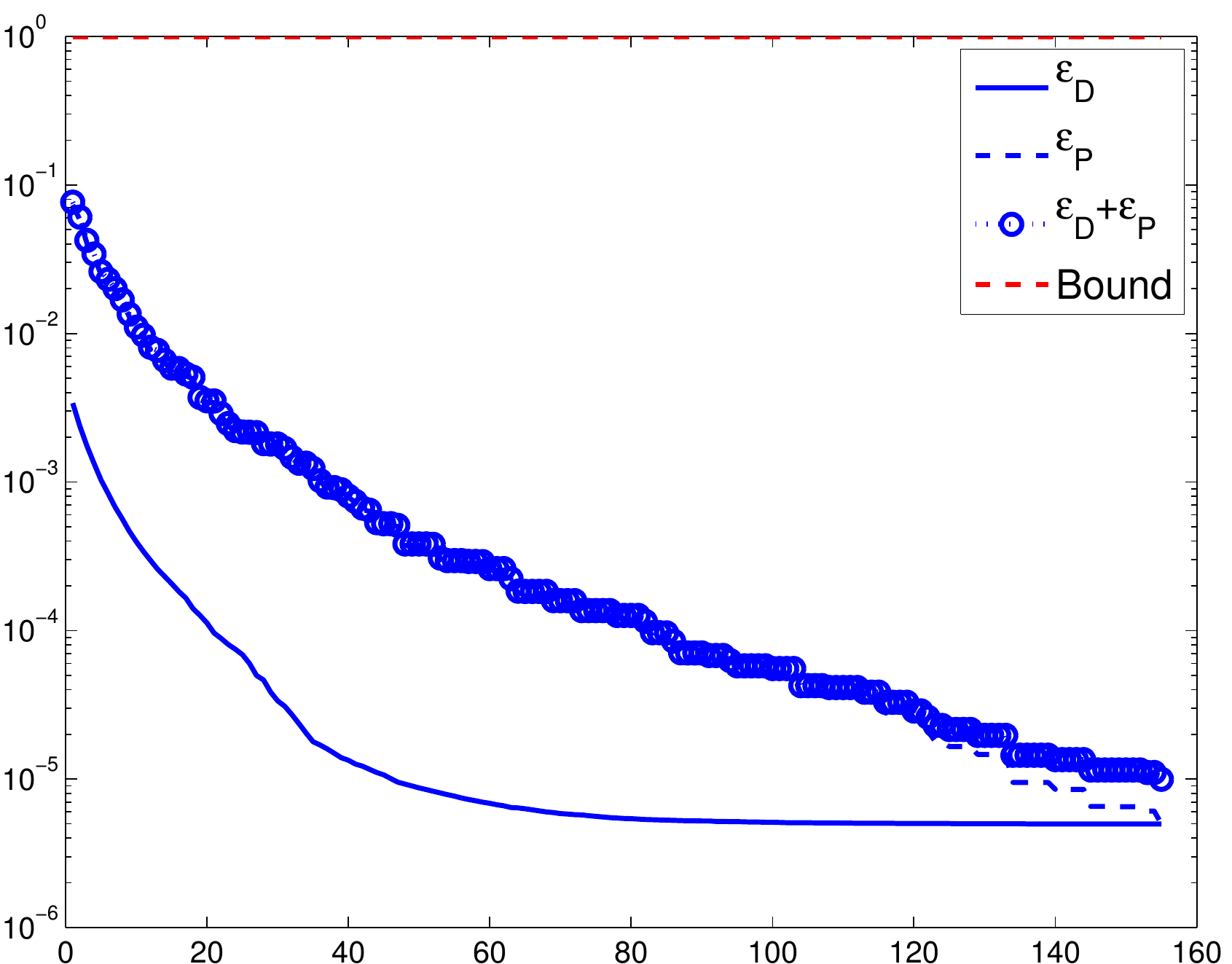} &
\includegraphics[width=0.31\textwidth]{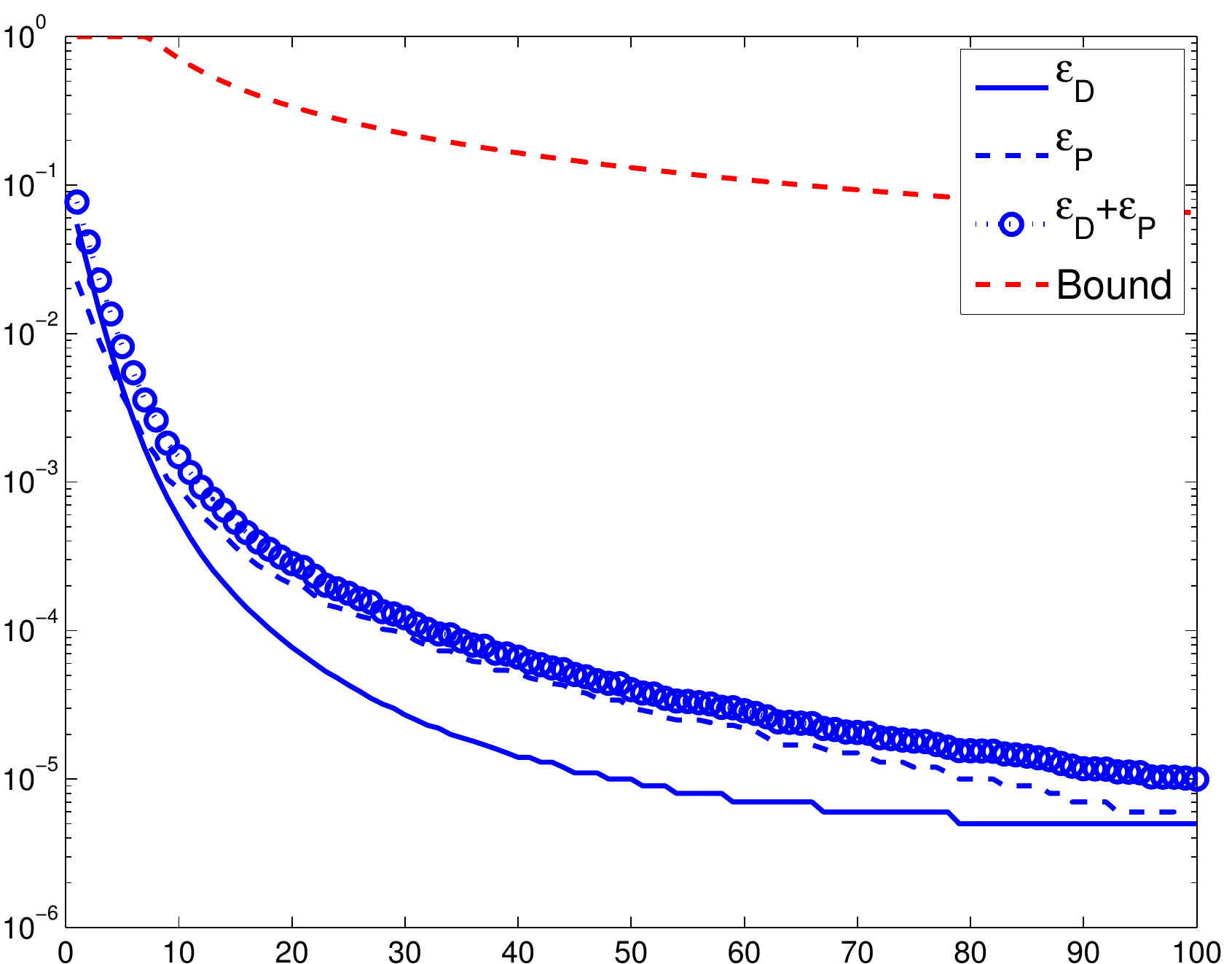} &
\includegraphics[width=0.31\textwidth]{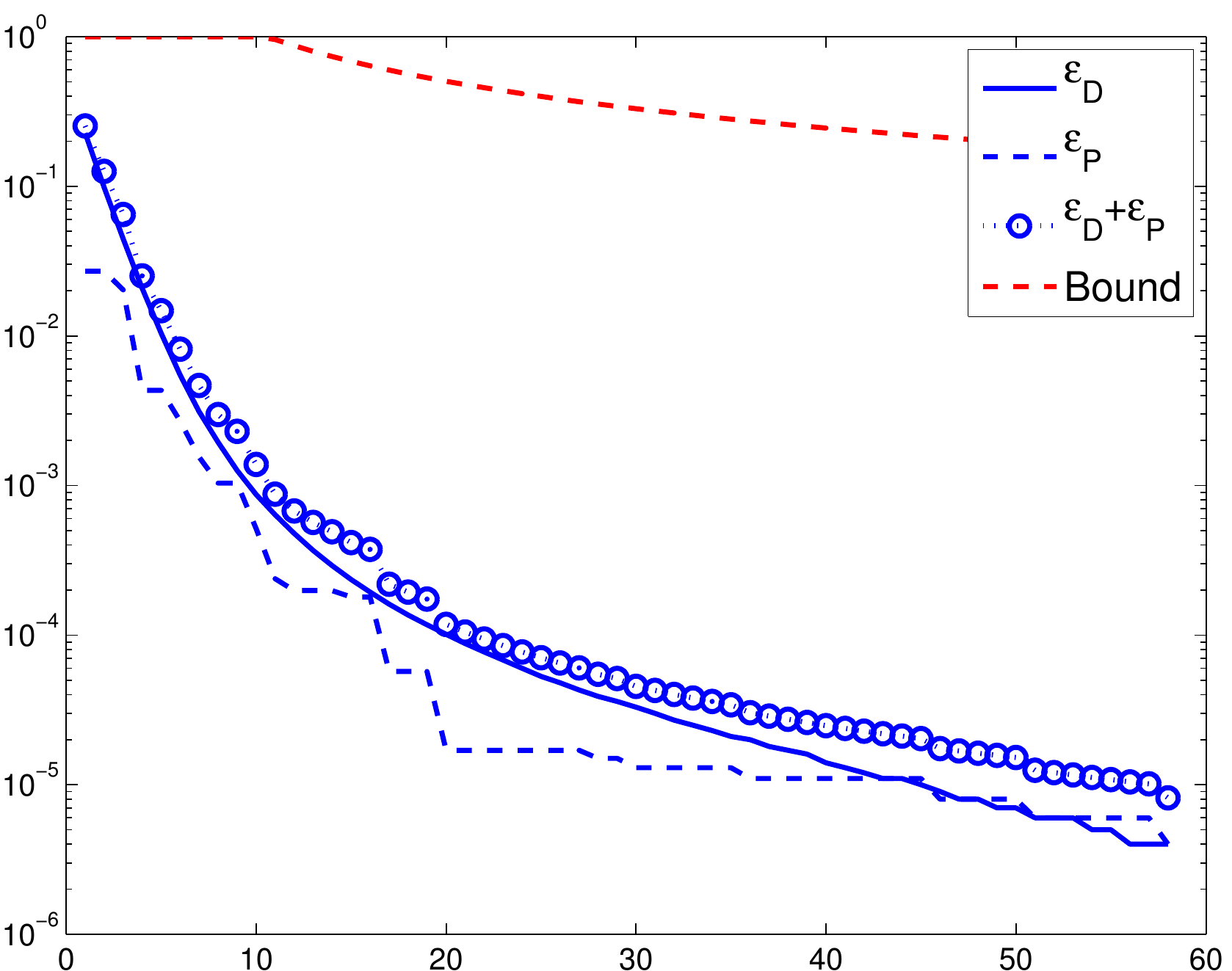}\\
\end{tabular}
\end{center}

\caption{\label{fig:non-smooth}Experiments with the hinge-loss
  (non-smooth). The primal and dual sub-optimality, the duality gap,
  and our bound are depicted as a function of the number of epochs, on
  the astro-ph (left), CCAT (center) and cov1 (right) datasets. In all
  plots the horizontal axis is the number of iterations divided by
  training set size (corresponding to the number of epochs through the
  data).}

\end{figure}

\begin{figure}

\begin{center}
\begin{tabular}{ @{} L | @{} S @{} S @{} S @{} }
$\lambda$ & \scriptsize{astro-ph} & \scriptsize{CCAT} & \scriptsize{cov1}\\ \hline
$10^{-3}$ & 
\includegraphics[width=0.31\textwidth]{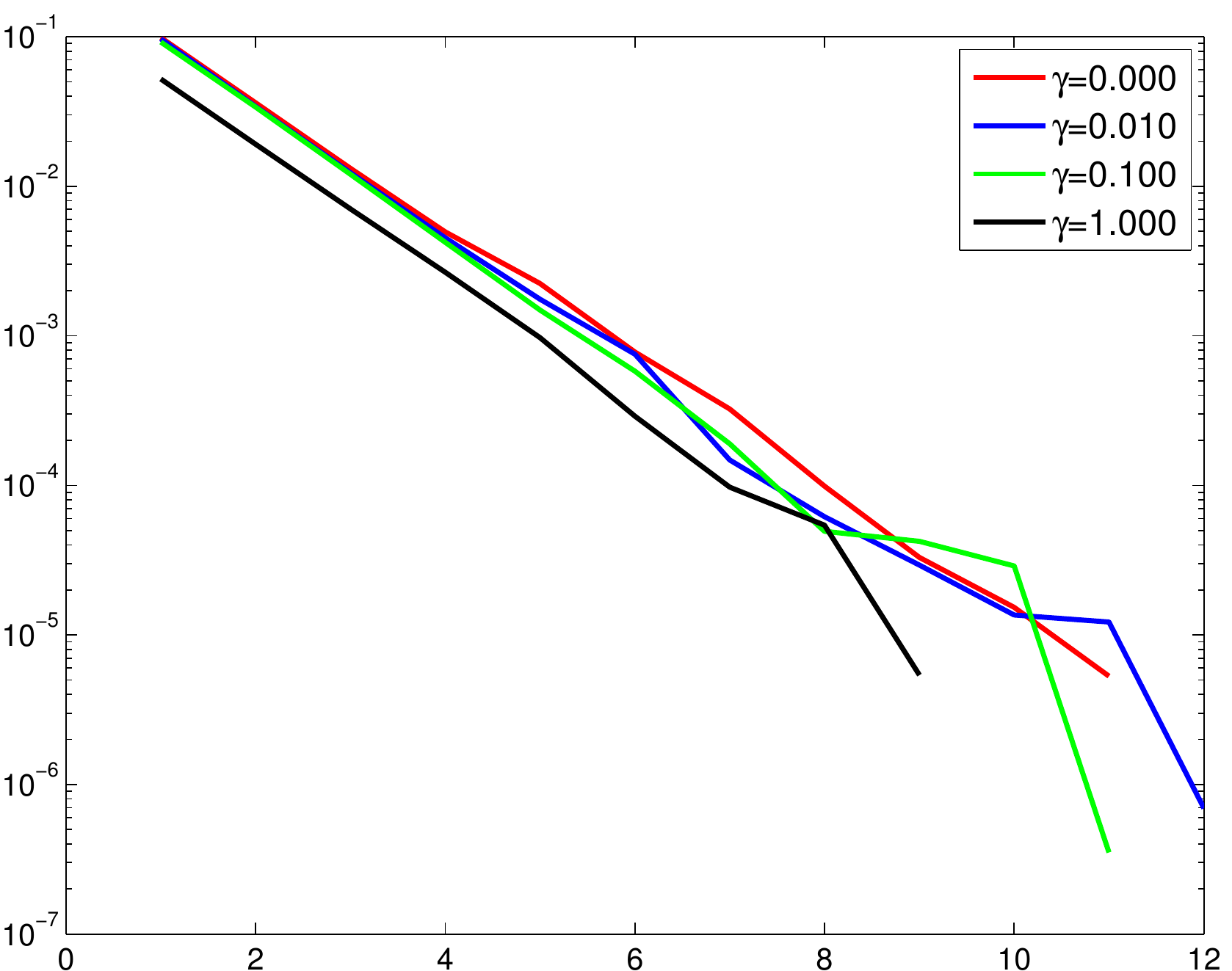} &
\includegraphics[width=0.31\textwidth]{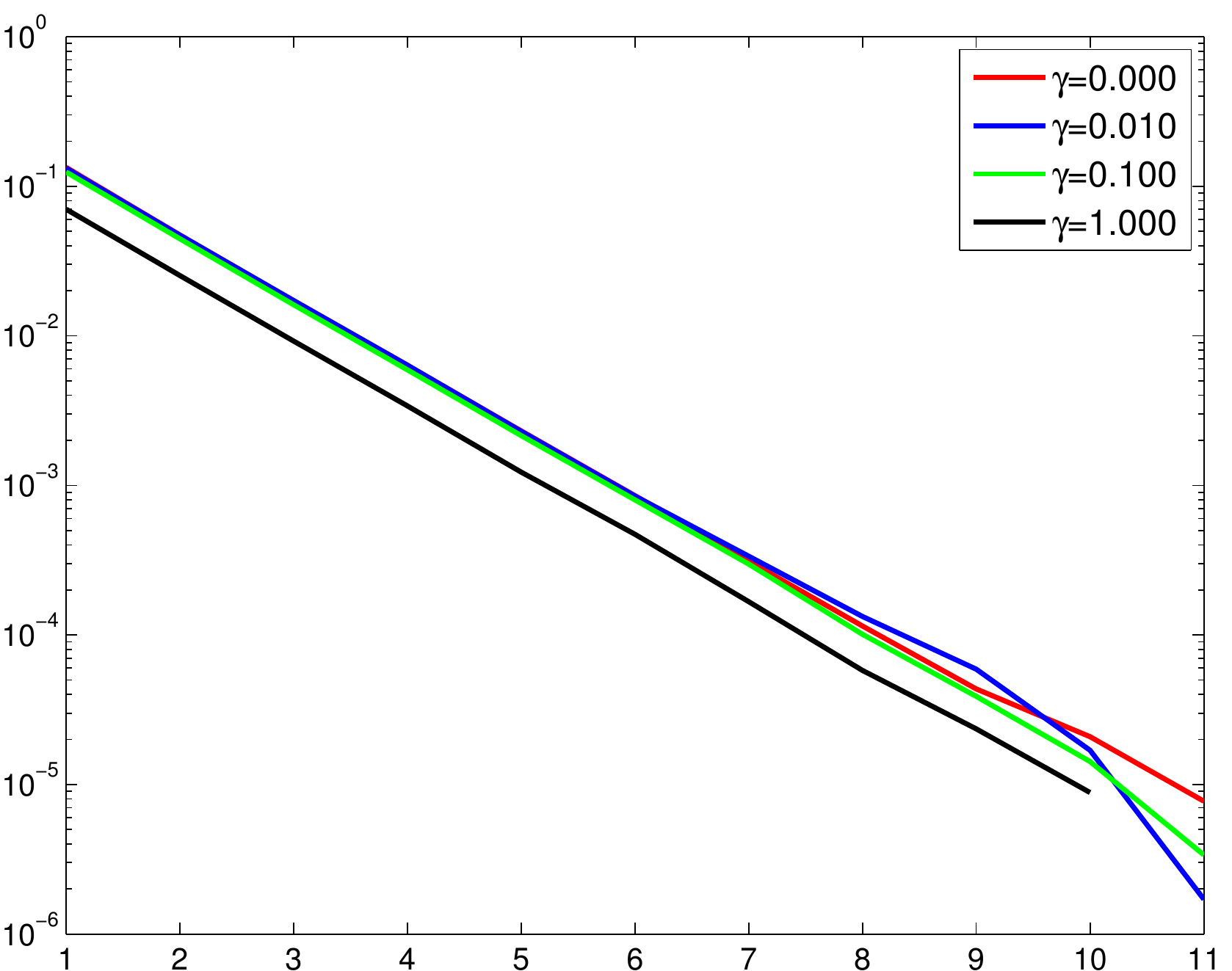} &
\includegraphics[width=0.31\textwidth]{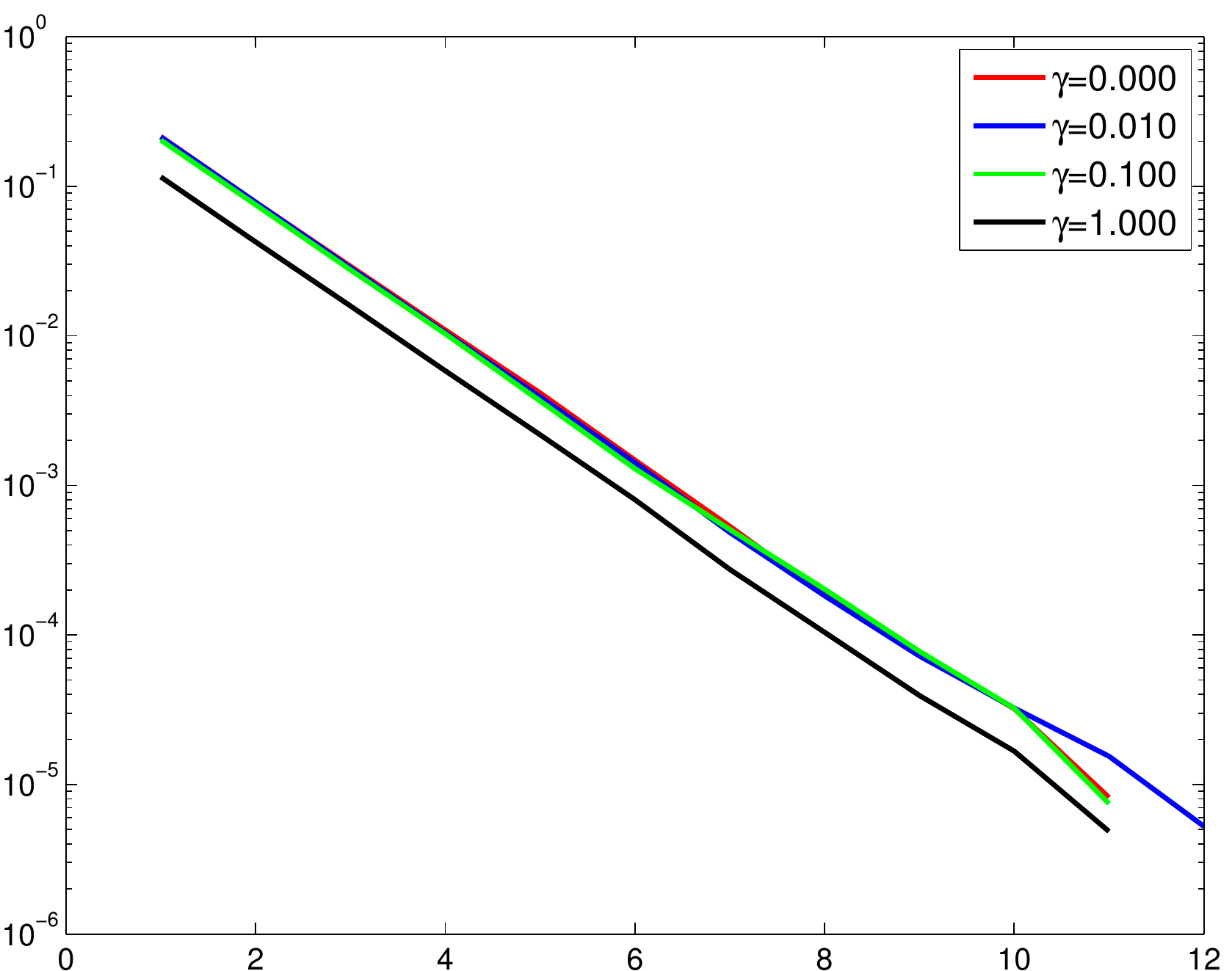}\\
$10^{-4}$ &
\includegraphics[width=0.31\textwidth]{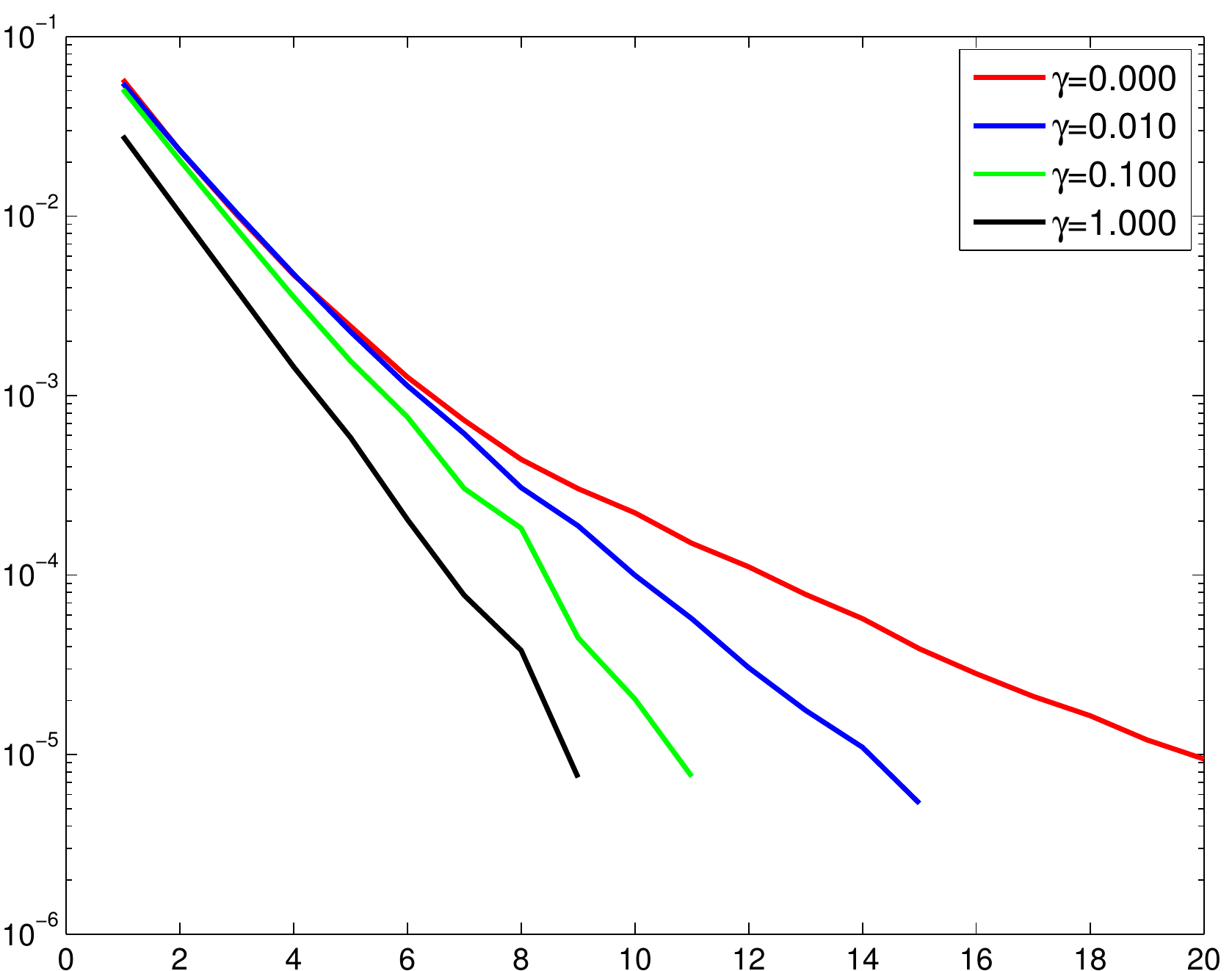} &
\includegraphics[width=0.31\textwidth]{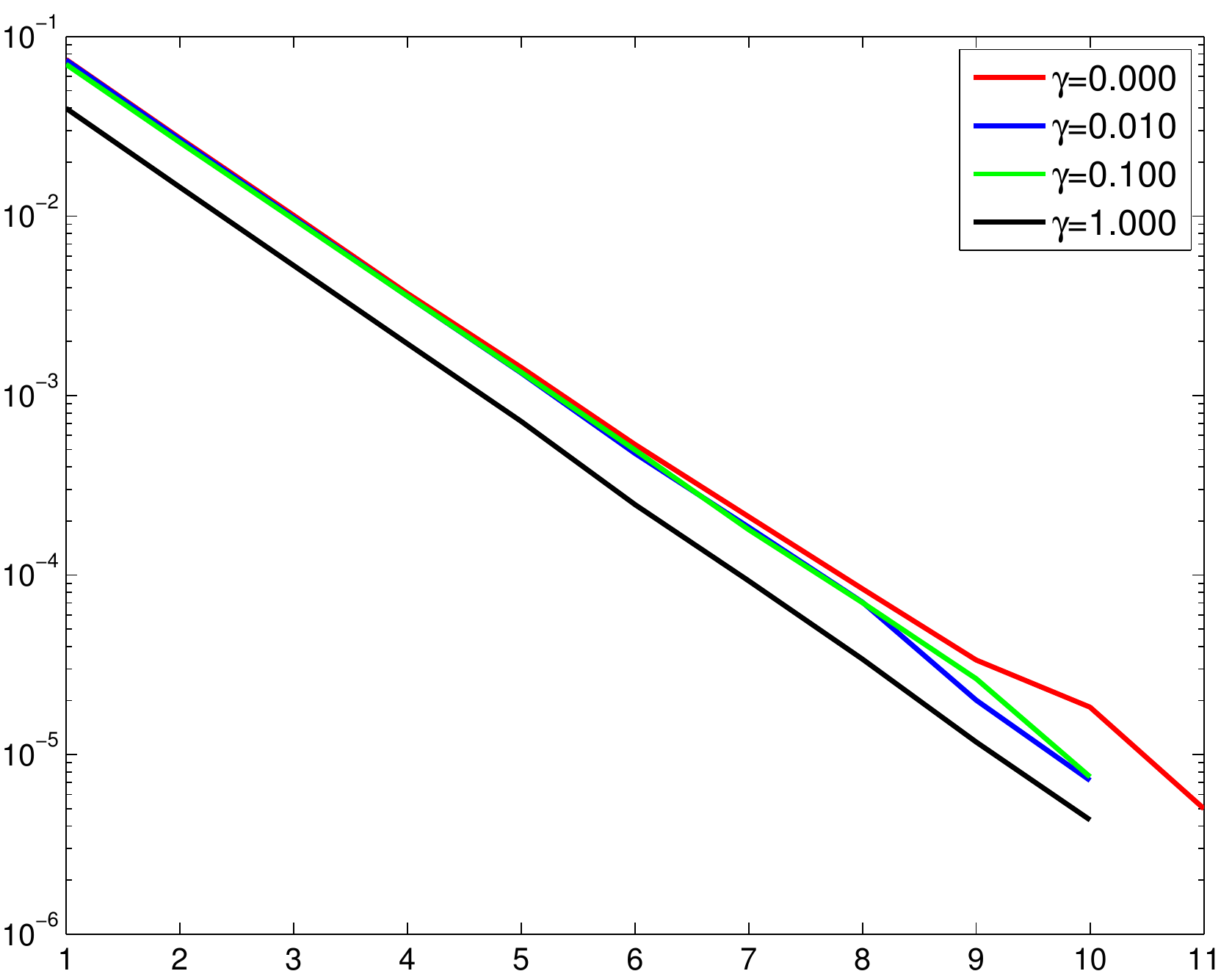} &
\includegraphics[width=0.31\textwidth]{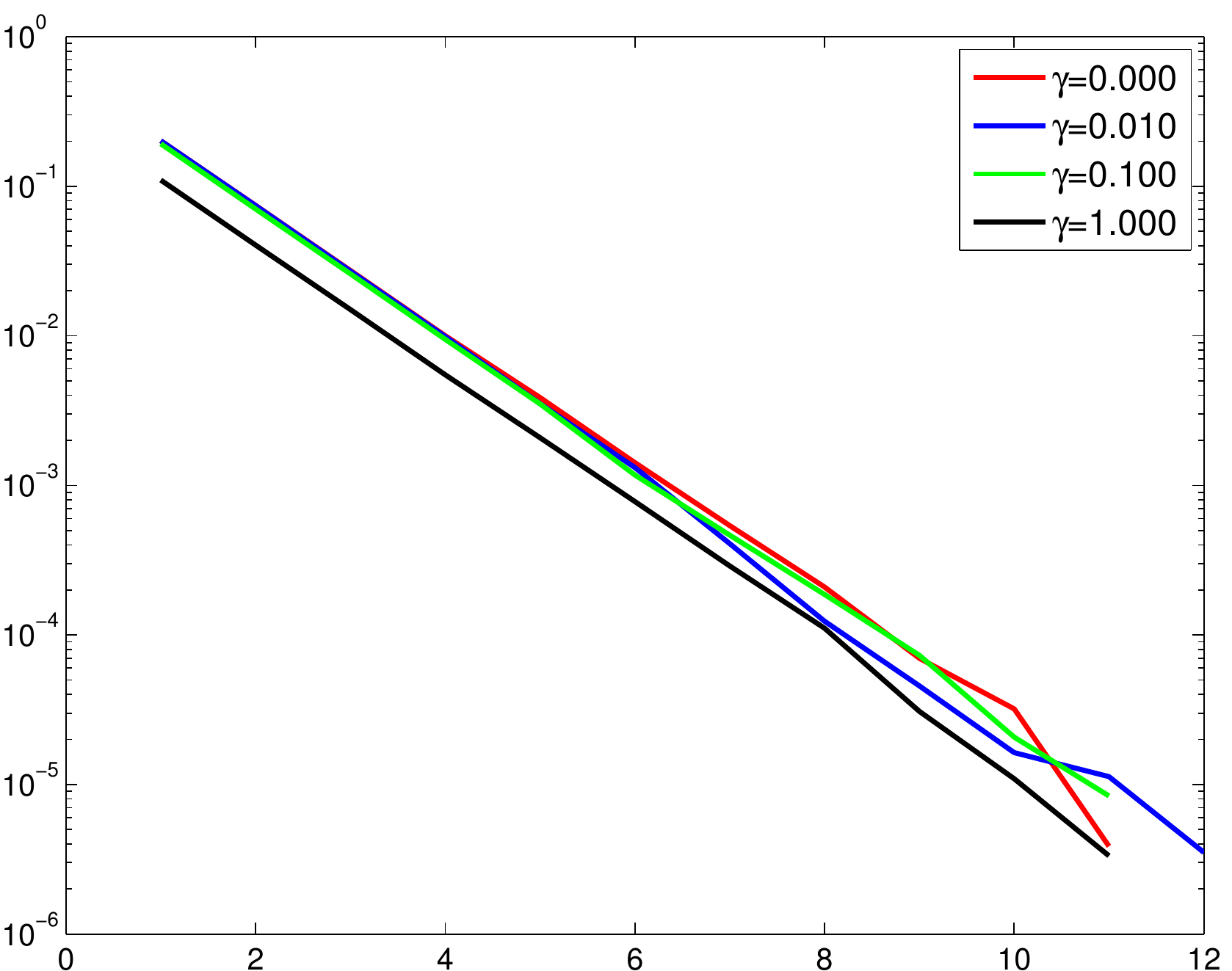}\\
$10^{-5}$ &
\includegraphics[width=0.31\textwidth]{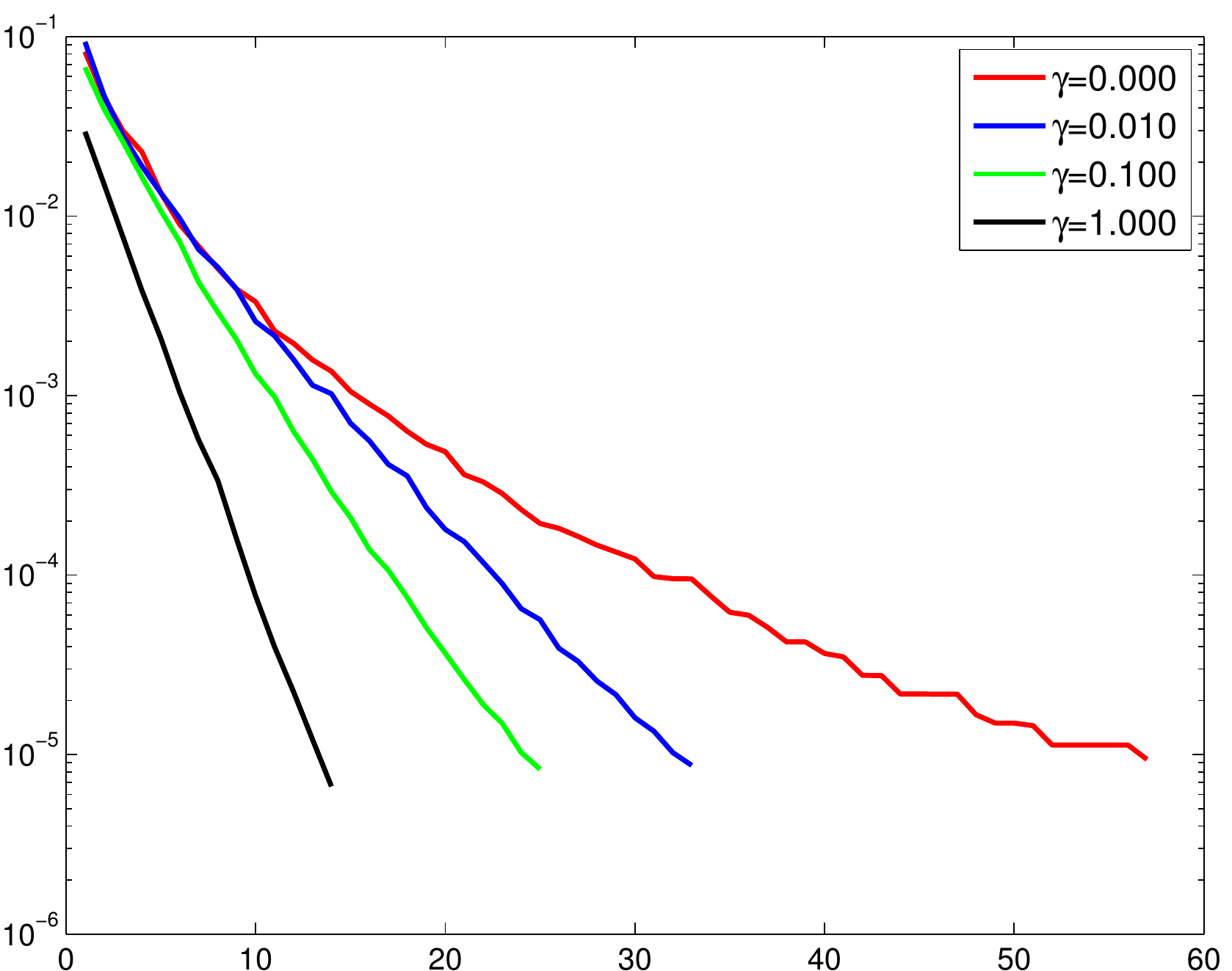} &
\includegraphics[width=0.31\textwidth]{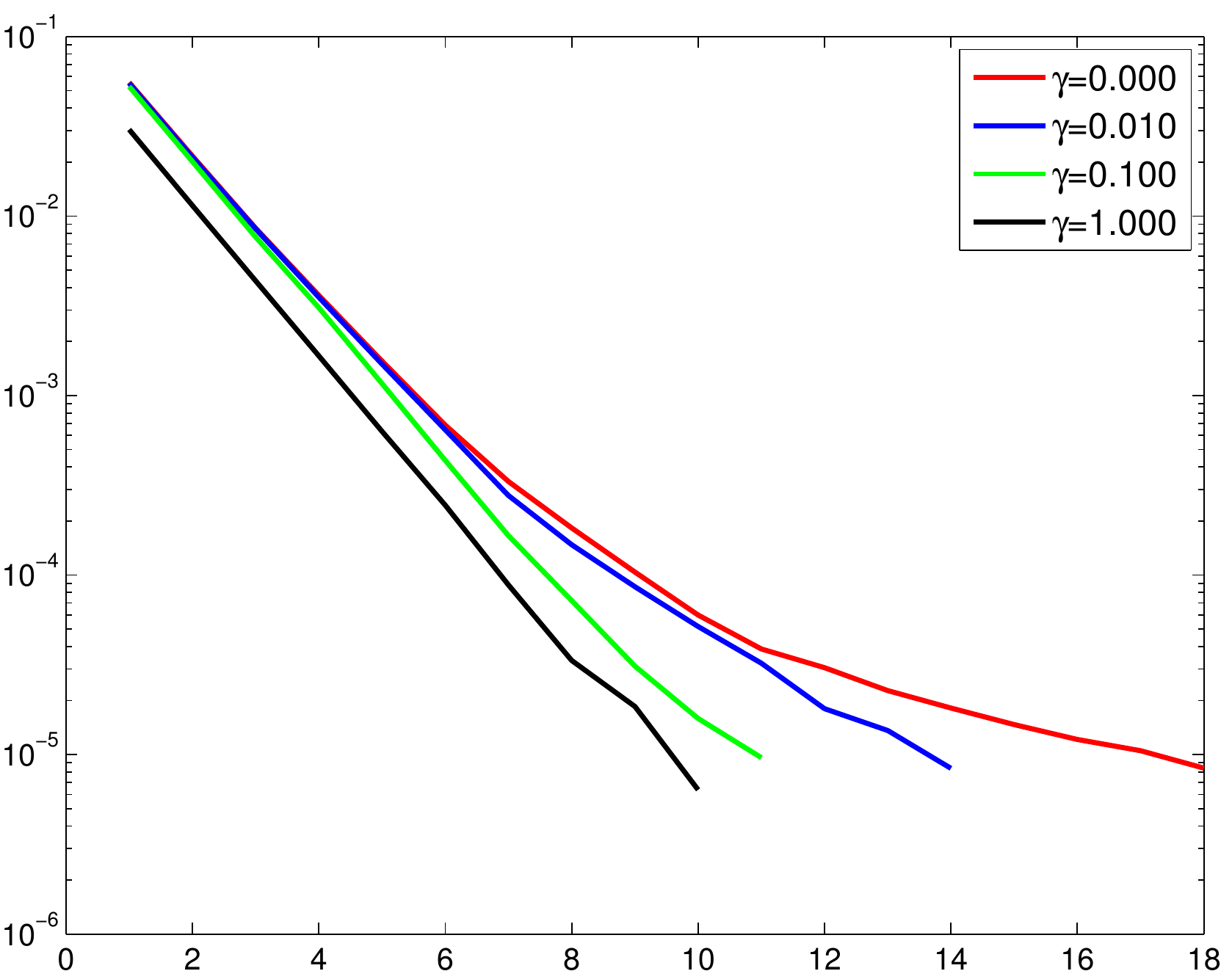} &
\includegraphics[width=0.31\textwidth]{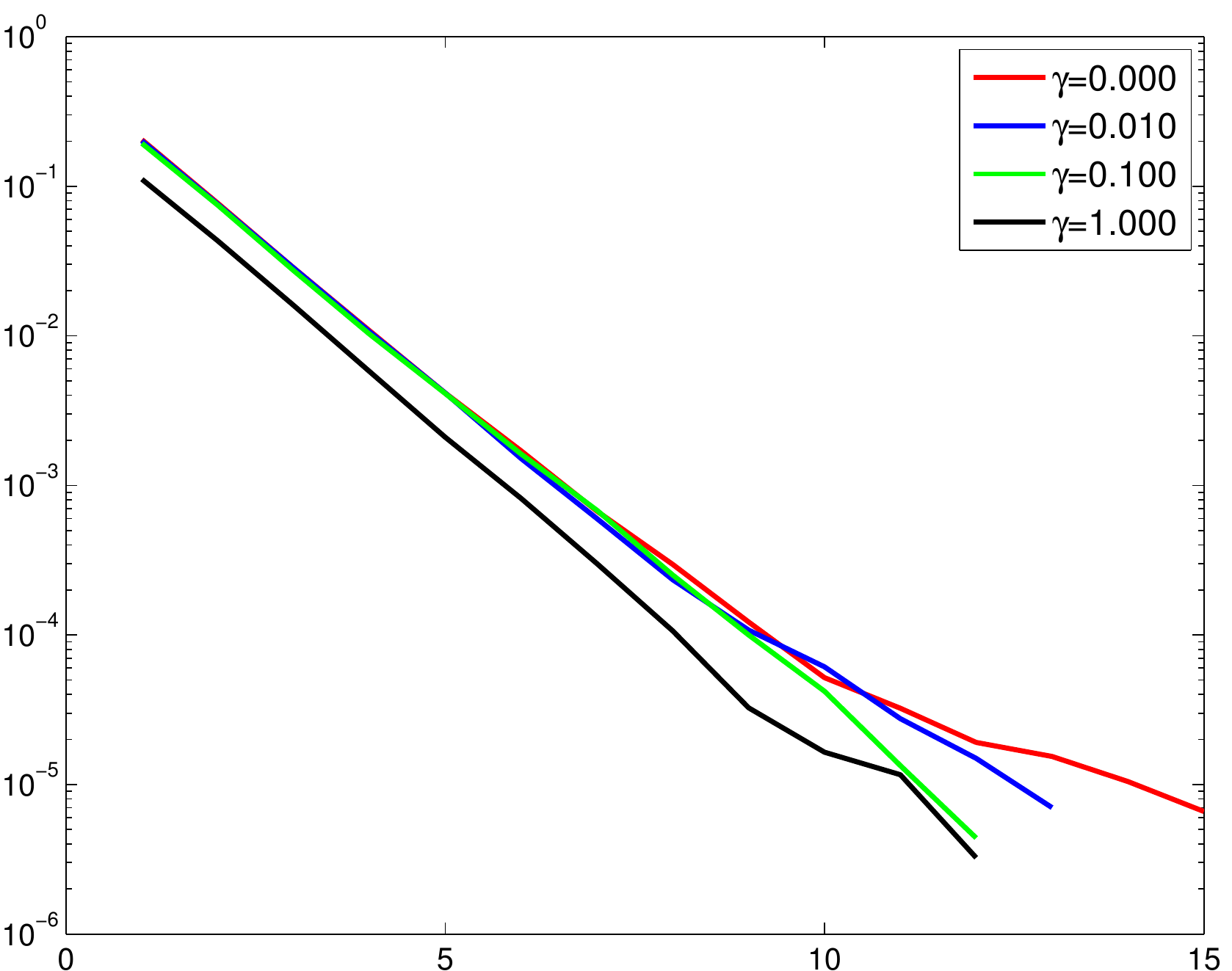}\\
$10^{-6}$ &
\includegraphics[width=0.31\textwidth]{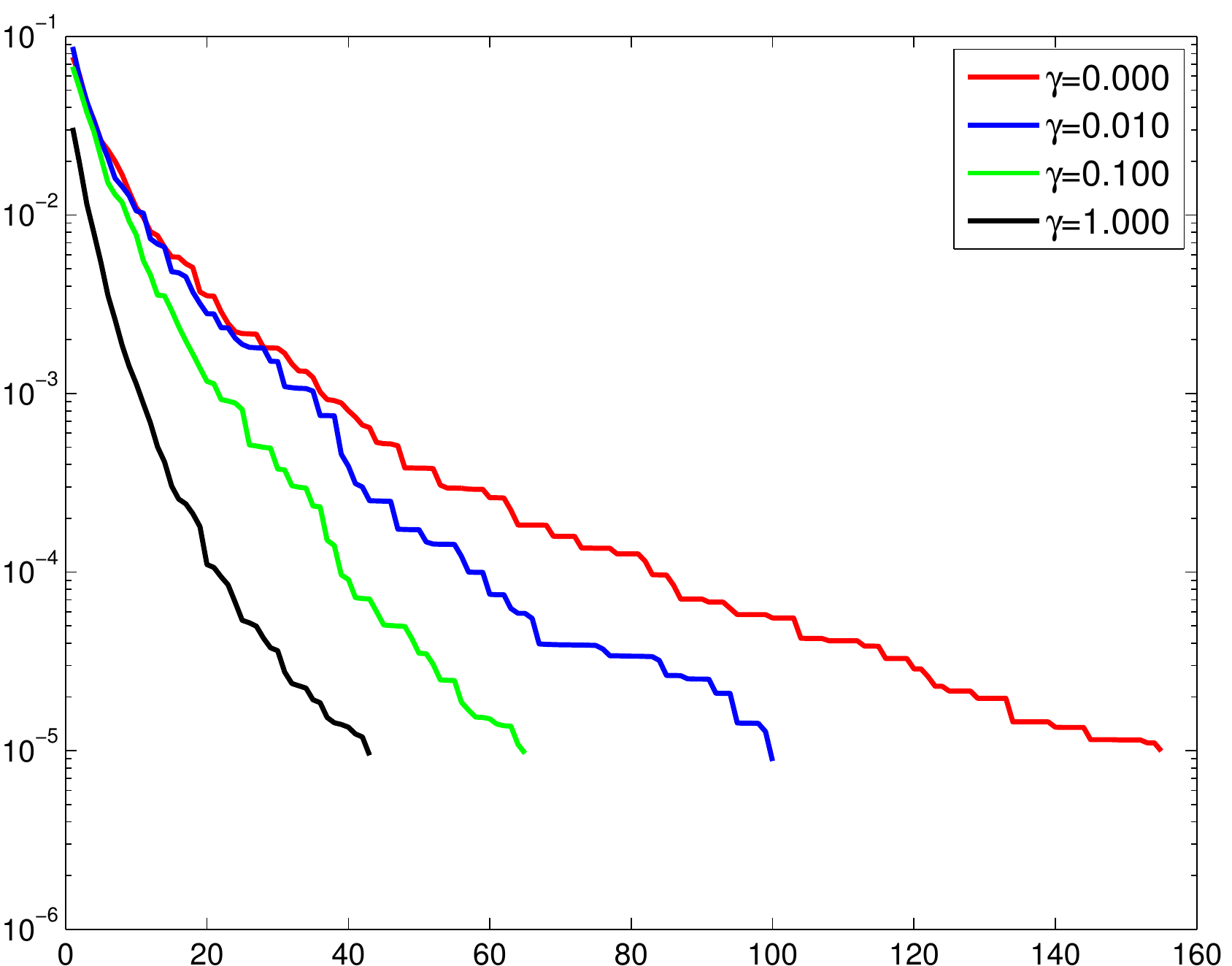} &
\includegraphics[width=0.31\textwidth]{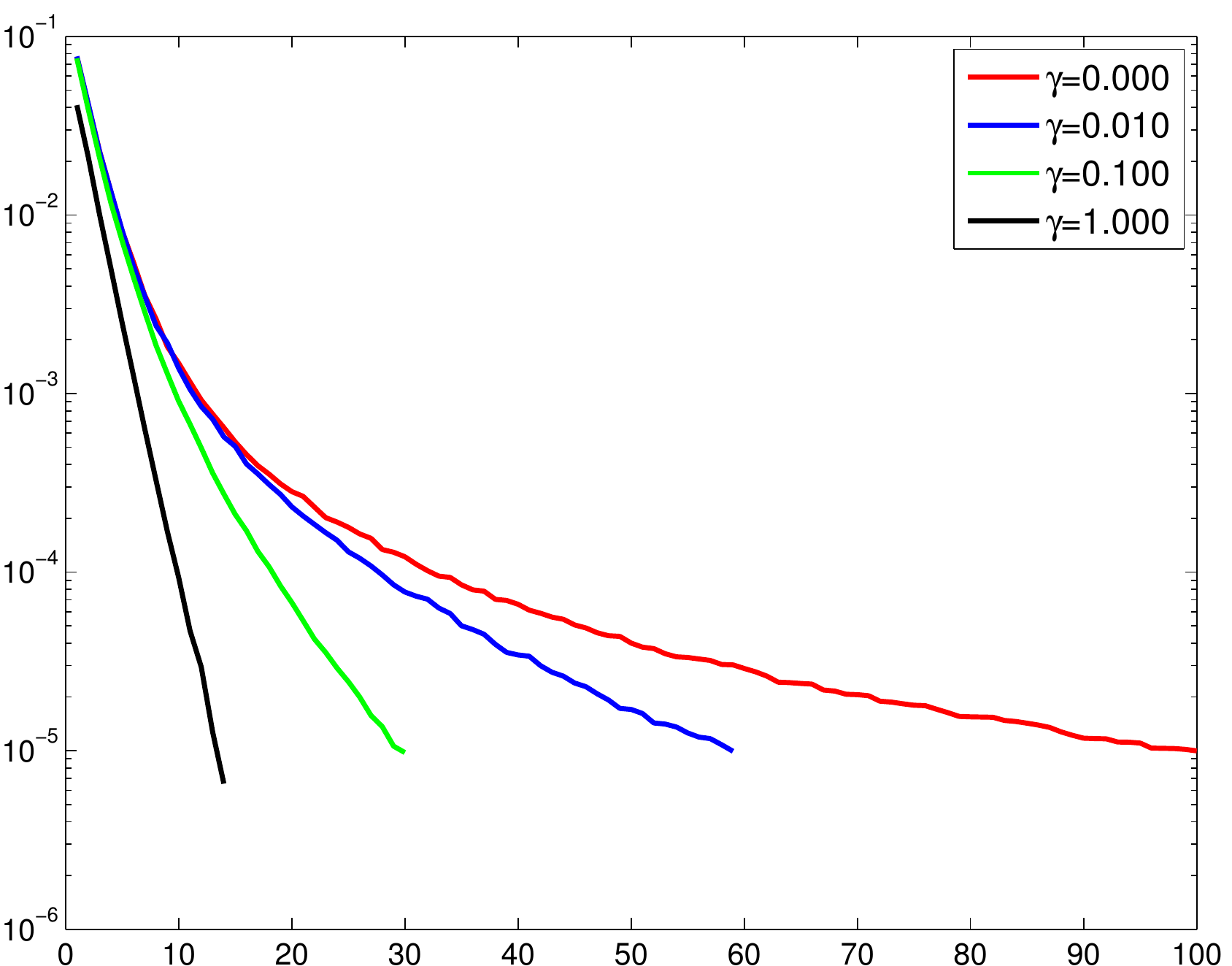} &
\includegraphics[width=0.31\textwidth]{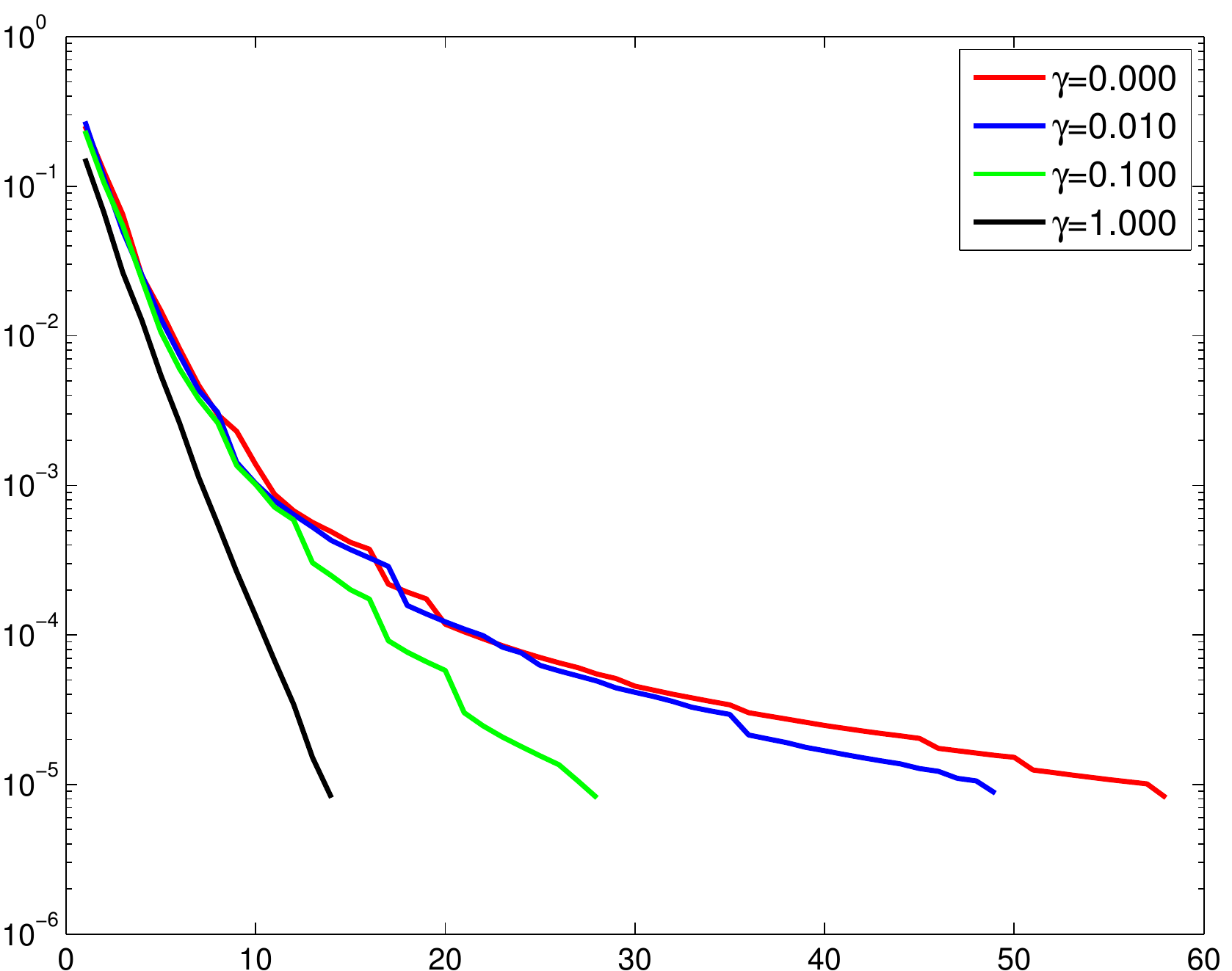}\\
\end{tabular}
\end{center}

\caption{\label{fig:gamma} Duality gap as a function of the number of
  rounds for different values of $\gamma$.}

\end{figure}

\begin{figure}

\begin{center}
\begin{tabular}{ @{} L | @{} S @{} S @{} S @{} }
$\lambda$ & \scriptsize{astro-ph} & \scriptsize{CCAT} & \scriptsize{cov1}\\ \hline
$10^{-3}$ & 
\includegraphics[width=0.31\textwidth]{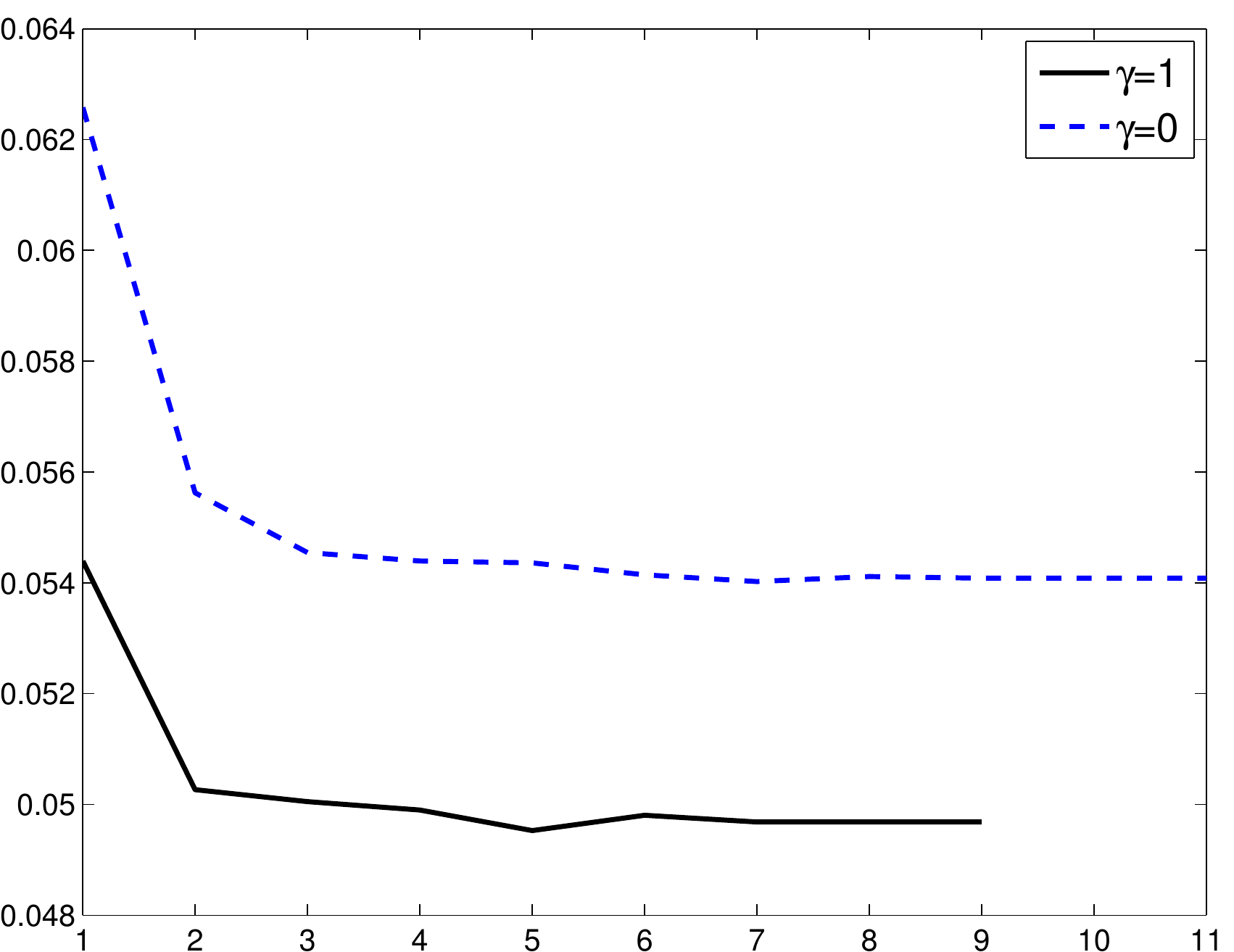} &
\includegraphics[width=0.31\textwidth]{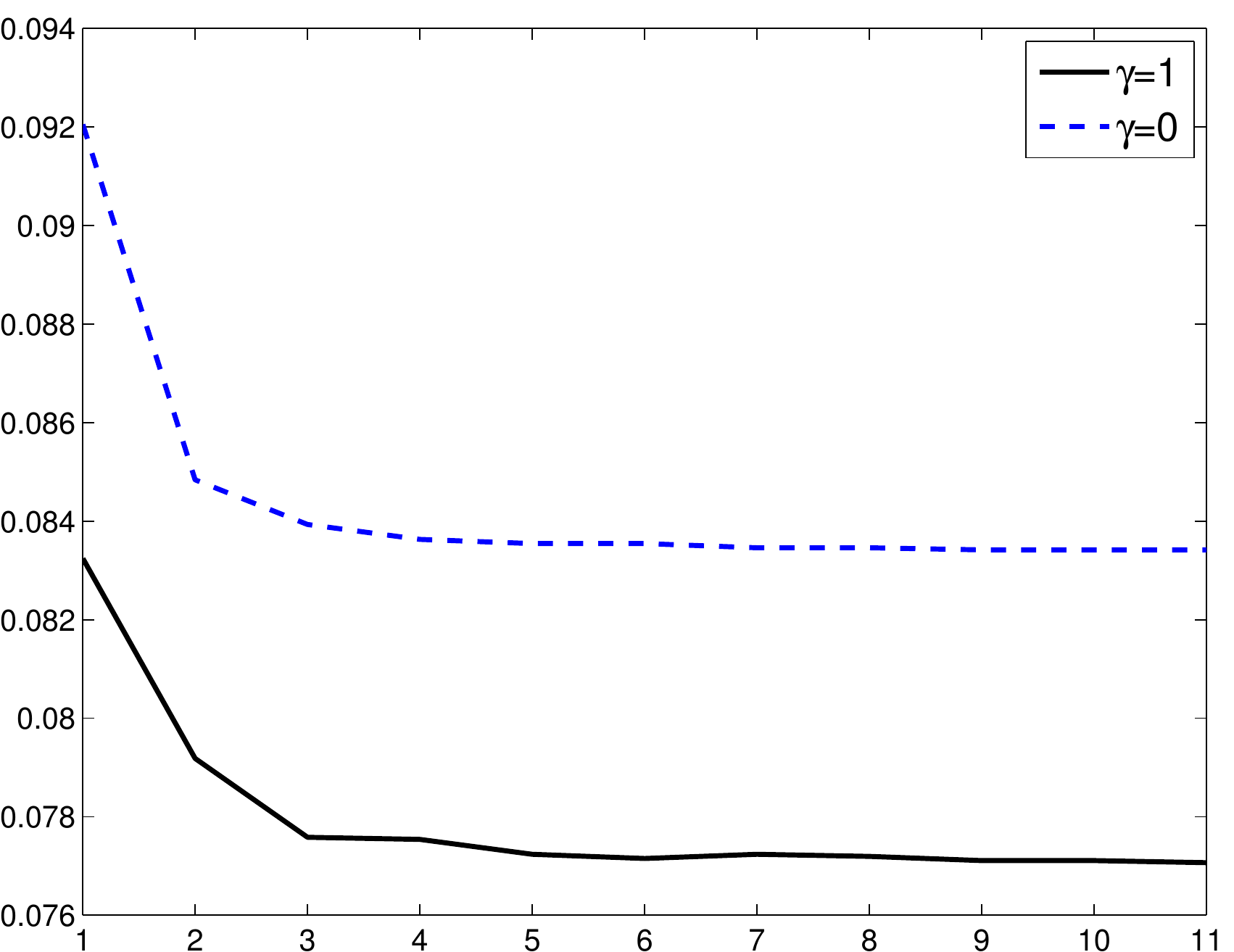} &
\includegraphics[width=0.31\textwidth]{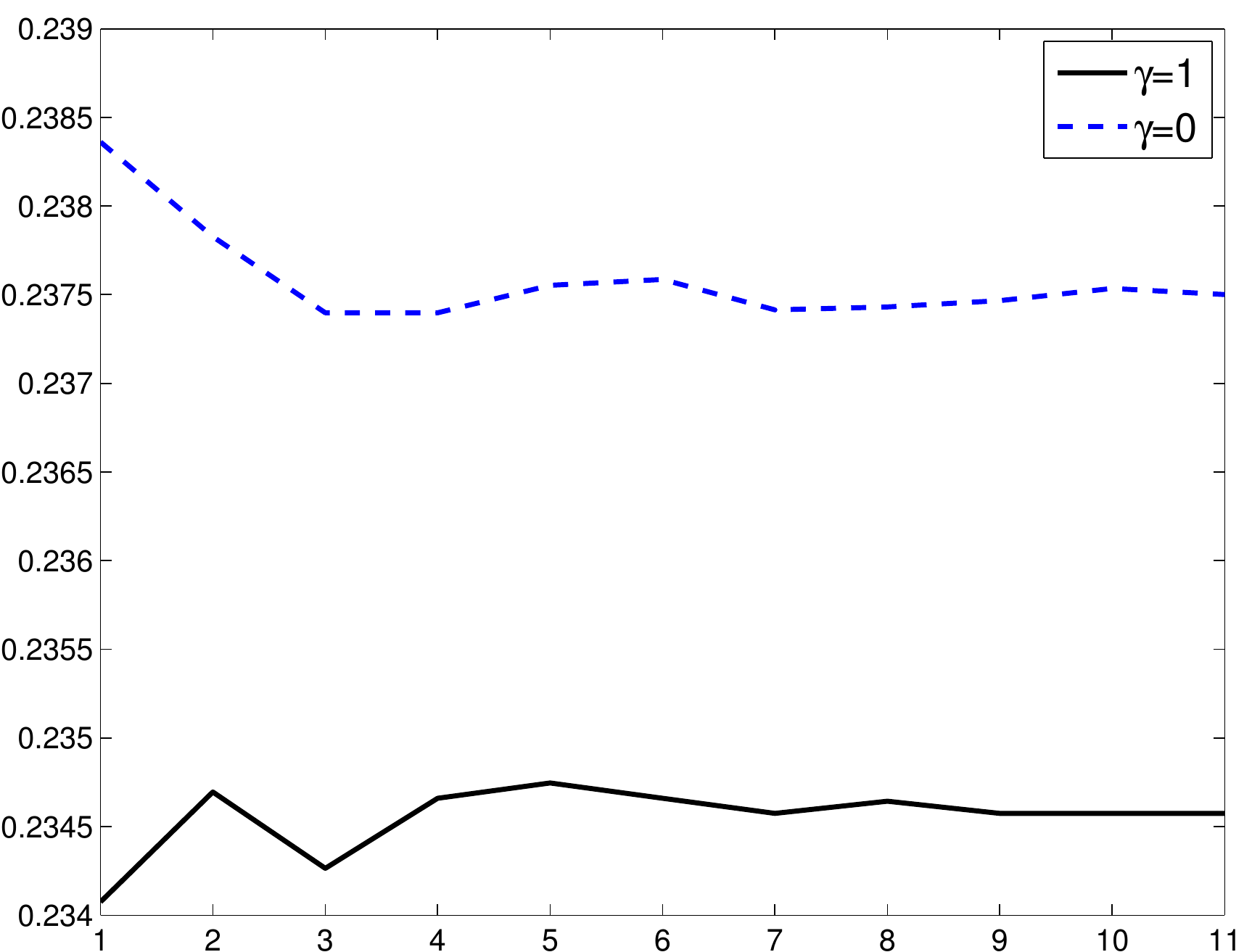}\\
$10^{-4}$ &
\includegraphics[width=0.31\textwidth]{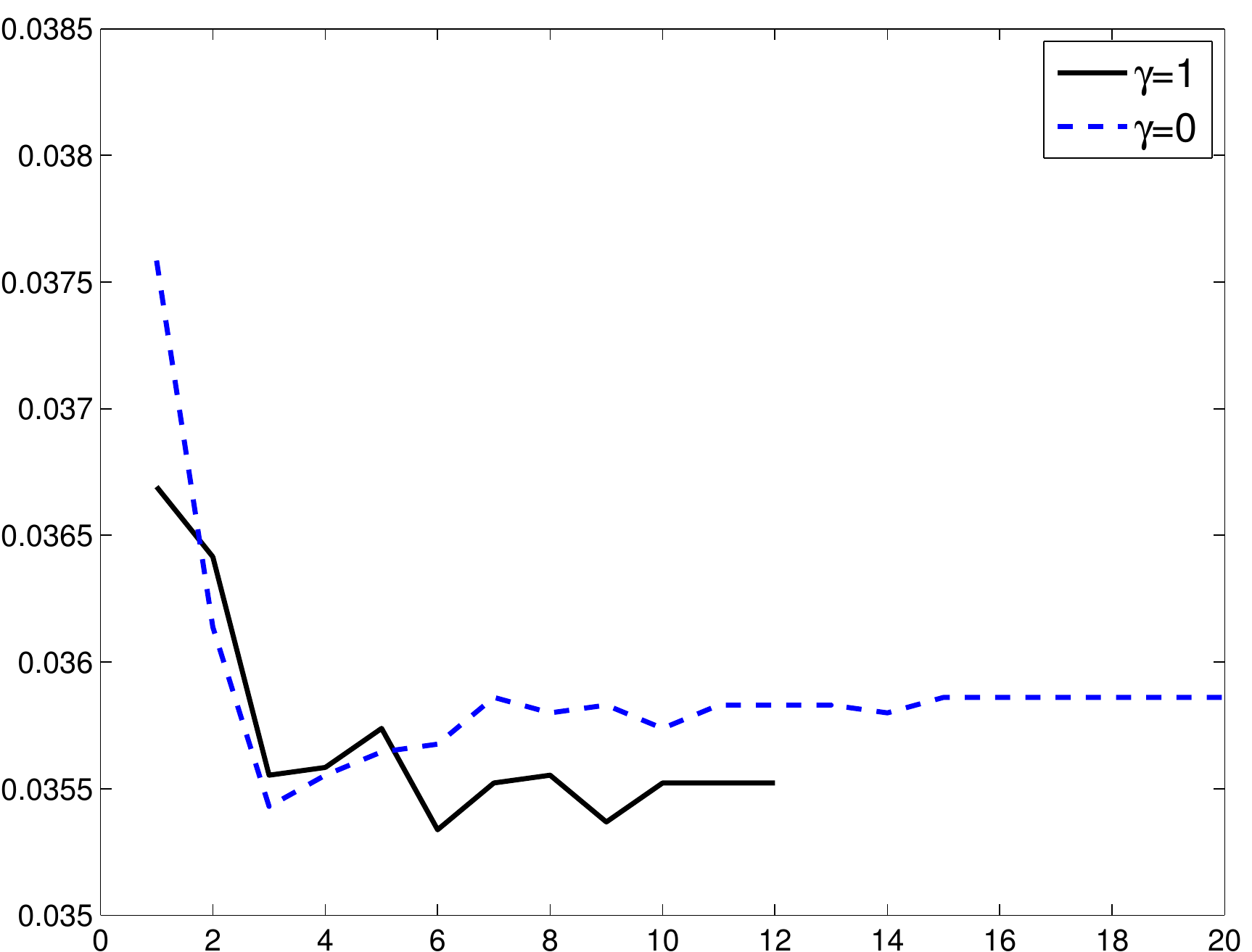} &
\includegraphics[width=0.31\textwidth]{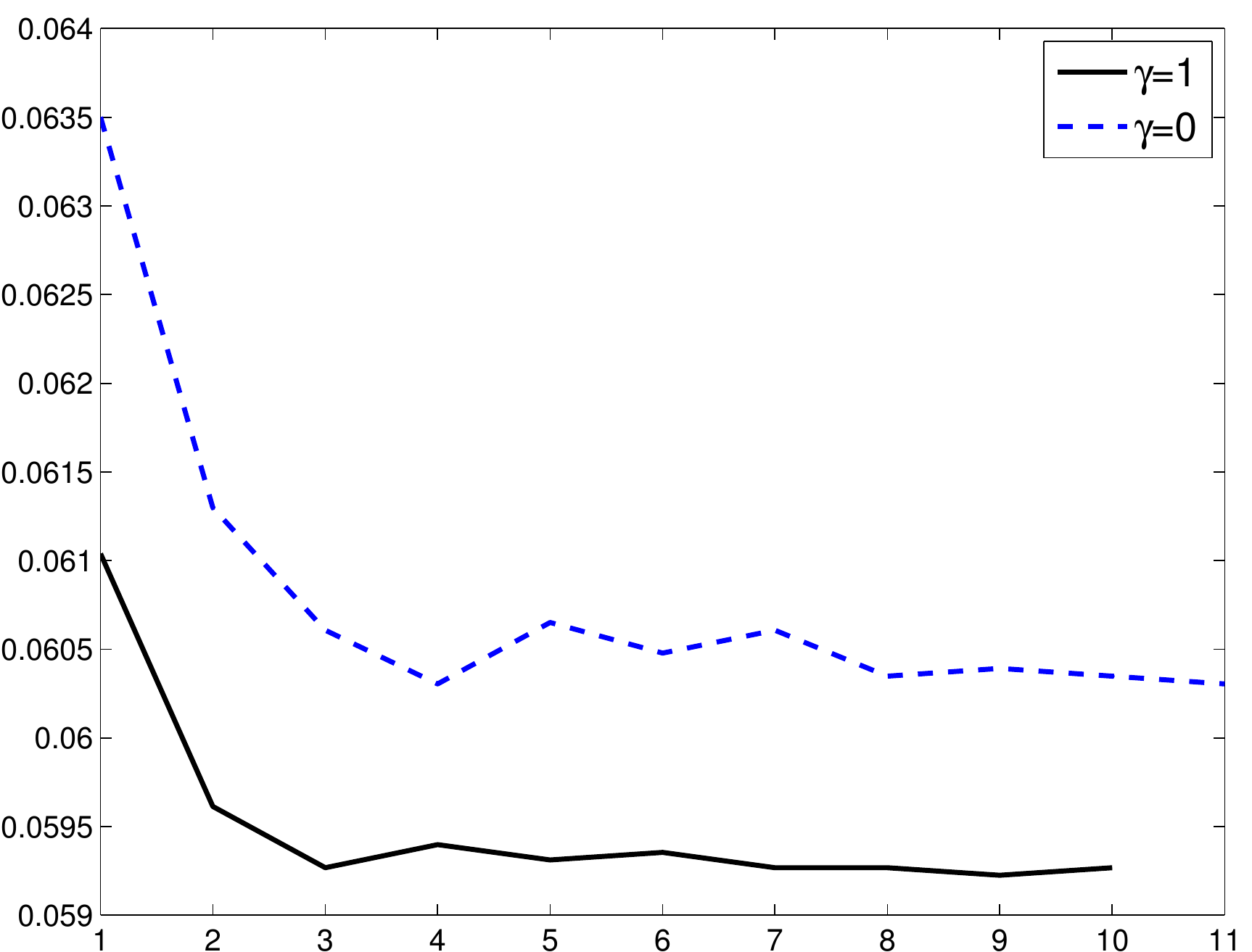} &
\includegraphics[width=0.31\textwidth]{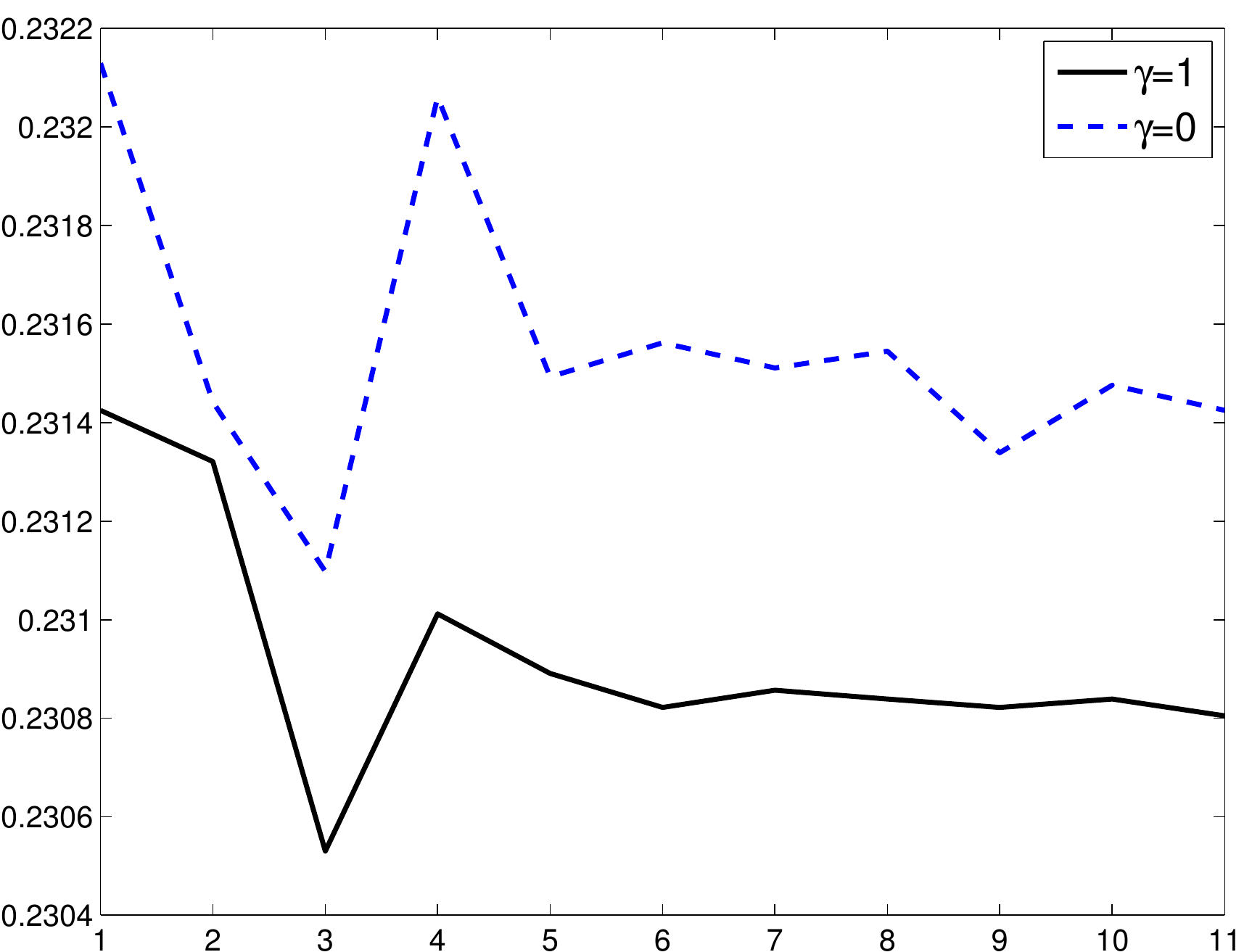}\\
$10^{-5}$ &
\includegraphics[width=0.31\textwidth]{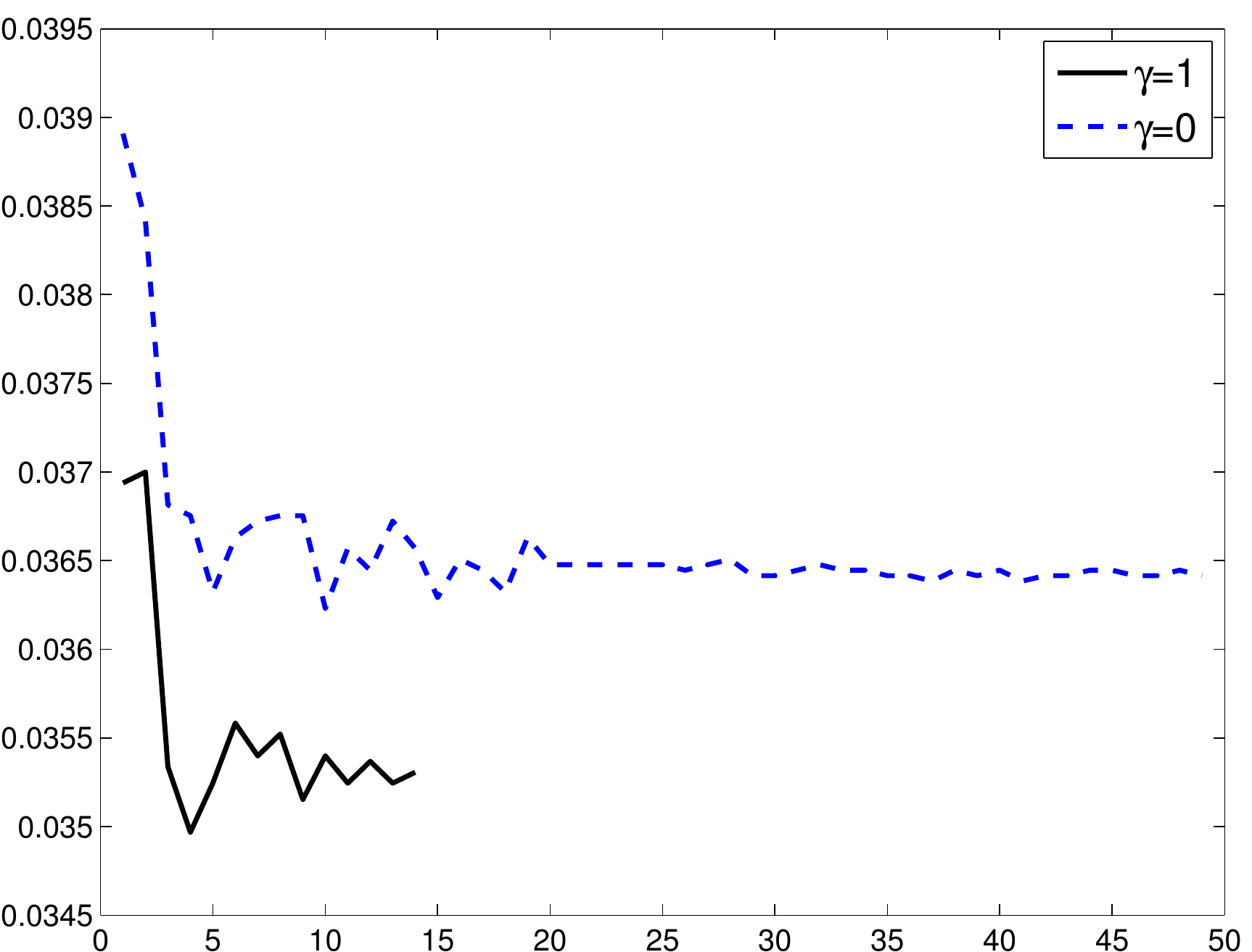} &
\includegraphics[width=0.31\textwidth]{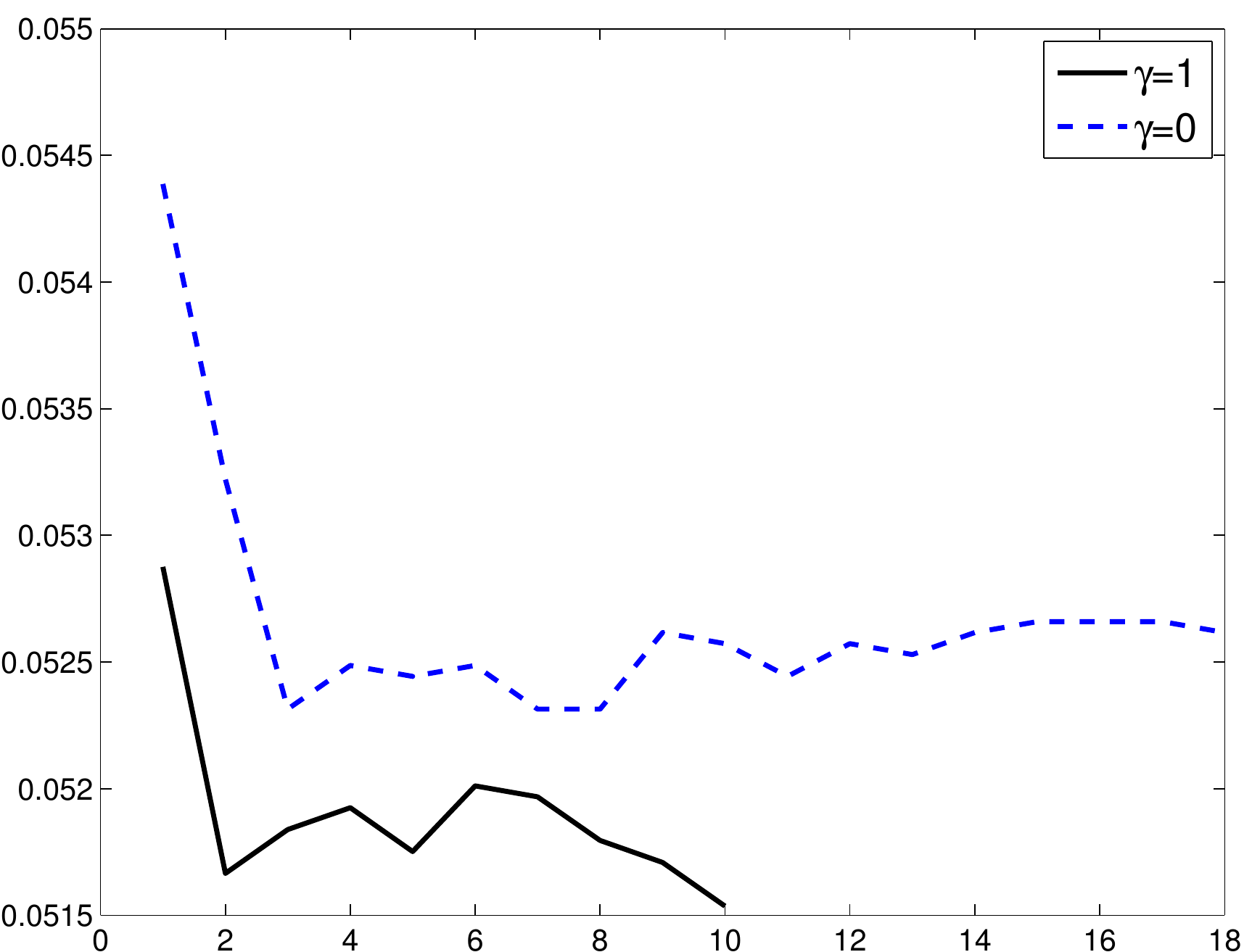} &
\includegraphics[width=0.31\textwidth]{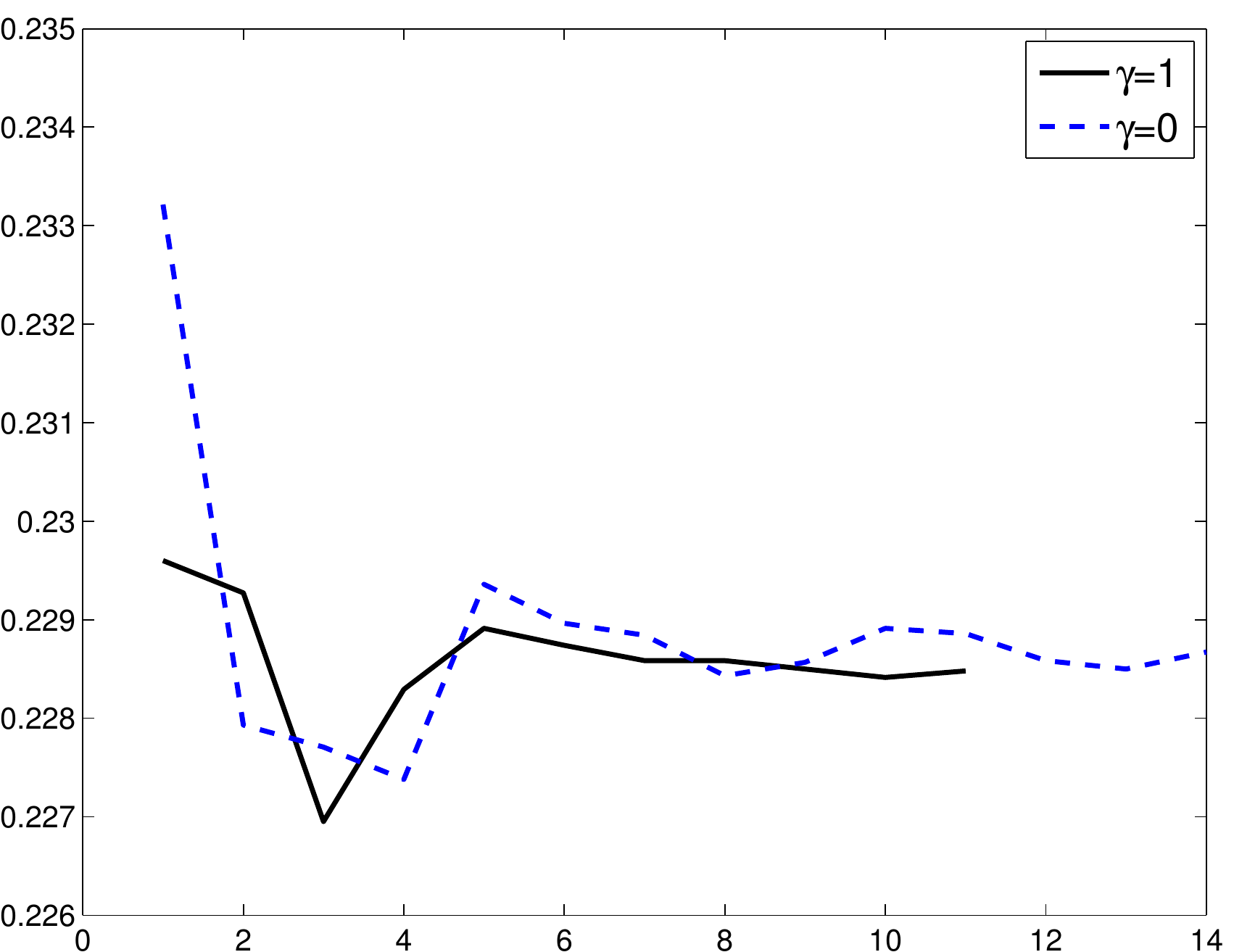}\\
$10^{-6}$ &
\includegraphics[width=0.31\textwidth]{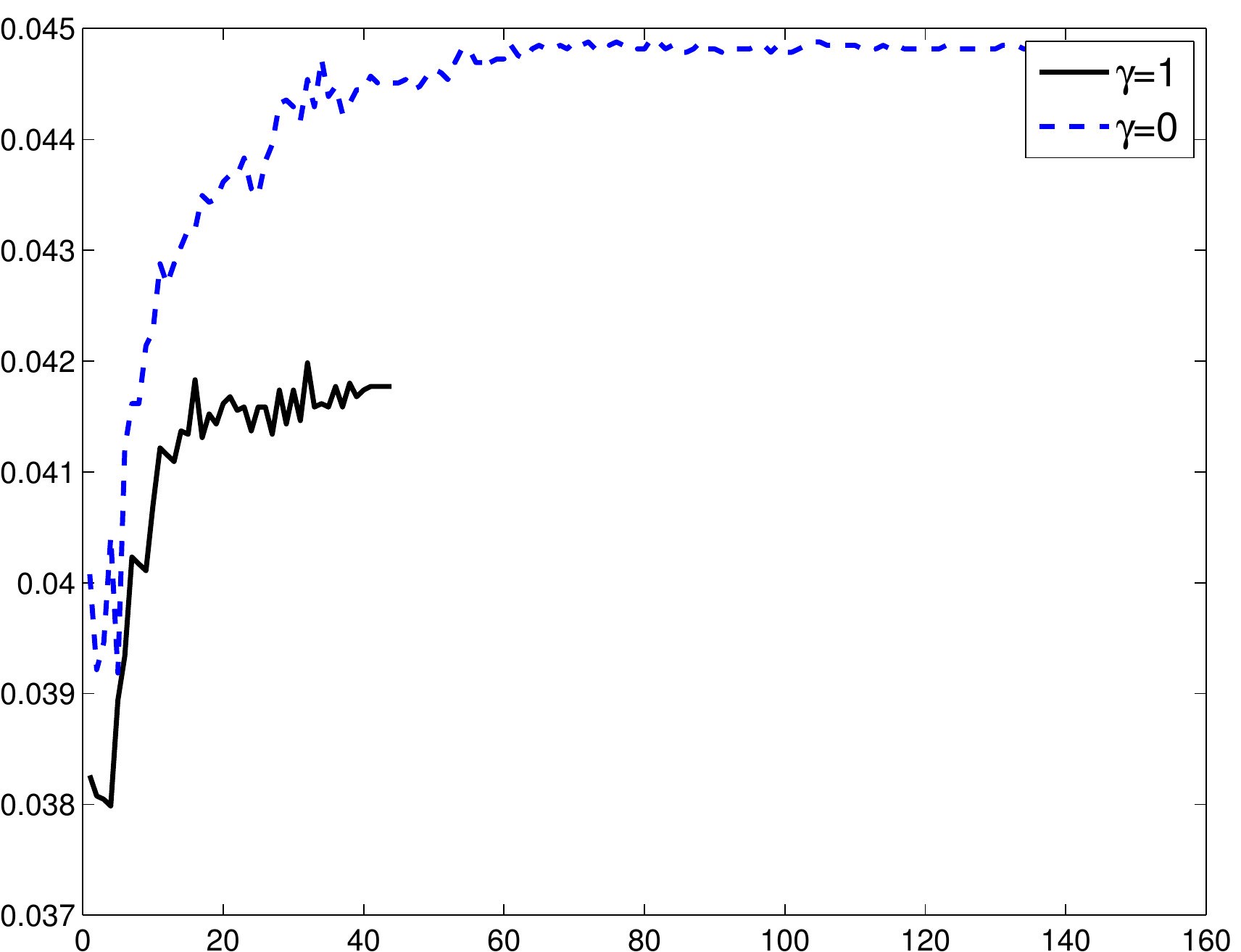} &
\includegraphics[width=0.31\textwidth]{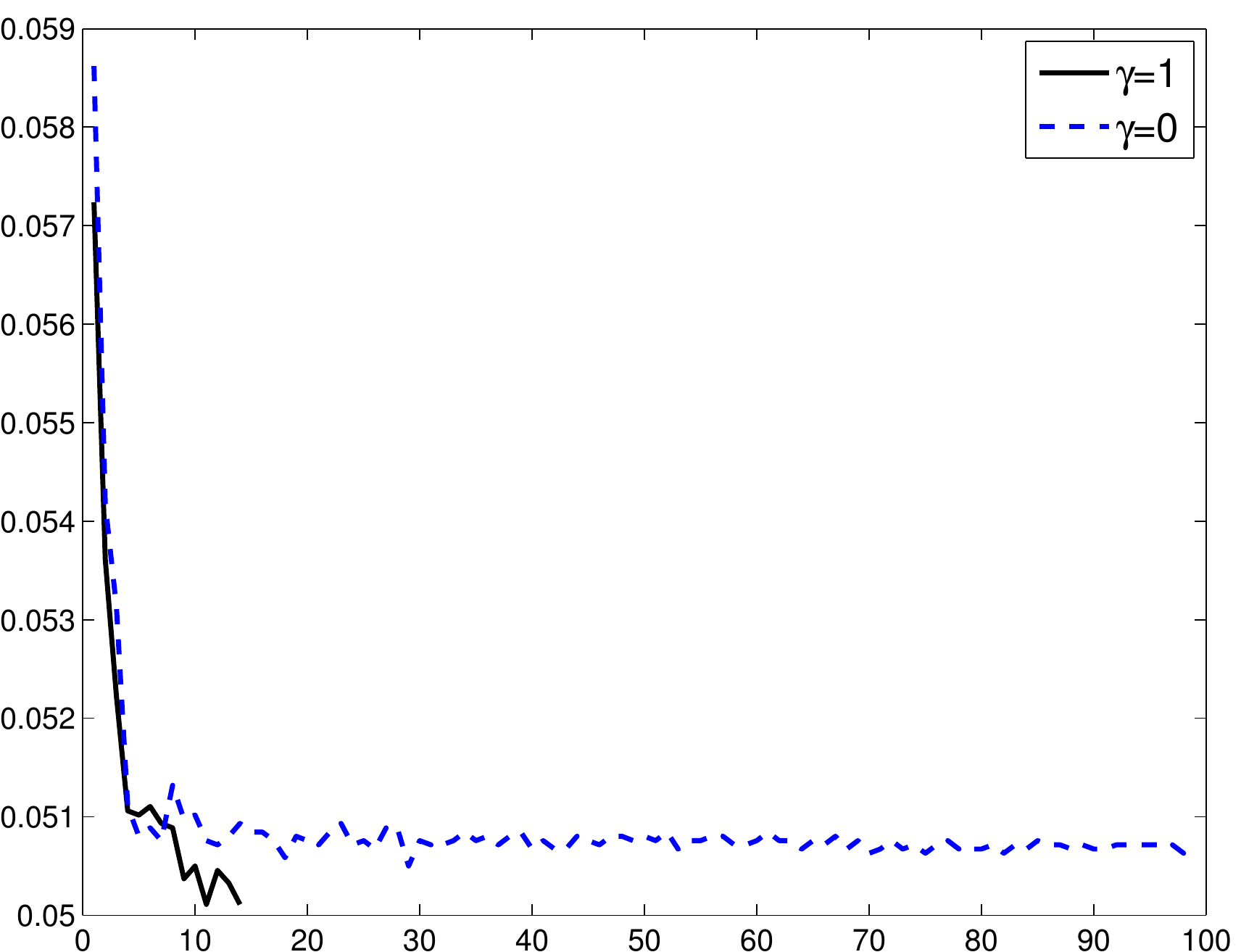} &
\includegraphics[width=0.31\textwidth]{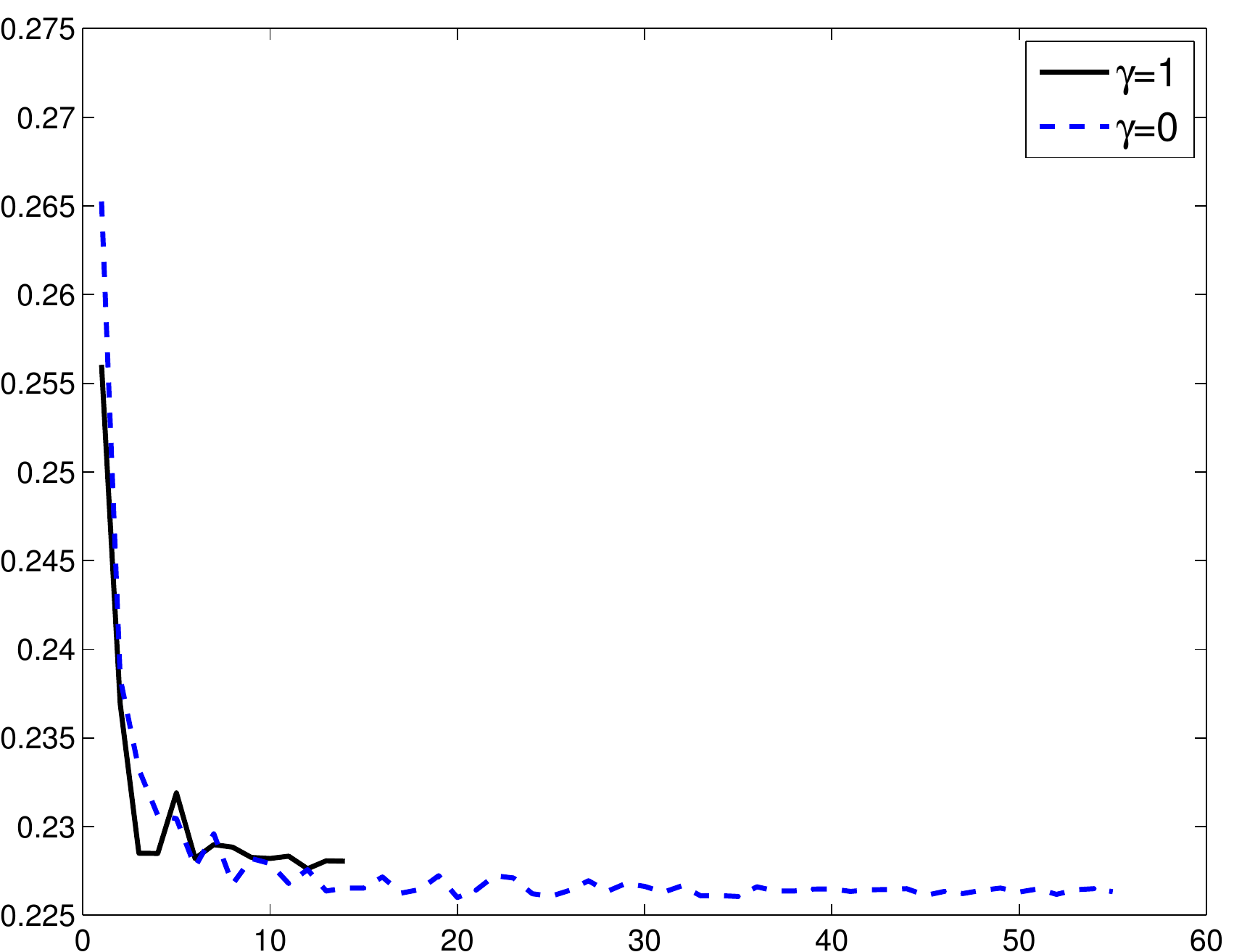}\\
\end{tabular}
\end{center}

\caption{\label{fig:SDCAgamma} Comparing the test zero-one error of
  SDCA for smoothed hinge-loss ($\gamma=1$) and non-smooth hinge-loss
  ($\gamma=0)$. In all plots the vertical axis is the zero-one error
  on the test set and the horizontal axis is the number of iterations
  divided by training set size (corresponding to the number of epochs
  through the data). We terminated each method when the duality gap
  was smaller than $10^{-5}$.}

\end{figure}





\begin{figure}

\begin{center}
\begin{tabular}{ @{} L | @{} S @{} S @{} S @{} }
$\lambda$ & \scriptsize{astro-ph} & \scriptsize{CCAT} & \scriptsize{cov1}\\ \hline
$10^{-3}$ & 
\includegraphics[width=0.31\textwidth]{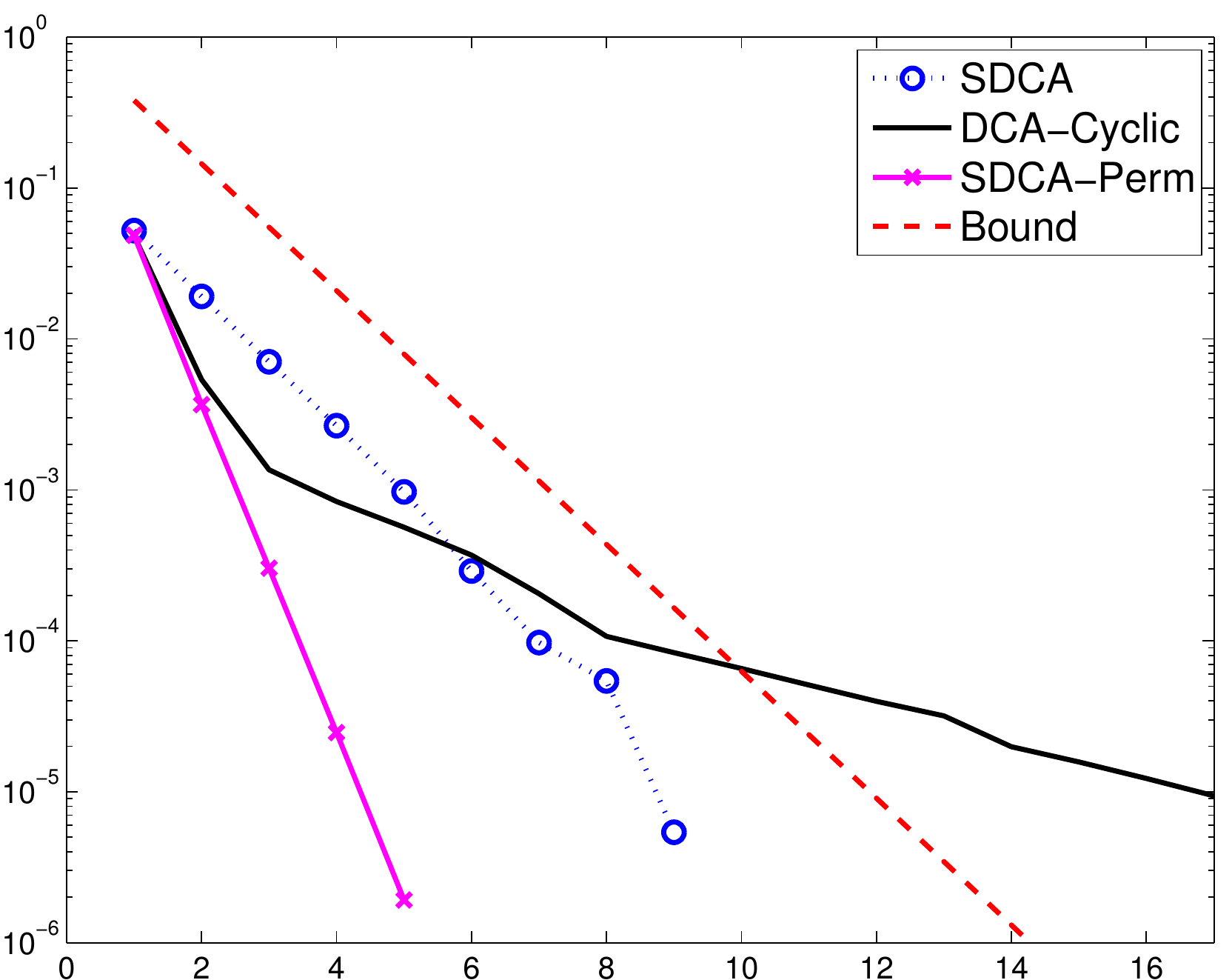} &
\includegraphics[width=0.31\textwidth]{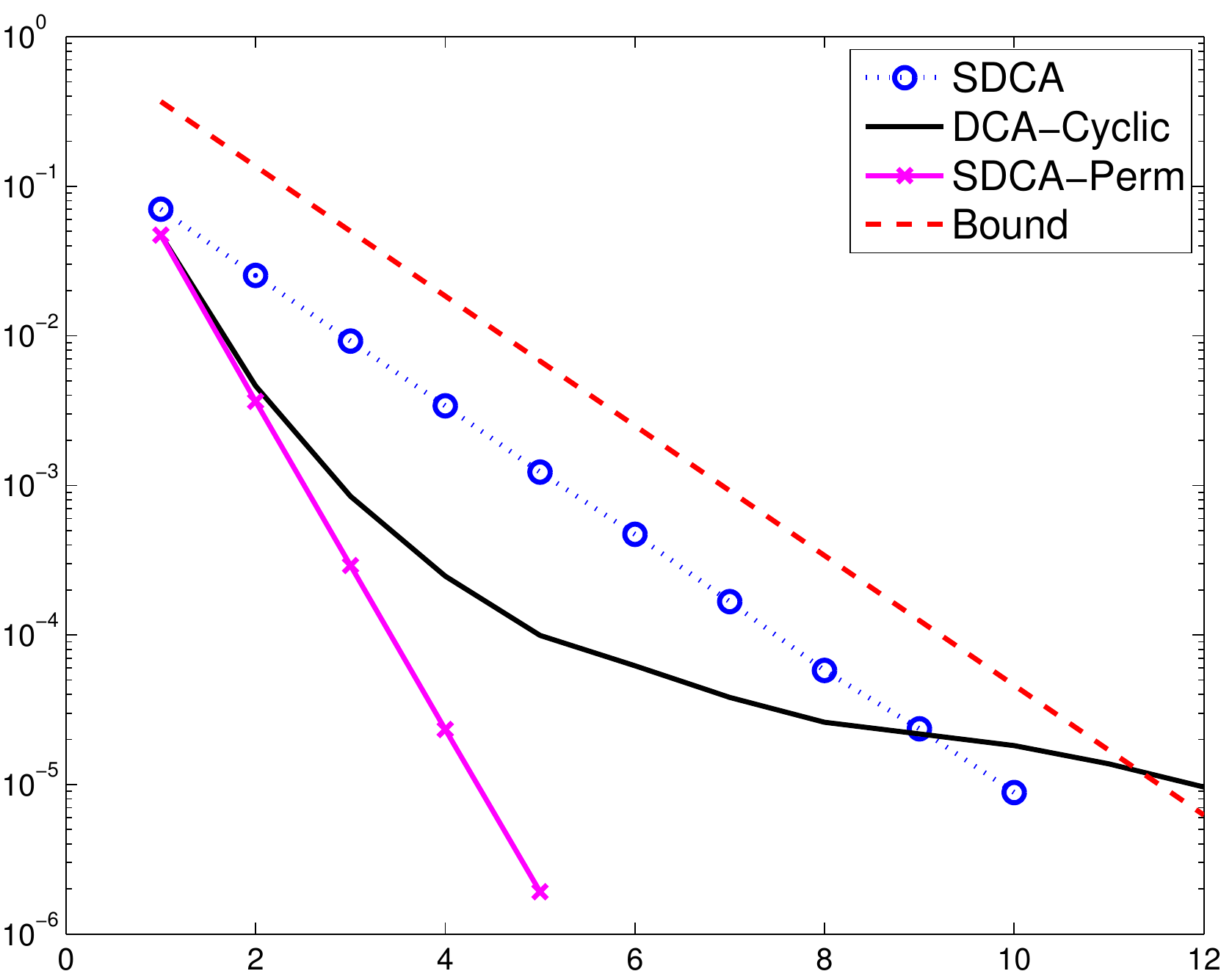} &
\includegraphics[width=0.31\textwidth]{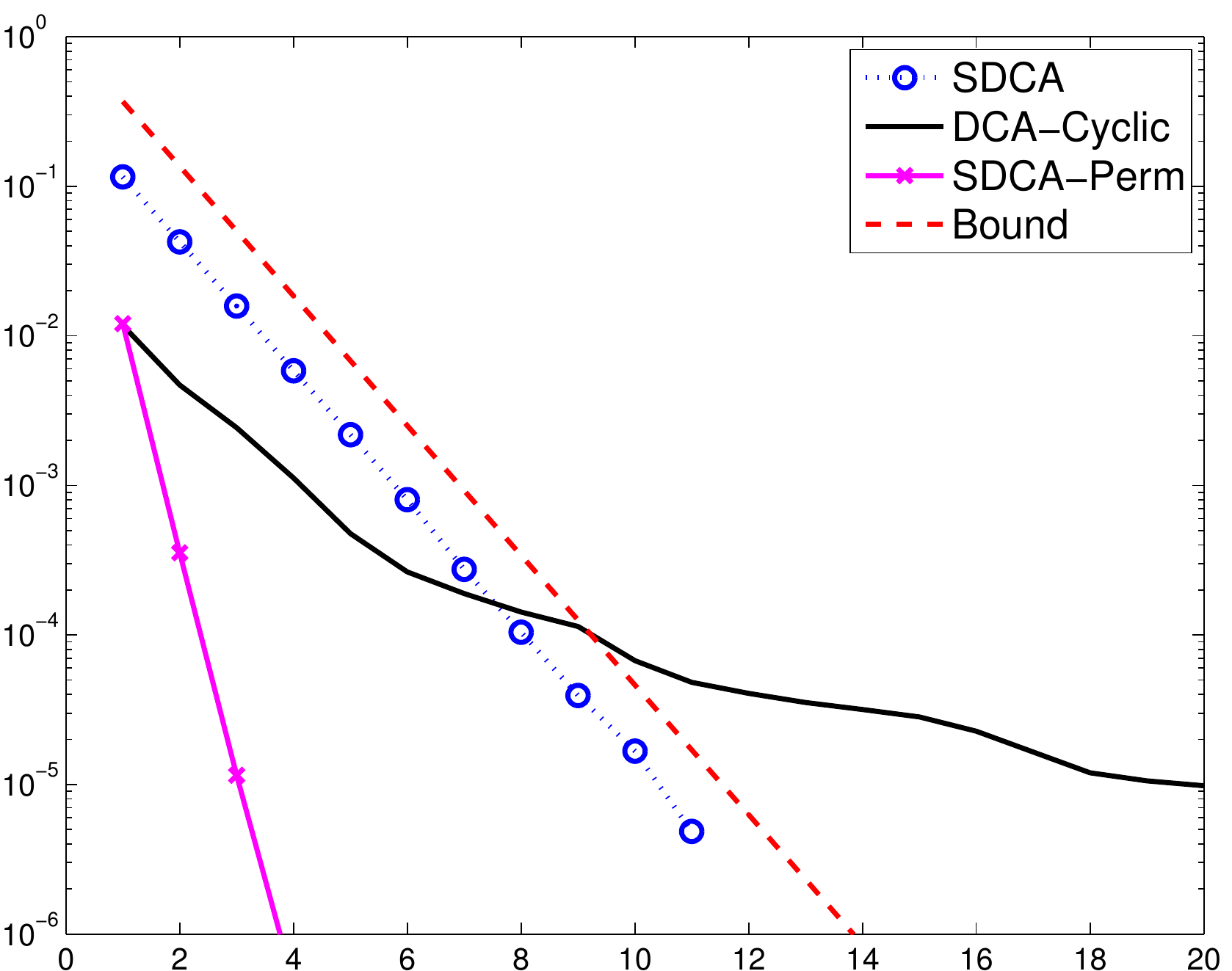}\\
$10^{-4}$ &
\includegraphics[width=0.31\textwidth]{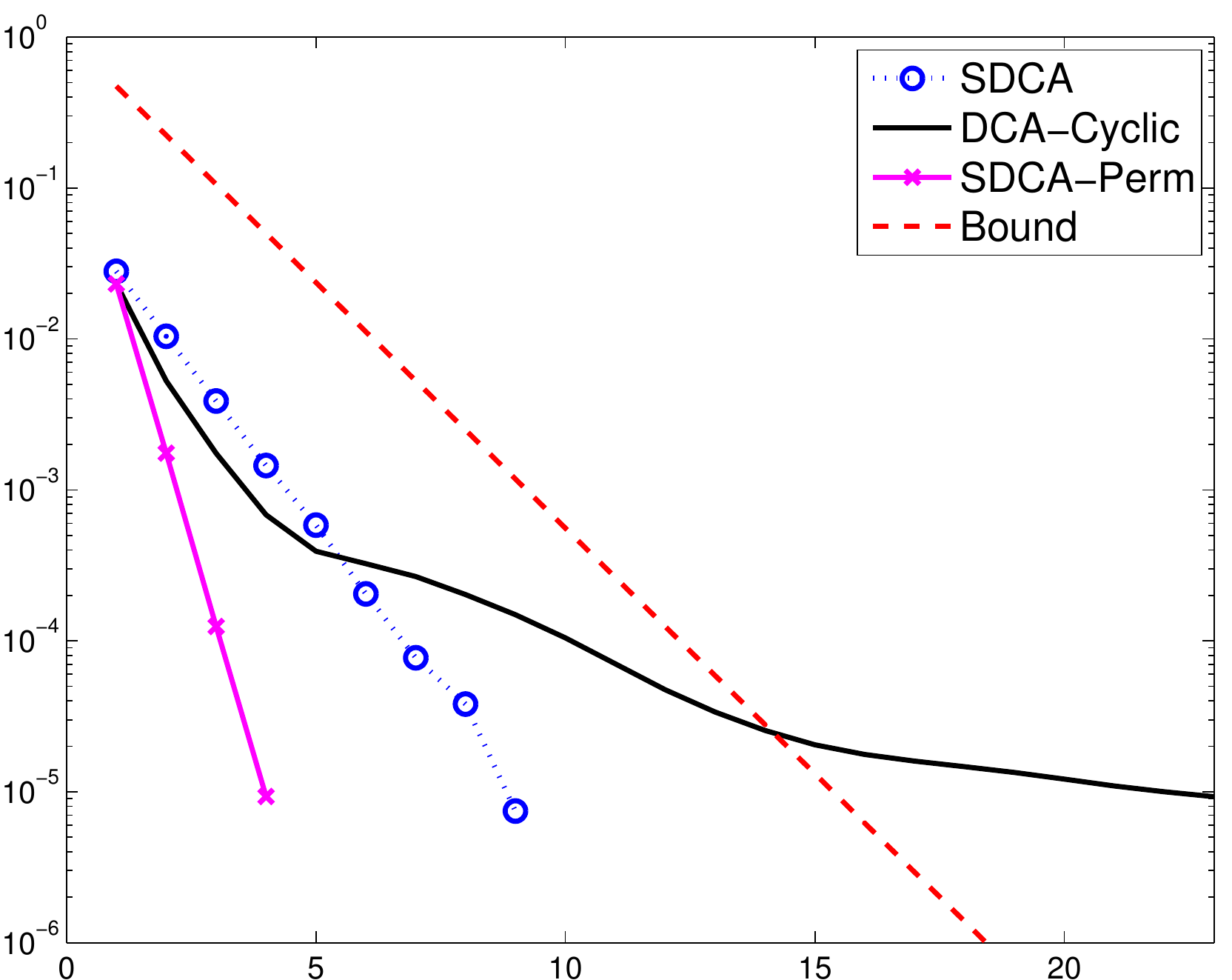} &
\includegraphics[width=0.31\textwidth]{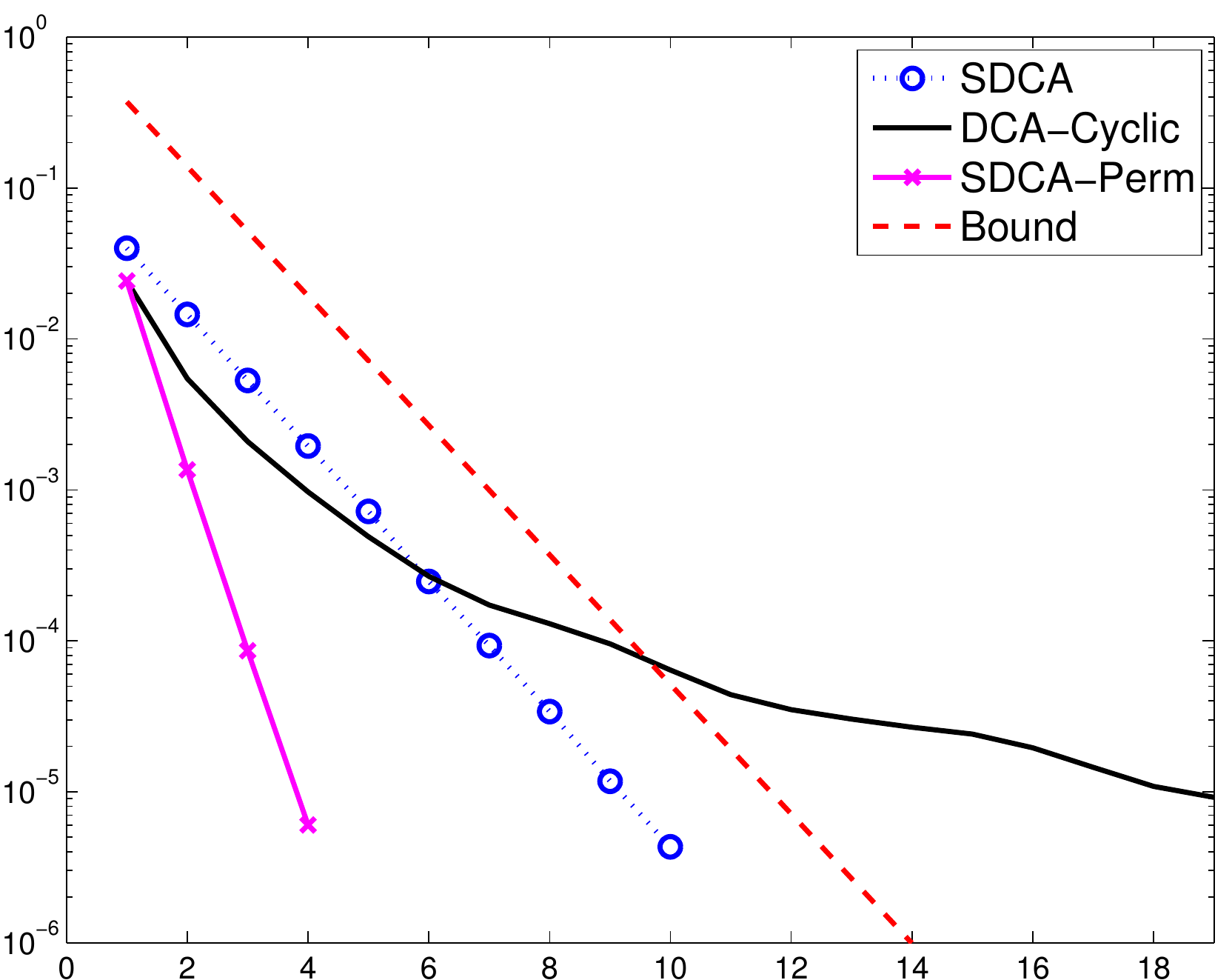} &
\includegraphics[width=0.31\textwidth]{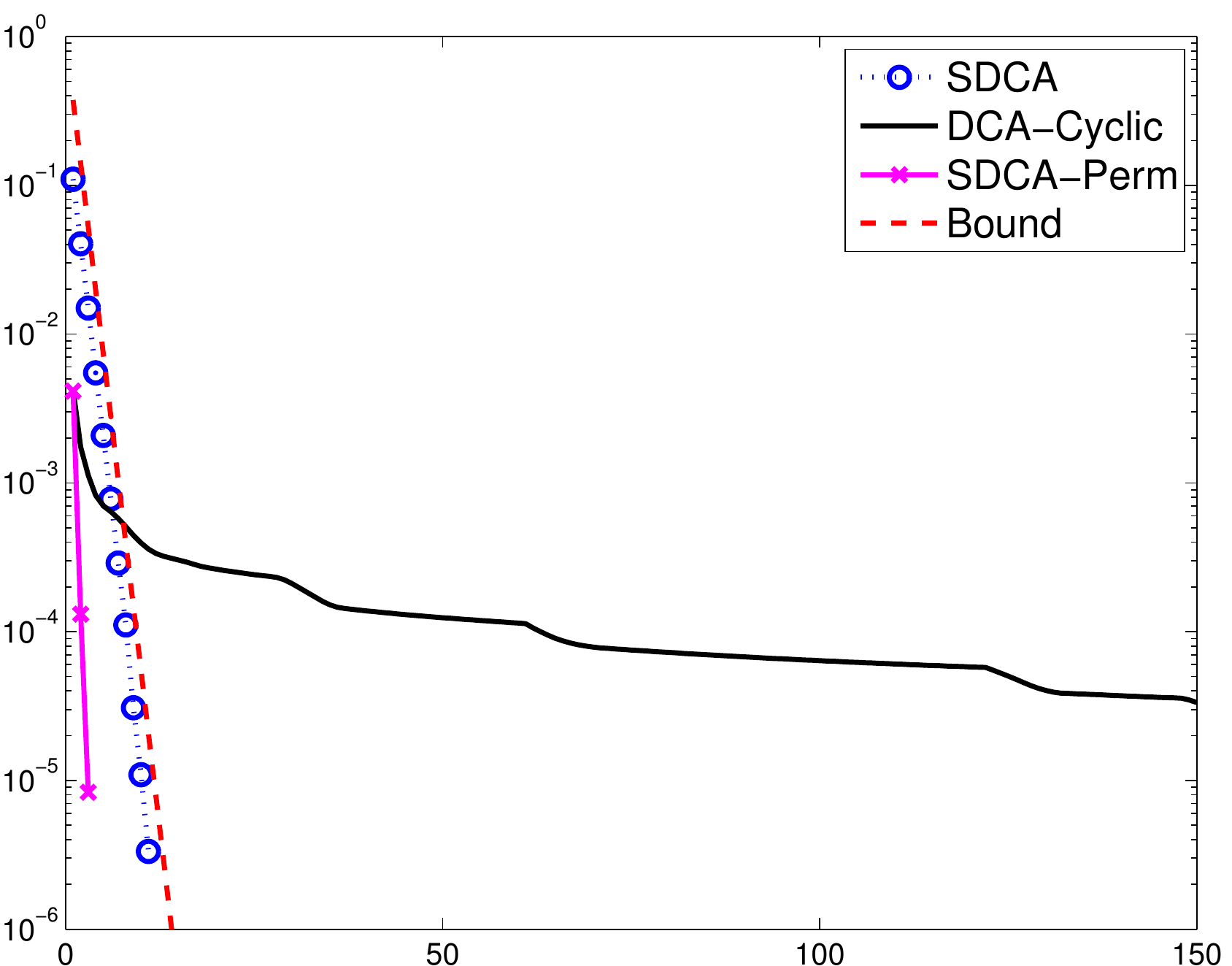}\\
$10^{-5}$ &
\includegraphics[width=0.31\textwidth]{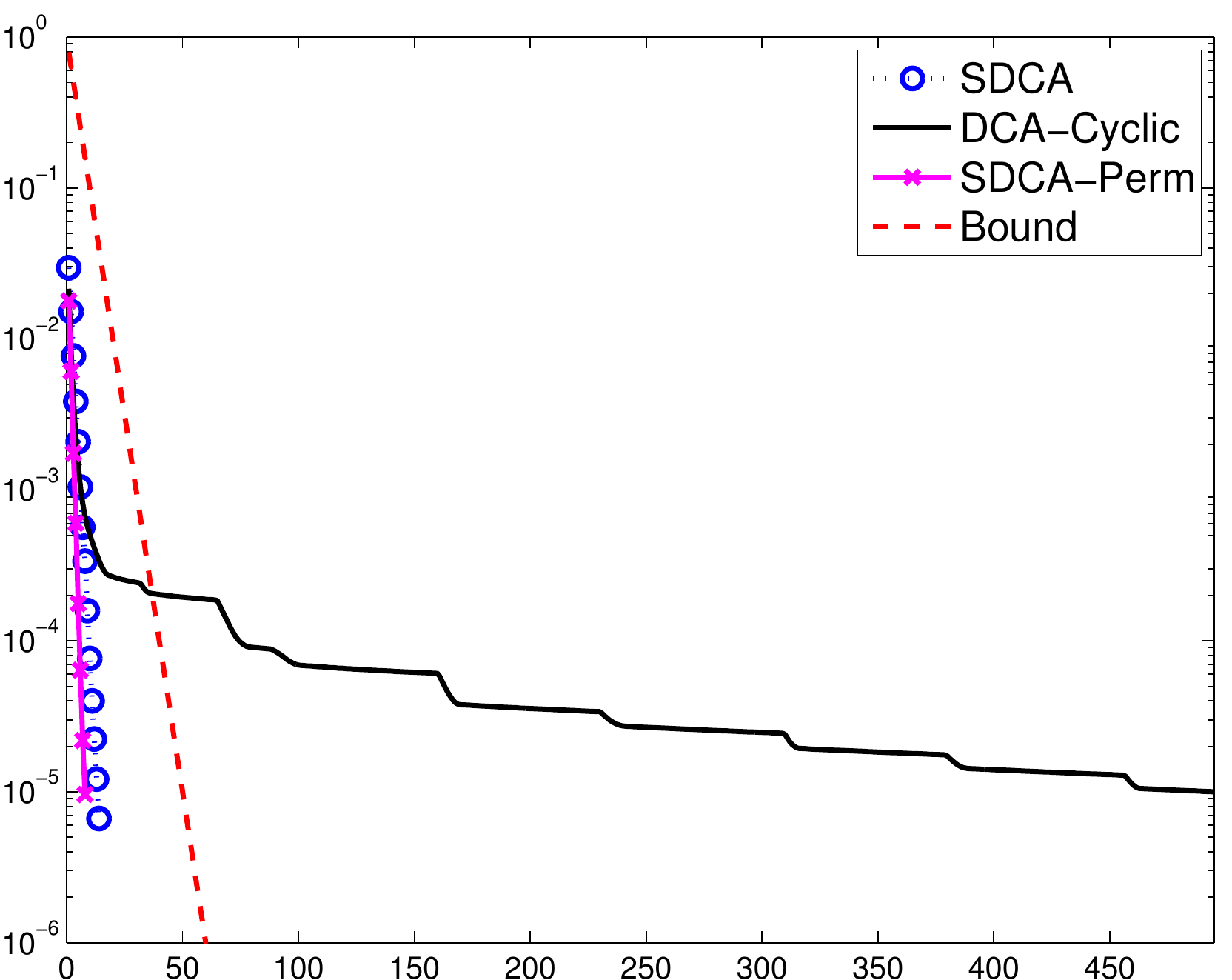} &
\includegraphics[width=0.31\textwidth]{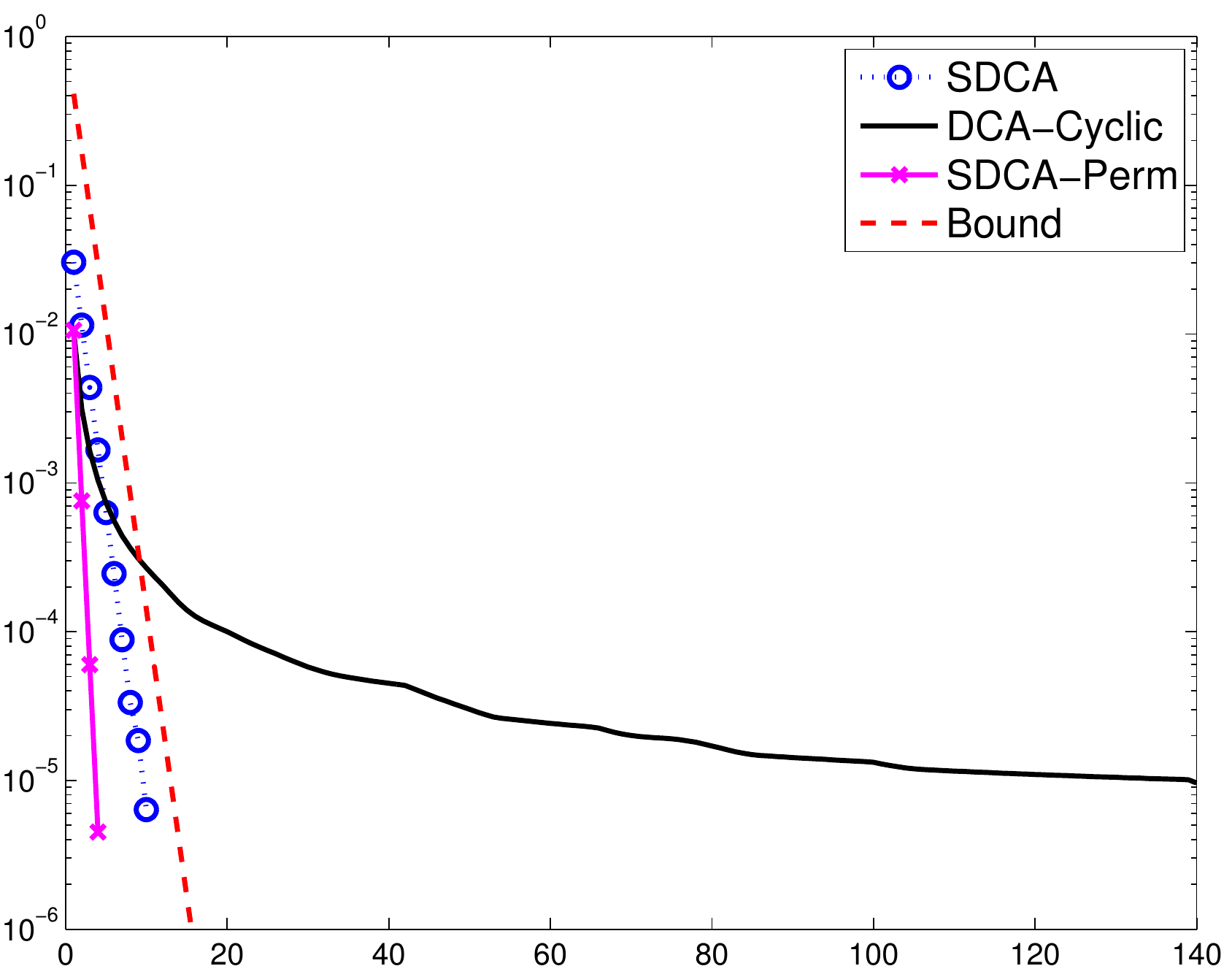} &
\includegraphics[width=0.31\textwidth]{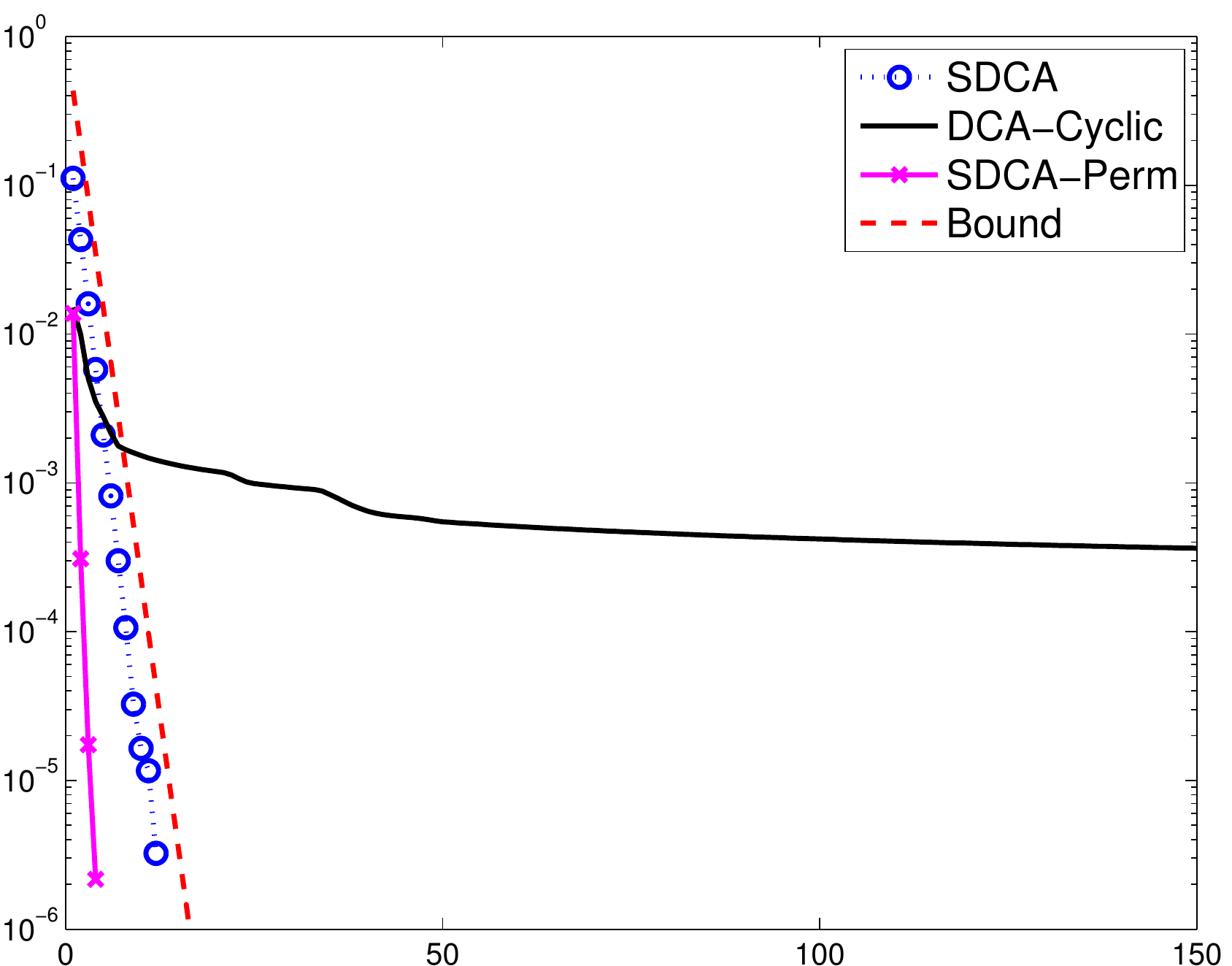}\\
$10^{-6}$ &
\includegraphics[width=0.31\textwidth]{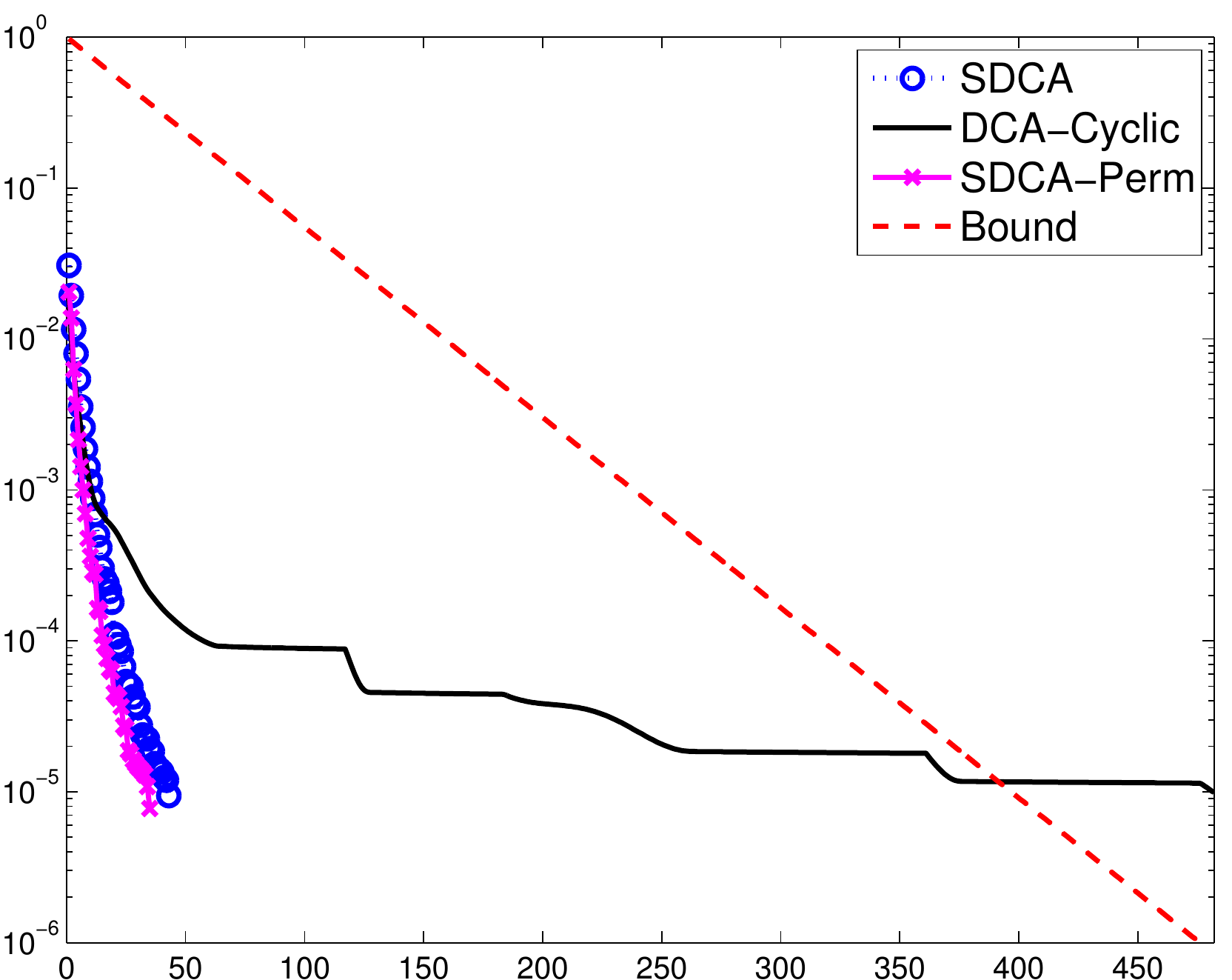} &
\includegraphics[width=0.31\textwidth]{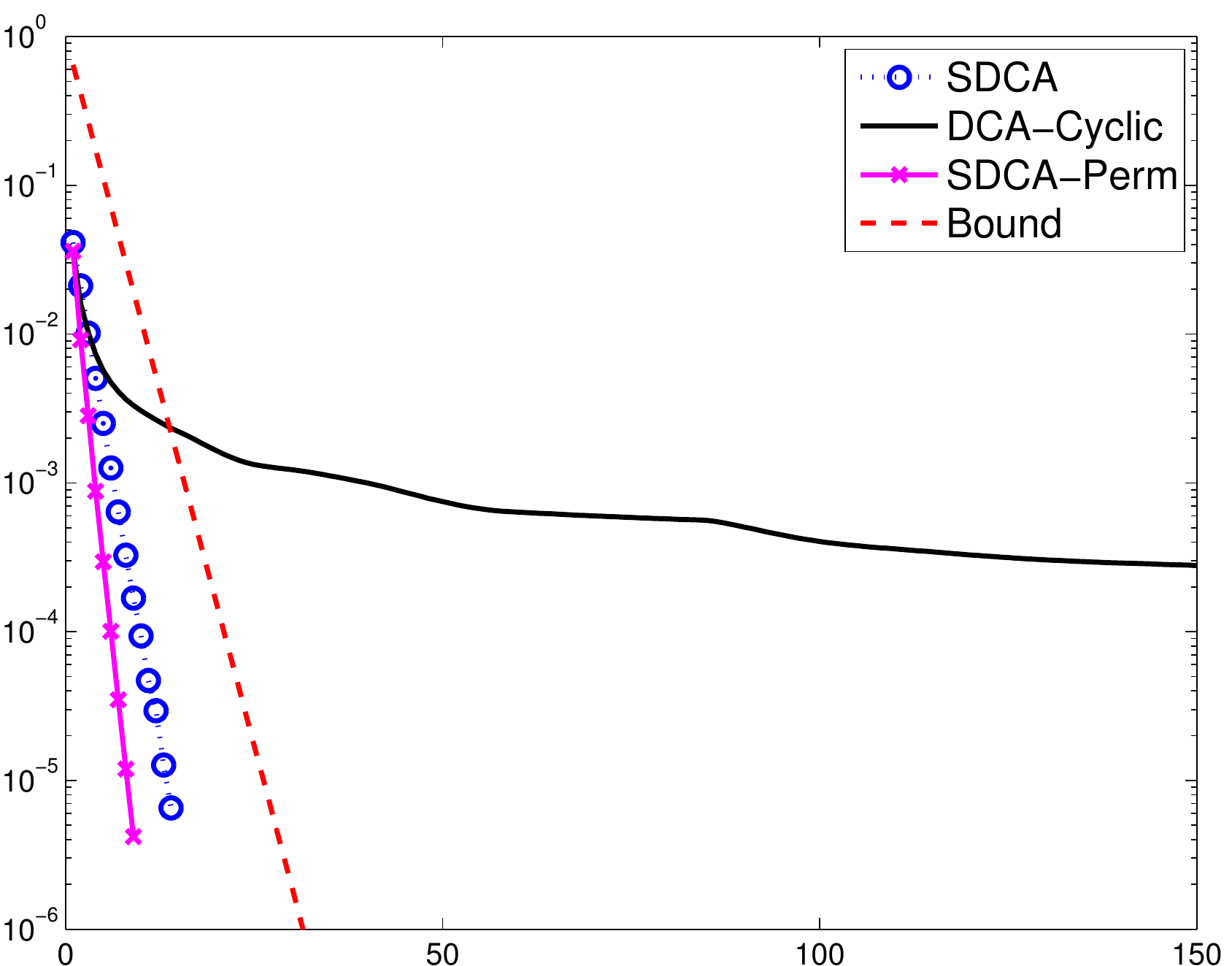} &
\includegraphics[width=0.31\textwidth]{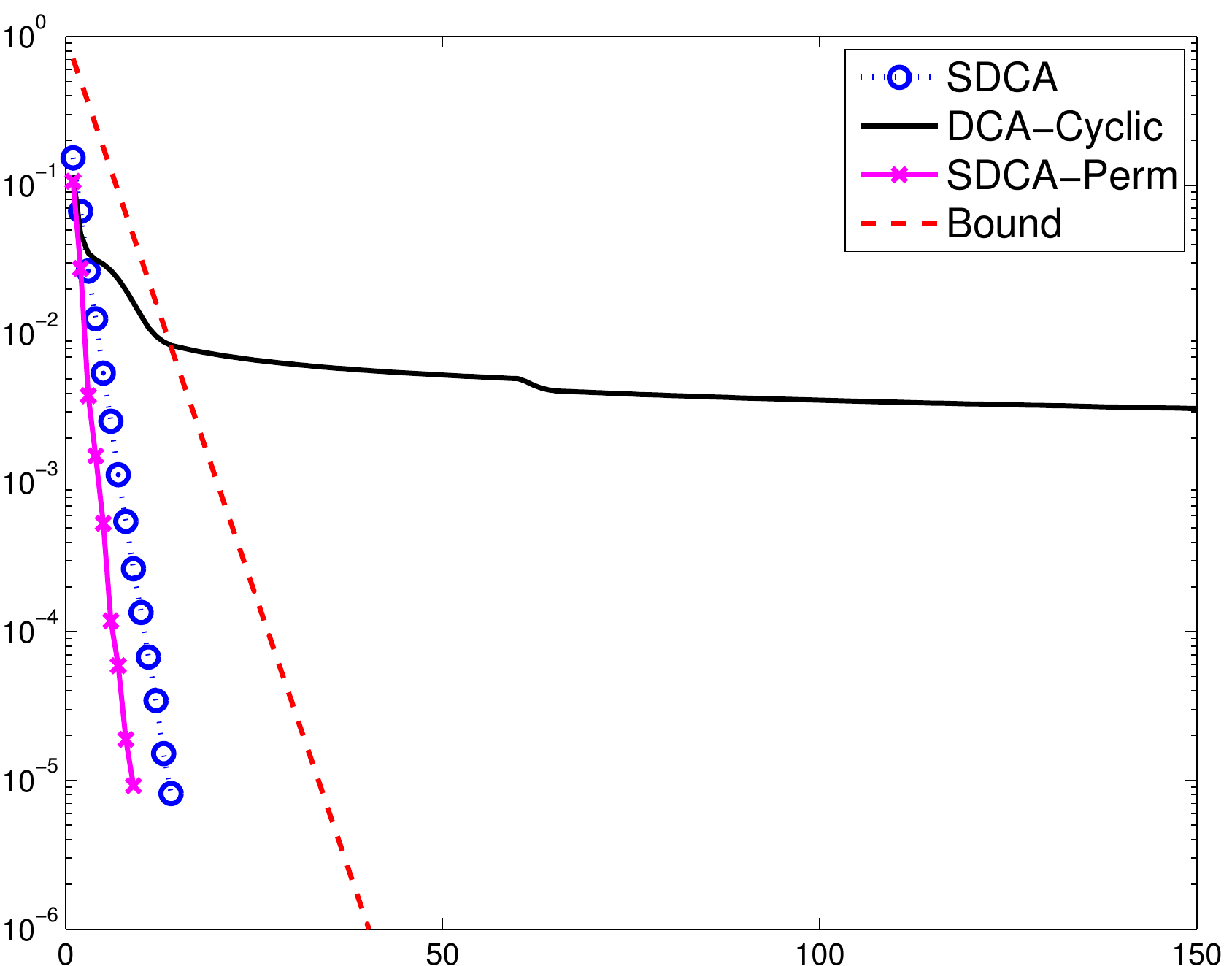}\\
\end{tabular}
\end{center}

\caption{\label{fig:cyclic} Comparing the duality gap achieved by
  choosing dual variables at random with repetitions (SDCA), choosing
  dual variables at random without repetitions (SDCA-Perm), or using a
  fixed cyclic order. In all cases, the duality gap is depicted as a
  function of the number of epochs for different values of
  $\lambda$. The loss function is the smooth hinge loss with
  $\gamma=1$.}

\end{figure}

\begin{figure}

\begin{center}
\begin{tabular}{ @{} L | @{} S @{} S @{} S @{} }
$\lambda$ & \scriptsize{astro-ph} & \scriptsize{CCAT} & \scriptsize{cov1}\\ \hline
$10^{-3}$ & 
\includegraphics[width=0.31\textwidth]{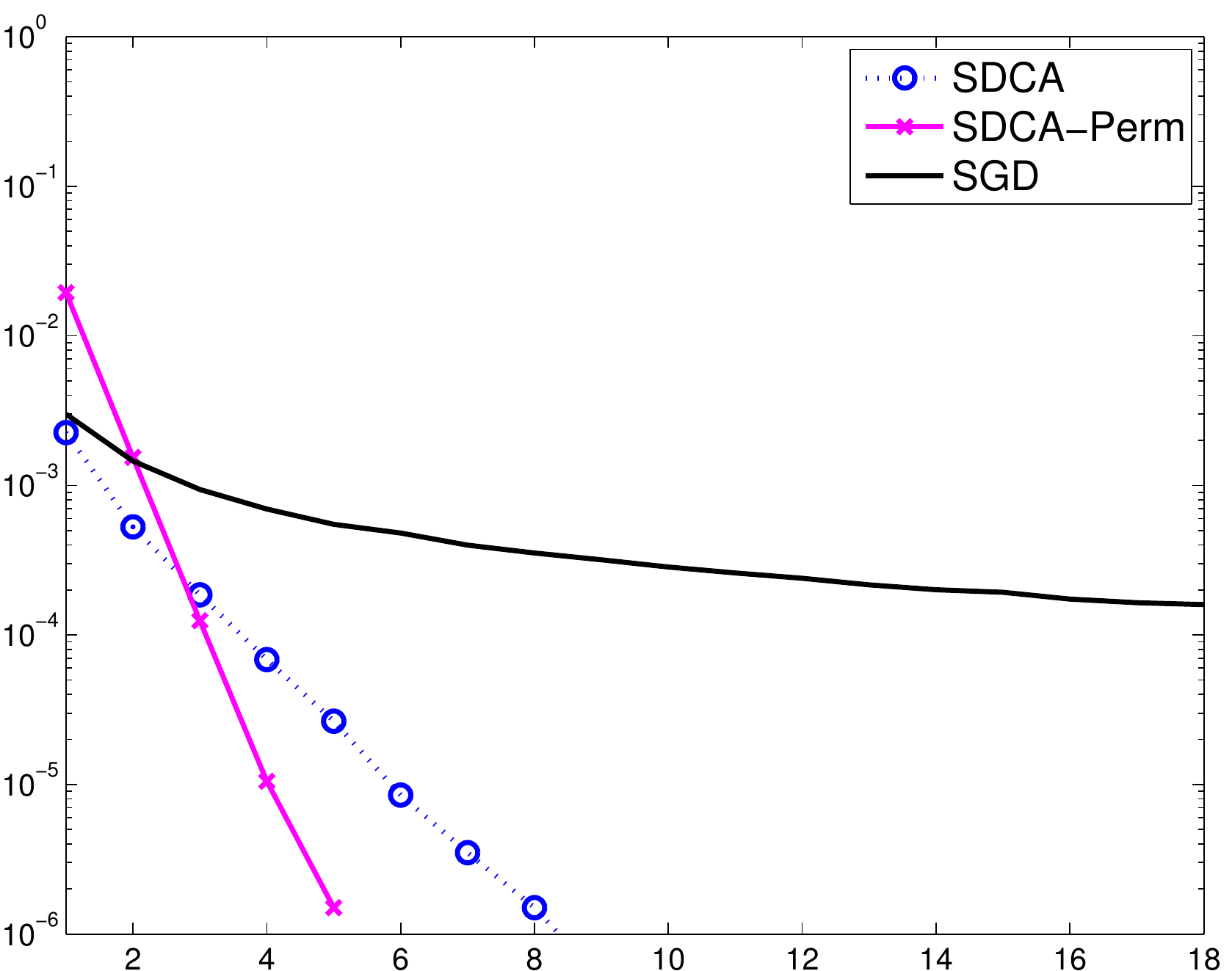} &
\includegraphics[width=0.31\textwidth]{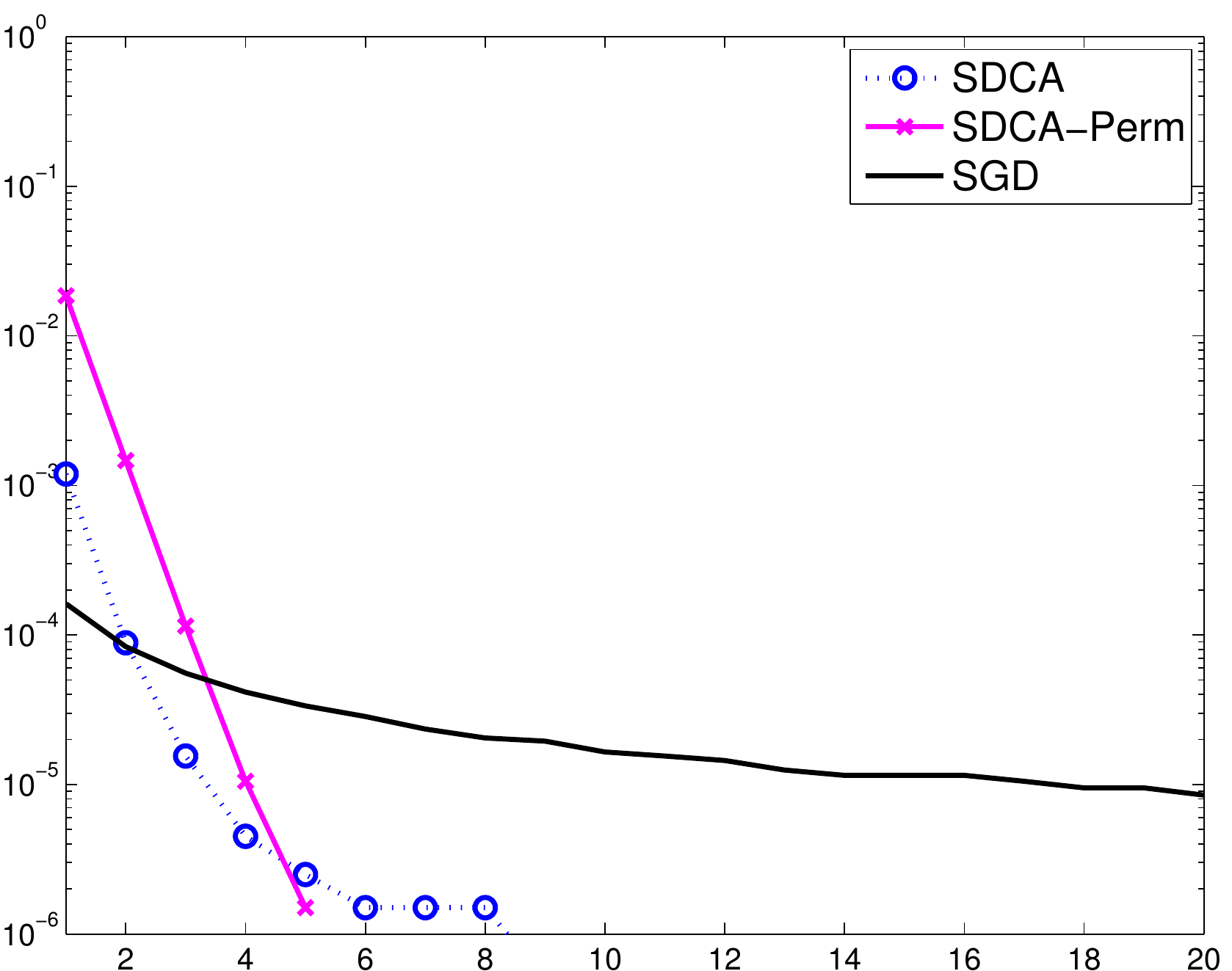} &
\includegraphics[width=0.31\textwidth]{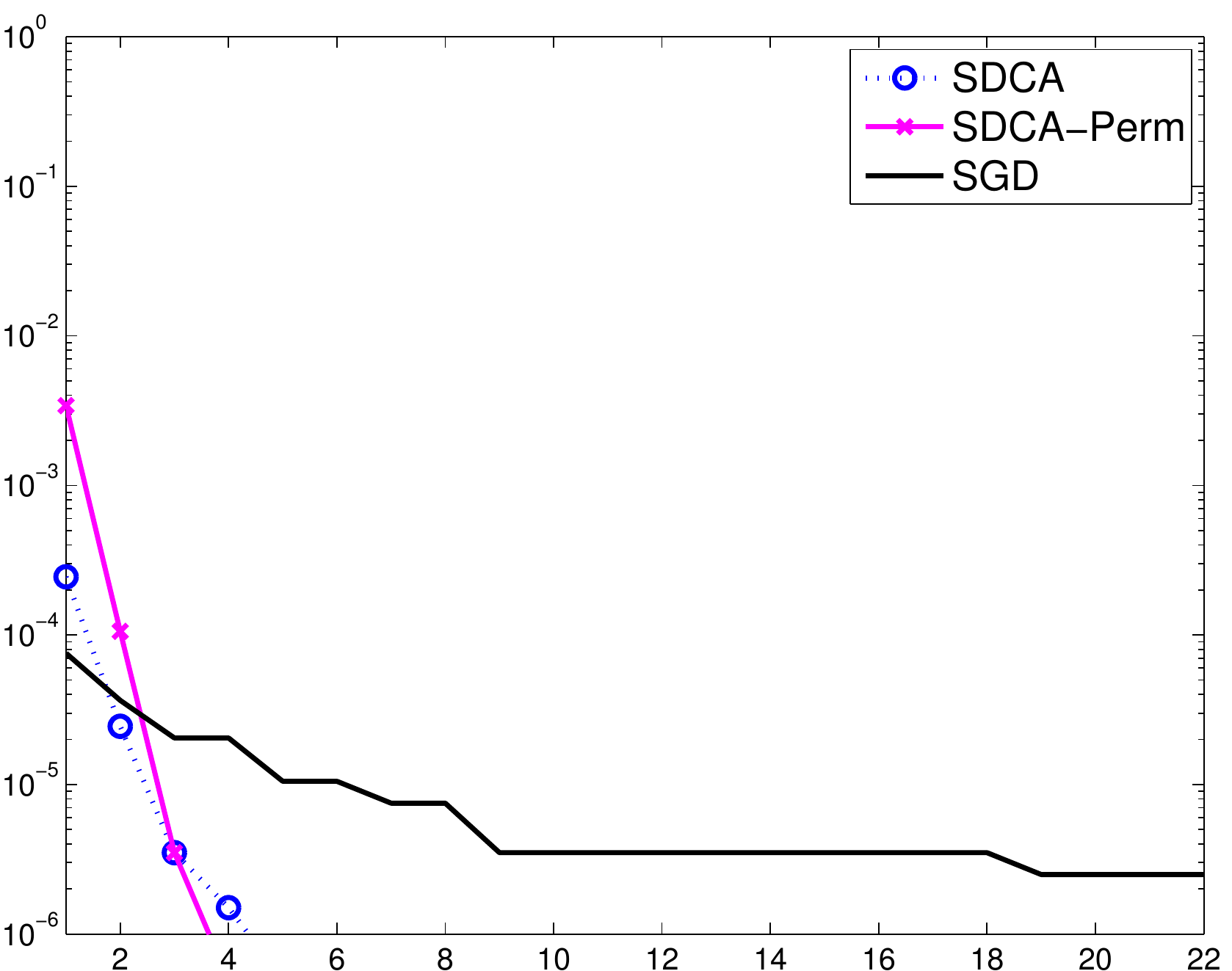}\\
$10^{-4}$ &
\includegraphics[width=0.31\textwidth]{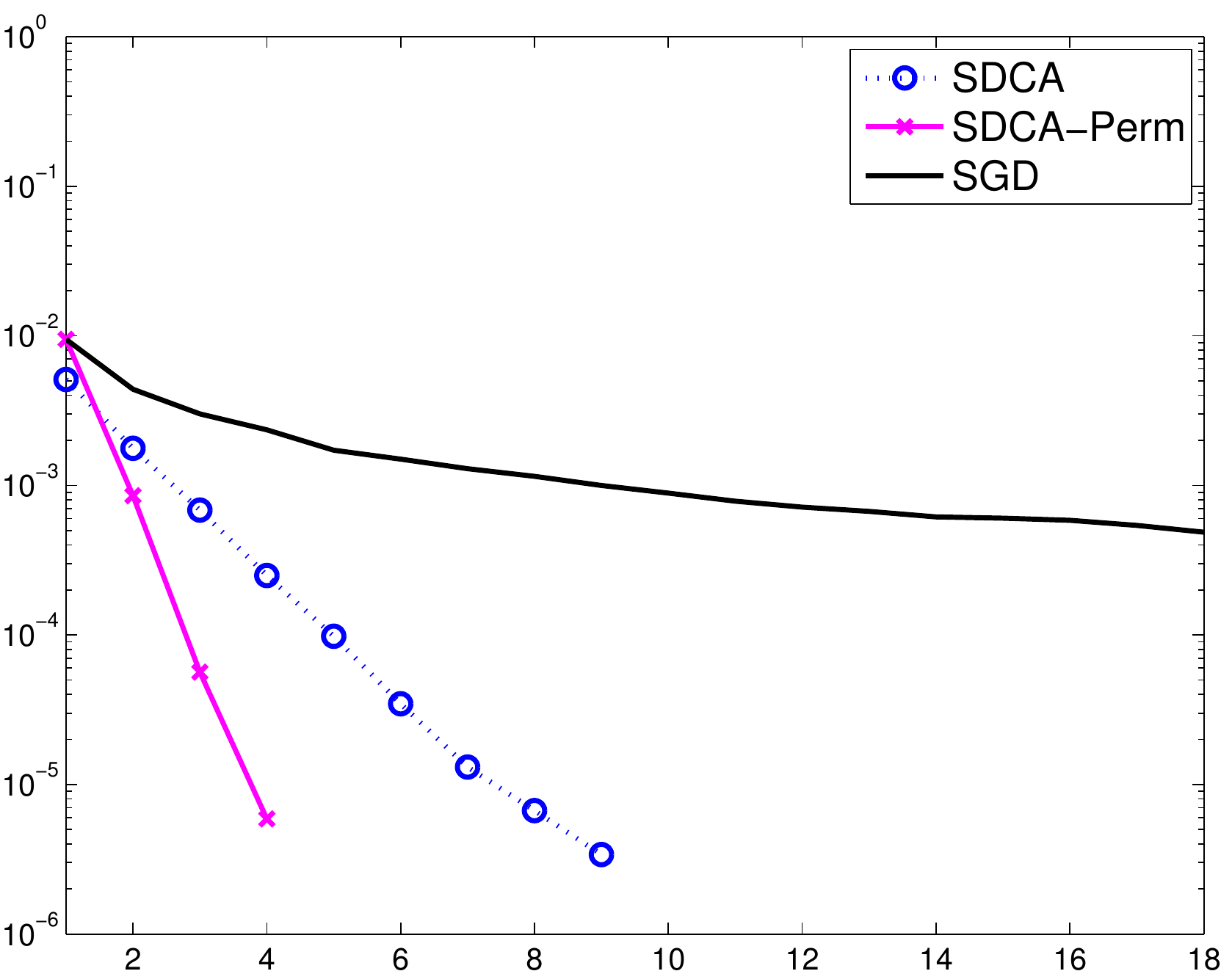} &
\includegraphics[width=0.31\textwidth]{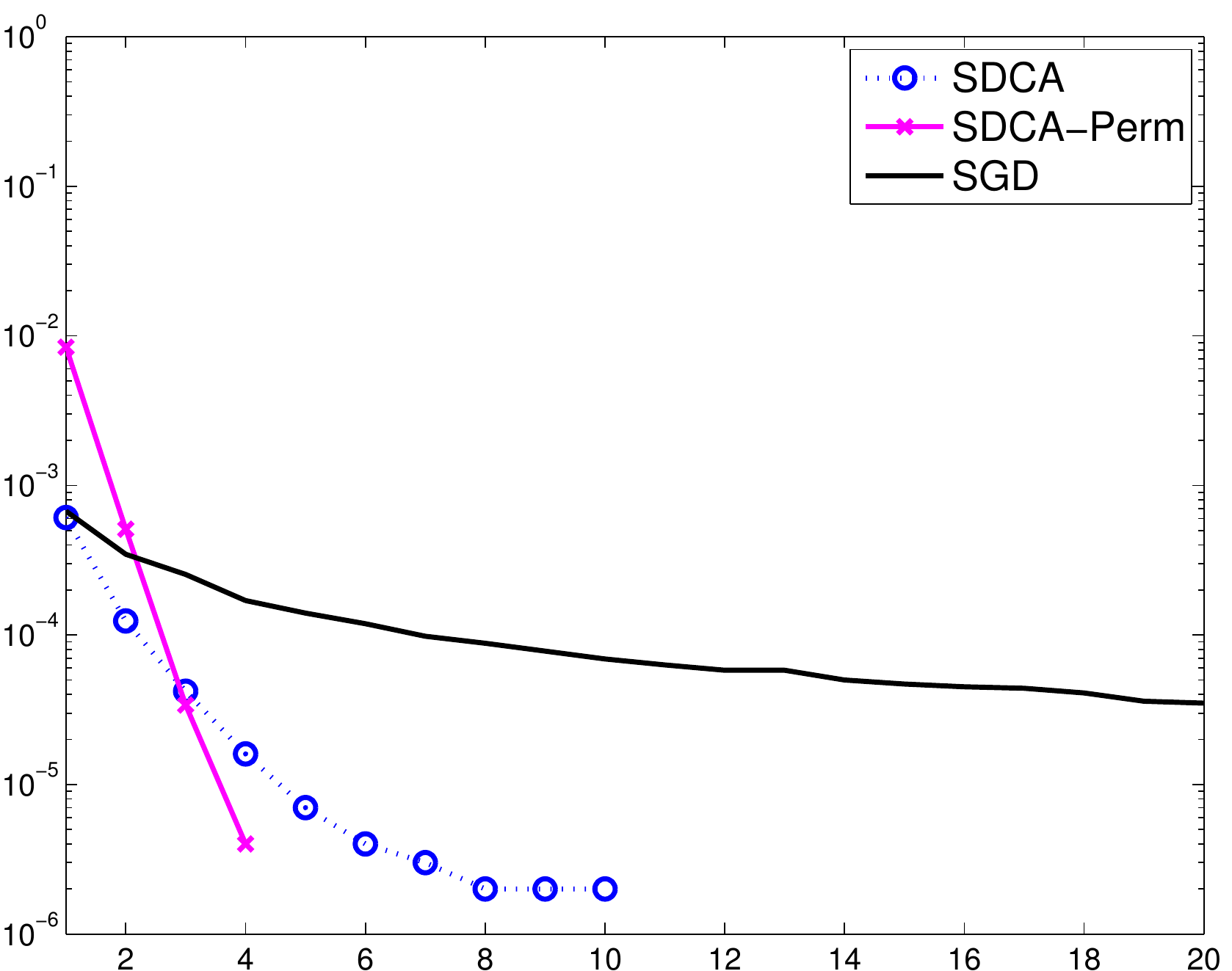} &
\includegraphics[width=0.31\textwidth]{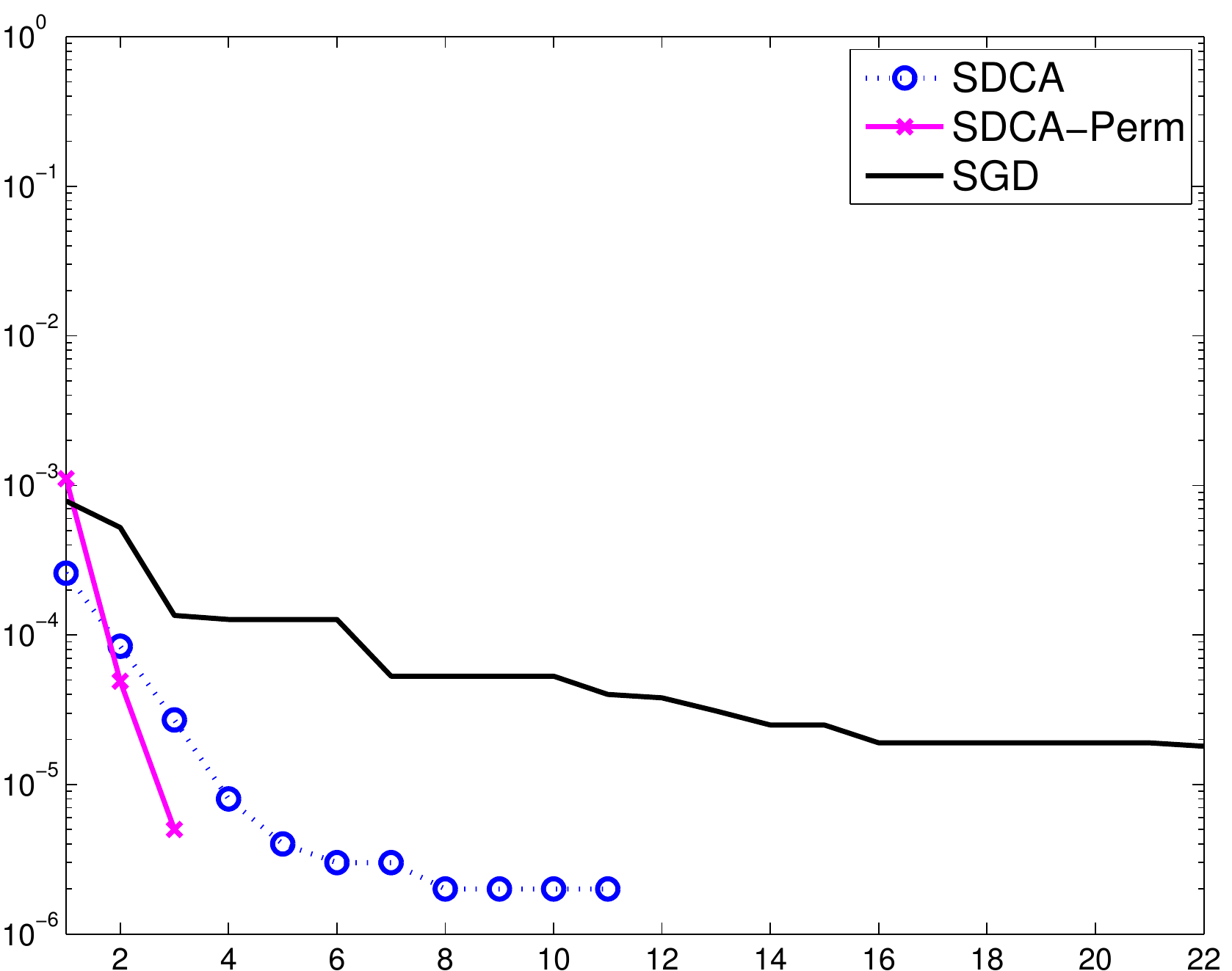}\\
$10^{-5}$ &
\includegraphics[width=0.31\textwidth]{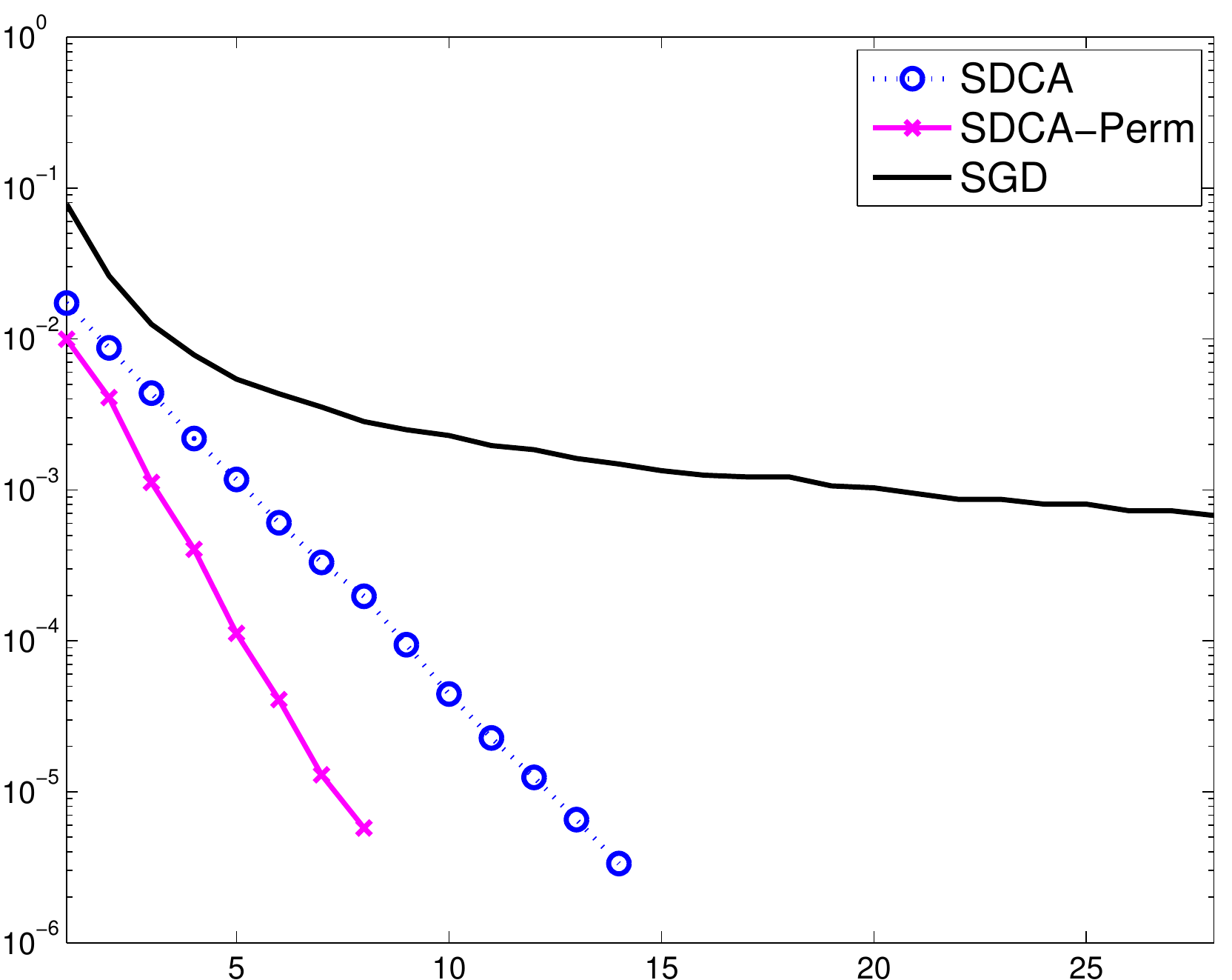} &
\includegraphics[width=0.31\textwidth]{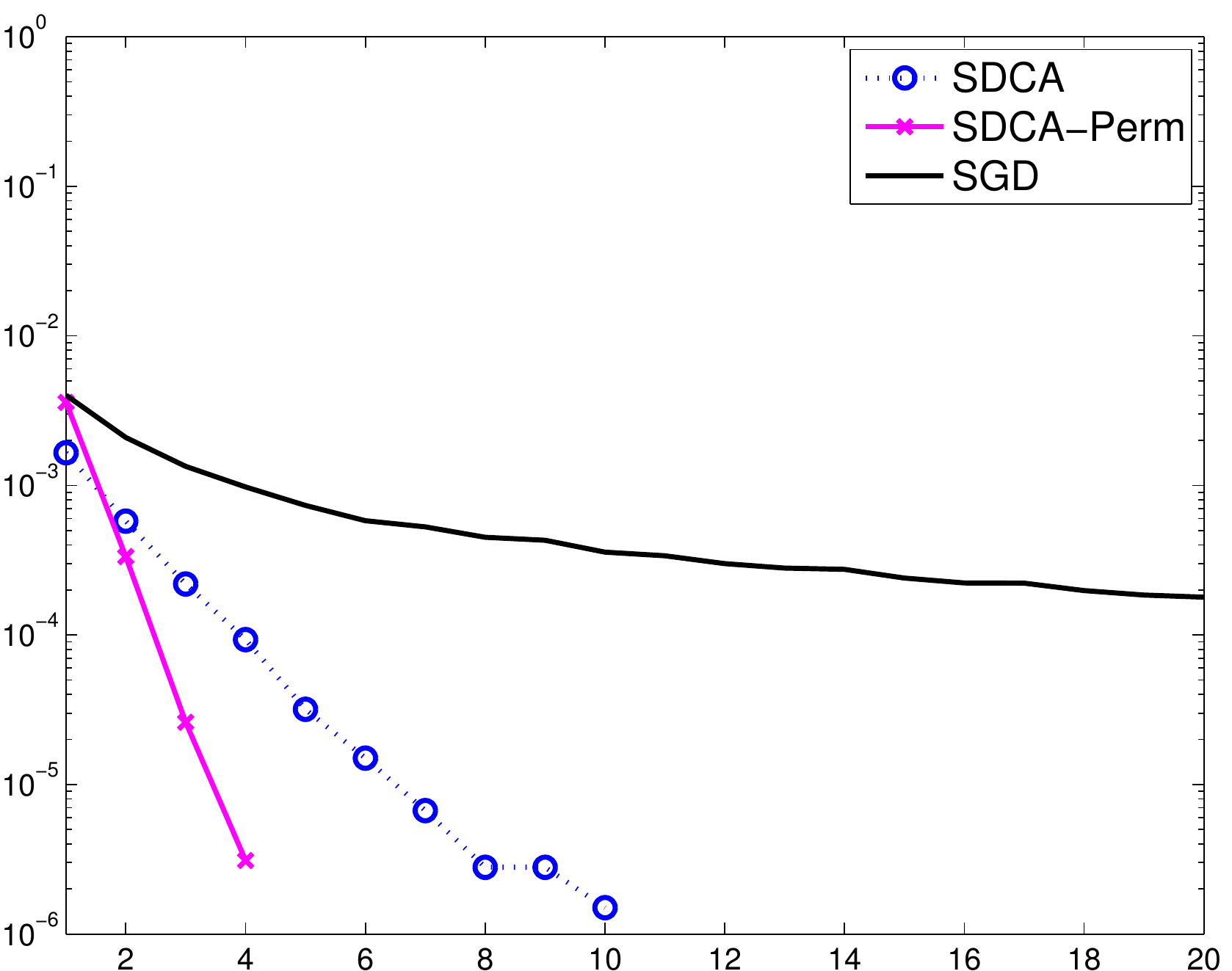} &
\includegraphics[width=0.31\textwidth]{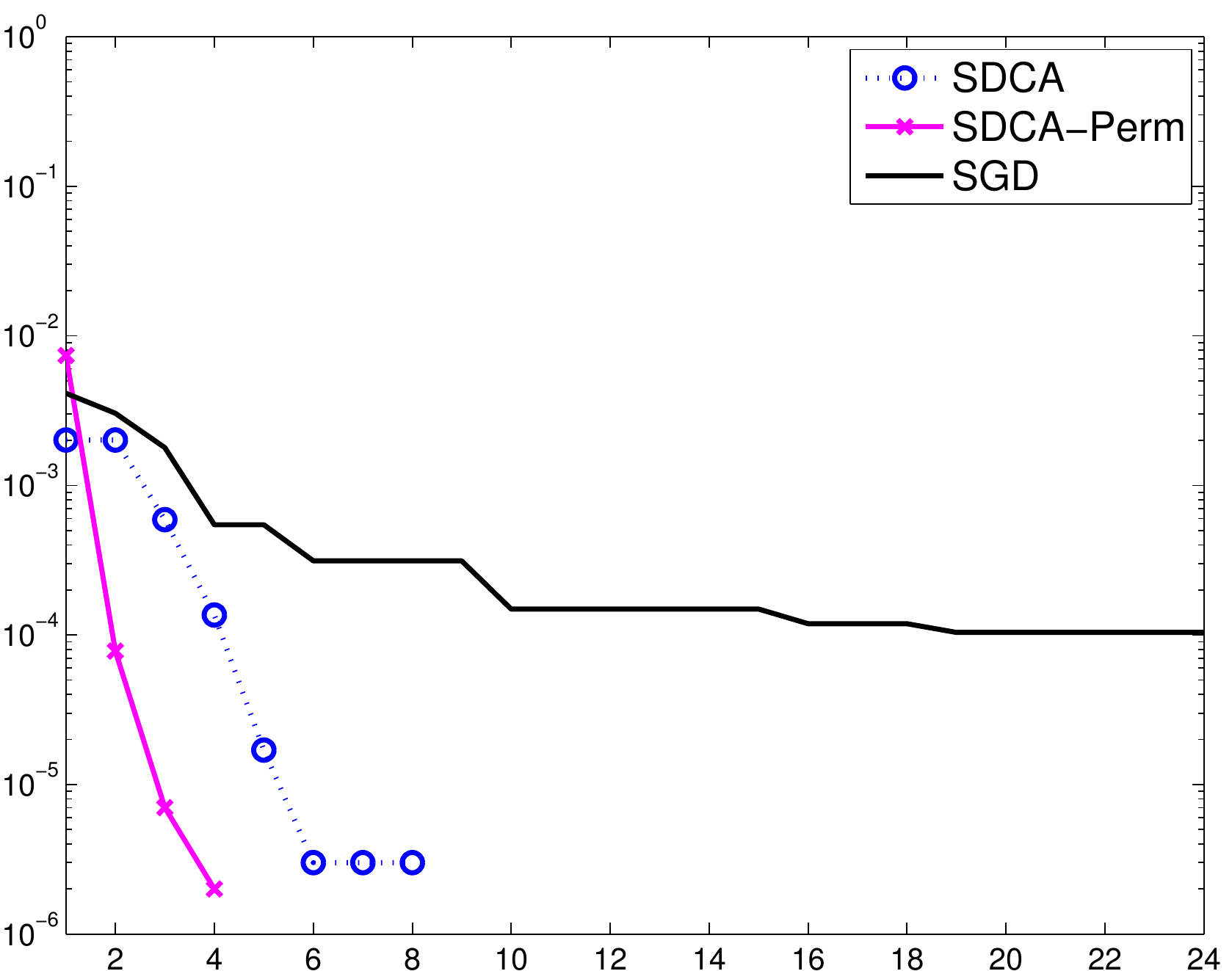}\\
$10^{-6}$ &
\includegraphics[width=0.31\textwidth]{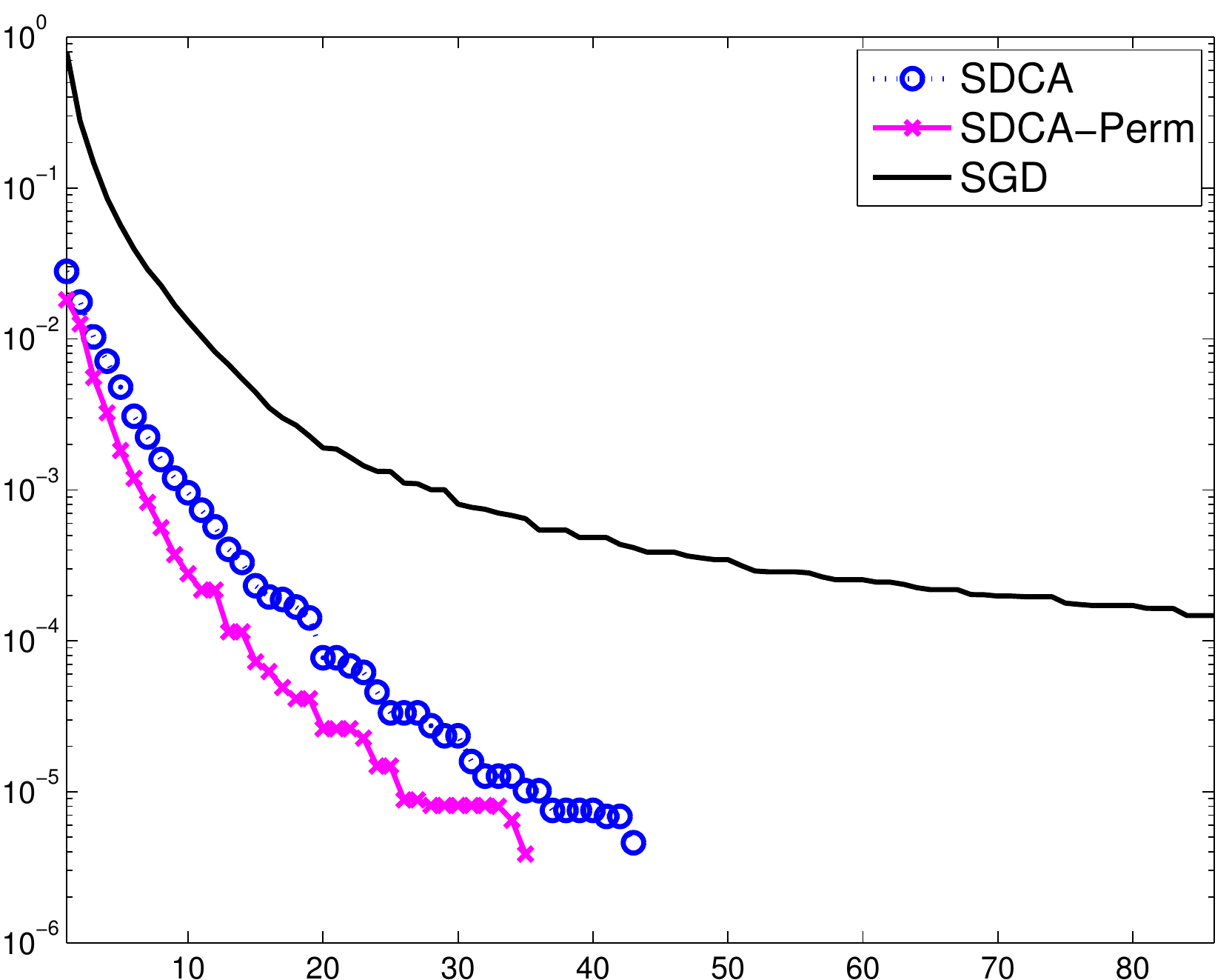} &
\includegraphics[width=0.31\textwidth]{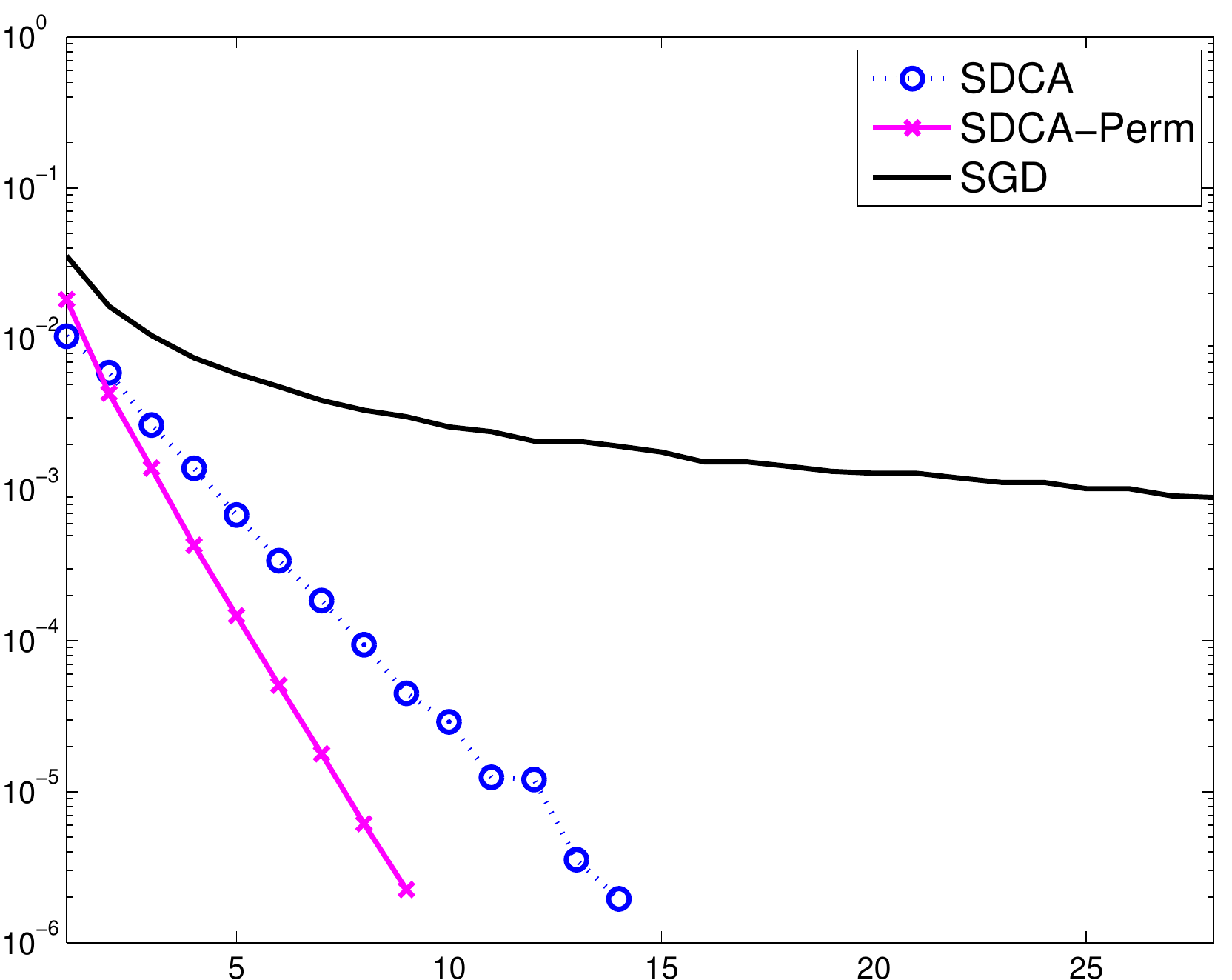} &
\includegraphics[width=0.31\textwidth]{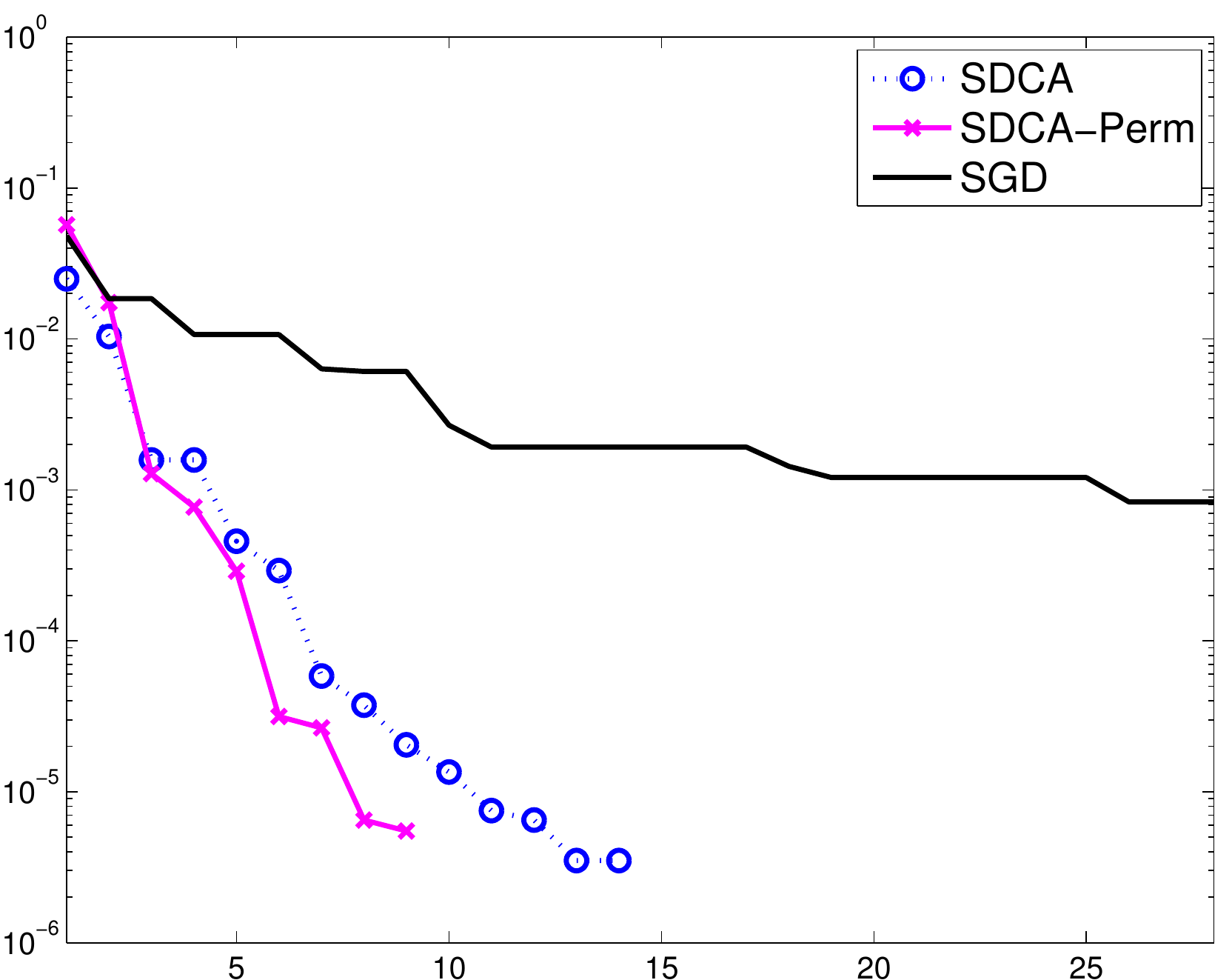}\\
\end{tabular}
\end{center}

\caption{\label{fig:SGDsmooth} Comparing the primal sub-optimality of
  SDCA and SGD for the smoothed hinge-loss ($\gamma=1$). In all plots the horizontal axis is the number of iterations divided by training set size (corresponding to the number of epochs through the data).}

\end{figure}

\begin{figure}

\begin{center}
\begin{tabular}{ @{} L | @{} S @{} S @{} S @{} }
$\lambda$ & \scriptsize{astro-ph} & \scriptsize{CCAT} & \scriptsize{cov1}\\ \hline
$10^{-3}$ & 
\includegraphics[width=0.31\textwidth]{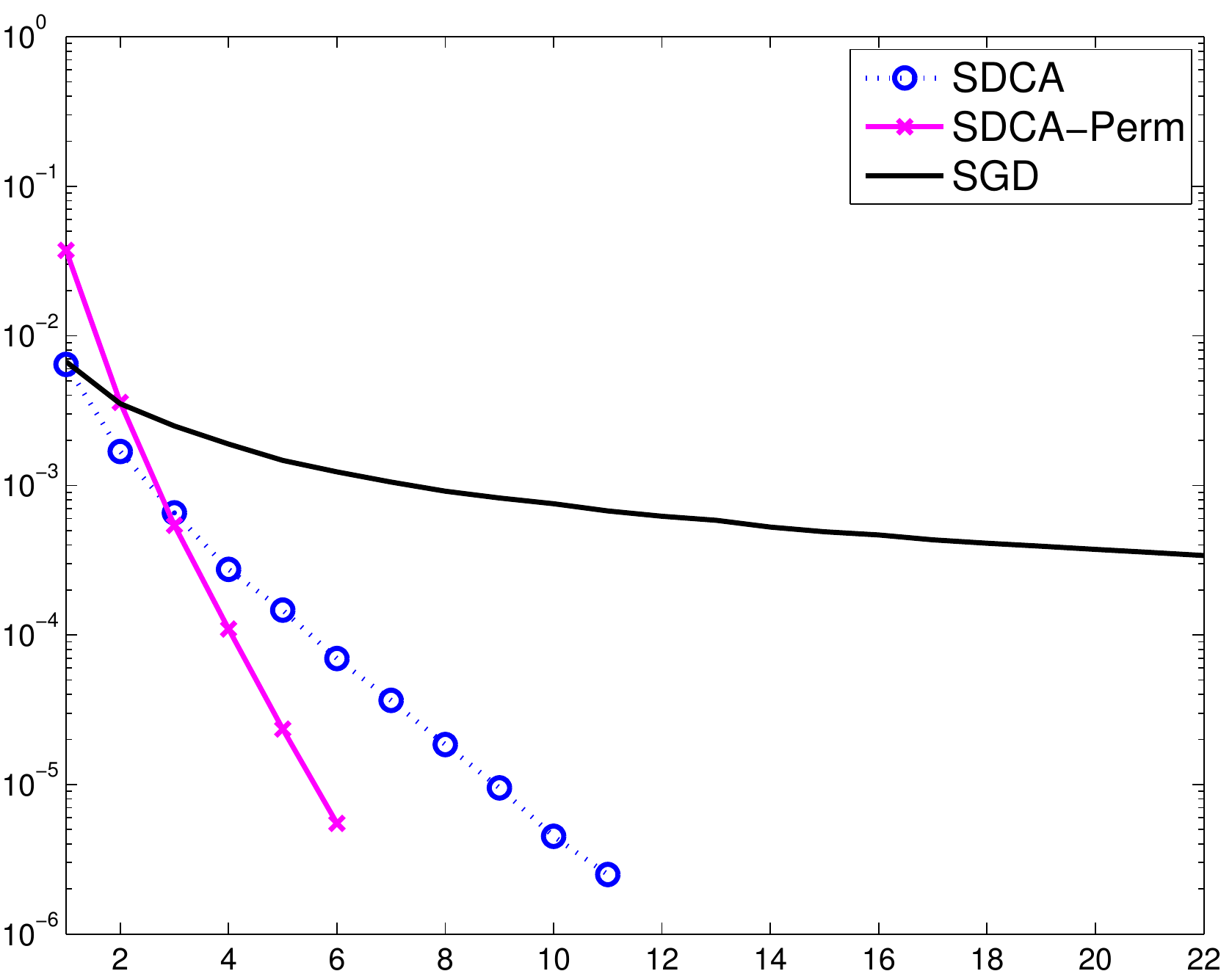} &
\includegraphics[width=0.31\textwidth]{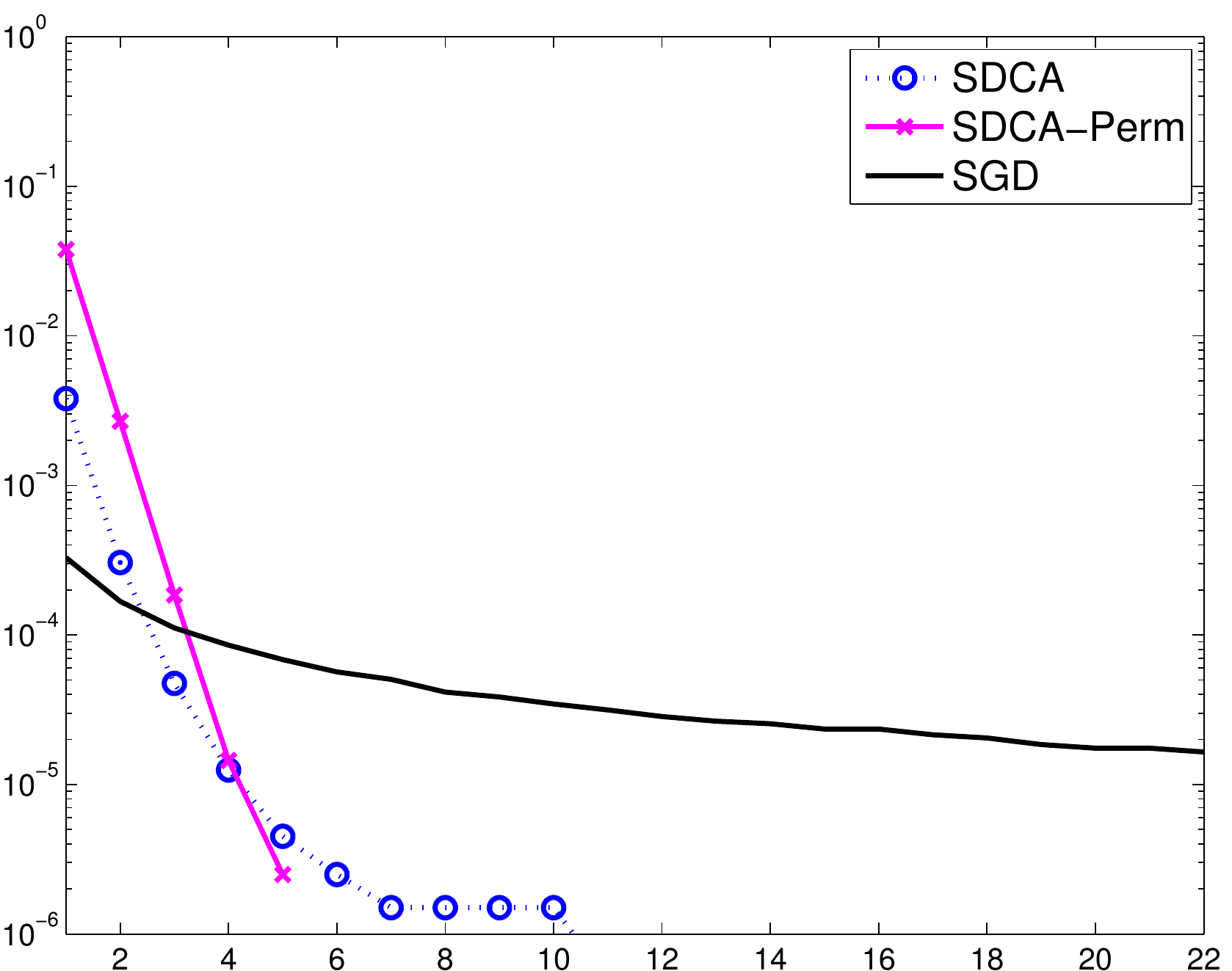} &
\includegraphics[width=0.31\textwidth]{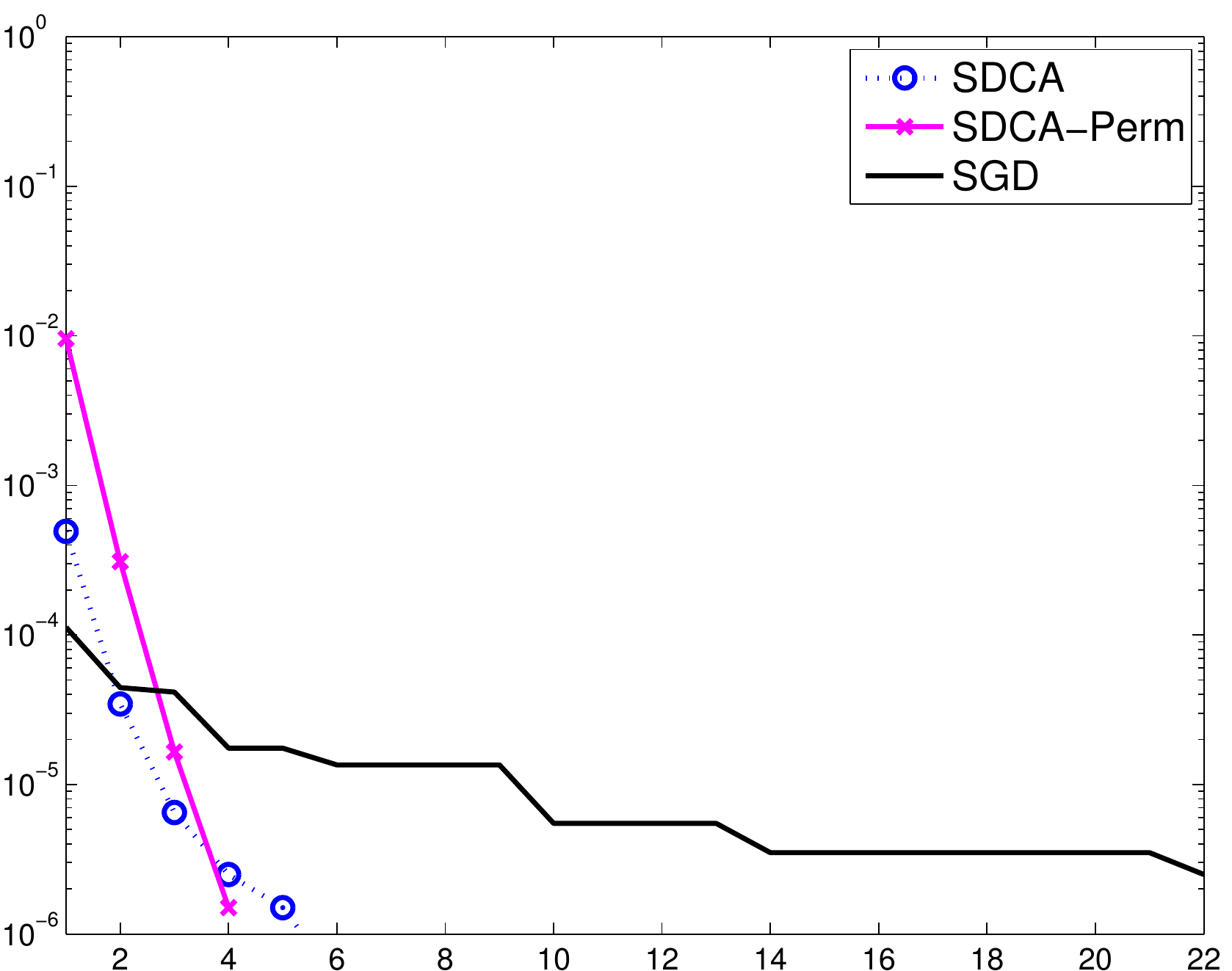}\\
$10^{-4}$ &
\includegraphics[width=0.31\textwidth]{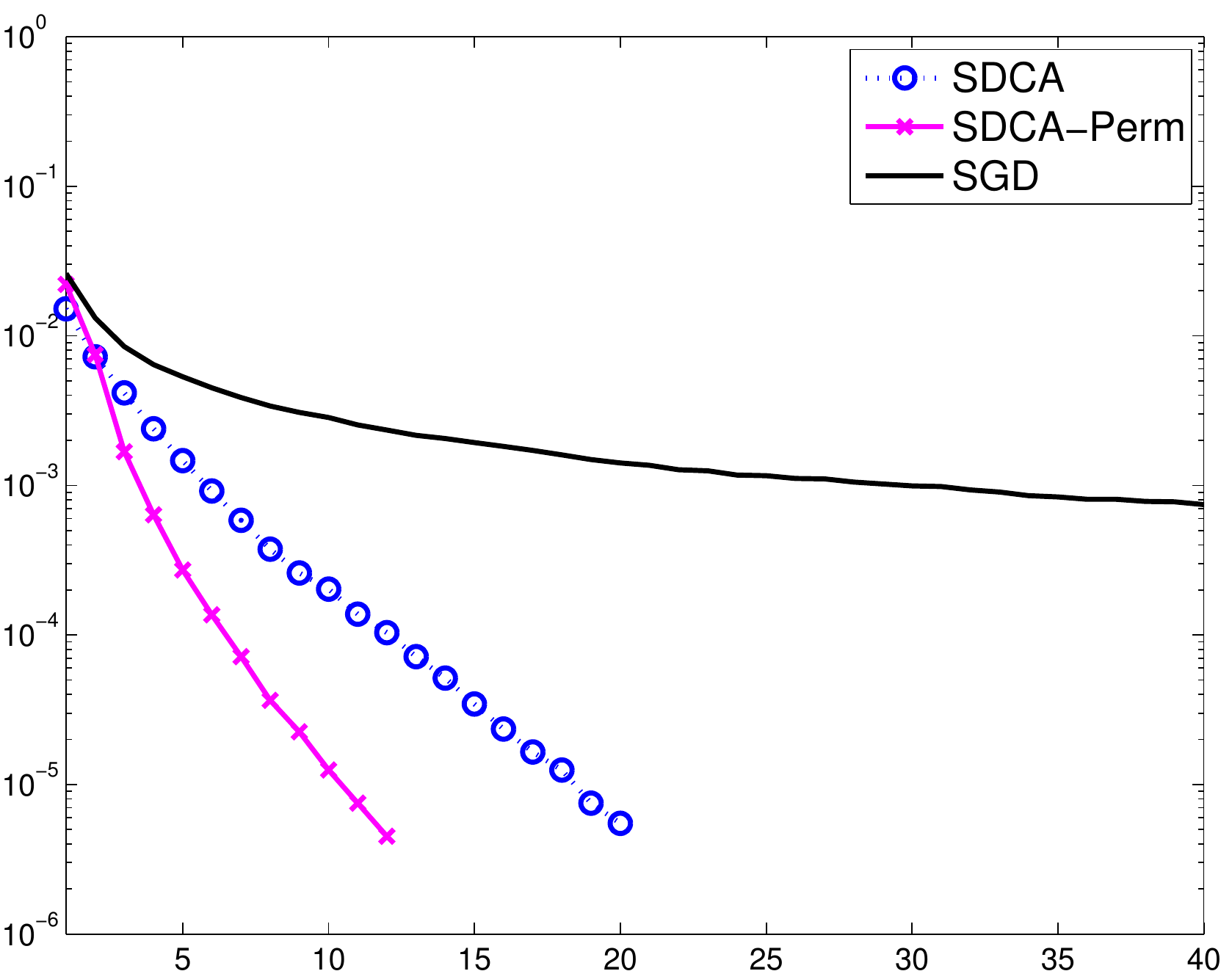} &
\includegraphics[width=0.31\textwidth]{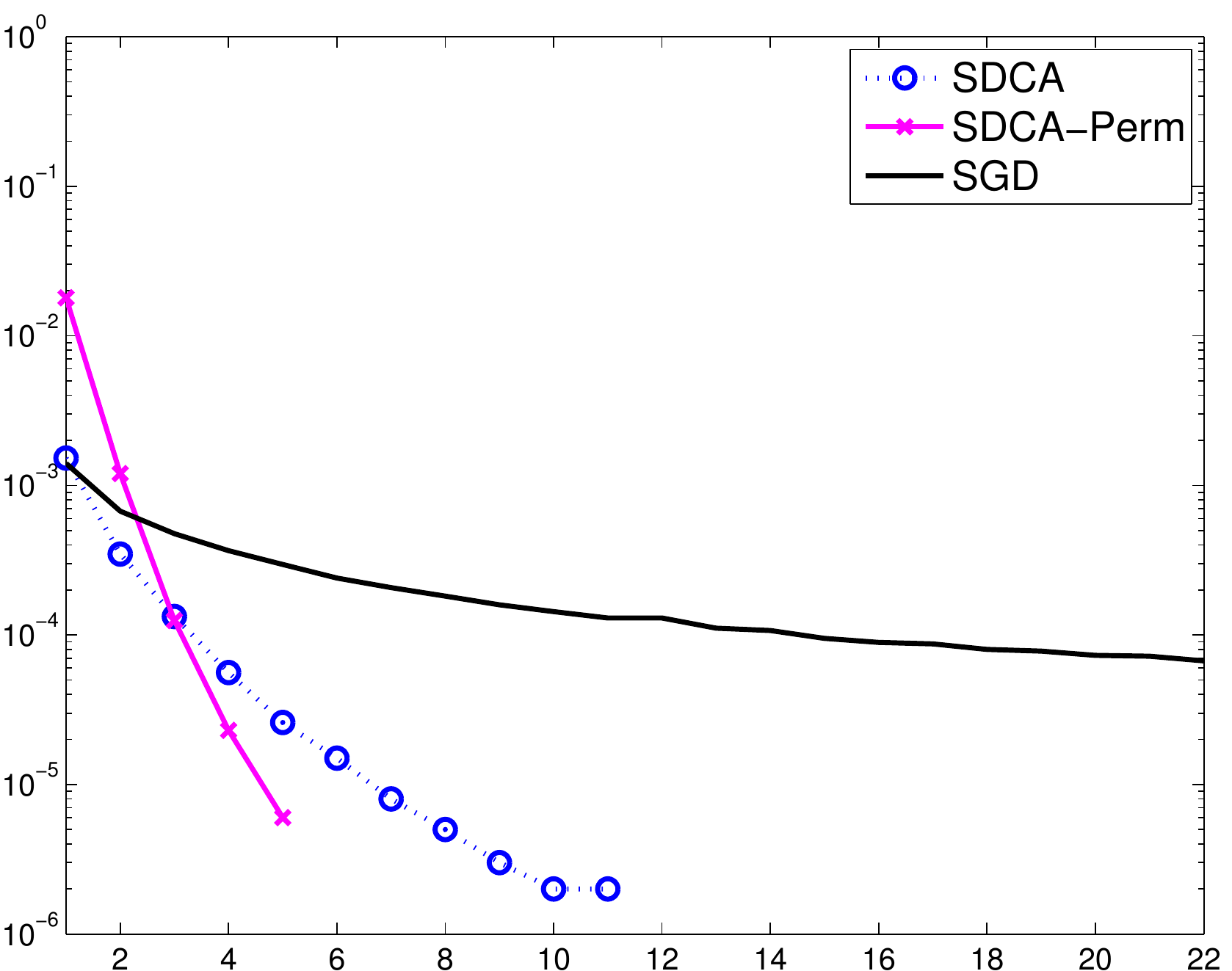} &
\includegraphics[width=0.31\textwidth]{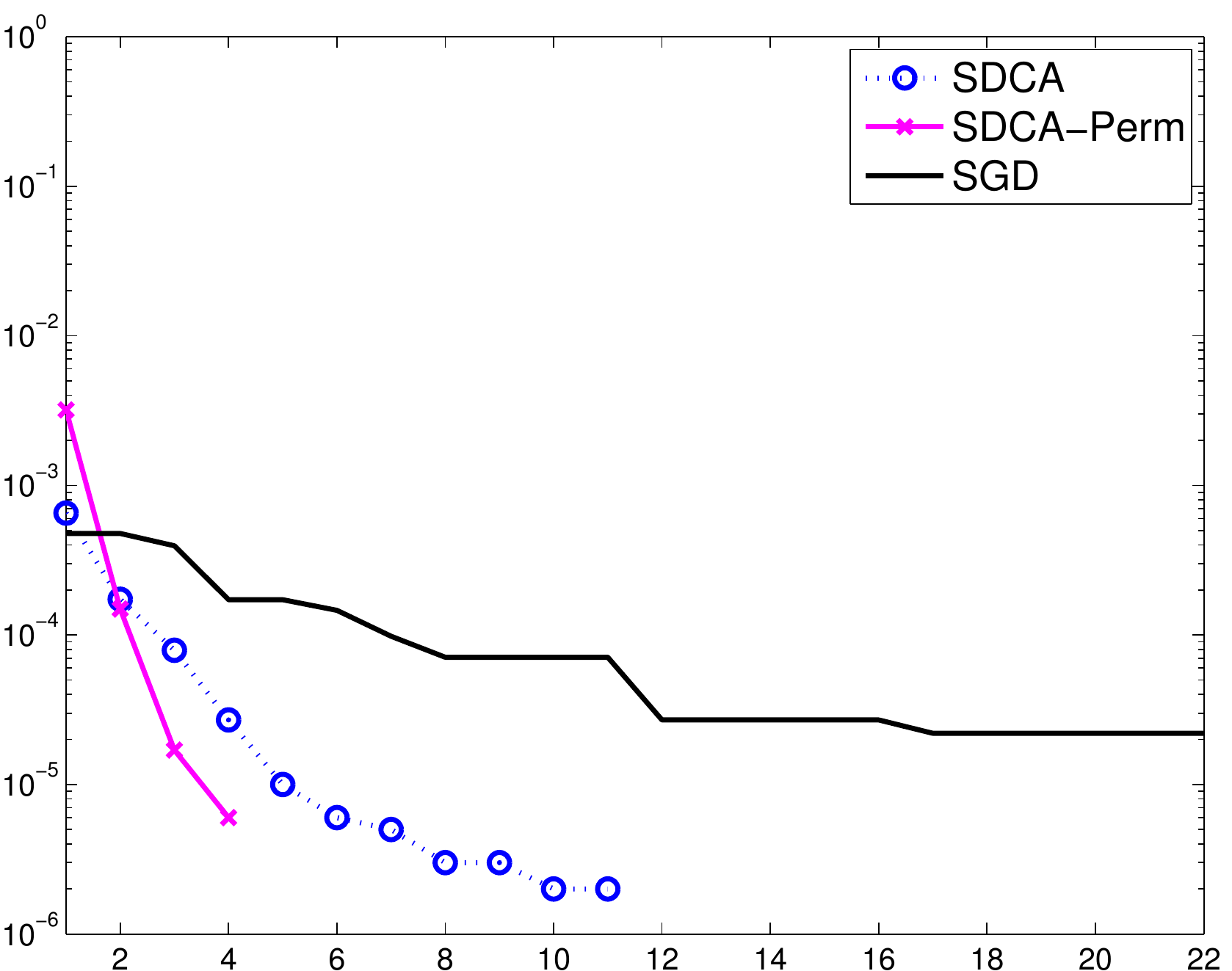}\\
$10^{-5}$ &
\includegraphics[width=0.31\textwidth]{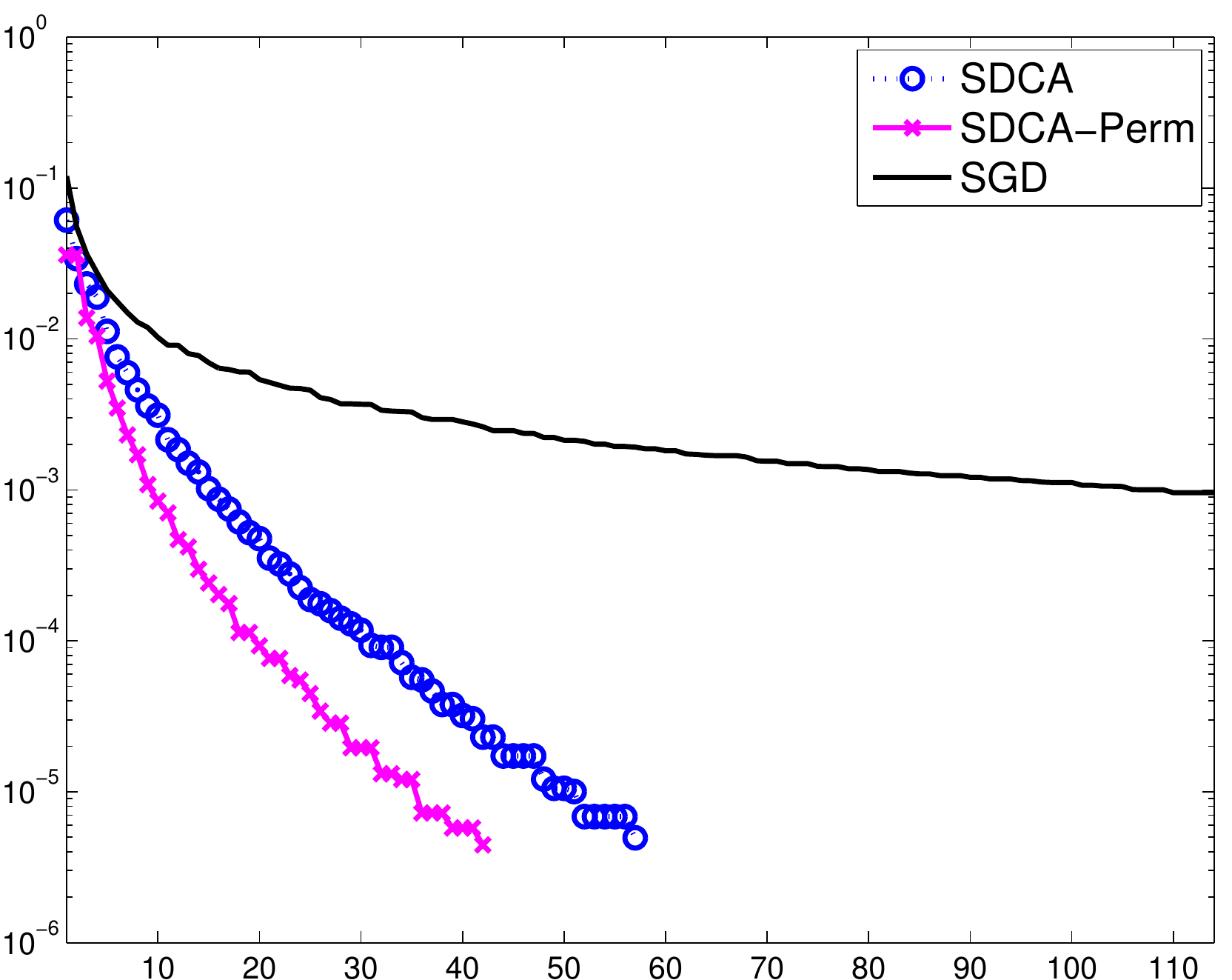} &
\includegraphics[width=0.31\textwidth]{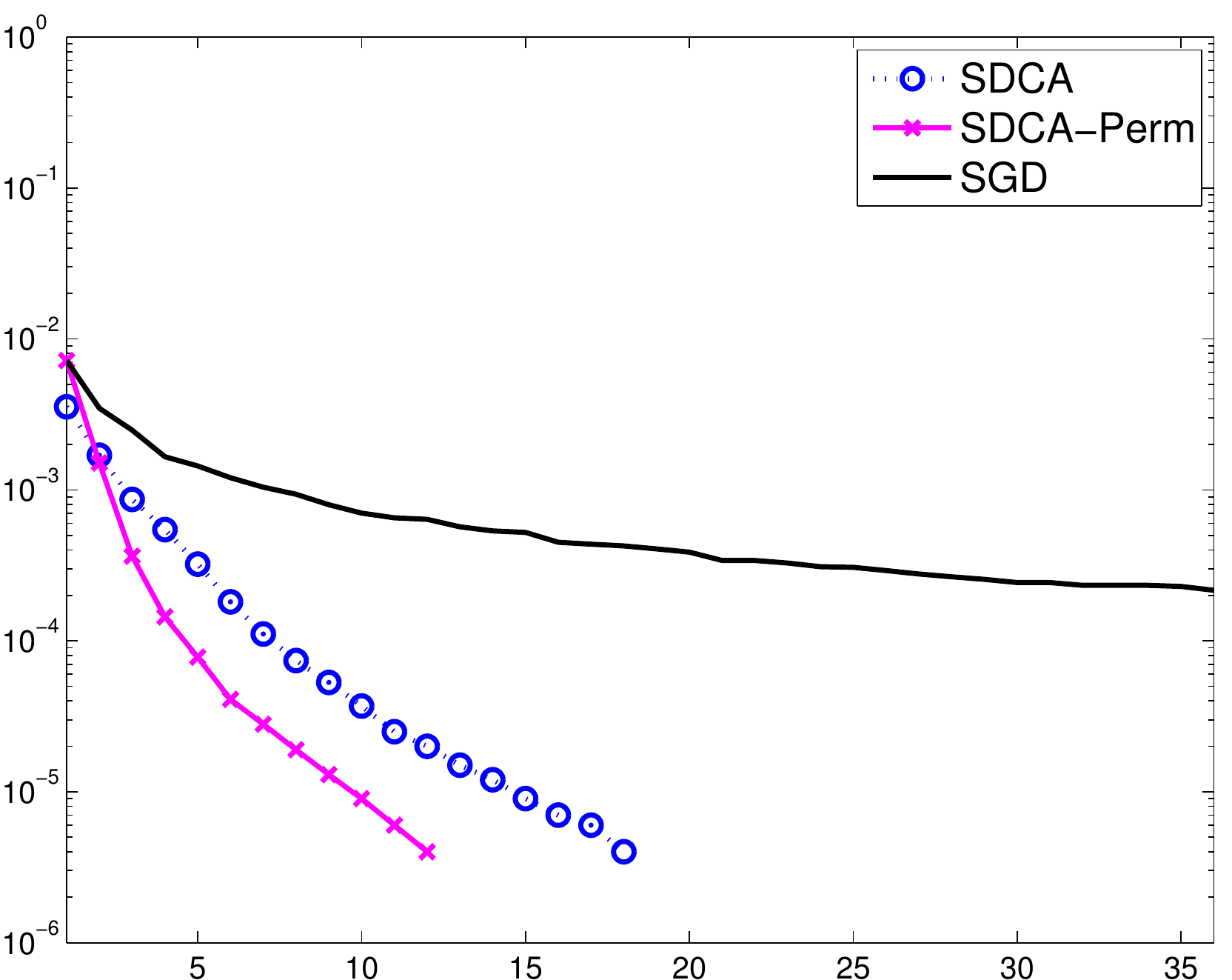} &
\includegraphics[width=0.31\textwidth]{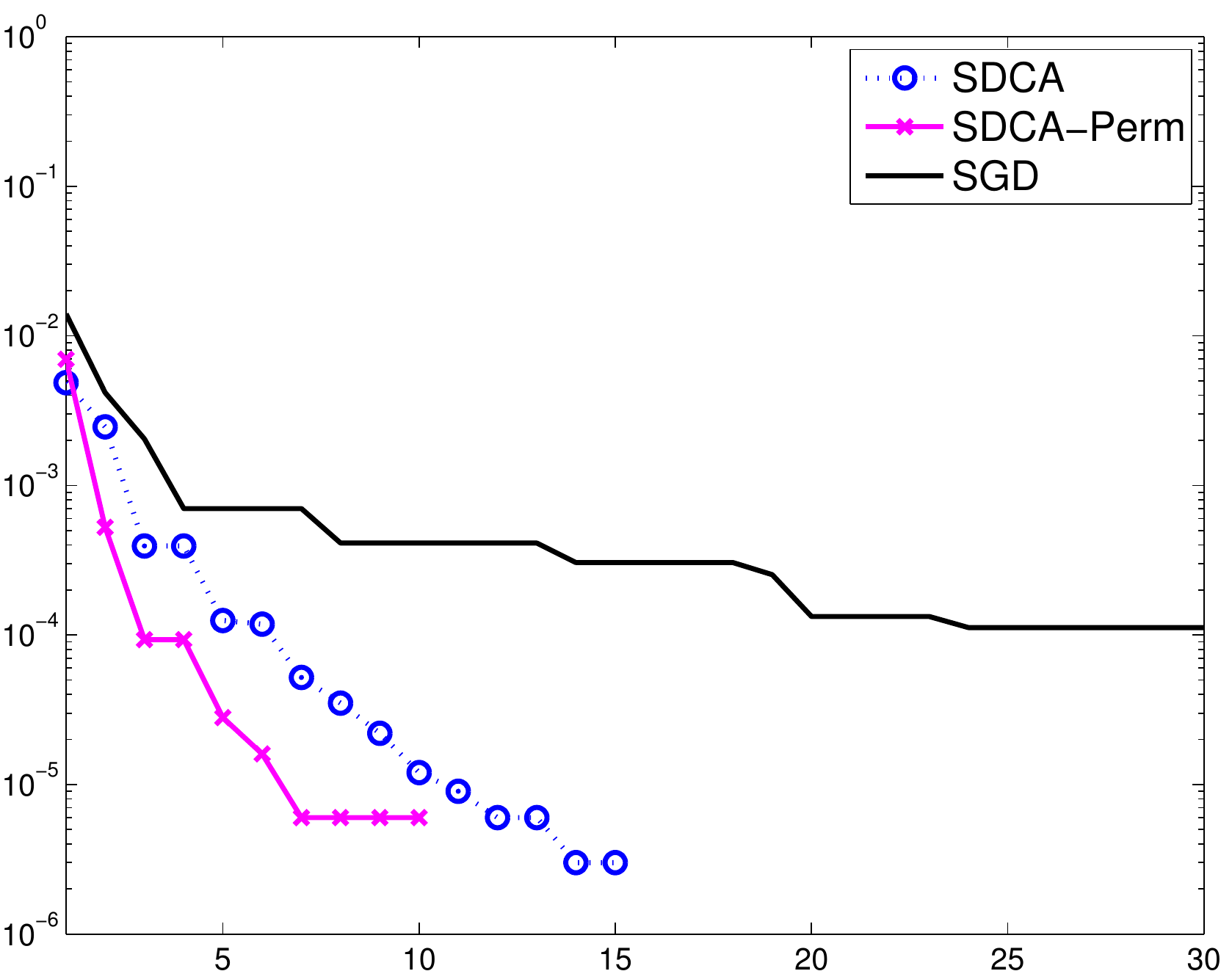}\\
$10^{-6}$ &
\includegraphics[width=0.31\textwidth]{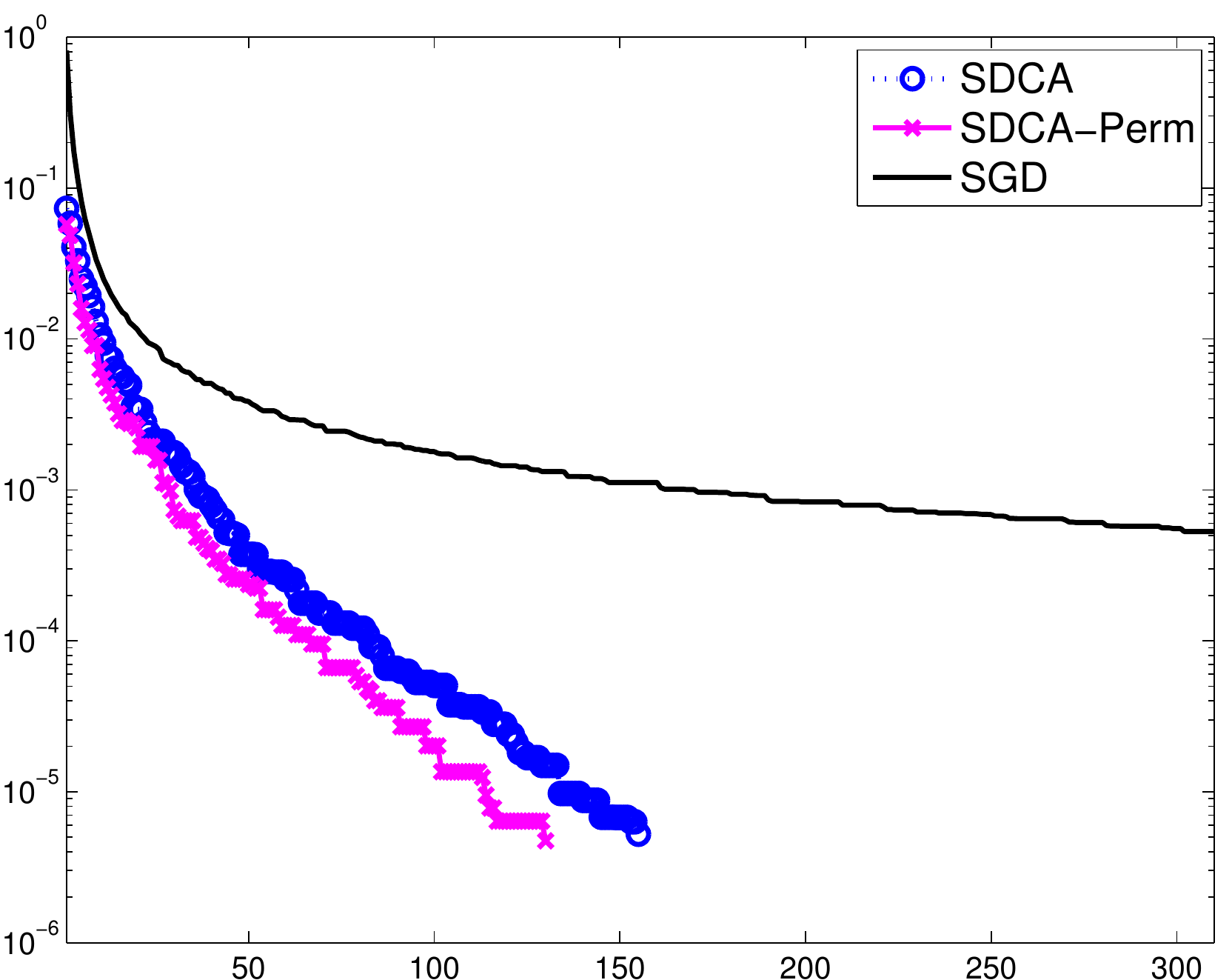} &
\includegraphics[width=0.31\textwidth]{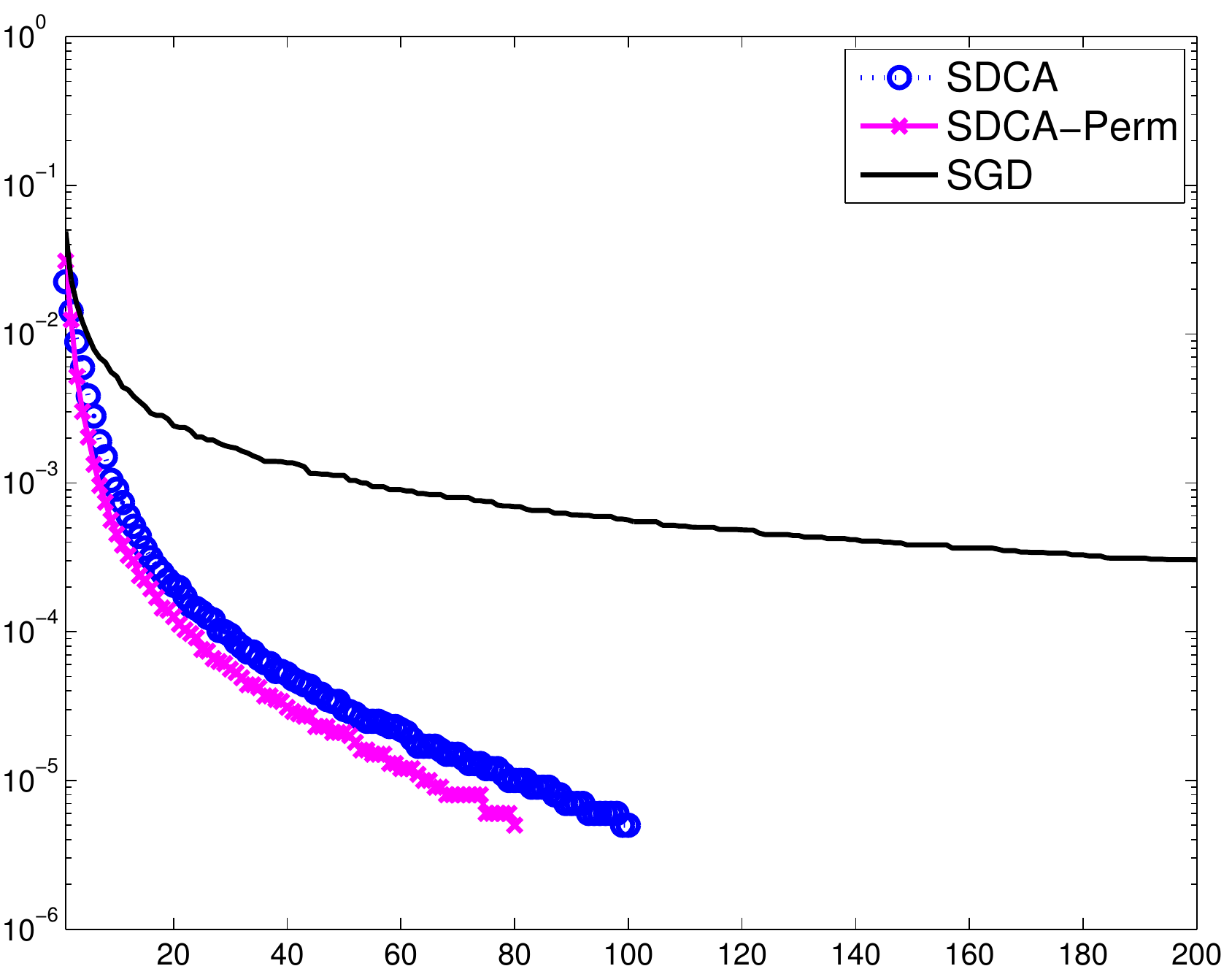} &
\includegraphics[width=0.31\textwidth]{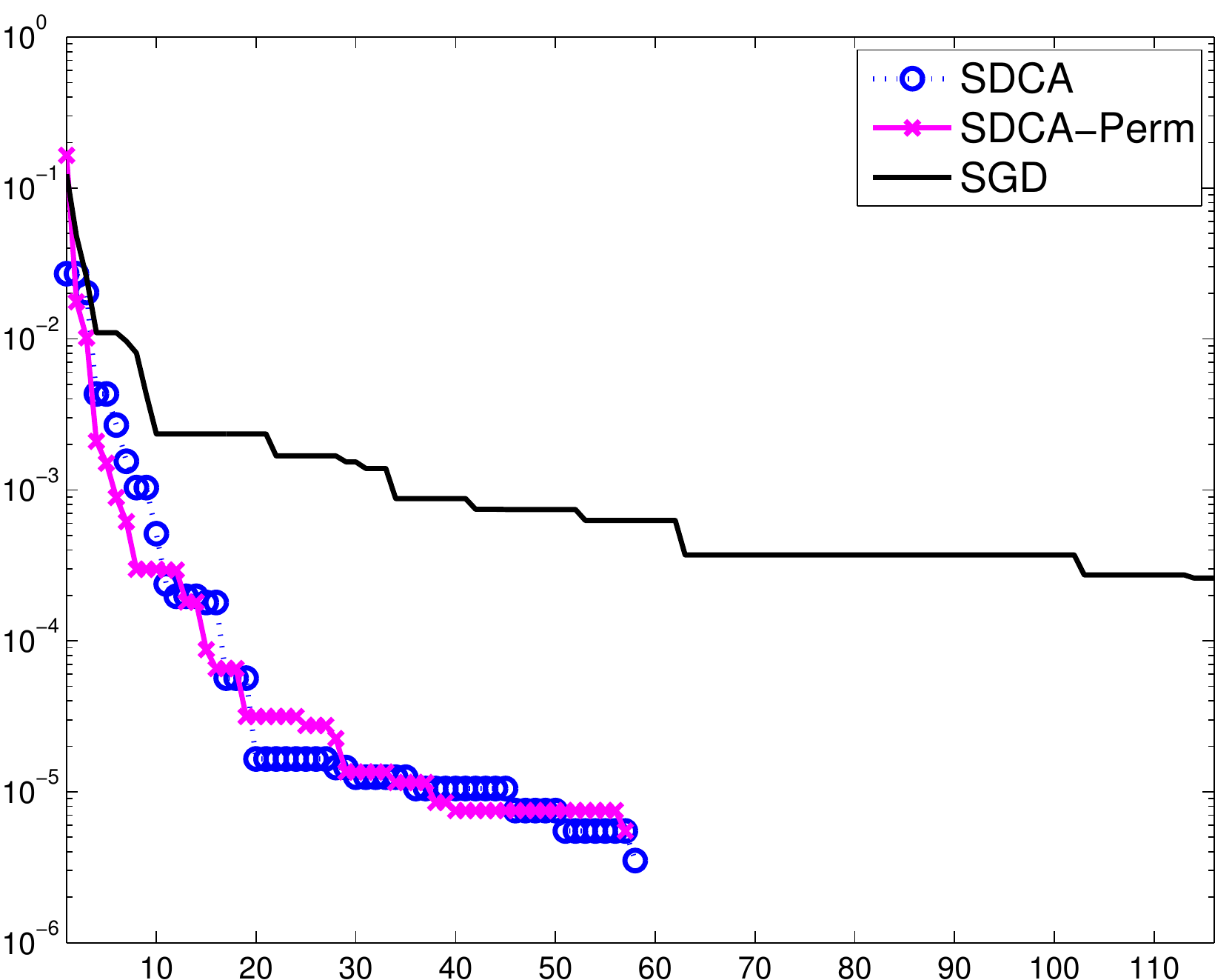}\\
\end{tabular}
\end{center}

\caption{\label{fig:SGD} Comparing the primal sub-optimality of
  SDCA and SGD for the non-smooth hinge-loss ($\gamma=0$). In all plots the horizontal axis is the number of iterations divided by training set size (corresponding to the number of epochs through the data).}

\end{figure}

In \figref{fig:SGdhingevsSDCAsmooth} we compare the zero-one test
error of SDCA, when working with the smooth hinge-loss ($\gamma=1$) to
the zero-one test error of SGD, when working with the non-smooth
hinge-loss. As can be seen, SDCA with the smooth hinge-loss achieves
the smallest zero-one test error faster than SGD. 

\begin{figure}

\begin{center}
\begin{tabular}{ @{} L | @{} S @{} S @{} S @{} }
$\lambda$ & \scriptsize{astro-ph} & \scriptsize{CCAT} & \scriptsize{cov1}\\ \hline
$10^{-3}$ & 
\includegraphics[width=0.31\textwidth]{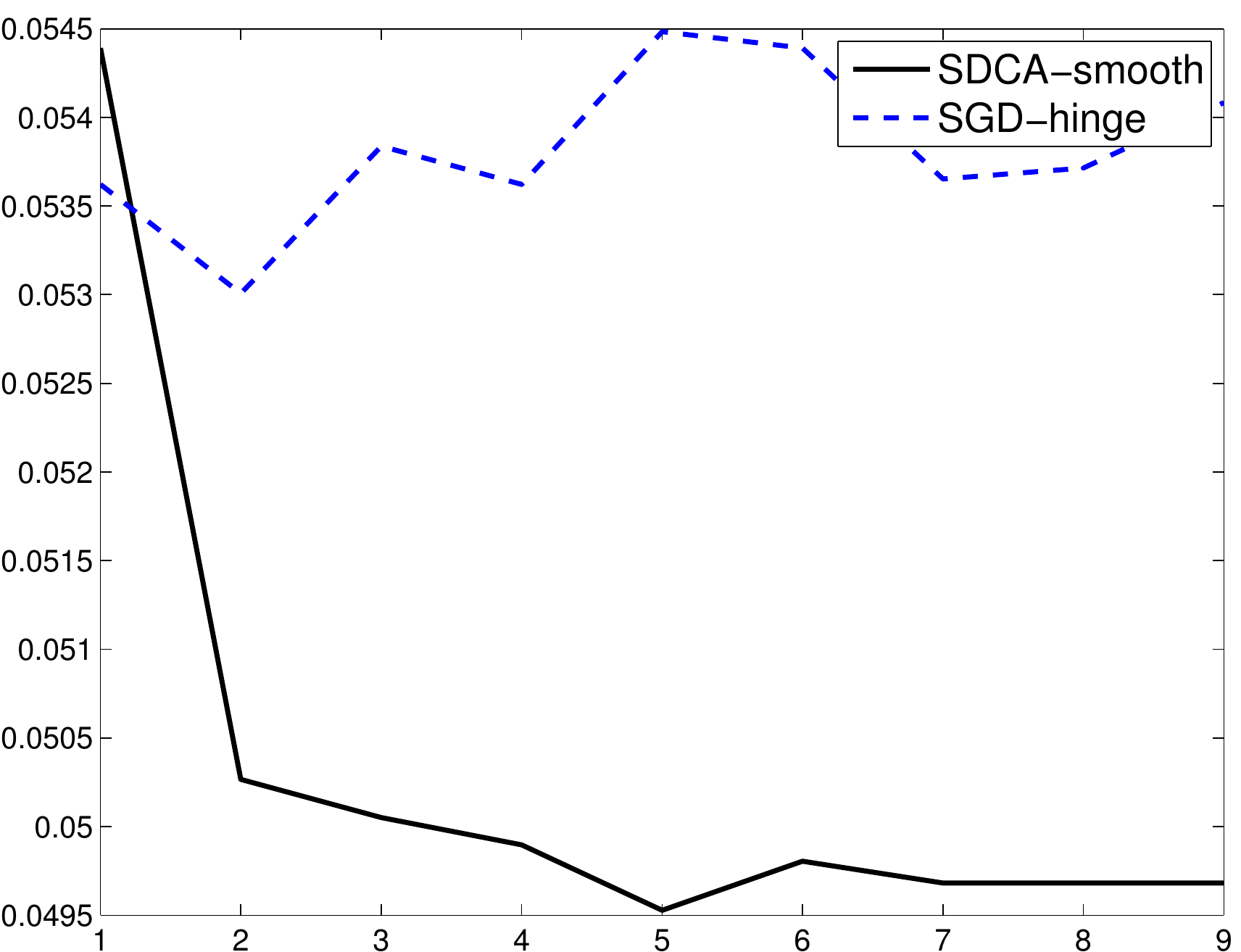} &
\includegraphics[width=0.31\textwidth]{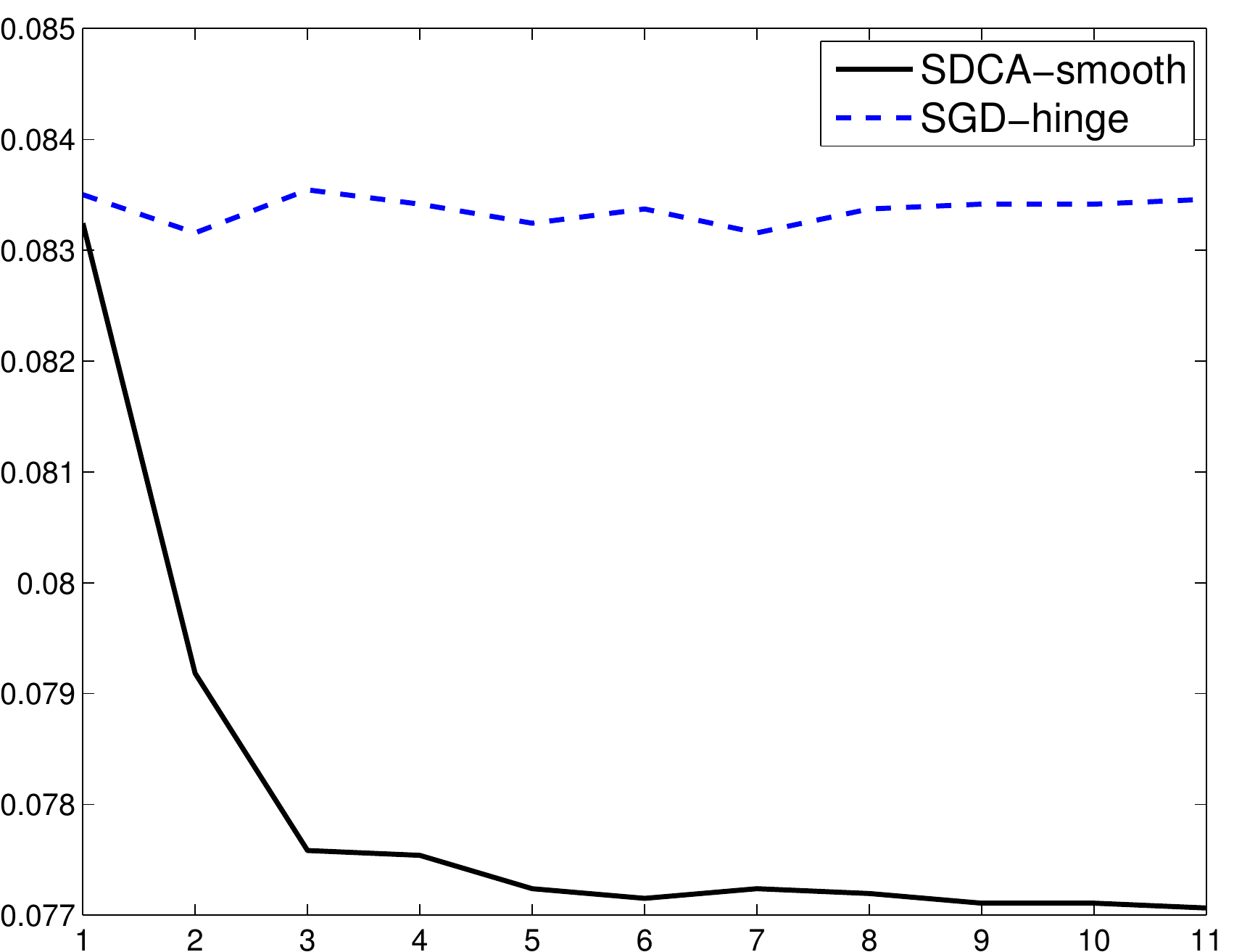} &
\includegraphics[width=0.31\textwidth]{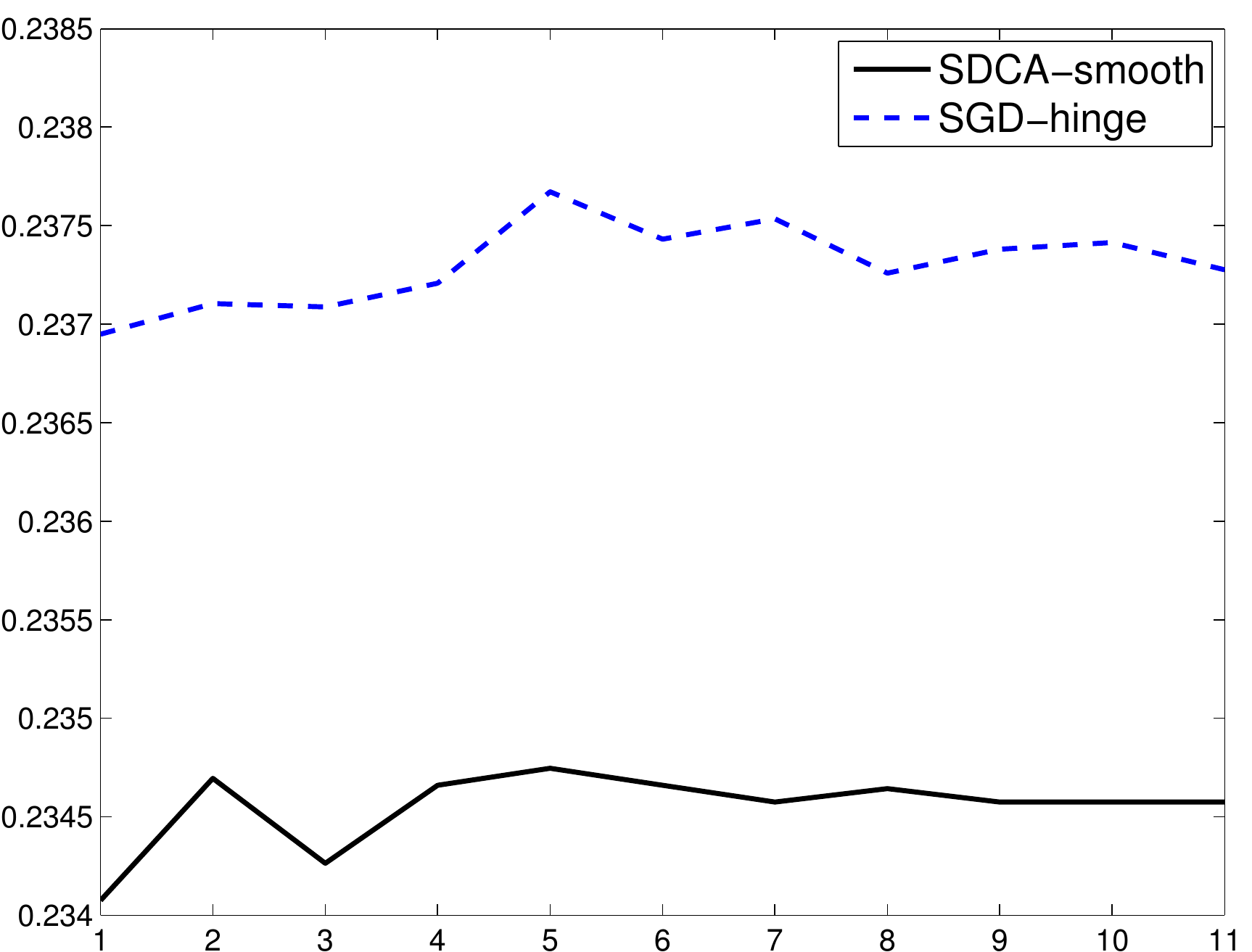}\\
$10^{-4}$ &
\includegraphics[width=0.31\textwidth]{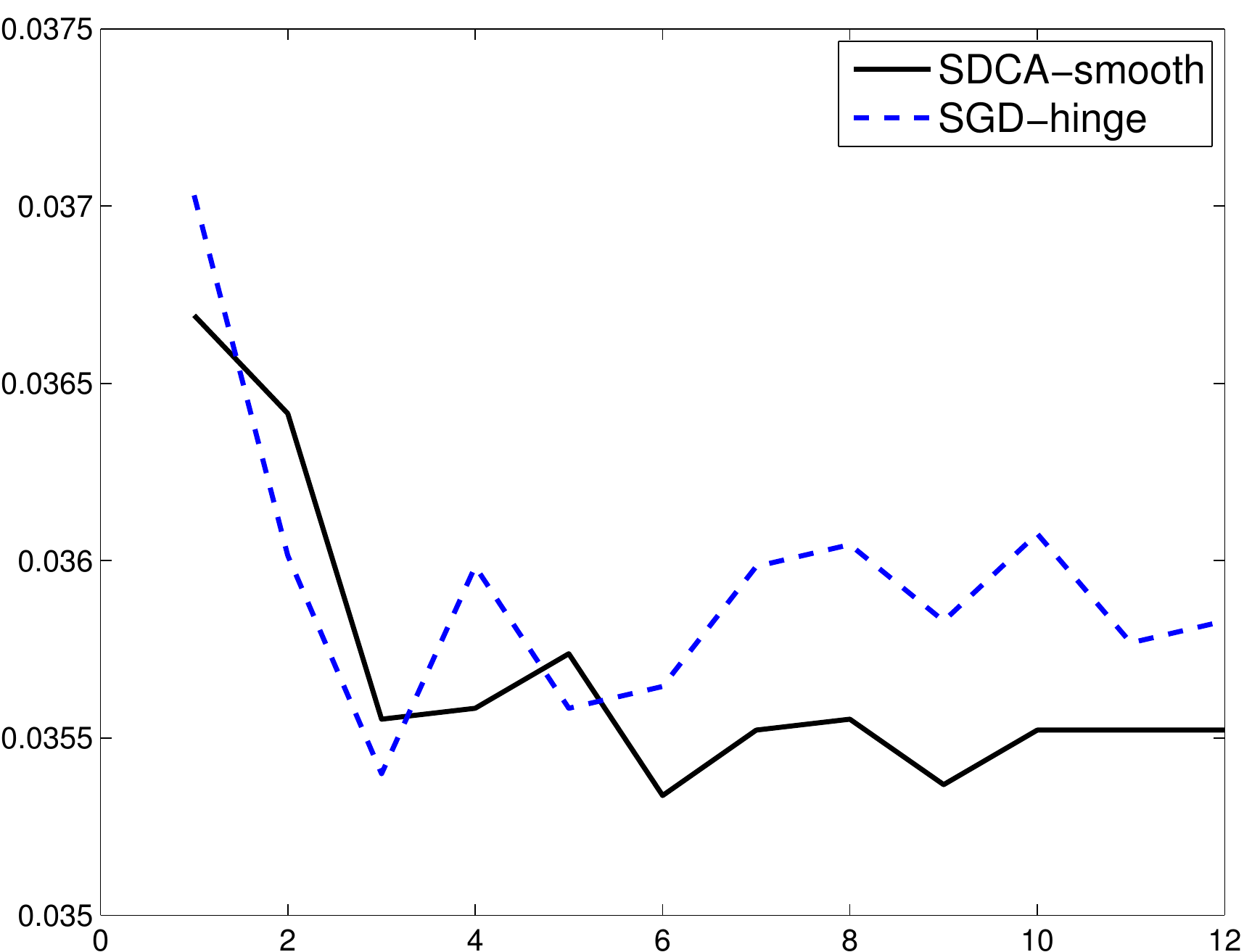} &
\includegraphics[width=0.31\textwidth]{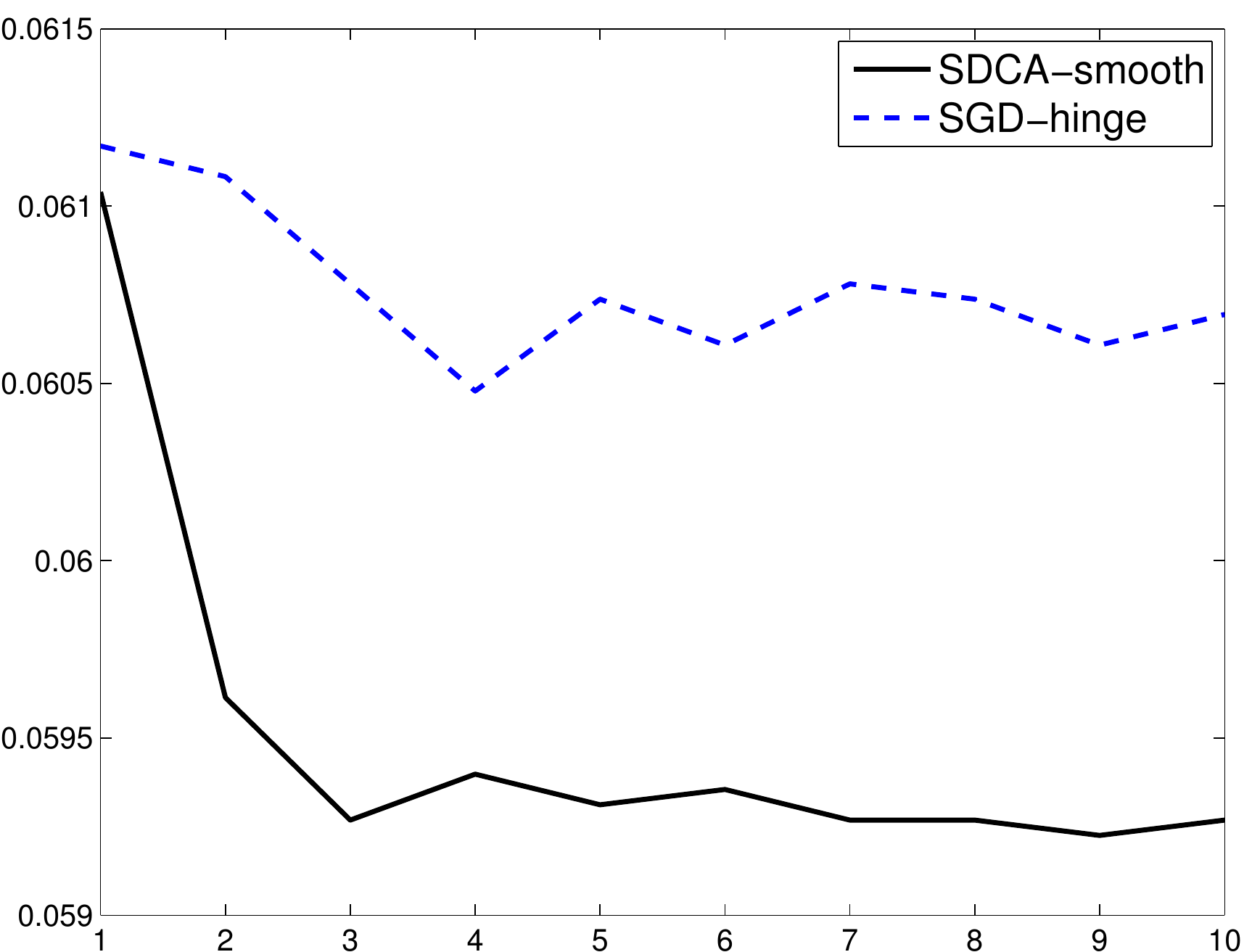} &
\includegraphics[width=0.31\textwidth]{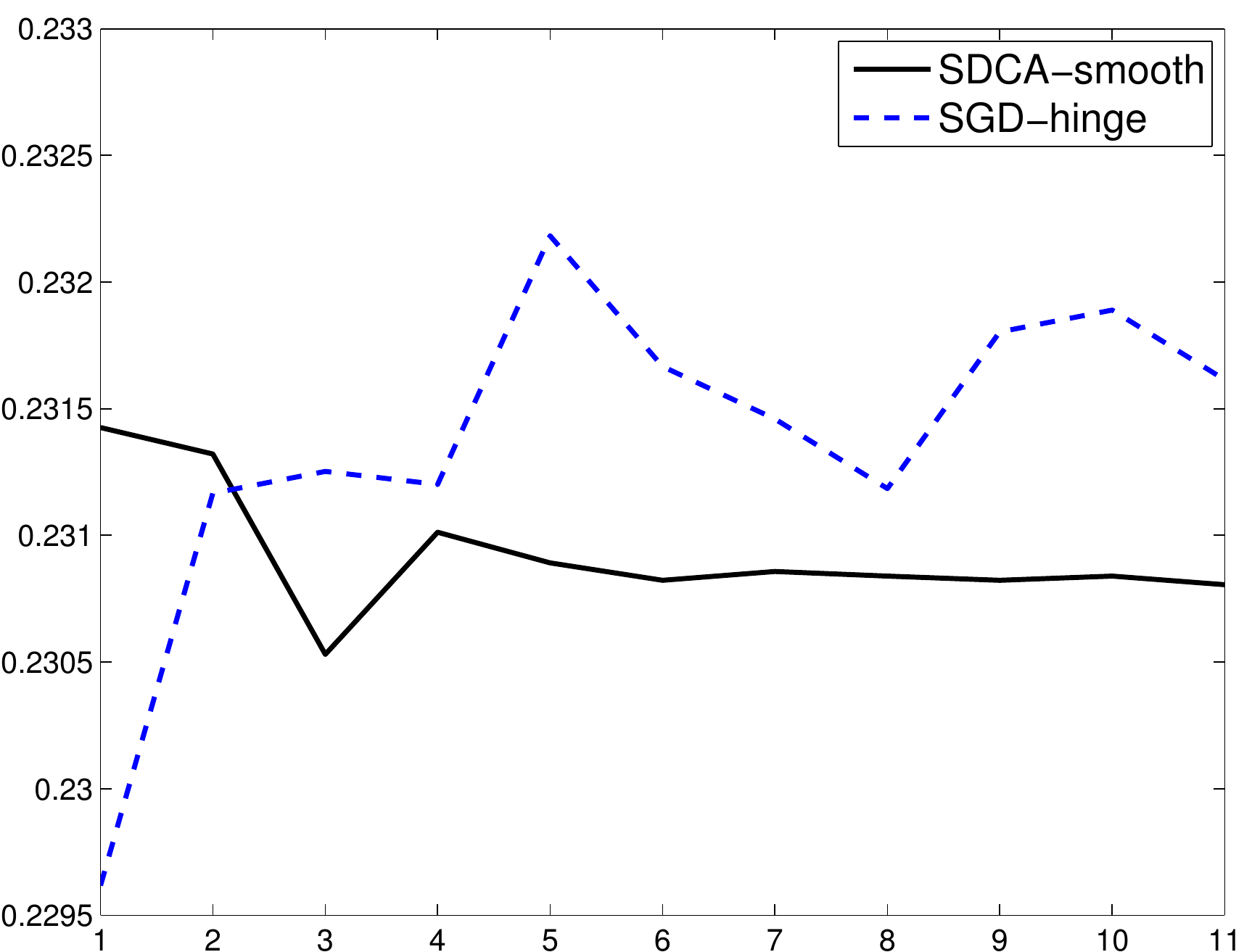}\\
$10^{-5}$ &
\includegraphics[width=0.31\textwidth]{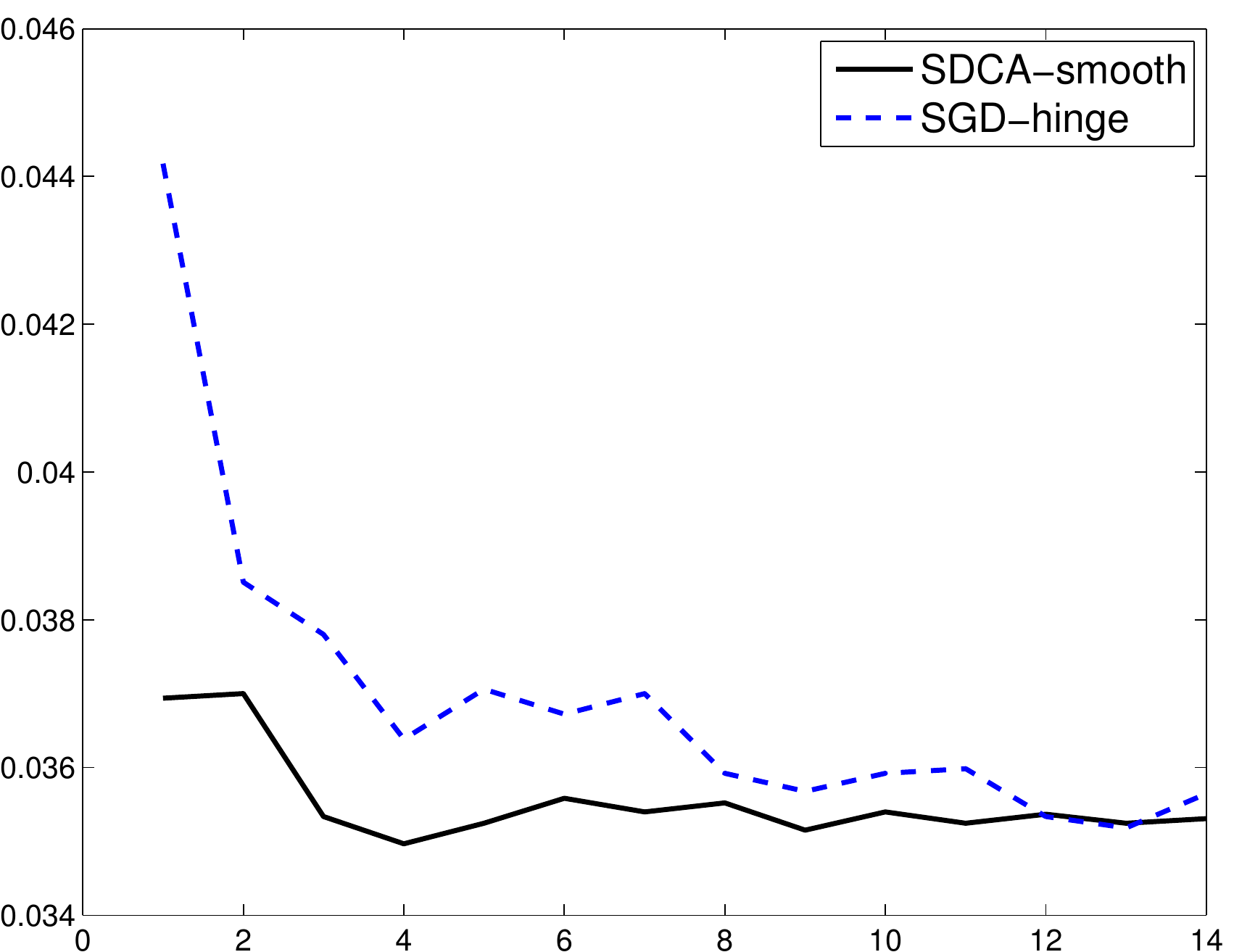} &
\includegraphics[width=0.31\textwidth]{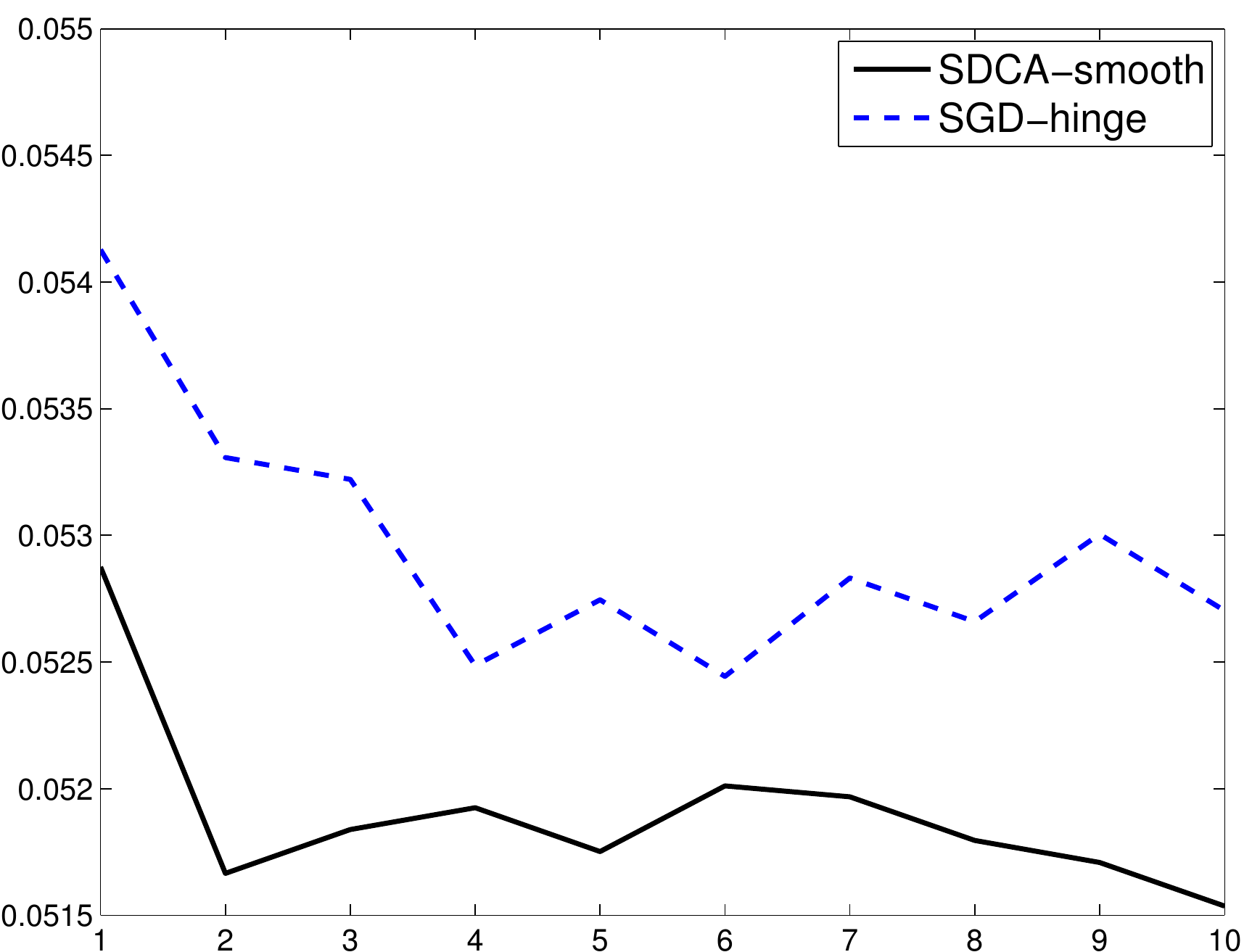} &
\includegraphics[width=0.31\textwidth]{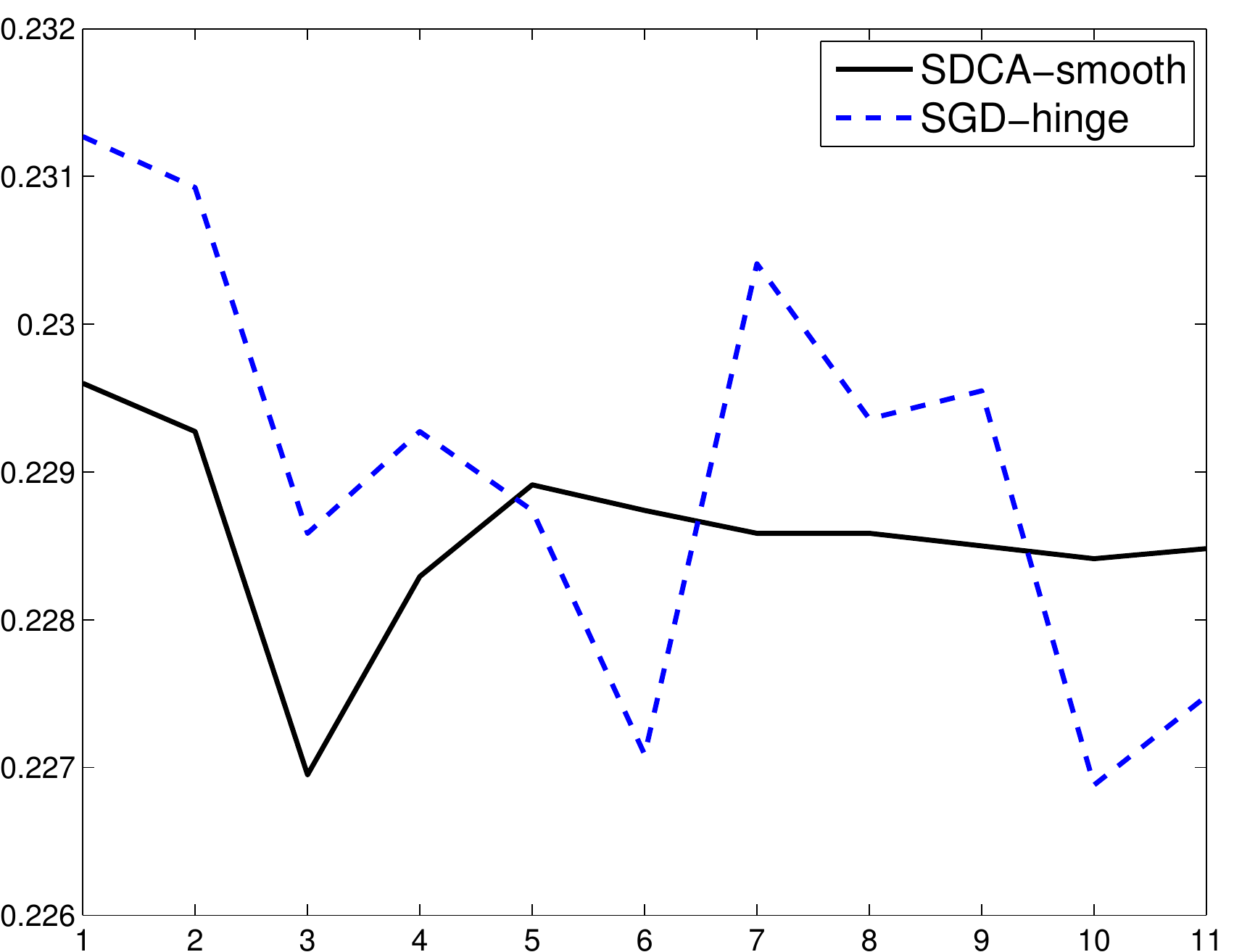}\\
$10^{-6}$ &
\includegraphics[width=0.31\textwidth]{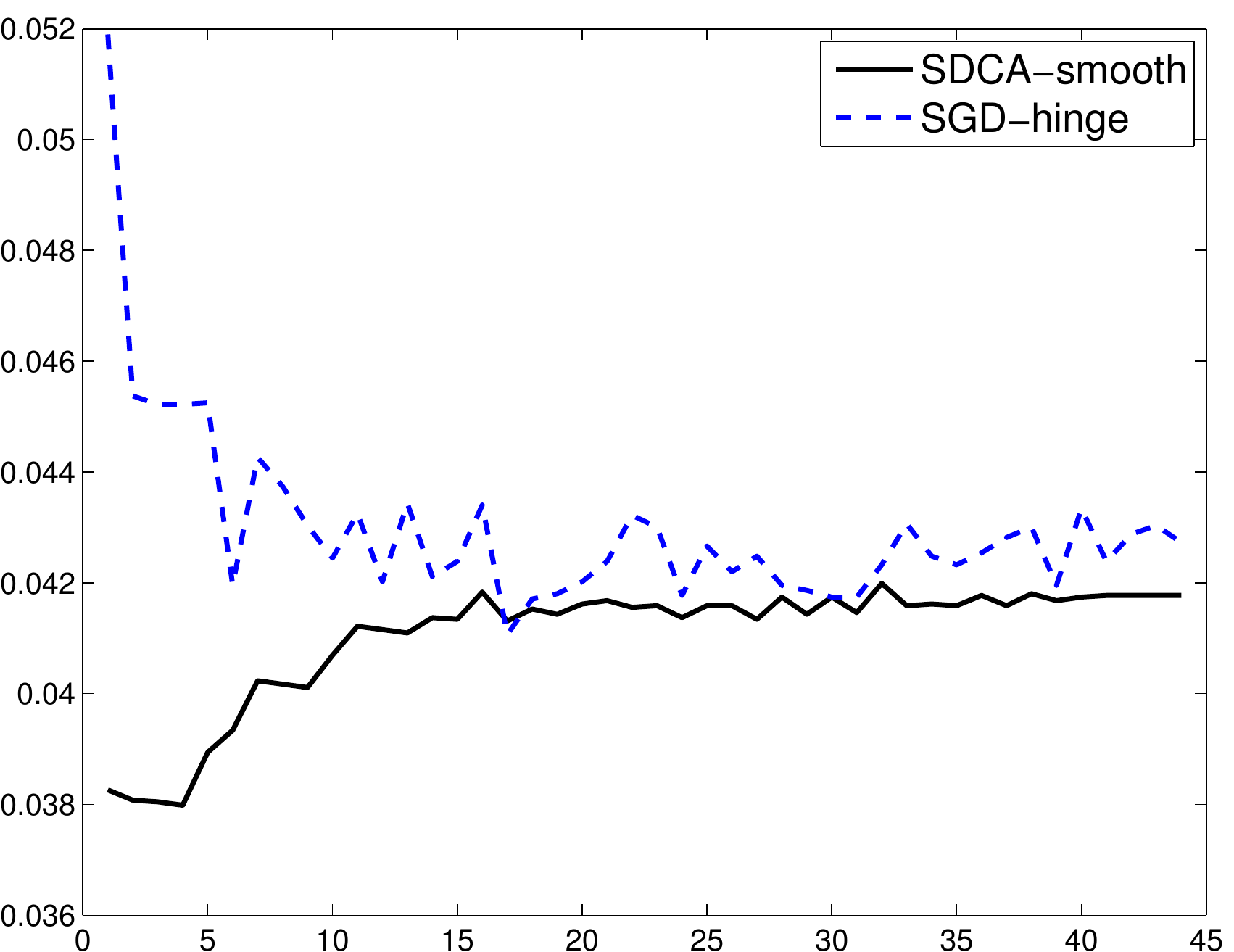} &
\includegraphics[width=0.31\textwidth]{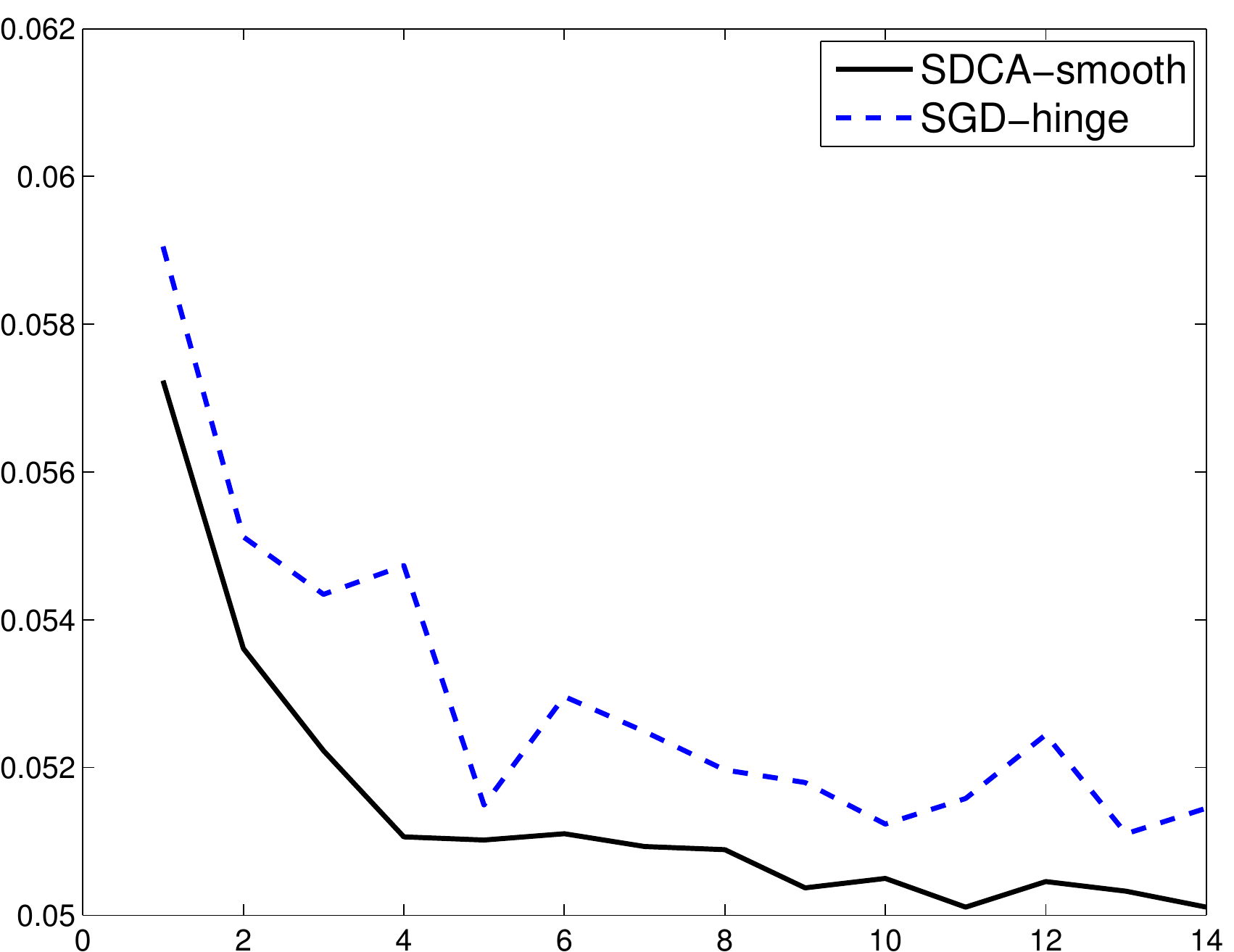} &
\includegraphics[width=0.31\textwidth]{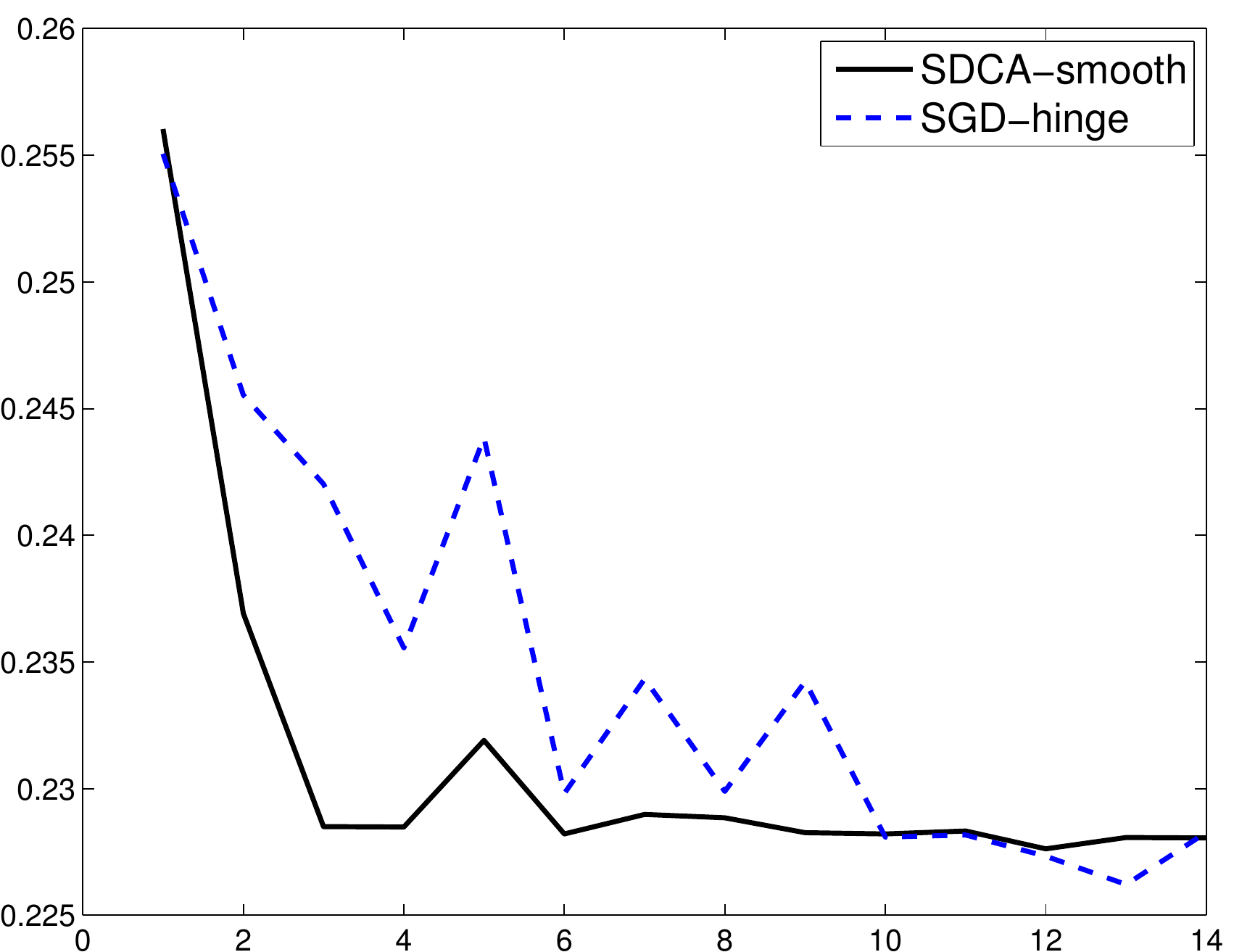}\\
\end{tabular}
\end{center}

\caption{\label{fig:SGdhingevsSDCAsmooth} Comparing the test error of
  SDCA with the smoothed hinge-loss ($\gamma=1$) to the test error of
  SGD with the non-smoothed hinge-loss. In all plots the vertical axis
  is the zero-one error on the test set and the horizontal axis is the
  number of iterations divided by training set size (corresponding to
  the number of epochs through the data). We terminated SDCA when the
  duality gap was smaller than $10^{-5}$.}

\end{figure}

\newpage

\bibliographystyle{plain}
\bibliography{curRefs}

\begin{thebibliography}{10}

\bibitem{BottouBo08}
L.~Bottou and O.~Bousquet.
\newblock The tradeoffs of large scale learning.
\newblock In {\em NIPS}, pages 161--168, 2008.

\bibitem{CollinsGlKoCaBa08}
M.~Collins, {A. Globerson}, T.~Koo, X.~Carreras, and P.~Bartlett.
\newblock Exponentiated gradient algorithms for conditional random fields and
  max-margin markov networks.
\newblock {\em Journal of Machine Learning Research}, 9:1775--1822, 2008.

\bibitem{HsiehChLiKeSu08}
C.J. Hsieh, K.W. Chang, C.J. Lin, S.S. Keerthi, and S.~Sundararajan.
\newblock A dual coordinate descent method for large-scale linear {SVM}.
\newblock In {\em ICML}, pages 408--415, 2008.

\bibitem{HuKeScSt06}
D.~Hush, P.~Kelly, C.~Scovel, and I.~Steinwart.
\newblock {QP} algorithms with guaranteed accuracy and run time for support
  vector machines.
\newblock {\em JMLR}, 7:733--769, 2006.

\bibitem{Joachims98}
T.~Joachims.
\newblock Making large-scale support vector machine learning practical.
\newblock In B.~Sch{\"o}lkopf, C.~Burges, and A.~Smola, editors, {\em Advances
  in Kernel Methods - Support Vector Learning}. MIT Press, 1998.

\bibitem{lacoste2012stochastic}
S.~Lacoste-Julien, M.~Jaggi, M.~Schmidt, and P.~Pletscher.
\newblock Stochastic block-coordinate frank-wolfe optimization for structural
  svms.
\newblock {\em arXiv preprint arXiv:1207.4747}, 2012.

\bibitem{le2004large}
L.B.Y. Le~Cun.
\newblock Large scale online learning.
\newblock In {\em Advances in Neural Information Processing Systems 16:
  Proceedings of the 2003 Conference}, volume~16, page 217. MIT Press, 2004.

\bibitem{LSB12-sgdexp}
Nicolas {Le Roux}, Mark {Schmidt}, and Francis {Bach}.
\newblock {A Stochastic Gradient Method with an Exponential Convergence Rate
  for Strongly-Convex Optimization with Finite Training Sets}.
\newblock {\em arXiv preprint arXiv:1202.6258}, 2012.

\bibitem{LuoTs92}
Z.Q. Luo and P.~Tseng.
\newblock On the convergence of coordinate descent method for convex
  differentiable minimization.
\newblock {\em J. Optim. Theory Appl.}, 72:7--35, 1992.

\bibitem{MangasarianMu99}
O.~Mangasarian and D.~Musicant.
\newblock Successive overrelaxation for support vector machines.
\newblock {\em IEEE Transactions on Neural Networks}, 10, 1999.

\bibitem{murata1998statistical}
N.~Murata.
\newblock A statistical study of on-line learning.
\newblock {\em Online Learning and Neural Networks. Cambridge University Press,
  Cambridge, UK}, 1998.

\bibitem{Nesterov10}
Y.~Nesterov.
\newblock Efficiency of coordinate descent methods on huge-scale optimization
  problems.
\newblock {\em SIAM Journal on Optimization}, 22(2):341--362, 2012.

\bibitem{Platt98}
J.~C. Platt.
\newblock Fast training of {Support Vector Machines} using sequential minimal
  optimization.
\newblock In B.~Sch{\"o}lkopf, C.~Burges, and A.~Smola, editors, {\em Advances
  in Kernel Methods - Support Vector Learning}. MIT Press, 1998.

\bibitem{robbins1951stochastic}
H.~Robbins and S.~Monro.
\newblock A stochastic approximation method.
\newblock {\em The Annals of Mathematical Statistics}, pages 400--407, 1951.

\bibitem{ShalevSiSr07}
S.~Shalev-Shwartz, Y.~Singer, and N.~Srebro.
\newblock Pegasos: {P}rimal {E}stimated sub-{G}r{A}dient {SO}lver for {SVM}.
\newblock In {\em ICML}, pages 807--814, 2007.

\bibitem{ShalevSr08}
S.~Shalev-Shwartz and N.~Srebro.
\newblock {SVM} optimization: Inverse dependence on training set size.
\newblock In {\em International Conference on Machine Learning}, pages
  928--935, 2008.

\bibitem{ShalevTe09}
Shai Shalev-Shwartz and Ambuj Tewari.
\newblock Stochastic methods for $l_1$ regularized loss minimization.
\newblock In {\em ICML}, page 117, 2009.

\bibitem{SridharanSrSh08}
K.~Sridharan, N.~Srebro, and S.~Shalev-Shwartz.
\newblock Fast rates for regularized objectives.
\newblock In {\em NIPS}, 2009.

\bibitem{Zhang04}
T.~Zhang.
\newblock Solving large scale linear prediction problems using stochastic
  gradient descent algorithms.
\newblock In {\em Proceedings of the Twenty-First International Conference on
  Machine Learning}, 2004.

\end{thebibliography}

\end{document}